\documentclass{article}

\usepackage{etoolbox}
\newtoggle{neurips}
\togglefalse{neurips}     

\newcommand{\neurips}[1]{\iftoggle{neurips}{#1}{}}
\newcommand{\arxiv}[1]{\iftoggle{neurips}{}{#1}}
\newcommand{\loose}{\looseness=-1}

\PassOptionsToPackage{round}{natbib}
\PassOptionsToPackage{dvipsnames}{xcolor}
\PassOptionsToPackage{hypertexnames=false}{hyperref}
\neurips{\usepackage{neurips_2026}}
\arxiv{\usepackage{natbib}}

\usepackage[utf8]{inputenc}
\usepackage[T1]{fontenc}
\usepackage{times}
\usepackage{amsmath}
\usepackage{amssymb}
\usepackage{amsfonts}
\usepackage{mathtools}
\usepackage{bm}
\usepackage{bbm}
\usepackage{nicefrac}
\usepackage[nopatch=footnote]{microtype}
\usepackage{ninecolors}
\usepackage{xcolor}
\usepackage{transparent}
\usepackage{xspace}
\usepackage{pifont}
\usepackage{dsfont}
\usepackage{xurl}
\usepackage{graphicx}
\usepackage{booktabs}
\usepackage{longtable}
\usepackage{multirow}
\usepackage{makecell}
\usepackage{standalone}
\usepackage{wrapfig}
\usepackage{enumitem}
\usepackage{setspace}
\usepackage{tocloft}
\arxiv{\usepackage{fancyhdr}}
\usepackage{breakcites}
\usepackage{mathrsfs}
\usepackage[scaled=.91]{helvet}
\usepackage{inconsolata}
\DeclareTextFontCommand{\texttt}{\upshape\ttfamily}
\usepackage{tablefootnote}
\usepackage{listings}
\usepackage{algorithm}
\usepackage[noend]{algorithmic}
\usepackage{tikz}
\usetikzlibrary{arrows.meta, automata, decorations.pathreplacing, positioning, bending}
\usepackage{amsthm}
\allowdisplaybreaks[4]
\usepackage{thm-restate}
\usepackage{hyperref}
\hypersetup{
  colorlinks=true,
  linkcolor=blue3,
  citecolor=blue3,
  urlcolor=blue3,
  pdftitle={When Does On-Policy Interaction Help? Representational Tradeoffs in Value-Based Imitation Learning},
  pdfauthor={Luca Viano, Antoine Moulin, Audrey Huang, Volkan Cevher, Philip Amortila, Dylan J. Foster}
}
\usepackage[nameinlink,capitalize]{cleveref}
\usepackage{caption}
\usepackage{subcaption}
\usepackage{array}

\newcommand{\appendixtocfooterfields}{%
  \fancyhf{}%
  \renewcommand{\headrulewidth}{0pt}%
  \renewcommand{\footrulewidth}{0pt}%
  \fancyfoot[C]{\thepage}%
  \fancyfoot[R]{\footnotesize\hyperlink{appendix-toc}{Appendix table of contents}}%
}
\newcommand{\enableappendixtocfooter}{%
  \fancypagestyle{appendixfooter}{\appendixtocfooterfields}%
  \fancypagestyle{plain}{\appendixtocfooterfields}%
  \pagestyle{appendixfooter}%
}

\makeatletter
\crefname{ALC@unique}{line}{lines}
\Crefname{ALC@unique}{Line}{Lines}
\crefname{ALC@line}{line}{lines}
\Crefname{ALC@line}{Line}{Lines}
\makeatother

\addtocontents{toc}{\protect\setcounter{tocdepth}{0}}

\arxiv{
  \usepackage[letterpaper, left=1in, right=1in, top=1in, bottom=1in]{geometry}
  \usepackage{parskip}
}

\usepackage{xargs}

\makeatletter
\newcommand{\neutralize}[1]{\expandafter\let\csname c@#1\endcsname\count@}
\makeatother

\declaretheoremstyle[
  spaceabove=0.2\baselineskip plus 1pt minus 0.5pt,
  spacebelow=0.05\baselineskip plus 1pt minus 0.5pt,
  headfont=\bfseries,
  bodyfont=\itshape,
  notefont=\normalfont,
  headpunct={.},
  postheadspace=0.5em,
]{papertheorem}

\declaretheorem[name=Theorem,parent=section,style=papertheorem]{theorem}
\declaretheorem[name=Lemma,parent=section,style=papertheorem]{lemma}
\declaretheorem[name=Assumption,parent=section,style=papertheorem]{assumption}
\declaretheorem[name=Definition,parent=section,style=papertheorem]{definition}
\declaretheorem[name=Condition,parent=section,style=papertheorem]{condition}
\declaretheorem[name=Corollary,parent=section,style=papertheorem]{corollary}

\declaretheorem[qed=$\triangleleft$,name=Example,parent=section,style=papertheorem]{example}

\declaretheorem[name=Proposition,parent=section,style=papertheorem]{proposition}

\newcommand{\pfref}[1]{Proof of \cref{#1}}

\renewcommand{\eqref}[1]{\texorpdfstring{\hyperref[#1]{(\ref*{#1})}}{(\ref*{#1})}}
\crefformat{equation}{#2Eq.\,(#1)#3}
\Crefformat{equation}{#2Eq.\,(#1)#3}
\Crefformat{figure}{#2Figure~#1#3}
\Crefformat{assumption}{#2Assumption~#1#3}
\Crefname{assumption}{Assumption}{Assumptions}
\crefname{fact}{Fact}{Facts}

\usepackage{crossreftools}
\usepackage{xpatch}
\makeatletter
\renewenvironment{proof}[1][Proof]%
{\par\noindent{\bfseries\upshape {#1.}\ }}%
{\qed\par}
\makeatother

\xpatchcmd{\proof}{\itshape}{\normalfont\proofnameformat}{}{}
\newcommand{\proofnameformat}{\bfseries}

\usepackage[most]{tcolorbox}

\tcbset{
  base/.style={
    arc=0mm,
    bottomtitle=0.5mm,
    boxrule=0mm,
    colbacktitle=!10!white,
    coltitle=black,
    fonttitle=\bfseries,
    left=2.5mm,
    leftrule=1mm,
    right=3.5mm,
    title={#1},
    toptitle=0.75mm,
  }
}

\newtcolorbox{mainbox}[1]{
  colframe=blue!10!black,
  colbacktitle=blue!50!black!30!white,
  colback=blue!2!white,
  enhanced,
  fonttitle=\bfseries,
  attach boxed title to top left={yshift=-2.5mm},
  boxed title style={size=small,colframe=blue!40!black,colback=blue!40!black},
  title={\small\textcolor{white}{\textsc{#1}}}
}

\newtcolorbox{minbox}[1]{
  colframe=blue!10!black,
  colbacktitle=blue!50!black!30!white,
  colback=blue!2!white,
  enhanced,
  fonttitle=\bfseries,
}

\DeclarePairedDelimiter{\prn}{(}{)}

\DeclareMathOperator*{\argmin}{arg\,min} 
\DeclareMathOperator*{\argmax}{arg\,max}             

\newcommand{\mb}[1]{\boldsymbol{#1}}
\newcommand{\wt}[1]{\widetilde{#1}}
\newcommand{\wh}[1]{\widehat{#1}}
\newcommand{\wb}[1]{\widebar{#1}}

\def\ddefloop#1{\ifx\ddefloop#1\else\ddef{#1}\expandafter\ddefloop\fi}
\def\ddef#1{\expandafter\def\csname bb#1\endcsname{\ensuremath{\mathbb{#1}}}}
\ddefloop ABCDEFGHIJKLMNOPQRSTUVWXYZ\ddefloop
\def\ddefloop#1{\ifx\ddefloop#1\else\ddef{#1}\expandafter\ddefloop\fi}
\def\ddef#1{\expandafter\def\csname b#1\endcsname{\ensuremath{\mathbf{#1}}}}
\ddefloop ABCDEFGHIJKLMNOPQRSTUVWXYZ\ddefloop
\def\ddef#1{\expandafter\def\csname sf#1\endcsname{\ensuremath{\mathsf{#1}}}}
\ddefloop ABCDEFGHIJKLMNOPQRSTUVWXYZ\ddefloop
\def\ddef#1{\expandafter\def\csname c#1\endcsname{\ensuremath{\mathcal{#1}}}}
\ddefloop ABCDEFGHIJKLMNOPQRSTUVWXYZ\ddefloop
\def\ddef#1{\expandafter\def\csname fk#1\endcsname{\ensuremath{\mathfrak{#1}}}}
\ddefloop ABCDEFGHIJKLMNOPQRSTUVWXYZ\ddefloop
\def\ddef#1{\expandafter\def\csname h#1\endcsname{\ensuremath{\widehat{#1}}}}
\ddefloop ABCDEFGHIJKLMNOPQRSTUVWXYZ\ddefloop
\def\ddef#1{\expandafter\def\csname hc#1\endcsname{\ensuremath{\widehat{\mathcal{#1}}}}}
\ddefloop ABCDEFGHIJKLMNOPQRSTUVWXYZ\ddefloop
\def\ddef#1{\expandafter\def\csname t#1\endcsname{\ensuremath{\widetilde{#1}}}}
\ddefloop ABCDEFGHIJKLMNOPQRSTUVWXYZ\ddefloop
\def\ddef#1{\expandafter\def\csname tc#1\endcsname{\ensuremath{\widetilde{\mathcal{#1}}}}}
\ddefloop ABCDEFGHIJKLMNOPQRSTUVWXYZ\ddefloop
\def\ddefloop#1{\ifx\ddefloop#1\else\ddef{#1}\expandafter\ddefloop\fi}
\def\ddef#1{\expandafter\def\csname scr#1\endcsname{\ensuremath{\mathscr{#1}}}}
\ddefloop ABCDEFGHIJKLMNOPQRSTUVWXYZ\ddefloop

\newcommand{\veps}{\varepsilon}

\newcommand{\ldef}{\vcentcolon=}

\let\underbar\undefined

\makeatletter
\let\save@mathaccent\mathaccent
\newcommand*\if@single[3]{%
  \setbox0\hbox{${\mathaccent"0362{#1}}^H$}%
  \setbox2\hbox{${\mathaccent"0362{\kern0pt#1}}^H$}%
  \ifdim\ht0=\ht2 #3\else #2\fi
  }
\newcommand*\rel@kern[1]{\kern#1\dimexpr\macc@kerna}
\newcommand*\widebar[1]{\@ifnextchar^{{\wide@bar{#1}{0}}}{\wide@bar{#1}{1}}}
\newcommand*\underbar[1]{\@ifnextchar_{{\under@bar{#1}{0}}}{\under@bar{#1}{1}}}
\newcommand*\wide@bar[2]{\if@single{#1}{\wide@bar@{#1}{#2}{1}}{\wide@bar@{#1}{#2}{2}}}
\newcommand*\under@bar[2]{\if@single{#1}{\under@bar@{#1}{#2}{1}}{\under@bar@{#1}{#2}{2}}}
\newcommand*\wide@bar@[3]{%
  \begingroup
  \def\mathaccent##1##2{%
    \let\mathaccent\save@mathaccent
    \if#32 \let\macc@nucleus\first@char \fi
    \setbox\z@\hbox{$\macc@style{\macc@nucleus}_{}$}%
    \setbox\tw@\hbox{$\macc@style{\macc@nucleus}{}_{}$}%
    \dimen@\wd\tw@
    \advance\dimen@-\wd\z@
    \divide\dimen@ 3
    \@tempdima\wd\tw@
    \advance\@tempdima-\scriptspace
    \divide\@tempdima 10
    \advance\dimen@-\@tempdima
    \ifdim\dimen@>\z@ \dimen@0pt\fi
    \rel@kern{0.6}\kern-\dimen@
    \if#31
      \overline{\rel@kern{-0.6}\kern\dimen@\macc@nucleus\rel@kern{0.4}\kern\dimen@}%
      \advance\dimen@0.4\dimexpr\macc@kerna
      \let\final@kern#2%
      \ifdim\dimen@<\z@ \let\final@kern1\fi
      \if\final@kern1 \kern-\dimen@\fi
    \else
      \overline{\rel@kern{-0.6}\kern\dimen@#1}%
    \fi
  }%
  \macc@depth\@ne
  \let\math@bgroup\@empty \let\math@egroup\macc@set@skewchar
  \mathsurround\z@ \frozen@everymath{\mathgroup\macc@group\relax}%
  \macc@set@skewchar\relax
  \let\mathaccentV\macc@nested@a
  \if#31
    \macc@nested@a\relax111{#1}%
  \else
    \def\gobble@till@marker##1\endmarker{}%
    \futurelet\first@char\gobble@till@marker#1\endmarker
    \ifcat\noexpand\first@char A\else
      \def\first@char{}%
    \fi
    \macc@nested@a\relax111{\first@char}%
  \fi
  \endgroup
}
\newcommand*\under@bar@[3]{%
  \begingroup
  \def\mathaccent##1##2{%
    \let\mathaccent\save@mathaccent
    \if#32 \let\macc@nucleus\first@char \fi
    \setbox\z@\hbox{$\macc@style{\macc@nucleus}_{}$}%
    \setbox\tw@\hbox{$\macc@style{\macc@nucleus}{}_{}$}%
    \dimen@\wd\tw@
    \advance\dimen@-\wd\z@
    \divide\dimen@ 3
    \@tempdima\wd\tw@
    \advance\@tempdima-\scriptspace
    \divide\@tempdima 10
    \advance\dimen@-\@tempdima
    \ifdim\dimen@>\z@ \dimen@0pt\fi
    \rel@kern{0.6}\kern-\dimen@
    \if#31
      \underline{\rel@kern{-0.6}\kern\dimen@\macc@nucleus\rel@kern{0.4}\kern\dimen@}%
      \advance\dimen@0.4\dimexpr\macc@kerna
      \let\final@kern#2%
      \ifdim\dimen@<\z@ \let\final@kern1\fi
      \if\final@kern1 \kern-\dimen@\fi
    \else
      \underline{\rel@kern{-0.6}\kern\dimen@#1}%
    \fi
  }%
  \macc@depth\@ne
  \let\math@bgroup\@empty \let\math@egroup\macc@set@skewchar
  \mathsurround\z@ \frozen@everymath{\mathgroup\macc@group\relax}%
  \macc@set@skewchar\relax
  \let\mathaccentV\macc@nested@a
  \if#31
    \macc@nested@a\relax111{#1}%
  \else
    \def\gobble@till@marker##1\endmarker{}%
    \futurelet\first@char\gobble@till@marker#1\endmarker
    \ifcat\noexpand\first@char A\else
      \def\first@char{}%
    \fi
    \macc@nested@a\relax111{\first@char}%
  \fi
  \endgroup
}
\makeatother

\makeatletter
\g@addto@macro\appendix{%
  \crefalias{section}{appendixsection}%
  \crefalias{subsection}{appendixsubsection}%
  \crefalias{subsubsection}{appendixsubsubsection}%
}
\makeatother
\crefname{appendixsection}{Appendix}{Appendices}
\Crefname{appendixsection}{Appendix}{Appendices}
\crefname{appendixsubsection}{Appendix}{Appendices}
\Crefname{appendixsubsection}{Appendix}{Appendices}
\crefname{appendixsubsubsection}{Appendix}{Appendices}
\Crefname{appendixsubsubsection}{Appendix}{Appendices}

\newcommand{\BC}{\texttt{BC}\xspace}
\newcommand{\SPOIL}{\texttt{SPOIL}\xspace}
\newcommand{\ATAC}{\texttt{ATAC}\xspace}
\newcommand{\IQLEARN}{\texttt{IQ-Learn}\xspace}
\newcommand{\ISPIL}{\texttt{OVI}\xspace}
\newcommand{\ISPILPLUS}{\texttt{OVI+}\xspace}
\newcommand{\DAEQUIL}{\texttt{DAeQuIL}\xspace}
\newcommand{\DAGGER}{\texttt{DAgger}\xspace}
\newcommand{\AGGREVATE}{\texttt{AggreVaTe}\xspace}
\newcommand{\DQN}{\texttt{DQN}\xspace}
\newcommand{\PPO}{\texttt{PPO}\xspace}
\newcommand{\SOFTDQN}{\texttt{Soft DQN}\xspace}

\newcommand{\CARTPOLE}{\texttt{CartPole-v1}\xspace}
\newcommand{\ACROBOT}{\texttt{Acrobot-v1}\xspace}

\newcommand{\LUNARLANDER}{\texttt{LunarLander-v2}\xspace}
\newcommand{\PENDULUM}{\texttt{Pendulum-v1}\xspace}

\newcommand{\QMAX}{Q_{\texttt{max}}}

\newcommand{\maxcovering}[2][\cQ]{\cN_{#2}^{\texttt{max}} \spr{#1}}

\newcommand{\experttag}{\texttt{E}}

\newcommand{\expert}{{\pi_{\experttag}}}
\newcommand{\expertone}{{\pi_{\experttag,1}}}
\newcommand{\experttwo}{{\pi_{\experttag,2}}}
\newcommand{\experti}{{\pi_{\experttag,i}}}
\newcommand{\experth}{{\pi_{\experttag, h}}}

\newcommand{\qexpert}{Q^{\expert}}
\newcommand{\qexperth}{Q^{\expert}_h}

\newcommand{\PiE}{\Pi}

\newcommand{\tauE}{n_{\experttag}}
\newcommand{\tauL}{n_{\texttt{L}}}
\newcommand{\cDE}{\mathcal{D}_{\experttag}}

\newcommand{\XEih}{\mb{x}_{\experttag, h}^i}

\newcommand{\AEh}{\mb{a}_{\experttag, h}}

\newcommand{\AEih}{\mb{a}_{\experttag, h}^i}
\newcommand{\AEkh}{\mb{a}_{\experttag, h}^k}

\newcommand{\expertsup}{\A^{\experttag}_{x,h}}

\newcommand{\simplex}{\Delta}
\newcommand{\initial}{\nu_0}

\newcommand{\smalltriangle}[1][0.6]{\mathord{\vcenter{\hbox{\scalebox{#1}{$\scriptstyle\triangle$}}}}}

\newcommand{\Xih}{\mb{x}_h^i}
\newcommand{\Xjkh}{\mb{x}_h^{jk}}
\newcommand{\Xkh}{\mb{x}_h^k}

\makeatletter
\newcommand{\piout}{\@ifnextchar_{\piout@sub}{{\pi_{\texttt{out}}}}}
\def\piout@sub_#1{{\pi_{\texttt{out}, #1}}}
\makeatother
\newcommand{\piouth}{\piout_h}

\newcounter{protocol}


\newcolumntype{P}[1]{>{\centering\arraybackslash}p{#1}}
\newcolumntype{M}[1]{>{\centering\arraybackslash}m{#1}}
\newcommand{\PreserveBackslash}[1]{\let\temp=\\#1\let\\=\temp}
\newcolumntype{C}[1]{>{\centering\arraybackslash}m{#1}}
\newcolumntype{R}[1]{>{\PreserveBackslash\raggedleft}p{#1}}
\newcolumntype{L}[1]{>{\PreserveBackslash\raggedright}p{#1}}

\definecolor{goldcolor}{HTML}{D4AF37}

\newcommand{\algcomment}[1]{\textcolor{blue!70!black}{\transparent{0.7}\small{\texttt{\textbf{\#\hspace{2pt}#1}}}}}

\makeatletter
\DeclareRobustCommand\onedot{\futurelet\@let@token\@onedot}
\def\@onedot{\ifx\@let@token.\else.\null\fi\xspace}
\makeatother

\def\cf{cf\onedot}
\def\eg{e.g\onedot}

\def\ie{i.e\onedot}

\def\vs{vs\onedot}

\newcommand{\delimgiven}[1][]{\nonscript\mathpunct{}#1\vert\nonscript\mathpunct{}}
\newcommand{\delimcontent}[1]{\renewcommand{\given}{\delimgiven[\delimsize]}#1}
\DeclarePairedDelimiterX{\abs}[1]{\lvert}{\rvert}{\delimcontent{#1}}
\DeclarePairedDelimiterX{\norm}[1]{\lVert}{\rVert}{\delimcontent{#1}}
\DeclarePairedDelimiterX{\inp}[1]{\langle}{\rangle}{\delimcontent{#1}}
\DeclarePairedDelimiterX{\spr}[1]{(}{)}{\delimcontent{#1}}
\DeclarePairedDelimiterX{\sbr}[1]{[}{]}{\delimcontent{#1}}
\DeclarePairedDelimiterX{\scbr}[1]{\{}{\}}{\delimcontent{#1}}
\DeclarePairedDelimiterX{\sdbr}[1]{[\![}{]\!]}{\delimcontent{#1}}
\DeclarePairedDelimiterX{\ceil}[1]{\lceil}{\rceil}{\delimcontent{#1}}
\DeclarePairedDelimiterX{\floor}[1]{\lfloor}{\rfloor}{\delimcontent{#1}}

\newcommand*\diff{\mathop{}\!\mathrm{d}}
\providecommand{\given}{}
\renewcommand{\given}{\delimgiven}
\newcommand{{\transpose}}{^\mathsf{\scriptscriptstyle T}}

\DeclareMathOperator{\softmax}{softmax}

\newcommand{\regret}{\mathfrak{R}}

\newcommand{\DTV}[2]{\mathcal{D}_{\mathsf{TV}}\spr{#1,#2}}

\renewcommand{\phi}{\varphi}  

\newcommand{\II}[1]{\bbI \scbr*{#1}}

\newcommand{\sumin}{\sum_{i=1}^n}

\newcommand{\sumkK}{\sum_{k=1}^K}

\newcommand{\sumhH}{\sum_{h=1}^H}

\newcommand{\bc}[1]{\left\{{#1}\right\}}
\newcommand{\br}[1]{\left({#1}\right)}

\def\A{\mathcal{A}}

\newcommand{\start}{\texttt{start}}

\lstdefinestyle{pythonstyle}{
    language=Python,
    basicstyle=\ttfamily\small,
    keywordstyle=\color{blue},
    stringstyle=\color{teal},
    commentstyle=\color{gray},
    numbers=left,
    numberstyle=\tiny\color{gray},
    stepnumber=1,
    numbersep=8pt,
    showstringspaces=false,
    breaklines=true,
    frame=single,
    tabsize=2,
    moredelim=**[is][\color{orange}]{@@}{@@},
}

\newtheorem{Lem}[lemma]{Lemma}

\usepackage{color-edits}
\addauthor{df}{ForestGreen}
\addauthor{ab}{red}
\addauthor{ah}{olive}
\addauthor{am}{BurntOrange}
\addauthor{lv}{purple}
\addauthor{pa}{cyan}

\usepackage{booktabs}
\usepackage{xcolor}
\usepackage{pifont}
\usepackage{capt-of}
\usepackage{diagbox}

\newcommand{\goodcell}{%
  \textcolor{green!50!black}{\ding{51}}%
}
\newcommand{\badcell}{%
  \textcolor{red!75!black}{\ding{55}}%
}

\let\oldparagraph\paragraph
\renewcommand{\paragraph}[1]{\oldparagraph{#1.}}

\newcommand{\para}[1]{\textbf{{#1}.}}

\title{When Does On-Policy Interaction Help? \\ Representational Tradeoffs in Value-Based Imitation Learning}

\newcommand{\email}[1]{\href{mailto:#1}{\scriptsize\texttt{#1}}}

\author{Luca Viano\thanks{Equal contribution.}~\thanks{Work done while at EPFL.}\\\email{luca.viano@ai.ethz.ch} \and Antoine Moulin\footnotemark[1]\\\email{antoine.moulin@berkeley.edu} \and Audrey Huang\\\email{audreyh5@illinois.edu} \and Volkan Cevher\\\email{volkan.cevher@epfl.ch} \and Philip Amortila\\\email{p.amortila@berkeley.edu} \and Dylan J. Foster\\\email{dylanfoster@microsoft.com}
}
\date{}

\begin{document}

\maketitle

\begin{abstract}
Imitation learning (IL)---training an agent to replicate expert behavior from demonstrations---underpins applications from robotics to language model training. Standard approaches such as Behavior Cloning (\BC) are known to suffer from compounding errors and performance plateaus, particularly when the learner cannot perfectly represent the expert's policy (as is typical, \eg, in distillation). Two interventions are widely understood empirically to improve performance: querying the expert \emph{interactively} along the learner's own trajectories, and using value function estimation en route to generating a policy.\loose

We investigate the nature of these improvements and their potentially surprising interplay. Our main finding is that expert interaction relaxes the \emph{representational} demands on the learner: one only needs a model capable of realizing the expert's value function, bypassing the (often stricter) requirement of realizing the expert's policy itself. Concretely, we introduce \ISPIL, an interactive on-policy IL algorithm that is statistically and computationally efficient whenever the learner can represent the expert's value function. We complement this with a negative result showing that interaction is necessary. Namely, without stronger assumptions beyond expert-value realizability alone, any offline IL algorithm must scale with the complexity of the expert policy class. Our findings bear out empirically. \ISPIL outperforms offline policy-based (\BC), interactive policy-based (\DAGGER), and offline value-based IL methods, with the largest gains when the learner network is substantially less expressive than the expert's.\loose

\end{abstract}

\section{Introduction}
\label{sec:introduction}
Training a model to replicate the behavior of a more capable expert---a process known as \emph{imitation learning} (IL)---has become a central paradigm in modern AI. In language modeling, distillation trains a smaller language model (LM) to reproduce the outputs of a larger one \citep{agarwal2024policy,lu2025onpolicydistillation}; in robotics and autonomous navigation, agents learn complex manipulation skills directly from expert demonstrations \citep{osa2018algorithmic,pomerleau1988alvinn}. Predominant approaches such as Behavior Cloning (\BC, \citealp{Pomerleau:1991}) often suffer from slow convergence, training instability, and performance plateaus even given ample offline expert data \citep{de2019causal,block2023butterfly,laskey2017dart,hu2025rac,spencer2021feedback}. These shortcomings are commonly attributed to two interrelated phenomena: error propagation along the horizon \citep{Ross:2010,Ross:2011}, and the inability of the learner to faithfully capture the expert's policy \citep{espinosa2025efficient,rohatgi2025computational}. Two interventions are commonly observed empirically to mitigate each of these problems in turn.

The first intervention is that of \emph{on-policy interaction}, \ie, the act of querying the expert along the learner's own trajectories. \emph{Interactive} (or \emph{on-policy}) IL is understood to avoid error amplification by learning to correct mistakes on-policy. This benefit has classically been formalized through improved horizon dependence \citep{Ross:2011,rajaraman2021value}. However, recent work by \citet{foster2024behavior} establishes that when the learner can accurately represent the expert policy, the apparent gap between \BC and interactive IL disappears, making the precise role of interaction less clear than it once seemed.

A second, parallel intervention is \emph{value-based} IL, where an estimated \emph{value or reward function} is used to assist policy learning \citep{Garg:2021,swamy2021moments,swamy2022sequence}, subsuming most inverse reinforcement learning (IRL) and apprenticeship learning methods \citep{Abbeel:2004,Syed:2007,Ho:2016b}. Rather than fitting the expert's full distribution over actions, including potentially arbitrary or exogenous choices with equal value, value-based IL seeks to recover only enough information to match the expert's return. This raises the hope that a value-based learner can succeed with a weaker representational burden, \eg, needing only to represent the expected value of each action even when representing the expert policy itself is too demanding. While previous works have sought to develop a theoretical understanding of value-based IL \citep{Abbeel:2004,Syed:2007,joshi2025learning,moulin2025inverse}, existing methods either do not address the sequential setting or require structural assumptions on the environment; as a result, our understanding of the benefits of value-based IL remains incomplete.

\begin{figure*}[!t]
    \centering
    \includegraphics[width=\linewidth]{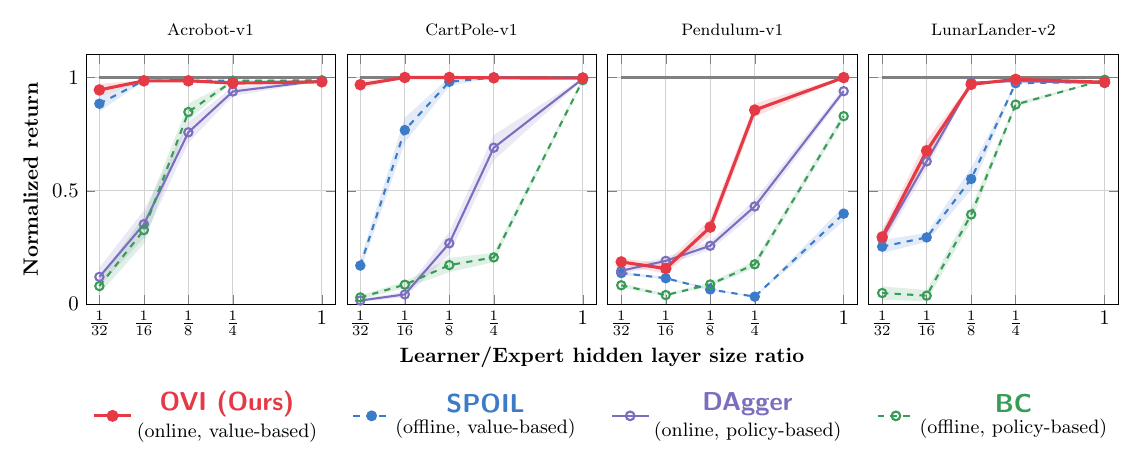}
    \caption{\arxiv{\vspace{-0.1em}%
    \textbf{\ISPIL achieves higher return with smaller learner networks.}} Average normalized return ($y$-axis, $1$ and $0$ correspond to the expert's and to a bad policy's returns, respectively) as a function of the learner-to-expert network-width ratio ($x$-axis), over $50$ seeds and $10$ expert trajectories for offline methods or $10$ rounds of expert queries along learner trajectories for interactive methods. The expert network has width $64$. We compare with offline and interactive policy-based methods (\BC; \citealp{pomerleau1988alvinn}; \DAGGER; \citealp{Ross:2011}), and the offline value-based method \SPOIL \citep{moulin2025inverse}. \textbf{When the learner network shrinks, \ISPIL outperforms each of the following methods alone: (1) offline value-based and (2) interactive policy-based---thereby supporting the theory} that predicts that representational gains can emerge only from expert interaction and value-based design jointly. See \cref{sec:experiments} for more details.\loose
    }
    \label{fig:size_scaling}
    \vspace{-0.8em}
\end{figure*}

In this work, we investigate the mechanisms through which value functions and interaction help in IL. Our results reveal a potentially surprising interplay: \textbf{interaction relaxes representational demands on the learner.}
\begin{center}
\neurips{\vspace{-.5\baselineskip}}
    \emph{Leveraging on-policy imitation, one only needs a model capable of realizing the expert's value function, bypassing the (often stricter) requirement of realizing the expert's policy itself. Furthermore, this is not possible in offline IL.\loose}
\neurips{\vspace{-.5\baselineskip}}
\end{center}

Our contributions are detailed below.

\para{Interaction permits efficient value-based IL (\cref{sec:ub})}
We introduce \ISPIL (\cref{alg:interactive_finite-H}), the first algorithm that is statistically and computationally oracle-efficient under only $\qexpert$\emph{-realizability}---a natural representational assumption requiring only that the learner can represent the expert's value function---\emph{provided that the learner can interactively query the expert on-policy.} \ISPIL queries the expert along the learner's own trajectories and leverages a saddle-point formulation of IL to identify a useful value function, building on prior work \citep{moulin2025inverse,swamy2021moments}. We complement this with a study on the representational advantages of $\qexpert$-realizability (\cref{sec:value-realizability-weaker}): we show in several natural problem classes that representing the expert's value function can be significantly easier than representing its policy, and furthermore that the weaker assumption of \emph{reward realizability} is in general insufficient for (offline or interactive) IL.\loose

\para{Interactive value-based IL empirically enables learning with fewer parameters (\cref{sec:experiments})}
We validate these findings empirically in Gymnasium \citep{towers2024gymnasium}; see \cref{fig:size_scaling}. Consistent with our theory, \ISPIL outperforms both interactive policy-based methods \citep[\DAGGER;][]{Ross:2011} and offline value-based methods \citep[\SPOIL;][]{moulin2025inverse}, especially as the learner's network shrinks relative to the expert's---a regime in which policy realizability is likely to be increasingly violated. The code for these experiments is available at \url{https://github.com/lviano/ovi}.\loose

\para{Interaction is necessary for value-based IL (\cref{sec:lb})}
We prove that without conditions stronger than $\qexpert$-realizability, all offline IL algorithms given only an offline dataset of expert trajectories necessarily scale with the number of states in the environment or the complexity of the expert policy class (\cref{thm:lower_any_algo}). We further establish that an important subclass of value-based learners (including popular empirical algorithms) can fail to compete with the expert policy in the offline setting even in the limit of infinite data (\cref{thm:lower}). Intuitively, the hardness arises because, unlike in traditional policy-based IL (\eg, via \BC or \DAGGER), value-based learners match the expert's return without matching its state visitation distribution (\cf \cref{fig:section4_experiments}); our results indicate that this can be drastically more efficient under interaction but fails to overcome well-studied distribution shift challenges \citep{Ross:2011,rajaraman2020toward,foster2024behavior} present in offline IL.\loose

\para{Computational-representational tradeoffs in chain-of-thought learning (\cref{sec:llm_tradeoff})}
Lastly, we provide a complementary perspective on the benefits of interaction in IL by showing that \ISPIL can achieve an exponential computational improvement over existing reward-based approaches \citep{joshi2025learning} when specialized to chain-of-thought learning. Prior theoretical approaches for value-based IL in chain-of-thought learning extract an imitating policy from a learned outcome reward function, but require intractable exponential-weights-type updates and response-level sampling. We show that, leveraging interaction and value function realizability, \ISPIL uses value functions to decompose the search over complete responses into a series of tractable token-level policy updates. This gives new theoretical understanding and motivation for the successful use of \emph{on-policy} updates \citep{agarwal2024policy,gu2024minillm,yang2025qwen3} and \emph{process reward models} \citep{lightman2023let,uesato2022solving} in LM training.\loose

\emph{When does interaction help in IL?}
Going beyond the classical understanding of improved horizon dependence, our results reveal additional (and perhaps stronger) \textbf{representational} benefits of interaction. Taken together, they provide a comprehensive picture of the interplay between interaction and representation in value-based IL.\loose

\newcommand{\verticaloffset}{0.8ex}
\newcommand{\spacingfirstcol}{-1.75ex}

\section{Problem Setting: Value-Based Imitation Learning}
\label{sec:preliminaries}
We first provide general definitions of Markov decision processes (MDPs) and the general IL problem, and then formally describe our problem setting of \emph{value-based imitation learning}.\loose

\subsection{Imitation Learning}
\label{sec:il-background}

\para{Episodic MDPs}
An \emph{episodic} MDP is defined as a tuple $\cM = \spr{\cX, \cA, H, P, r, \initial}$, where $\cX$ is a (potentially large) finite state space, $\cA$ is a finite action space with $A \ge 2$ actions, $H \in \bbN$ is the horizon, $P = \spr{P_h\colon \cX \times \cA \to \simplex \spr{\cX}}_{h = 1}^H$ is the transition kernel, $r = \spr{r_h\colon \cX \times \cA \to \sbr{0, 1}}_{h=1}^H$ is the reward function, and $\initial \in \simplex \spr{\cX}$ is the initial state distribution. A nonstationary Markov policy (henceforth simply \emph{policy}) is a sequence of \emph{decision rules} $\pi = \spr{\pi_h\colon \cX \to \simplex \spr{\cA}}_{h=1}^H$, it induces a distribution $\bbP^\pi$ over trajectories $\spr{\mb{x}_h, \mb{a}_h, \mb{r}_h}^H_{h=1}$ via the following interaction protocol: an initial state $\mb{x}_1 \sim \initial$ is drawn, and then for each $h = 1, \ldots, H$, we have $\mb{a}_h \sim \pi_h \spr{\cdot \given \mb{x}_h}$, $\mb{r}_h = r_h \spr{\mb{x}_h, \mb{a}_h}$, and $\mb{x}_{h+1} \sim P_h \spr{\cdot \given \mb{x}_h, \mb{a}_h}$. We denote by $\bbE^\pi$ the corresponding expectation. The expected return of a policy $\pi$ is given by $J^\pi = \bbE^\pi \bigl[\sumhH \mb{r}_h \bigr]$, and the value functions are defined for any state-action pair $\spr{x, a}$ by
\begin{equation*}
    V^\pi_h \spr*{x}
    =
    \bbE^\pi \sbr*{{\textstyle\sum_{h'=h}^H} \mb{r}_{h'} \given \mb{x}_h = x}
    \,\, \text{ and } \,\,
    Q^\pi_h \spr*{x, a}
    =
    \bbE^\pi \sbr*{{\textstyle\sum_{h'=h}^H} \mb{r}_{h'} \given \mb{x}_h = x, \mb{a}_h = a}.
\end{equation*}
The value functions satisfy the Bellman equations
\begin{equation} \label{eq:bellman}
    Q^\pi_h \spr*{x, a} = r_h \spr*{x, a} + \bbE_{\mb{x}' \sim P_h \spr*{\cdot \given x, a}} \sbr*{V^\pi_{h+1} \spr*{\mb{x}'}},
    \quad
    V^\pi_h \spr*{x} = \bbE_{\mb{a} \sim \pi_h \spr*{\cdot \given x}} \sbr*{Q^\pi_h \spr*{x, \mb{a}}},
\end{equation}
where $V_{H+1}^\pi = 0$. We define the occupancy measures via $d^\pi_h \spr{x, a} = \bbP^\pi \sbr{\mb{x}_h = x, \mb{a}_h = a}$ and $d^\pi_h \spr{x} = \bbP^\pi \sbr{\mb{x}_h = x}$. To disentangle the effects of reward scaling from the horizon $H$, we assume rewards are normalized such that for some $\QMAX > 0$ and any feasible trajectory, $\sum_{h=1}^H \mb{r}_h \in \sbr{0, \QMAX}$ \citep{jiang2018open,wang2020long,zhang2021reinforcement}.\footnote{In general (the \emph{dense reward} setting), $\QMAX$ can be as large as $H$, but in some settings (\eg, in LM reasoning tasks with binary outcome-level rewards) we can have $\QMAX = \cO \spr{1}$ (the \emph{sparse reward} setting).} For any policy $\pi$ and function $f\colon \cX \times \cA \to \bbR$, we sometimes use the notation $f \spr{x, \pi} = \sum_a \pi \spr{a \given x} f \spr{x, a}$.\loose

\para{Offline and interactive imitation learning}
In the imitation learning setting that we consider, there is an unknown \emph{expert policy} $\expert$ and an unknown MDP $\cM$, and the goal is to learn a policy that performs as well as $\expert$, as measured by expected return. Formally, for error tolerance $\varepsilon > 0$ and failure probability $\delta > 0$, the learner's objective is to output a policy $\piout$ such that, with probability at least $1 - \delta$,\loose
\begin{equation} \label{eq:suboptimality}
    J^{\expert} - J^{\piout}
    \leq
    \varepsilon.
\end{equation}
Central to our paper is the distinction between offline and interactive imitation learning. In \emph{offline IL}, the learner is given a dataset of pre-collected independent expert trajectories
\begin{equation} \label{eq:expert-dataset}
    \cDE = \scbr*{\spr*{\spr*{\XEih, \AEih}}_{h=1}^H}^{\tauE}_{i=1}, \quad \text{ where } \quad \spr*{\XEih, \AEih} \sim d^{\expert}_h,
\end{equation}
where $\tauE$ is the number of samples. The learner cannot interact with the MDP or the expert further.

In contrast, in \emph{interactive IL} the learner can query the expert $\expert$ for actions while rolling out in the unknown MDP $\cM$---importantly, without observing the rewards. For the purpose of our paper, we formalize \emph{on-policy interaction} as follows: in each round of interaction, the learner selects a policy $\pi$ to generate a state trajectory $\spr{\mb{x}_h}^H_{h=1}$ where $\mb{x}_h \sim d^\pi_h$ and might query expert actions at states along the trajectory, \ie, $\AEh \sim \experth \spr{\cdot \given \mb{x}_h}$.

\subsection{Value-Based Imitation Learning}
\label{sec:vb-il}

Our focus is on \emph{value-based imitation learning}, by which we mean methods that derive policies from learned value or reward functions. Classical \emph{policy-based} methods (\eg, \BC, \DAGGER) aim to imitate the full conditional action distribution of the expert $\expert$ and to match its performance as a consequence. By contrast, \emph{value-based} learners aim to recover only enough information to match the expert's return, which may be a simpler task whenever there is redundancy in the expert's trajectory distribution (\cref{ex:QfromPI}) or relevant structure in the dynamics that can be exploited (\cref{ex:exbmdp}).\loose

In particular, our goal is to develop value-based IL algorithms that succeed with \emph{only realizability of the expert's value function}, a natural representational assumption for value-based IL. Stated in learning-theoretic terms, the learner has access to a value function class $\cQ \subseteq \scbr{\spr{Q_h\colon \cX \times \cA \to \sbr{0, \QMAX}}_{h=1}^H}$ that can be used to recover a good policy, and we assume only the following on its expressivity.
\begin{assumption}[$Q^\expert$-realizability] \label{asp:q-expert-realizability}
   The class $\cQ$ contains the expert's value function, \ie, $Q^{\expert} \in \cQ$.
\end{assumption}
By contrast, policy-based methods, such as \BC and \DAGGER, commonly require policy realizability assumptions \citep{foster2024behavior}, which render them inapplicable in this setting. In \cref{sec:value-realizability-weaker}, we highlight several natural settings in which \emph{representing the expert's value function is drastically less stringent than realizing the expert's policy}. Accordingly, in \cref{sec:algo} we will develop algorithms for interactive IL whose statistical rates depend only on the statistical capacity of the value function class $\cQ$ (\eg, $\log(\abs{\cQ})$ for finite function classes) and not the policy class used by the learner.\loose

Value-based realizability assumptions comparable to \cref{asp:q-expert-realizability} have been analyzed in previous works, although they further required assumptions about the dynamics of the MDP. In particular, \citeauthor{joshi2025learning} only consider the $H=1$ (``contextual bandit'') setting, where \cref{asp:q-expert-realizability} simplifies to realizability of the reward function, and \citeauthor{moulin2025inverse} require closure-type assumptions on $\cQ$, \ie, that it can realize $Q^\pi$ for any policy $\pi$ generated by the learner.\footnote{Such ``completeness'' assumptions implicitly place restrictions on the dynamics \citep{chen2019information,foster2022offline,jiang2025offline}, and by contrast with \cref{asp:q-expert-realizability} cannot be satisfied by simply taking a more expressive $\cQ$ class.} In \cref{sec:value-realizability-weaker}, we further establish that mere reward realizability is insufficient for general environments when $H>1$ (\cref{thm:reward-realizability-lb}), indicating that $Q^{\expert}$-realizability is a natural minimal assumption for value-based IL.\loose

\para{Comparing interactive value-based IL to related settings}
Finally, we clarify the differences between our value-based IL setting and related IL settings considered in the literature. All of the settings mentioned below allow the learner to roll out in the MDP, but use protocols for interacting with the expert that differ from ours. In \emph{inverse RL}, the learner interacts with the MDP or with known MDP dynamics, but does not query the expert interactively \citep{Abbeel:2004, Syed:2007,Ziebart:2008}. In \emph{value-feedback} algorithms such as \AGGREVATE \citep{ross2014reinforcement,sun2017deeply}, the learner interacts with the MDP and can observe the expert value $Q^{\expert} \spr{\mb{x}_h, \mb{a}_h}$, whereas it is unclear how to directly estimate these values in our setting given that the learner does not observe rewards. Lastly, in a setting we call \emph{RL with expert advice}, the learner both interacts with the MDP and queries the expert\arxiv{ (or a reference policy that covers it well)} interactively, but additionally observes rewards in the MDP \citep{amortila2022few, tiapkin2023regularized, foster2025good}. See \cref{sec:additional_related_work} for a comprehensive discussion of related work.\loose


\section{Interaction Enables \texorpdfstring{$\qexpert$}{q-expert}-Realizable Value-Based Imitation}
\label{sec:ub}
The main result of this section is \ISPIL, the first provably efficient algorithm for value-based IL under $\qexpert$-realizability. In \cref{sec:algo}, we describe the algorithm and its guarantees. We then show in \cref{sec:value-realizability-weaker} that $\qexpert$-realizability can be weaker than policy realizability, while reward realizability alone is insufficient. Finally, in \cref{sec:experiments} we present empirical evidence that \ISPIL outperforms both policy-based methods \citep[\DAGGER;][]{Ross:2011} and offline value-based methods \citep[\SPOIL;][]{moulin2025inverse} in regimes where policy realizability is likely to be violated.\loose

\subsection{\ISPIL{}: Value-Based Imitation Learning With Only \texorpdfstring{$\qexpert$}{q-expert}-Realizability}
\label{sec:algo}

\para{Algorithm design} Following an approach suggested by the ``On-Q'' moment matching template \citep{swamy2021moments}, \ISPIL treats the problem of minimizing the suboptimality gap in \cref{eq:suboptimality} as a min-max game between a $\pi$-player and a $Q$-player. The starting point is the performance difference lemma \citep{howard1960dynamic,kakade2002approximately}:
\begin{align*}
    J^{\expert} - J^\pi &= \sumhH \cL_h^{d^\pi} \spr*{\pi_h, \qexpert_h}, \\
    ~~\text{where, for any $d = \spr*{d_h \in \simplex \spr*{\cX}}_{h=1}^H$,}~~&\cL_h^d \spr*{\pi, Q} := \underbrace{\sum_{x \in \cX} d_h \spr*{x}}_{\texttt{(I)}} \underbrace{\sum_{a \in \cA} \spr*{\experth \spr*{a \given x} - \pi \spr*{a \given x}} Q \spr*{x, a}}_{\texttt{(II)}}.
\end{align*}
This form is not enough yet to derive an algorithm because $\qexpert$ is unknown. However, assuming $\qexpert \in \cQ$, we bound the suboptimality by taking the supremum over $\cQ_h \ldef \scbr{Q_h : \exists Q' \in \cQ, Q_h' = Q_h}$ for every $h$ on the right hand side.\loose
\begin{equation*}
    J^{\expert} - J^\pi \leq \sumhH \sup_{Q_h \in \cQ_h} \cL_h^{d^\pi} \spr*{\pi_h, Q_h}.
\end{equation*}
Accordingly, what remains is to find an approximate saddle point of the expected advantage function, \ie, a policy $\piout$ such that, for all $h \in \sbr{H}$, $\sup_{Q_h \in \cQ_h} \cL_h^{d^{\piout}} \spr{\piout_h, Q_h} \lesssim \varepsilon / H$.

Algorithmically, the main point of departure from prior value-based IL algorithms is the \emph{layer-wise} learning of a good policy, in the vein of the \texttt{Forward} algorithm \citep{Ross:2010}. To understand this design choice, we recall that the approach employed by past work \citep{moulin2025inverse} is to compute a sequence $\spr{\pi^k, Q^k}_{k=1}^K$ where $\pi^k$ performs online mirror ascent \citep{Beck:2003} and $Q^k$ performs a best response to (an empirical estimate of) the loss function $\cL$. This technique does not seem to apply in the $\qexpert$-realizability setting, as the sequence $\pi^k$ is also dictating the current sampling distribution $d$ under which the error is measured (see \texttt{(I)} above), preventing us from simultaneously forming an empirical estimate of $\cL$ and controlling the error in \texttt{(II)} via online learning. This is addressed in \ISPIL by learning $\piout$ one layer at a time (note the outer loop over $h = 1, \ldots, H$ in \cref{line:h-loop} of \cref{alg:interactive_finite-H}), as $d^{\piout}_h(x)$ only depends on $(\piout_1,\ldots,\piout_{h-1})$ and we can thus fix the sampling distribution for stage $h$ before optimizing for $\piouth$.\loose

We view this algorithm as a principled algorithmic realization of the ``On-Q'' moment matching template introduced by \cite{swamy2021moments}. The pseudocode for \ISPIL is given in \cref{alg:interactive_finite-H}. The main sample complexity guarantee is stated below. It establishes that \ISPIL only requires a number of expert queries that scales with the statistical capacity of the value function class $\cQ$, as represented by its log-cardinality $\log \abs{\cQ}$.\loose

\begin{algorithm}[t]
    \caption{\textbf{\ISPIL}: \textbf{O}n-Policy \textbf{V}alue-Based \textbf{I}mitation Learning \label{alg:interactive_finite-H}}
    \begin{algorithmic}[1]
        \STATE \textbf{input:} Learning rate $\eta$, iterations $K$, expert queries per stage $\tauE$.
        \FOR{$h = 1, \ldots, H$}\label{line:h-loop}
        \STATE Create dataset at stage $h$: for $i \in \sbr{\tauE}$, sample $\Xih \sim  d^{\piout}_h$, and query $\AEih \sim \experth \spr{\cdot \given \Xih}$.
        \STATE Initialize $\pi^1_h = \mathrm{Unif} \spr{\cA}$.
        \FOR{$k = 1, \ldots, K$}
        \STATE Set $Q^{k}_h \in \argmax_{Q_h \in \cQ_h} \sum^{\tauE}_{i=1} \spr{Q_h \spr{\Xih, \AEih} - Q_h \spr{\Xih, \pi^k_h}}$.\label{line:test}
        \STATE Set $\pi^{k+1}_h \spr{a \given x} \propto \pi^k_h \spr{a \given x} e^{\eta Q^k_h \spr{x, a}}$, for any state-action pair $\spr{x, a}$.\label{line:pi-alg}
        \ENDFOR
        \STATE Create the output policy at layer $h$: $\piouth = \frac1K \sumkK \pi^k_h$.
        \ENDFOR
    \end{algorithmic}
\end{algorithm}

\begin{restatable}[Sample complexity of \ISPIL]{theorem}{pifirstmain} \label{thm:pi_first_main}
    Let \Cref{asp:q-expert-realizability} hold for a finite class $\cQ$. Then, for any $\varepsilon, \delta \in (0, 1)$, \ISPIL (\Cref{alg:interactive_finite-H}) with parameters $\tauE = \tcO \spr{H^4 \QMAX^4 \log \spr{A} \log \spr{\abs{\cQ} / \delta} \varepsilon^{-4}}$, $\eta = \spr{\log \spr{A} / \spr{K \QMAX^2}}^{1/2}$, and $K = \cO \spr{H^2 \QMAX^2 \log \spr{A} \varepsilon^{-2}}$, outputs a policy $\piout$ such that with probability at least $1 - \delta$, $J^\expert - J^{\piout} \leq \varepsilon$ after
    \begin{equation*}
        \tcO \spr*{ \frac{H^5 \QMAX^4 \log \spr*{A} \log \spr*{\abs*{\cQ} / \delta}}{\varepsilon^4}} \qquad \text{expert queries}.
    \end{equation*}
\end{restatable}

Note that the total number of learner interactions with the environment obeys the same bound with an extra factor of $H$. Up to the usual polynomial factors in\arxiv{ the horizon} $H$,\arxiv{ the suboptimality} $1/\varepsilon$, and\arxiv{ the log failure probability} $\log \spr{\delta^{-1}}$, the main appealing feature of \ISPIL is that its sample complexity is controlled by the statistical complexity of \emph{only the value function class $\cQ$}, through its log-cardinality. In particular, the guarantee imposes no statistical assumption on the learner's policy class. We show in \cref{sec:value-realizability-weaker} that, in many natural settings, a function class $\cQ$ satisfying $\qexpert$-realizability can be drastically smaller than a policy class $\Pi$ satisfying $\expert$-realizability, while it is the latter that governs the sample complexity of policy-based methods such as \BC or \DAGGER. We note that the dependence on the problem horizon is larger than that obtained by policy-based methods such as \BC or \DAGGER \citep{foster2024behavior}; it is an interesting question whether this can be improved further or if this is fundamental to value-based IL.\loose

\Cref{thm:pi_first_main} is a special case of \Cref{thm:ovi-main-covering}, which extends this guarantee to infinite value function classes via covering numbers. The result also shows that, when $\cQ$ is convex, the expert-query complexity improves to $\tcO \spr{\varepsilon^{-2}}$. 

Computationally, \ISPIL only requires access to a standard linear maximization oracle for the $Q$ updates (\Cref{line:test} in \cref{alg:interactive_finite-H}) and standard softmax policy updates for the $\pi$ player (\Cref{line:pi-alg} in \cref{alg:interactive_finite-H}). One drawback is that the layer-wise updates learn a potentially nonstationary policy and introduce computation and memory requirements that scale linearly with the horizon. However, in our experiments (\cref{sec:experiments}), we find that a stationary approximation of \ISPIL performs well. We refer to \cref{sec:experiments_details} for details.\loose

\para{Improved rates for nonconvex classes} The rate attained by \ISPIL for nonconvex classes $\cQ$ has a suboptimal $\tcO \spr{\varepsilon^{-4}}$ dependence. In \Cref{sec:ispil_plus}, we introduce the variant \ISPILPLUS (\Cref{alg:interactive_finite-H-expert-buffer}) that reuses the expert labels across iterations while collecting fresh learner rollouts to estimate the learner term in the Q-player's objective. It improves the number of expert queries to $\tcO \spr{H^3 \QMAX^2 \log \spr{\abs{\cQ} / \delta} \varepsilon^{-2}}$, but the number of learner rollouts remains of order $\tcO \spr{H^5 \QMAX^4 \log \spr{A} \log \spr{\abs{\cQ} / \delta} \varepsilon^{-4}}$.
A promising direction for reducing the remaining learner-rollout cost is recursive variance reduction of the learner term, which we leave to future work (see \Cref{sec:ispil_plus}).
Both the number of expert queries and learner rollouts can be improved to an $\tcO \spr{\varepsilon^{-2}}$ rate via a variant of \ISPIL, named $Q$-\ISPIL (\Cref{alg:exp_weights_finite_H}), where the roles of the value functions and the policies are reversed, \ie, the $Q$ player performs online learning and the $\pi$ player performs a best response (\Cref{thm:Q_first_main}). While statistically more efficient, this relies on explicitly discretizing the value-function class $\cQ$, making the method computationally unattractive for large or continuous classes such as neural networks. It is an interesting question to obtain a rate of $\tcO \spr{\varepsilon^{-2}}$ for both expert queries and learner rollouts in a computationally efficient manner. Another advantage of $Q$-\ISPIL is that it can be used to compute a stationary expert-matching policy (see \Cref{sec:infinite_horizon}), whereas \ISPIL learns a nonstationary policy. Whether a computationally efficient value-based IL algorithm can compute such a stationary policy remains open.\loose

\subsection{Representational Advantages of \texorpdfstring{$\qexpert$}{q-expert}-Realizability}
\label{sec:value-realizability-weaker}

\ISPIL requires only access to a value-function class that realizes the expert's value function, $\qexpert$. In this section, we show that, in many natural settings, $\qexpert$-realizability is substantially weaker than the assumptions used by prior policy- and value-based approaches. We also establish that reward realizability alone is insufficient.

\arxiv{\subsubsection{Value realizability can be weaker than policy realizability}}
\neurips{\para{Values can be easier to represent than policies}}
Classical IL algorithms, such as \BC, \DAGGER, and their derivatives, are \emph{policy-based}: they operate by directly fitting the expert's entire conditional distribution of actions, and the performance of these algorithms classically relies on a \emph{policy realizability} assumption \citep{foster2024behavior}.
\begin{assumption}[Policy realizability] \label{ass:policy-realizability}
    The policy class $\Pi$ contains the expert policy, \ie, $\expert \in \Pi$.
\end{assumption}
This assumption is natural if the goal is to reproduce the expert's trajectory distribution \emph{in a reward-free manner}. But the objective in IL is weaker: to match the expert's return. For this objective, it may be unnecessary and wasteful to attempt to fit the specific details of the expert distribution (\eg, how it breaks ties among equally good actions, or how it responds to features that are irrelevant for value). The examples below formalize this intuition.

\citet{joshi2025learning} consider the special case of our setting in which $H = 1$, and show that value-based IL can be strictly weaker than policy-based IL (note that when $H = 1$, $\qexpert$-realizability reduces to reward realizability). In particular, they show that when the expert is optimal and realized by the policy class $\Pi$, the class can be used to construct a realizable reward-function class $\cR$\footnote{Up to a reward shaping term that does not alter the optimal policy.} satisfying $\abs{\cR} \leq \abs{\Pi}$, simply by taking $\cR = \scbr{r: \exists \pi \in \Pi, r \spr{x, a} = \mathbbm{1} \scbr{a \in \texttt{supp} \spr{\pi \spr{\cdot \given x}}}}$. By contrast, they show that the existence of a small realizable reward class $\cR$ cannot be used to construct a small realizable policy class $\Pi$ under which \BC will succeed.\footnote{While the learner knows that policies in $\Pi$ must be supported on reward-maximizing actions, the set of such policies is prohibitively large.} Our first example extends this to the multi-step setting.
\begin{example}[Learning to answer from correct demonstrations, with $H > 1$] \label{ex:QfromPI}
    Given a finite class $\Pi$ that realizes an optimal expert $\expert$, we show in \cref{sec:Qisweaker1} that it is possible to construct a class $\cQ$ such that $\abs{\cQ} \leq \abs{\Pi}$ and $\cQ$ realizes a function $\tQ^\expert$ which assigns maximal value to actions within the expert support, defined by $\tQ^\expert \spr{x, a} = \QMAX \ \mathbbm{1} \scbr{a \in \texttt{supp} \spr{\expert \spr{\cdot \given x}}}$. \ISPIL still succeeds, even if we cannot guarantee that $\qexpert \in \cQ$ (see \Cref{sec:relaxedQrea}). Conversely, as above, when given a small $\cQ$ class satisfying $\qexpert$-realizability, the learner cannot construct a small policy class $\Pi$ satisfying $\expert \in \Pi$. We come back to this example in \Cref{sec:llm_tradeoff}.\loose
\end{example}

We note that this reduction is enabled by the relaxed $\qexpert$-realizability condition detailed in \Cref{sec:relaxedQrea}. In the $H>1$ setting, the closest comparable value-based IL work is the \SPOIL algorithm of \citeauthor{moulin2025inverse}, which requires that $\cQ$ can realize $Q^\pi$ for all policies $\pi$ generated by the learner. In that setting, it is highly unclear whether there exists a similar reduction from a small policy class $\Pi$ to a small function class $\cQ$ that satisfies this closure-type condition. Thus, under prior value expressivity conditions, value-based IL may not be representationally weaker than policy-based IL.

Our second example shows that, even with suboptimal experts, the same separation can arise from structure in the dynamics, as the value function may ignore high-dimensional information needed to reproduce the expert's exact actions.\loose
\begin{example}[Structure in the dynamics] \label{ex:exbmdp}
    In \Cref{app:hard-decoder-exbmdp}, we study an Exogenous Block MDP \citep{efroni2022sample,mhammedi2024power,amortila2024reinforcement}, where the dynamics and reward depend only on low-dimensional \emph{endogenous} state components, while the observations contain high-dimensional exogenous noise. We construct a family of MDPs and expert policies and show that any expert-agnostic policy class $\Pi$ satisfying \cref{ass:policy-realizability} must be exponentially large. In contrast, the same two-dimensional class $\cQ$ satisfies \Cref{asp:q-expert-realizability} for the entire family.\loose
\end{example}

Although the two assumptions need not be comparable in general, we expect value realizability to be weaker than policy realizability in a wide range of analogous settings. This is empirically supported by our experiments in \cref{sec:experiments}.\loose

\subsubsection{Reward realizability is insufficient}

The expert's value function $\qexpert$ encodes information about the expert policy $\expert$, the reward function, and the transitions of the MDP via the Bellman equations (\cref{eq:bellman}). A natural question is whether the assumption can be weakened to reward realizability alone while retaining efficient learning. Specifically, assume that the learner has access to a reward function class $\cR \subseteq \scbr{\spr{r'_h\colon \cX \times \cA \to \sbr{0, 1}}_{h=1}^H}$ satisfying the following expressivity condition.
\begin{assumption}[Reward realizability] \label{ass:reward-realizability}
   The class $\cR$ contains the MDP's reward function, \ie, $r \in \cR$.
\end{assumption}
The following result shows that mere reward realizability is insufficient for efficient IL. Our lower bound is established against the class of interactive IL algorithms defined below, formalizing the interaction protocol described in \cref{sec:il-background}.\loose

\begin{definition}[Interactive IL algorithm] \label{def:interactive-il-algo}
    An interactive IL algorithm $\texttt{Alg}$ with a fixed interaction budget $\tauE \in \bbN^\star$ is a (potentially randomized) procedure that, in each episode $k \in \sbr{\tauE}$, selects a policy $\pi^k$ based on the preceding interaction history, rolls it out in the MDP to generate states $\Xkh \sim d^{\pi^k}_h$ for all $h \in [H]$, and \textbf{is allowed} to query the expert for actions $\AEkh \sim \experth \spr{\cdot \given \Xkh}$ for all $h \in \sbr{H}$, without observing rewards. After the $\tauE$ episodes, it maps the resulting interaction history to a nonstationary policy $\spr{\piout_h\colon \cX \to \Delta \spr{\cA}}_{h=1}^H$.
\end{definition}

Our negative result for this setting establishes a lower bound on the \emph{rounds of interactions} used by any algorithm following the protocol defined in \cref{def:interactive-il-algo}. Given an MDP $\cM$, we write $J^{\pi}_{\!\scriptscriptstyle{\cM}}$ for the return of a policy $\pi$, and $\pi^\star_{\!\scriptscriptstyle{\cM}}$ for one of its optimal deterministic policies.\loose
\begin{restatable}[Lower bound under reward realizability]{theorem}{rewardrealizabilitylb} \label{thm:reward-realizability-lb}
    For every $X \geq 1$ and $\varepsilon \in (0, 1 / 4]$, there exists a family $\cF$ of MDPs with: i) stochastic initial states but deterministic rewards and transitions, ii) horizon $H=2$, iii) state space size $\abs{\cX} = \cO \spr{X}$, and iv) family size $\log \spr{\abs{\cF}} = \cO \spr{X}$, such that for every $\cM \in \cF$, \cref{ass:reward-realizability} is satisfied by the same singleton reward class $\cR = \scbr{r^\star}$. Let $\texttt{Alg}$ be any interactive IL algorithm (\cref{def:interactive-il-algo}). Then, there exists an MDP $\cM \in \cF$, with corresponding expert $\expert = \pi^\star_{\scriptscriptstyle{\cM}}$, such that $\texttt{Alg}$ needs $\Omega\prn*{\frac{X}{\varepsilon}}$ \emph{rounds of interactions} to output a policy $\piout$ such that
    $
        \bbE \sbr*{J_{\scriptscriptstyle{\cM}}^{\expert} - J_{\scriptscriptstyle{\cM}}^{\piout}}
        \leq
        \varepsilon
    $.\footnote{
    The expectation is over all randomness in the protocol, including the trajectories sampled from the MDP, the actions sampled from the expert policy, and the internal randomness of $\texttt{Alg}$.
    }
\end{restatable}

We recall that the results of \citet{joshi2025learning} establish that reward realizability is sufficient in the offline setting when $H=1$. By sharp contrast, our lower bound establishes that this is insufficient even when $H = 2$, and even with interaction. As a simple corollary of \cref{thm:reward-realizability-lb}, either the number of environment interactions \emph{or} the number of expert queries must be $\Omega \spr{\min \scbr{\abs{\cX}, \log \spr{\abs{\cF}}} / \varepsilon}$, and thus it is not possible to be polynomial in the relevant problem parameters ($\log \spr{\abs{\cR}}, H, \varepsilon^{-1}, \log \spr{\delta^{-1}}$) for both resources simultaneously.

Our results provide some worst-case impossibility results for the ``reward'' moment matching template of \citet{swamy2021moments}. In general, to avoid the ``tabular'' rate $\Omega \spr{\abs{\cX}}$, we expect that algorithms for this setting must estimate functionals of the dynamics, and thus additional structural conditions on the environment or additional representational ability, e.g. of the dynamics \citep{Liu:2022,kidambi2021mobile,viano2024imitation,moulin2025optimistically} or temporal differences of candidate $Q$ functions \citep{Garg:2021}, may be needed. We leave this as an interesting direction for future work.\loose 

Intuitively, the source of hardness leading to \cref{thm:reward-realizability-lb} is that, even though the reward provides information on which states are rewarding, it does not provide information on which actions lead to those states. Consequently, the learner cannot use the reward information to learn how to act at step $h=1$, and must resort to \emph{either} using an expert query and cloning the received action \emph{or} applying a learner's action in the environment and using the observed next state to learn the dynamics. Overall this takes $\min \scbr{\abs{\cX}, \log \spr{\abs{\cF}}} / \varepsilon$ samples. By contrast, \cref{thm:pi_first_main} shows that this transition-dependent information is encoded in $\qexpert$ and can be exploited by \ISPIL.

All in all, \textbf{we view $\qexpert$-realizability as a minimal representational assumption for efficient learning}: it is more informative than reward realizability alone, which is insufficient, but less demanding than other assumptions (\eg, policy realizability, realizability of $Q^\pi$ for infinitely-many policies $\pi$, or realizability of rewards and transitions) discussed in this paper and prior work. We next provide empirical evidence that the minimality of $\qexpert$-realizability translates to more effective IL when using learners with limited expressivity.\loose

\subsection{Experiments}
\label{sec:experiments}

We test whether the representational benefits suggested by \cref{sec:value-realizability-weaker} translate into practical performance when the learner has limited capacity. In each of four Gymnasium environments \citep{towers2024gymnasium}, the expert is an RL-trained policy represented by a network with two hidden layers of $64$ neurons each, while the learners use the same depth but have widths in $\{2,4,8,16,32,64\}$. We use network width as a proxy for representational capacity: as the learner's network becomes smaller relative to that of the expert, representing the expert's action distribution should become increasingly difficult, and policy-based methods such as \BC and \DAGGER may therefore degrade relative to \ISPIL. The code for these experiments is available at \url{https://github.com/lviano/ovi}.

We evaluate \ISPIL, \SPOIL, \BC, and \DAGGER across the above range of widths. We use $10$ expert trajectories for the offline methods, $10$ learner-trajectory labeling rounds for the interactive methods, and $50$ seeds (see \cref{sec:experiments_details} for more experimental details). \Cref{fig:size_scaling} shows that when the learner also has $64$ neurons per layer, \BC and \DAGGER perform well in all four environments. Their performance degrades, however, as the learner network shrinks. In contrast, \ISPIL dominates across all learner sizes and attains good performance even at the smallest width. \SPOIL generally deteriorates faster than \ISPIL, especially in \CARTPOLE and \ACROBOT, consistent with the possibility that $\qexpert$-realizability is easier to satisfy in these experiments than the closure-type realizability conditions required by \SPOIL. We use a stationary approximation to \ISPIL in these experiments; see \Cref{alg:approximateOVI} for pseudocode.\loose

Finally, we compare \DAGGER and \ISPIL in a synthetic $\qexpert$-realizable MDP to highlight that, despite both IL algorithms being interactive, \ISPIL implements a different imitation mechanism than \DAGGER.

\para{(\cref{fig:section4_experiments}) \DAGGER \vs \ISPIL: a different imitation mechanism}
Although also an interactive IL method, \DAGGER is surprisingly suboptimal in this environment, which is an extended version of the simple $\qexpert$-realizable instance used in \cref{thm:lower} (defined for \cref{thm:lower_any_algo} and illustrated in \Cref{fig:extended_hard}). As displayed in \Cref{fig:section4_experiments}, the reward suboptimality, \ie, $J^{\expert} - J^{\piout}$, of \ISPIL (solid red line) converges significantly faster than the suboptimality of \DAGGER (solid purple line). The advantage of \ISPIL stems from the fact that it directly matches the expert's return, while \DAGGER, being a policy-based method, attempts to match the expert's actions almost everywhere, which is a harder problem in the MDP under consideration because there are $\cO \spr{2^{\abs{\cX}}}$ possible optimal policies from which the expert can choose. In comparison, \ISPIL acts approximately\footnote{The \ISPIL iterates become more and more greedy with respect to $\qexpert$ as $K$ increases.} greedily with respect to its learned $\qexpert$, which is sufficient to achieve the optimal return with significantly fewer samples but at the cost of potentially learning a policy different from the expert's. To illustrate the difference between the learning mechanisms of these two algorithms, we plot in \cref{fig:section4_experiments} their trajectory-level TV distance to the expert, $\DTV{\bbP^\expert}{\bbP^{\piout}}$, which decreases under \DAGGER (dashed purple line) but remains constant for \ISPIL (dashed red line).\loose

%

\section{Interaction Is Necessary for \texorpdfstring{$\qexpert$}{q-expert}-Realizable Value-Based Imitation}
\label{sec:lb}
In \cref{sec:ub}, we showed that interaction enables \ISPIL to succeed under only $\qexpert$-realizability, a weaker representational requirement than those imposed by comparable algorithms. Here, we prove that interaction is in fact \emph{necessary}: under $\qexpert$-realizability alone, \emph{no offline IL method can learn efficiently}. We formally define an offline IL algorithm below.
\begin{definition}[Offline IL Algorithm] \label{def:IL}
    An offline IL algorithm $\mathrm{Alg}$ is a (potentially randomized) mapping from an expert dataset $\cDE$ (sampled according to \cref{eq:expert-dataset}) to a nonstationary policy $\spr{\piout_h\colon \cX \to \Delta \spr{\cA}}_{h=1}^H$.\loose
\end{definition}

Our main negative result (\Cref{thm:lower_any_algo}), stated formally below, establishes that there exists a large family of MDPs (and corresponding optimal expert policies) with a small value function class exhibiting $\qexpert$-realizability for each MDP in the class, but for which the sample complexity of any offline IL algorithm must scale with either the size of the state space or the number of MDPs in the class. For the statement below, we recall that, for an MDP $\cM$, we write $J^{\pi}_{\scriptscriptstyle{\cM}}$ for the return of a policy $\pi$, and $\pi^\star_{\scriptscriptstyle{\cM}}$ for one of its optimal deterministic policies.
\begin{restatable}[Main lower bound for offline IL]
{theorem}{LowerAny} \label{thm:lower_any_algo}
    For any $X \in \bbN^\star$ and any $\varepsilon \in (0,\frac{1}{8}]$,
    there exists a family $\cF$ of MDPs with: i) stochastic initial states but deterministic rewards and transitions, ii) horizon $H = 2$, iii) state space size $\abs{\cX} = \cO \spr{X}$, and iv) family size $\log \spr{\abs{\cF}} = \cO \spr{X}$, as well as a value function class $\cQ$ with $\abs{\cQ} = 2$ such that for every $\cM \in \cF$, \cref{asp:q-expert-realizability} is satisfied for expert policy $\expert = \pi^\star_{\scriptscriptstyle\!\cM}$. Let $\texttt{Alg}$ be any offline IL algorithm (\cref{def:IL}). Then, there exists an MDP $\cM \in \cF$ such that $\mathrm{Alg}$ needs  $\tauE \geq \Omega \spr{\frac{X}{\varepsilon}}$ trajectories sampled according to \cref{eq:expert-dataset} with expert policy $\pi^\star_{\scriptscriptstyle\!\cM}$, to output a policy $\piout$ such that $\bbE \sbr{J_{\scriptscriptstyle\!\cM}^{\pi^\star_{\scriptscriptstyle\!\cM}} - J_{\scriptscriptstyle\!\cM}^{\piout}} \leq \varepsilon$.\footnote{The expectation is over the training data and the internal randomness of $\texttt{Alg}$.}
\end{restatable}

In other words, any offline imitation learner needs at least $\Omega \spr{\min \scbr{\abs{\cX}, \log \spr{\abs{\cF}}} / \varepsilon}$ samples to obtain a near-optimal policy, even when there exists a small value function class satisfying $\qexpert$-realizability for every MDP in the problem class.\footnote{By contrast, \ISPIL succeeds on the MDPs defined in the proof of \cref{thm:lower_any_algo} with one expert query outside the support of its state distribution.} We note that similar ``tabular'' lower bound rates have appeared for offline IL \citep{rajaraman2020toward,rajaraman2021provablybreaking, foster2024behavior}, though with a focus on establishing the optimal dependence on horizon. By comparison, our lower bound shows that the worst-case offline IL rate continues to depend on the number of states even when there exists a small value function class for the problem class, and even for $H = 2$.

In \cref{tab:tractability}, we summarize various representational assumptions considered in this work and whether they are statistically tractable or intractable in the offline or interactive IL settings.

\begin{table}[H]
\centering
\renewcommand{\arraystretch}{1.2}
\begin{tabular*}{0.86\textwidth}{@{\extracolsep{\fill}}lccc}
\toprule
\textbf{Representation}
&
\textbf{Reward $r$}
&
\textbf{Value} $\qexpert$
&
\textbf{Policy} $\expert$\\
\midrule
\textbf{Offline IL} (\cref{def:IL})
& \badcell\,\, (\cref{thm:reward-realizability-lb})
& \badcell\,\, (\cref{thm:lower_any_algo})
& \goodcell \,\,(\eg, \BC) \\
\textbf{Interactive IL} (\cref{def:interactive-il-algo})
& \badcell \,\, (\cref{thm:reward-realizability-lb})
& \goodcell$^\star$ \,\, (\ISPIL; \cref{thm:pi_first_main})
& \goodcell \,\, (\eg, \DAGGER)\\
\bottomrule
\end{tabular*}
\caption{Tractability of imitation learning (for $H > 1$) under different representational assumptions. \underline{\goodcell}: A $\mathrm{poly} \spr{\log \spr{\abs{\cF}}, H, \varepsilon^{-1}, \log \spr{\delta^{-1}}}$ sample complexity is possible, where $\cF$ is a given finite reward/value/policy class satisfying realizability (Assumptions \ref{ass:reward-realizability}, \ref{asp:q-expert-realizability}, and \ref{ass:policy-realizability}, respectively). \underline{\badcell}: Said polynomial sample complexity is not possible. \underline{$^\star$}: incurs additional $\log \spr{A}$ dependence.\loose}
\label{tab:tractability}
\end{table}

\pagebreak[4]
\vspace{-\baselineskip}
\begin{samepage}
    Next, we show a stronger negative result for any algorithm which only uses \emph{value-induced} policies, defined below.\loose\par\nopagebreak[4]
    \begin{definition}[Value-induced policy class] \label{def:valueIL}
        An offline IL algorithm $\mathrm{Alg}$ (\cref{def:IL}) uses a \emph{value-induced} (VI) policy class (for value function class $\cQ$) if it outputs policies of the form $\piout_h \spr{a \given x} \propto f \spr{w\transpose \spr{Q_h^{1:K}} \spr{x, a}}$, for some parameter $K \in \bbN$, sequence $Q^{1:K} \in \cQ^K$, linear parameter $w \in \bbR^K$, and function $f\colon \bbR \to \bbR_+$.
    \end{definition}
\end{samepage}
The VI policy class encompasses the policies used by a broad class of algorithms that extract policies from value functions by applying activation functions (\eg, softmax, ReLU, greedy) to linear combinations of value functions from $\cQ$.\footnote{We note that greedy policies are only included as pointwise limits of other VI policies, as shown in \Cref{sec:proofs_lb}.} In particular, it includes the implicit policy classes used by prior theoretical and empirical value-based IL methods \citep{moulin2025inverse,joshi2025learning,Garg:2021,watson2023coherent} as well as by \ISPIL. See \Cref{sec:proofs_lb} for details. The following lower bound shows that, under $\qexpert$-realizability alone, offline learners with VI policy classes fail to compete with the expert's performance, even in the limit of infinite data.\loose

\begin{restatable}[Unidentifiability under VI policies]{theorem}{Lower} \label{thm:lower}
    There exists a family $\cF$ of MDPs with: i) deterministic initial states, rewards, and transitions, ii) horizon $H = 2$, iii) state space size $\abs{\cX} = 3$, and iv) family size $\abs{\cF} = 2$, as well as a value function class $\cQ$ with $\abs{\cQ} = 2$ such that for every $\cM \in \cF$, \cref{asp:q-expert-realizability} is satisfied for the expert $\expert = \pi^\star_{\!\scriptscriptstyle{\cM}}$. Moreover, for any algorithm $\texttt{Alg}$ (\cref{def:IL}) with a VI policy class (\cref{def:valueIL}), for any $\tauE \in \bbN^\star$, there exists an MDP $\cM \in \cF$ such that, given $\tauE$ expert trajectories sampled from $\expert$, $\texttt{Alg}$ outputs policies $\piout$ such that $\bbE\sbr{J_{\scriptscriptstyle\!\cM}^\expert - J_{\scriptscriptstyle\!\cM}^{\piout}} \geq \frac14$.\footnote{The expectation is over the training data and the internal randomness of $\texttt{Alg}$.}
\end{restatable}

\begin{figure}
    \centering
    \includegraphics[width=0.65\linewidth]{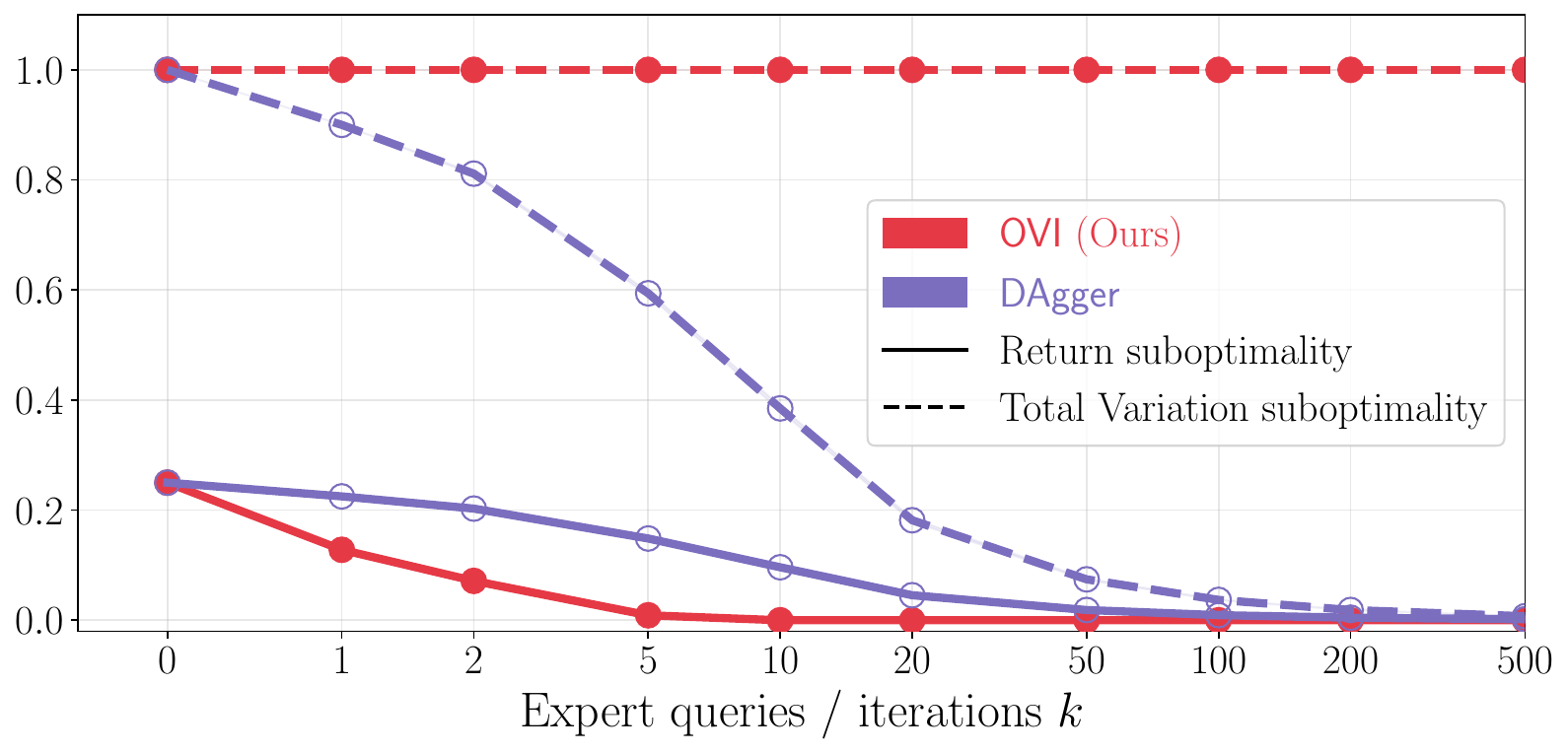}
    \caption{Comparison of reward suboptimality, $J^{\expert} - J^{\pi_{\texttt{out}}}$, and total variation distance, $\DTV{\bbP^{\expert}}{\bbP^{\pi_{\texttt{out}}}}$ for \ISPIL and \DAGGER in a synthetic environment satisfying \cref{asp:q-expert-realizability} (illustrated in \Cref{fig:extended_hard}). \ISPIL matches expert return (see solid lines) without matching the expert trajectory distribution (see dashed lines).\loose}
    \label{fig:section4_experiments}
    \vspace{-0.5em}
\end{figure}

The intuition for \cref{thm:lower} is that the value-induced policies must randomize uniformly among actions with equally high value, even though the expert itself may only play a subset of those actions. In the construction of \cref{thm:lower} (illustrated in \Cref{fig:hardMDPenv}), this implies that the learner's output policy goes out of distribution with constant probability, thus incurring a constant error rate for all dataset sizes. We expect that unidentifiability continues to hold under broader definitions of the VI policy class which assign probabilities to actions based on their \emph{value profiles} \citep{sun2019model}.

Our lower bounds in this section (\Cref{thm:lower_any_algo} and \Cref{thm:lower}) show that interaction is necessary in general for learning with only $\qexpert$-realizability. We conclude this section by briefly discussing how the interaction and representational requirements can be relaxed in certain settings.

\para{Coverage circumvents the offline lower bound}
If the expert induces a distribution over the state space with sufficient coverage, then distribution shift no longer causes issues for offline algorithms that output a VI policy such as \SPOIL (\Cref{sec:break_lb_with_coverage}). In particular, in \Cref{thm:SPOIL}, we prove sample complexity guarantees for \SPOIL under $\qexpert$-realizability, expert optimality and a standard coverage condition \citep{munos2003error,antos2008learning,chen2019information,xie2023role,amortila2024scalable,jiang2025offline}. However, we do not consider this an easy assumption to satisfy in practice.\loose

\para{Representational benefits of mixing expert and learner rollouts}
Finally, we show that additional representational benefits can arise when the dataset includes trajectories from \emph{both} the learner and the expert, with the expert being queried at all such states (a technique commonly used in LM distillation; \citealp{agarwal2024policy,li2026revisiting}). We show that a variant of \ISPIL that mixes learner and expert trajectories succeeds whenever $\cQ$ satisfies either $\qexpert$-realizability (\cref{asp:q-expert-realizability}) or the \SPOIL-like closure condition, which requires $Q^\pi \in \cQ$ for every policy $\pi$ generated by the learner, without knowing in advance which condition holds (\Cref{thm:adaptive-hybrid}).\loose


\section{\mbox{Chain-of-Thought Learning: Computational-Representational Tradeoffs}}
\label{sec:llm_tradeoff}
The preceding sections showed that interaction relaxes the representational requirements for the learner. Here, we show how \ISPIL can be instantiated in the setting of LM reasoning to achieve exponential improvements in computational efficiency over existing approaches \citep{joshi2025learning}, provided that the expert's value can be realized, thereby providing a complementary perspective on the value of interaction in IL.\loose

Reasoning tasks such as mathematics and code generation have become a central focus of modern LM research \citep{wei2022chain,li2022competition}. In these domains, objective evaluators (such as unit and integration tests for coding and formal verifiers for math) are often available and can be used as a ``ground truth'' reward function. To solve these problems, LMs are trained to generate a sequence of tokens, called a \emph{chain of thought} (CoT), before outputting the final answer. To acquire CoT reasoning, LMs can be trained via IL, commonly via next-token prediction \citep{ouyang2022training,cobbe2021training} on expert reasoning traces consisting of a prompt, the expert's complete chain-of-thought, and the correct final answer.\loose

CoT learning \citep{malach2023auto,joshi2025theory} has been formalized as follows. The LM is modeled as an autoregressive policy that, given a prompt $\mb{x} \in \cX$, samples a response $\mb{y} \in \cY$ via $\mb{y} \sim \pi \spr{\cdot \given \mb{x}}$, where $\mb{y} = \spr{\mb{y}_1, \ldots, \mb{y}_{H-1}, \mb{y}_H}$ is a sequence of $H$ tokens, belonging to a token vocabulary denoted by $\Sigma$, generated autoregressively as $\mb{y}_h \sim \pi \spr{\cdot \given \mb{x}, \mb{y}_{1}, \ldots, \mb{y}_{h-1}}$, $h = 1, \ldots, H$, and the last token $\mb{y}_H$ is the final answer. This interaction defines a \emph{token-level MDP} with horizon $H$, where states are the partial generations so far, $\mb{x}_h = \spr{\mb{x}, \mb{y}_{1:h-1}} \in \cX \times \Sigma^{h-1}$, the action is the next token $\mb{y}_h \in \Sigma$, and the next state is $\mb{x}_{h+1} = \spr{\mb{x}, \mb{y}_{1:h}}$. The learner is given access to a dataset $\cD = \scbr{\spr{\mb{x}, \mb{y}}}$, where the prompt $\mb{x}$ is sampled from a given prompt distribution, and the response is generated by the expert, $\mb{y} \sim \expert\spr{\cdot \given \mb{x}}$. The reward function is a binary \emph{outcome verifier} defined as $r\colon \cX \times \cY \rightarrow \scbr{0, 1}$ so that the expected return of a policy, $J^\pi = \bbE^\pi \sbr{r \spr{\mb{x}, \mb{y}}}$, is equivalent to its answer accuracy. The goal of IL then translates to the problem of producing a policy $\piout$ whose final-answer accuracy competes with the expert's.\loose

\para{Non-interactive reward-based IL is intractable computationally}
A natural simplification to CoT learning, exploiting that the dynamics are deterministic and known, is to treat the CoT problem as an $H=1$ problem (\ie, a contextual bandit), where the prompt $\mb{x}$ is the context and the response $\mb{y} \in \Sigma^H$ is an action. Under a reward (or outcome verifier) realizability assumption (\cref{ass:reward-realizability}), \citet{joshi2025learning} estimate a reward $\wh{r}$ from the demonstrations $\cD$ using an exponential-weights-type procedure which, at each iteration, must solve $\argmax_{y \in \Sigma^{H}} \wh{r} \spr{\mb{x}, y}$ to extract a response for a given prompt $\mb{x}$.\footnote{We note that our reward-based lower bound \cref{thm:reward-realizability-lb} does not apply in CoT learning since the dynamics in the token-level MDP are known.} Their procedure presents two computational barriers. First, the reward optimization is over the full response space $\cY = \Sigma^H$, which is computationally intractable for long horizons. Second, performing exponential weight updates over the reward class is infeasible as it requires enumerating over the reward class.

\para{Computational benefits of value-based interactive IL}
To exploit the inherently sequential nature of CoT learning, we can apply \ISPIL directly to the token-level MDP generated by the interaction of the prompt distribution, the LM, and the outcome verifier $r$. At each step $h$, \ISPIL learns a value function that naturally induces a softmax token sampling distribution (\Cref{line:pi-alg} of \cref{alg:interactive_finite-H}). This decomposes the previously intractable search over the response space $\Sigma^H$ into $H$ local search steps. Moreover, each stage of the problem is solved by \ISPIL using simple linear maximization oracles for the value function search (\Cref{line:test} of \cref{alg:interactive_finite-H}) and softmax updates for the policy (\Cref{line:pi-alg} of \cref{alg:interactive_finite-H}), which can be implemented at scale. Indeed, our approach suggests a loss function for the value network which can be minimized via backpropagation coupled with modern optimizers.\loose

It is worth noting that the computational benefits of \ISPIL over sequence-level methods like \citet{joshi2025learning} come at the cost of a stronger representational requirement (realizing the expert's value function for each step rather than simply the reward of the final answer) as well as interactive query access to the expert. Both of these additional assumptions have been widely shown to improve performance; the first can naturally be interpreted as requiring an architecture capable of expressing a process verifier (or process reward model; \citealp{lightman2023let,uesato2022solving}), as opposed to simply an outcome verifier, and the second corresponds to the setting of on-policy distillation \citep{agarwal2024policy,gu2024minillm,yang2025qwen3} which has been widely employed in language modeling. Our results therefore provide a novel theoretical justification for the use of process reward models and on-policy interaction in LM training. Whether these gains translate to practical value-based IL algorithms for language models is a very interesting empirical question.


\section{Conclusion}
\label{sec:conclusion}
What is the value of interaction in IL? In this work, we investigated the mechanisms through which interaction and value function estimation help in IL, revealing a perhaps surprising interplay: interaction relaxes the representational demands on the learner, allowing it to succeed while realizing only the expert's value function, as opposed to the expert's full policy or the value functions of all learner-generated policies. We formalized this through \ISPIL, an interactive value-based IL algorithm that is statistically and computationally efficient under $\qexpert$-realizability alone, and a complementary lower bound establishing that interaction is necessary for IL under this minimal assumption. Specializing to chain-of-thought reasoning, we further showed that these representational benefits can translate into exponential computational gains over non-interactive approaches \citep{joshi2025learning}. These findings suggest that the benefits of interaction in IL extend well beyond the classical understanding of improved horizon dependence.\loose

Several interesting future directions remain open. On the technical side, it would first be valuable to obtain an $\cO \spr{\veps^{-2}}$ rate for both expert queries and learner rollouts when $\cQ$ is nonconvex with a computationally efficient algorithm. Second, while under $\expert$-realizability nearly horizon-free bounds are possible \citep{foster2024behavior}, whether the same is achievable for value-based IL is an interesting open question. Finally, it would be interesting to study value-based IL in the presence of value function misspecification; we expect that the relevant forms of misspecification differ from their counterparts in value-based RL \citep{chen2019information,du2020good,amortila2023optimal,amortila2024mitigating,maran2026beyond}. On the empirical side, it would be valuable to understand whether \ISPIL (or value-based IL methods more broadly) can be effective in language model distillation, where expert policy realizability is unlikely to hold.\loose


\clearpage
\section*{Acknowledgments}
We thank Gergely Neu and Nan Jiang for interesting discussions about this project. LV was supported by the Swiss Data Science Center under fellowship number P22$\_$03. AM received funding from the European Research Council (ERC), under the European Union's Horizon 2020 research and innovation programme (Grant agreement No. 950180). PA gratefully acknowledges the support of DARPA through award No. HR00112520022. Part of this research was performed while AM and AH were visiting the Institute for Mathematical and Statistical Innovation (IMSI), which is supported by the National Science Foundation (Grant No. DMS-2425650).


\bibliography{ref_neurips}

@inproceedings{gu2024minillm,
  title        = {{MiniLLM}: Knowledge Distillation of Large Language Models},
  author       = {Gu, Yuxian and Dong, Li and Wei, Furu and Huang, Minlie},
  booktitle    = {International Conference on Learning Representations},
  year         = {2024},
  url          = {https://openreview.net/forum?id=5h0qf7IBZZ}
}

@inproceedings{amortila2024mitigating,
  title={Mitigating covariate shift in misspecified regression with applications to reinforcement learning},
  author={Amortila, Philip and Cao, Tongyi and Krishnamurthy, Akshay},
  booktitle={Conference on Learning Theory},
  year={2024},
  url = {https://arxiv.org/abs/2401.12216}
}

@inproceedings{amortila2024scalable,
  title={Scalable Online Exploration via Coverability},
  author={Amortila, Philip and Foster, Dylan J and Krishnamurthy, Akshay},
  booktitle={International Conference on Machine Learning},
  year={2024},
  url = {https://arxiv.org/abs/2403.06571}
}

@article{antos2008learning,
  title={Learning near-optimal policies with {Bellman}-residual minimization based fitted policy iteration and a single sample path},
  author={Antos, Andr{\'a}s and Szepesv{\'a}ri, Csaba and Munos, R{\'e}mi},
  journal={Machine Learning},
  year={2008},
  publisher={Springer},
  url = {https://link.springer.com/article/10.1007/s10994-007-5038-2}
}

@article{yang2025qwen3,
  title        = {{Qwen3} Technical Report},
  author       = {Yang, An and Li, Anfeng and Yang, Baosong and Zhang, Beichen and Hui, Binyuan and Zheng, Bo and Yu, Bowen and Gao, Chang and Huang, Chengen and others},
  journal      = {arXiv preprint arXiv:2505.09388},
  year         = {2025},
  url          = {https://arxiv.org/abs/2505.09388}
}

@InProceedings{efroni2022sample,
  title = 	 {Sample-Efficient Reinforcement Learning in the Presence of Exogenous Information},
  author =       {Efroni, Yonathan and Foster, Dylan J and Misra, Dipendra and Krishnamurthy, Akshay and Langford, John},
  booktitle = 	 {Conference on Learning Theory},
  year = 	 {2022},
  url = 	 {https://proceedings.mlr.press/v178/efroni22a.html},
}

@inproceedings{mhammedi2024power,
title={The Power of Resets in Online Reinforcement Learning},
author={Zakaria Mhammedi and Dylan J Foster and Alexander Rakhlin},
booktitle={Advances in Neural Information Processing Systems},
year={2024},
url={https://openreview.net/forum?id=7sACcaOmGi}
}

@inproceedings{malach2023auto,
  title={Auto-Regressive Next-Token Predictors are Universal Learners},
  author={Malach, Eran},
  booktitle={International Conference on Machine Learning},
  year={2024},
  url={https://arxiv.org/abs/2309.06979}
}

@inproceedings{
li2025near,
title={Near-Optimal Second-Order Guarantees for Model-Based Adversarial Imitation Learning},
author={Shangzhe Li and Dongruo Zhou and Weitong Zhang},
booktitle={International Conference on Learning Representations},
year={2026},
url={https://openreview.net/forum?id=PD8wnZOV1J}
}

@article{wulfmeier2015maximum,
  title={Maximum entropy deep inverse reinforcement learning},
  author={Wulfmeier, Markus and Ondruska, Peter and Posner, Ingmar},
  journal={arXiv preprint arXiv:1507.04888},
  year={2015},
  url={https://arxiv.org/abs/1507.04888}
}

@inproceedings{xu2026non,
title={Non-Adversarial Imitation Learning Provably Free of Compounding Errors: The Value Flow Mechanism},
author={Tian Xu and Chenyang Wang and Xiaochen Zhai and Ziniu Li and Yi-Chen Li and Yang Yu},
booktitle={International Conference on Machine Learning},
year={2026},
url={https://openreview.net/forum?id=1UUJrgYr20}
}

@InProceedings{laskey2017dart,
  title = 	 {{DART}: Noise Injection for Robust Imitation Learning},
  author = 	 {Laskey, Michael and Lee, Jonathan and Fox, Roy and Dragan, Anca and Goldberg, Ken},
  booktitle = 	 {Conference on Robot Learning},
  year = 	 {2017},
  url = 	 {https://proceedings.mlr.press/v78/laskey17a.html},
}

@inproceedings{amortila2024reinforcement,
title={Reinforcement Learning Under Latent Dynamics: Toward Statistical and Algorithmic Modularity},
author={Philip Amortila and Dylan J Foster and Nan Jiang and Akshay Krishnamurthy and Zakaria Mhammedi},
booktitle={Advances in Neural Information Processing Systems},
year={2024},
url={https://openreview.net/forum?id=qf2uZAdy1N}
}

@InProceedings{chen2019information,
  title = 	 {Information-Theoretic Considerations in Batch Reinforcement Learning},
  author =       {Chen, Jinglin and Jiang, Nan},
  booktitle = 	 {International Conference on Machine Learning},
  year = 	 {2019},
  url = 	 {https://proceedings.mlr.press/v97/chen19e.html},
}

@inproceedings{foster2022offline,
  title={Offline Reinforcement Learning: Fundamental Barriers for Value Function Approximation},
  author={Foster, Dylan J and Krishnamurthy, Akshay and Simchi-Levi, David and Xu, Yunzong},
  booktitle={Conference on Learning Theory},
  year={2022},
  url={https://arxiv.org/abs/2111.10919}
}

@InProceedings{foster2025good,
  title = 	 {Is a Good Foundation Necessary for Efficient Reinforcement Learning? {The} Computational Role of the Base Model in Exploration},
  author =       {Foster, Dylan J and Mhammedi, Zakaria and Rohatgi, Dhruv},
  booktitle = 	 {Conference on Learning Theory},
  year = 	 {2025},
  url = 	 {https://proceedings.mlr.press/v291/foster25a.html},
}

@article{jiang2025offline,
  title={Offline reinforcement learning in large state spaces: Algorithms and guarantees},
  author={Jiang, Nan and Xie, Tengyang},
  journal={Statistical Science},
  year={2025},
  url={https://arxiv.org/abs/2510.04088}
}

@inproceedings{block2023butterfly,
title={Butterfly Effects of {SGD} Noise: Error Amplification in Behavior Cloning and Autoregression},
author={Adam Block and Dylan J Foster and Akshay Krishnamurthy and Max Simchowitz and Cyril Zhang},
booktitle={International Conference on Learning Representations},
year={2024},
url={https://openreview.net/forum?id=CgPs04l9TO}
}

@inproceedings{de2019causal,
 author = {de Haan, Pim and Jayaraman, Dinesh and Levine, Sergey},
 booktitle = {Advances in Neural Information Processing Systems},
 title = {Causal Confusion in Imitation Learning},
 url = {https://proceedings.neurips.cc/paper_files/paper/2019/file/947018640bf36a2bb609d3557a285329-Paper.pdf},
 year = {2019}
}

@article{spencer2021feedback,
  title={Feedback in imitation learning: The three regimes of covariate shift},
  author={Spencer, Jonathan and Choudhury, Sanjiban and Venkatraman, Arun and Ziebart, Brian and Bagnell, J Andrew},
  journal={arXiv preprint arXiv:2102.02872},
  year={2021},
  url={https://arxiv.org/abs/2102.02872}
}

@inproceedings{hu2025rac,
title={{RaC}: Robot Learning for Long-Horizon Tasks by Scaling Recovery and Correction},
author={Zheyuan Hu and Robyn Wu and Naveen Enock and Jasmine Jia-ni Li and Riya Kadakia and Zackory Erickson and Aviral Kumar},
booktitle={Workshop on Making Sense of Data in Robotics: Composition, Curation, and Interpretability at Scale at CoRL 2025},
year={2025},
url={https://openreview.net/forum?id=y8wskVS7BV}
}

@inproceedings{Abbeel:2004,
	title = {Apprenticeship learning via inverse reinforcement learning},
	author = {Abbeel, Pieter and Ng, Andrew Y.},
	year = {2004},
	booktitle = {International Conference on Machine Learning},
	url = {https://icml.cc/Conferences/2004/proceedings/papers/335.pdf}
}

@inproceedings{Abbeel:2008,
	title = {Apprenticeship learning for motion planning with application to parking lot navigation},
	author = {Abbeel, Pieter and Dolgov, Dmitri and Ng, Andrew Y. and Thrun, Sebastian},
	year = {2008},
	booktitle = {IEEE/RSJ International Conference on Intelligent Robots and Systems},
	url = {https://ai.stanford.edu/~ang/papers/iros08-ApprenticeshipLearningParkingLotNavigation.pdf}
}

@inproceedings{agarwal2024policy,
	title = {On-policy distillation of language models: Learning from self-generated mistakes},
	author = {Agarwal, Rishabh and Vieillard, Nino and Zhou, Yongchao and Stanczyk, Piotr and Garea, Sabela Ramos and Geist, Matthieu and Bachem, Olivier},
	year = {2024},
	booktitle = {International Conference on Learning Representations},
	url = {https://arxiv.org/abs/2306.13649}
}

@inproceedings{amortila2022few,
	title = {A Few Expert Queries Suffices for Sample-Efficient {RL} with Resets and Linear Value Approximation},
	author = {Philip Amortila and Nan Jiang and Dhruv Madeka and Dean Foster},
	year = {2022},
	booktitle = {Advances in Neural Information Processing Systems},
	url = {https://openreview.net/forum?id=d19Dsqtw421}
}

@article{Beck:2003,
	title = {Mirror descent and nonlinear projected subgradient methods for convex optimization},
	author = {Beck, Amir and Teboulle, Marc},
	year = {2003},
	journal = {Operations Research Letters},
	url = {https://www.tau.ac.il/~becka/3.pdf}
}

@book{Boucheron:2013,
	title = {Concentration inequalities: A nonasymptotic theory of independence},
	author = {Boucheron, St{\'e}phane and Lugosi, G{\'a}bor and Massart, Pascal},
	year = {2013},
	publisher = {Oxford University Press}
}

@InProceedings{boularias2011relative,
  title = 	 {Relative Entropy Inverse Reinforcement Learning},
  author = 	 {Boularias, Abdeslam and Kober, Jens and Peters, Jan},
  booktitle = 	 {International Conference on Artificial Intelligence and Statistics},
  year = 	 {2011},
  url = 	 {https://proceedings.mlr.press/v15/boularias11a.html},
}

@article{bubeck2012regret,
	title = {Regret analysis of stochastic and nonstochastic multi-armed bandit problems},
	author = {Bubeck, S{\'e}bastien and Cesa-Bianchi, Nicolo},
	year = {2012},
	journal = {Foundations and Trends in Machine Learning},
	publisher = {Now Publishers, Inc.},
	url = {https://arxiv.org/abs/1204.5721}
}

@article{cao2026understanding,
	title = {Understanding Behavior Cloning with Action Quantization},
	author = {Cao, Haoqun and Xie, Tengyang},
	year = {2026},
	journal = {arXiv preprint arXiv:2603.20538},
	url = {https://arxiv.org/abs/2603.20538}
}

@book{Cesa-Bianchi:2006,
	title = {Prediction, learning, and games},
	author = {Cesa-Bianchi, Nicolo and Lugosi, G{\'a}bor},
	year = {2006},
	publisher = {Cambridge University Press}
}

@article{cheng2018fast,
	title = {Fast policy learning through imitation and reinforcement},
	author = {Cheng, Ching-An and Yan, Xinyan and Wagener, Nolan and Boots, Byron},
	year = {2018},
	journal = {arXiv preprint arXiv:1805.10413},
	url = {https://arxiv.org/abs/1805.10413}
}

@article{cobbe2021training,
	title = {Training verifiers to solve math word problems},
	author = {Cobbe, Karl and Kosaraju, Vineet and Bavarian, Mohammad and Chen, Mark and Jun, Heewoo and Kaiser, Lukasz and others},
	year = {2021},
	journal = {arXiv preprint arXiv:2110.14168},
	url = {https://arxiv.org/abs/2110.14168}
}

@inproceedings{dadashi2021continuous,
	title = {Continuous Control with Action Quantization from Demonstrations},
	author = {Dadashi, Robert and Hussenot, L{\'e}onard and Vincent, Damien and Girgin, Sertan and Raichuk, Anton and Geist, Matthieu and Pietquin, Olivier},
	year = {2022},
	booktitle = {International Conference on Machine Learning},
	url = {https://proceedings.mlr.press/v162/dadashi22a.html}
}

@inproceedings{espinosa2025efficient,
	title = {Efficient Imitation under Misspecification},
	author = {Espinosa-Dice, Nicolas and Choudhury, Sanjiban and Sun, Wen and Swamy, Gokul},
	year = {2025},
	booktitle = {International Conference on Learning Representations},
	url = {https://openreview.net/forum?id=fn36V5qsCw}
}

@InProceedings{finn2016guided,
  title = 	 {Guided Cost Learning: Deep Inverse Optimal Control via Policy Optimization},
  author = 	 {Finn, Chelsea and Levine, Sergey and Abbeel, Pieter},
  booktitle = 	 {International Conference on Machine Learning},
  year = 	 {2016},
  url = 	 {https://proceedings.mlr.press/v48/finn16.html},
}

@inproceedings{foster2024behavior,
	title = {Is behavior cloning all you need? {Understanding} horizon in imitation learning},
	author = {Foster, Dylan J and Block, Adam and Misra, Dipendra},
	year = {2024},
	booktitle = {Advances in Neural Information Processing Systems},
	url = {https://openreview.net/forum?id=8KPyJm4gt5}
}

@inproceedings{Fu:2018,
  title={Learning Robust Rewards with Adversarial Inverse Reinforcement Learning},
  author={Justin Fu and Katie Luo and Sergey Levine},
  booktitle={International Conference on Learning Representations},
  year={2018},
  url={https://openreview.net/forum?id=rkHywl-A-},
}

@inproceedings{Garg:2021,
	title = {{IQ-Learn}: Inverse soft-{Q} Learning for Imitation},
	author = {Divyansh Garg and Shuvam Chakraborty and Chris Cundy and Jiaming Song and Matthieu Geist and Stefano Ermon},
	year = {2021},
	booktitle = {Advances in Neural Information Processing Systems},
	url = {https://arxiv.org/abs/2106.12142},
	note = {Note: Read arXiv version for correct version and complete author list}
}

@inproceedings{haarnoja2017reinforcement,
	title = {Reinforcement learning with deep energy-based policies},
	author = {Haarnoja, Tuomas and Tang, Haoran and Abbeel, Pieter and Levine, Sergey},
	year = {2017},
	booktitle = {International Conference on Machine Learning},
	url = {https://proceedings.mlr.press/v70/haarnoja17a.html}
}

@inproceedings{Ho:2016,
	title = {Model-Free Imitation Learning with Policy Optimization},
	author = {Ho, Jonathan and Gupta, Jayesh and Ermon, Stefano},
	year = {2016},
	booktitle = {International Conference on Machine Learning},
	url = {https://proceedings.mlr.press/v48/ho16.html}
}

@inproceedings{Ho:2016b,
 author = {Ho, Jonathan and Ermon, Stefano},
 booktitle = {Advances in Neural Information Processing Systems},
 title = {Generative Adversarial Imitation Learning},
 url = {https://proceedings.neurips.cc/paper_files/paper/2016/file/cc7e2b878868cbae992d1fb743995d8f-Paper.pdf},
 year = {2016}
}

@book{howard1960dynamic,
	title = {Dynamic programming and {Markov} processes.},
	author = {Howard, Ronald A},
	year = {1960},
	publisher = {John Wiley}
}

@inproceedings{joshi2025learning,
	title = {Learning to Answer from Correct Demonstrations},
	author = {Nirmit Joshi and Gene Li and Siddharth Bhandari and Shiva Kasiviswanathan and Cong Ma and Nathan Srebro},
	year = {2026},
	booktitle = {International Conference on Learning Representations},
	url = {https://openreview.net/forum?id=69fIHgLjyH}
}

@inproceedings{joshi2025theory,
	title = {A Theory of Learning with Autoregressive Chain of Thought},
	author = {Joshi, Nirmit and Vardi, Gal and Block, Adam and Goel, Surbhi and Li, Zhiyuan and Misiakiewicz, Theodor and Srebro, Nathan},
	year = {2025},
	booktitle = {Conference on Learning Theory},
	url = {https://arxiv.org/abs/2503.07932}
}

@inproceedings{kakade2002approximately,
	title = {Approximately optimal approximate reinforcement learning},
	author = {Kakade, Sham and Langford, John},
	year = {2002},
	booktitle = {International Conference on Machine Learning},
	url = {https://homes.cs.washington.edu/~sham/papers/rl/aoarl.pdf}
}

@inproceedings{kidambi2021mobile,
	title = {{MobILE}: Model-based imitation learning from observation alone},
	author = {Kidambi, Rahul and Chang, Jonathan and Sun, Wen},
	year = {2021},
	booktitle = {Advances in Neural Information Processing Systems},
	url = {https://arxiv.org/abs/2102.10769}
}

@inproceedings{Kostrikov:2019,
title={Discriminator-Actor-Critic: Addressing Sample Inefficiency and Reward Bias in Adversarial Imitation Learning},
author={Ilya Kostrikov and Kumar Krishna Agrawal and Debidatta Dwibedi and Sergey Levine and Jonathan Tompson},
booktitle={International Conference on Learning Representations},
year={2019},
url={https://openreview.net/forum?id=Hk4fpoA5Km},
}

@inproceedings{Kostrikov:2020,
title={Imitation Learning via Off-Policy Distribution Matching},
author={Ilya Kostrikov and Ofir Nachum and Jonathan Tompson},
booktitle={International Conference on Learning Representations},
year={2020},
url={https://openreview.net/forum?id=Hyg-JC4FDr}
}

@inproceedings{
li2022efficient,
title={On Efficient Online Imitation Learning via Classification},
author={Yichen Li and Chicheng Zhang},
booktitle={Advances in Neural Information Processing Systems},
year={2022},
url={https://openreview.net/forum?id=h2imPVlCCyN}
}

@inproceedings{lightman2023let,
	title = {Let's verify step by step},
	author = {Lightman, Hunter and Kosaraju, Vineet and Burda, Yuri and Edwards, Harrison and Baker, Bowen and Lee, Teddy and Leike, Jan and Schulman, John and Sutskever, Ilya and Cobbe, Karl},
	year = {2024},
	booktitle = {International Conference on Learning Representations},
	url = {https://openreview.net/forum?id=v8L0pN6EOi}
}

@inproceedings{Liu:2022,
	title = {Learning from Demonstration: Provably Efficient Adversarial Policy Imitation with Linear Function Approximation},
	author = {Liu, Zhihan and Zhang, Yufeng and Fu, Zuyue and Yang, Zhuoran and Wang, Zhaoran},
	year = {2022},
	booktitle = {International Conference on Machine Learning},
	url = {https://proceedings.mlr.press/v162/liu22u.html}
}

@inproceedings{liu2023exponential,
	title = {Exponential hardness of reinforcement learning with linear function approximation},
	author = {Liu, Sihan and Mahajan, Gaurav and Kane, Daniel and Lovett, Shachar and Weisz, Gell{\'e}rt and Szepesv{\'a}ri, Csaba},
	year = {2023},
	booktitle = {Conference on Learning Theory},
	url = {https://proceedings.mlr.press/v195/liu23b.html}
}

@inproceedings{xie2023role,
  title={The Role of Coverage in Online Reinforcement Learning},
  author={Xie, Tengyang and Foster, Dylan J and Bai, Yu and Jiang, Nan and Kakade, Sham M},
  booktitle={International Conference on Learning Representations},
  year={2023},
  url={https://arxiv.org/abs/2210.04157}
}

@inproceedings{sun2019model,
  title={Model-based {RL} in contextual decision processes: {PAC} bounds and exponential improvements over model-free approaches},
  author={Sun, Wen and Jiang, Nan and Krishnamurthy, Akshay and Agarwal, Alekh and Langford, John},
  booktitle={Conference on Learning Theory},
  year={2019},
  url={https://arxiv.org/abs/1811.08540}
}

@inproceedings{jiang2018open,
  title={Open problem: The dependence of sample complexity lower bounds on planning horizon},
  author={Jiang, Nan and Agarwal, Alekh},
  booktitle={Conference on Learning Theory},
  year={2018},
  url={https://proceedings.mlr.press/v75/jiang18a.html}
}

@inproceedings{wang2020long,
  title={Is long horizon {RL} more difficult than short horizon {RL}?},
  author={Wang, Ruosong and Du, Simon S and Yang, Lin and Kakade, Sham},
  booktitle={Advances in Neural Information Processing Systems},
  year={2020},
  url={https://arxiv.org/abs/2005.00527}
}

@inproceedings{wei2022chain,
  title={Chain-of-thought prompting elicits reasoning in large language models},
  author={Wei, Jason and Wang, Xuezhi and Schuurmans, Dale and Bosma, Maarten and Ichter, Brian and Xia, Fei and Chi, Ed and Le, Quoc V and Zhou, Denny and others},
  booktitle={Advances in Neural Information Processing Systems},
  year={2022},
  url={https://proceedings.neurips.cc/paper/2022/hash/9d5609613524ecf4f15af0f7b31abca4-Abstract-Conference.html}
}

@article{li2022competition,
  title={Competition-level code generation with {AlphaCode}},
  author={Li, Yujia and Choi, David and Chung, Junyoung and Kushman, Nate and Schrittwieser, Julian and Leblond, R{\'e}mi and Eccles, Tom and Keeling, James and Gimeno, Felix and Dal Lago, Agustin and others},
  journal={Science},
  year={2022},
  publisher={American Association for the Advancement of Science},
  url={https://arxiv.org/abs/2203.07814}
}

@inproceedings{ouyang2022training,
  title={Training language models to follow instructions with human feedback},
  author={Ouyang, Long and Wu, Jeffrey and Jiang, Xu and Almeida, Diogo and Wainwright, Carroll and Mishkin, Pamela and Zhang, Chong and Agarwal, Sandhini and Slama, Katarina and Ray, Alex and others},
  booktitle={Advances in Neural Information Processing Systems},
  year={2022},
  url={https://proceedings.neurips.cc/paper/2022/hash/b1efde53be364a73914f58805a001731-Abstract.html}
}

@inproceedings{zhang2021reinforcement,
  title={Is reinforcement learning more difficult than bandits? {A} near-optimal algorithm escaping the curse of horizon},
  author={Zhang, Zihan and Ji, Xiangyang and Du, Simon},
  booktitle={Conference on Learning Theory},
  year={2021},
  url={https://proceedings.mlr.press/v134/zhang21b}
}

@inproceedings{du2020good,
  title={Is a Good Representation Sufficient for Sample Efficient Reinforcement Learning?},
  author={Du, Simon S and Kakade, Sham M and Wang, Ruosong and Yang, Lin F},
  booktitle={International Conference on Learning Representations},
  year={2020},
  url={https://arxiv.org/abs/1910.03016}
}

@inproceedings{amortila2023optimal,
  title={The optimal approximation factors in misspecified off-policy value function estimation},
  author={Amortila, Philip and Jiang, Nan and Szepesv{\'a}ri, Csaba},
  booktitle={International Conference on Machine Learning},
  year={2023},
  url={https://proceedings.mlr.press/v202/amortila23a.html}
}

@inproceedings{maran2026beyond,
  title={Beyond least squares: Uniform approximation and the hidden cost of misspecification},
  author={Maran, Davide and Szepesv{\'a}ri, Csaba},
  booktitle={Advances in Neural Information Processing Systems},
  year={2025},
  url={https://papers.nips.cc/paper_files/paper/2025/file/f5ecc67e8b2f4941137bcbde901ac6be-Paper-Conference.pdf}
}

@article{lu2025onpolicydistillation,
	title = {On-Policy Distillation},
	author = {Kevin Lu and {Thinking Machines Lab}},
	year = {2025},
	journal = {Thinking Machines Lab: Connectionism},
	url = {https://thinkingmachines.ai/blog/on-policy-distillation}
}

@article{Mnih:2015,
	title = {Human-level control through deep reinforcement learning},
	author = {Mnih, Volodymyr and Kavukcuoglu, Koray and Silver, David and Rusu, Andrei A. and Veness, Joel and Bellemare, Marc G. and Graves, Alex and Riedmiller, Martin and Fidjeland, Andreas K. and Ostrovski, Georg and Petersen, Stig and Beattie, Charles and Sadik, Amir and Antonoglou, Ioannis and King, Helen and Kumaran, Dharshan and Wierstra, Daan and Legg, Shane and Hassabis, Demis},
	year = {2015},
	journal = {Nature},
	url = {https://www.nature.com/articles/nature14236}
}

@inproceedings{moulin2025inverse,
	title = {Inverse {Q}-Learning Done Right: Offline Imitation Learning in {$Q^\pi$}-Realizable {MDP}s},
	author = {Antoine Moulin and Gergely Neu and Luca Viano},
	year = {2025},
	booktitle = {Advances in Neural Information Processing Systems},
	url = {https://openreview.net/forum?id=tvEE9KQcLi}
}

@inproceedings{moulin2025optimistically,
	title = {Optimistically Optimistic Exploration for Provably Efficient Infinite-Horizon Reinforcement and Imitation Learning},
	author = {Moulin, Antoine and Neu, Gergely and Viano, Luca},
	year = {2025},
	booktitle = {Conference on Learning Theory},
	url = {https://proceedings.mlr.press/v291/moulin25a.html}
}

@inproceedings{munos2003error,
	title = {Error bounds for approximate policy iteration},
	author = {Munos, R{\'e}mi},
	year = {2003},
	booktitle = {International Conference on Machine Learning},
	url = {https://cdn.aaai.org/ICML/2003/ICML03-074.pdf}
}

@inproceedings{Nachum:2019a,
 author = {Nachum, Ofir and Chow, Yinlam and Dai, Bo and Li, Lihong},
 booktitle = {Advances in Neural Information Processing Systems},
 title = {{DualDICE}: Behavior-Agnostic Estimation of Discounted Stationary Distribution Corrections},
 url = {https://proceedings.neurips.cc/paper_files/paper/2019/file/cf9a242b70f45317ffd281241fa66502-Paper.pdf},
 year = {2019}
}

@inproceedings{Ng:2000,
	title = {Algorithms for inverse reinforcement learning},
	author = {Ng, Andrew Y. and Russell, Stuart J.},
	year = {2000},
	booktitle = {International Conference on Machine Learning},
	url = {https://ai.stanford.edu/~ang/papers/icml00-irl.pdf}
}

@article{orabona2023modern,
	title = {Online Learning: A Modern Introduction Using Convex Optimization},
	author = {Francesco Orabona},
	year = {2026},
	journal = {arXiv preprint arXiv:1912.13213},
	url = {https://arxiv.org/abs/1912.13213v10},
	eprint = {1912.13213},
	archiveprefix = {arXiv},
	primaryclass = {cs.LG}
}

@article{osa2018algorithmic,
	title = {An algorithmic perspective on imitation learning},
	author = {Osa, Takayuki and Pajarinen, Joni and Neumann, Gerhard and Bagnell, J Andrew and Abbeel, Pieter and Peters, Jan},
	year = {2018},
	journal = {Foundations and Trends in Robotics},
	publisher = {Emerald Publishing Limited},
	url = {https://arxiv.org/abs/1811.06711}
}

@article{Pomerleau:1991,
	title = {Efficient Training of Artificial Neural Networks for Autonomous Navigation},
	author = {Dean A. {Pomerleau}},
	year = {1991},
	journal = {Neural Computation},
	url = {https://ieeexplore.ieee.org/document/6796843}
}

@inproceedings{pomerleau1988alvinn,
	title = {{ALVINN}: An Autonomous Land Vehicle in a Neural Network},
	author = {Pomerleau, Dean A.},
	year = {1988},
	booktitle = {Advances in Neural Information Processing Systems},
	url = {https://proceedings.neurips.cc/paper_files/paper/1988/file/812b4ba287f5ee0bc9d43bbf5bbe87fb-Paper.pdf}
}

@inproceedings{rajaraman2020toward,
	title = {Toward the fundamental limits of imitation learning},
	author = {Rajaraman, Nived and Yang, Lin and Jiao, Jiantao and Ramchandran, Kannan},
	year = {2020},
	booktitle = {Advances in Neural Information Processing Systems},
	url = {https://proceedings.neurips.cc/paper_files/paper/2020/hash/1e7875cf32d306989d80c14308f3a099-Abstract.html}
}

@article{rajaraman2021provablybreaking,
	title = {Provably Breaking the Quadratic Error Compounding Barrier in Imitation Learning, Optimally},
	author = {Nived Rajaraman and Yanjun Han and Lin F. Yang and Kannan Ramchandran and Jiantao Jiao},
	year = {2021},
	journal = {arXiv preprint arXiv:2102.12948},
	url = {https://arxiv.org/abs/2102.12948},
	eprint = {2102.12948},
	archiveprefix = {arXiv},
	primaryclass = {cs.LG}
}

@inproceedings{rajaraman2021value,
	title = {On the value of interaction and function approximation in imitation learning},
	author = {Rajaraman, Nived and Han, Yanjun and Yang, Lin and Liu, Jingbo and Jiao, Jiantao and Ramchandran, Kannan},
	year = {2021},
	booktitle = {Advances in Neural Information Processing Systems},
	url = {https://proceedings.neurips.cc/paper_files/paper/2021/hash/09dbc1177211571ef3e1ca961cc39363-Abstract.html}
}

@inproceedings{Ratliff:2006,
	title = {Maximum margin planning},
	author = {Ratliff, N. D. and Bagnell, J. A. and Zinkevich, M. A.},
	year = {2006},
	booktitle = {International Conference on Machine Learning},
	url = {https://dl.acm.org/doi/10.1145/1143844.1143936}
}

@article{Reddy:2020,
	title = {{SQIL}: Imitation Learning via Reinforcement Learning with Sparse Rewards},
	author = {Siddharth Reddy and Anca D. Dragan and Sergey Levine},
	year = {2019},
	journal = {arXiv preprint arXiv:1905.11108},
	url = {https://arxiv.org/abs/1905.11108}
}

@inproceedings{rohatgi2025computational,
	title = {Computational-Statistical Tradeoffs at the Next-Token Prediction Barrier: Autoregressive and Imitation Learning under Misspecification},
	author = {Rohatgi, Dhruv and Block, Adam and Huang, Audrey and Krishnamurthy, Akshay and Foster, Dylan J.},
	year = {2025},
	booktitle = {Conference on Learning Theory},
	url = {https://proceedings.mlr.press/v291/rohatgi25a.html}
}

@inproceedings{Ross:2010,
	title = {Efficient reductions for imitation learning},
	author = {Ross, St{\'e}phane and Bagnell, Drew},
	year = {2010},
	booktitle = {International Conference on Artificial Intelligence and Statistics},
	url = {https://proceedings.mlr.press/v9/ross10a.html}
}

@inproceedings{Ross:2011,
	title = {A Reduction of Imitation Learning and Structured Prediction to No-Regret Online Learning},
	author = {Ross, Stephane and Gordon, Geoffrey and Bagnell, Drew},
	year = {2011},
	booktitle = {International Conference on Artificial Intelligence and Statistics},
	url = {https://proceedings.mlr.press/v15/ross11a.html}
}

@article{ross2014reinforcement,
	title = {Reinforcement and imitation learning via interactive no-regret learning},
	author = {Ross, Stephane and Bagnell, J Andrew},
	year = {2014},
	journal = {arXiv preprint arXiv:1406.5979},
	url = {https://arxiv.org/abs/1406.5979}
}

@inproceedings{Russell:1998,
	title = {Learning agents for uncertain environments (Extended Abstract)},
	author = {Stuart Russell},
	year = {1998},
	booktitle = {Conference on Computational Learning Theory},
	url = {https://dl.acm.org/doi/pdf/10.1145/279943.279964}
}

@InProceedings{Schulman:2015,
  title = 	 {Trust Region Policy Optimization},
  author = 	 {Schulman, John and Levine, Sergey and Abbeel, Pieter and Jordan, Michael and Moritz, Philipp},
  booktitle = 	 {International Conference on Machine Learning},
  year = 	 {2015},
  url = 	 {https://proceedings.mlr.press/v37/schulman15.html},
}

@article{Schulman:2017,
	title = {Proximal policy optimization algorithms},
	author = {Schulman, John and Wolski, Filip and Dhariwal, Prafulla and Radford, Alec and Klimov, Oleg},
	year = {2017},
	journal = {arXiv preprint arXiv:1707.06347},
	url = {https://arxiv.org/abs/1707.06347}
}

@inproceedings{Shani:2021,
	title = {Online Apprenticeship Learning},
	author = {Shani, Lior and Zahavy, Tom and Mannor, Shie},
	year = {2022},
	booktitle = {AAAI Conference on Artificial Intelligence},
	url = {https://ojs.aaai.org/index.php/AAAI/article/view/20798}
}

@inproceedings{sun2017deeply,
	title = {Deeply {A}ggre{V}a{T}e{D}: Differentiable Imitation Learning for Sequential Prediction},
	author = {Wen Sun and Arun Venkatraman and Geoffrey J. Gordon and Byron Boots and J. Andrew Bagnell},
	year = {2017},
	booktitle = {International Conference on Machine Learning},
	url = {https://proceedings.mlr.press/v70/sun17d.html}
}

@inproceedings{sun2019provably,
	title = {Provably Efficient Imitation Learning from Observation Alone},
	author = {Sun, Wen and Vemula, Anirudh and Boots, Byron and Bagnell, Drew},
	year = {2019},
	booktitle = {International Conference on Machine Learning},
	url = {https://proceedings.mlr.press/v97/sun19b.html}
}

@inproceedings{swamy2021moments,
	title = {Of moments and matching: A game-theoretic framework for closing the imitation gap},
	author = {Swamy, Gokul and Choudhury, Sanjiban and Bagnell, J Andrew and Wu, Steven},
	year = {2021},
	booktitle = {International Conference on Machine Learning},
	url = {https://proceedings.mlr.press/v139/swamy21a.html}
}

@inproceedings{swamy2022minimax,
	title = {Minimax optimal online imitation learning via replay estimation},
	author = {Swamy, Gokul and Rajaraman, Nived and Peng, Matt and Choudhury, Sanjiban and Bagnell, J and Wu, Steven Z and Jiao, Jiantao and Ramchandran, Kannan},
	year = {2022},
	booktitle = {Advances in Neural Information Processing Systems},
	url = {https://proceedings.neurips.cc/paper_files/paper/2022/hash/2e809adc337594e0fee330a64acbb982-Abstract-Conference.html}
}

@inproceedings{swamy2022sequence,
	title = {Sequence Model Imitation Learning with Unobserved Contexts},
	author = {Gokul Swamy and Sanjiban Choudhury and Drew Bagnell and Steven Wu},
	year = {2022},
	booktitle = {Advances in Neural Information Processing Systems},
	url = {https://openreview.net/forum?id=3nbKUphLBg5}
}

@inproceedings{Syed:2007,
	title = {A game-theoretic approach to apprenticeship learning},
	author = {Syed, Umar and Schapire, Robert E.},
	year = {2007},
	booktitle = {Advances in Neural Information Processing Systems},
	url = {https://papers.nips.cc/paper_files/paper/2007/hash/ca3ec598002d2e7662e2ef4bdd58278b-Abstract.html}
}

@inproceedings{Syed:2008,
	title = {Apprenticeship learning using linear programming},
	author = {Syed, Umar and Bowling, Michael and Schapire, Robert E.},
	year = {2008},
	booktitle = {International Conference on Machine Learning},
	url = {https://icml.cc/Conferences/2008/papers/645.pdf}
}

@inproceedings{tiapkin2023regularized,
	title = {Demonstration-Regularized {RL}},
	author = {Daniil Tiapkin and Denis Belomestny and Daniele Calandriello and Eric Moulines and Alexey Naumov and Pierre Perrault and Michal Valko and Pierre Menard},
	year = {2024},
	booktitle = {International Conference on Learning Representations},
	url = {https://arxiv.org/abs/2310.17303}
}

@inproceedings{torabi2018behavioral,
	title = {Behavioral Cloning from Observation},
	author = {Faraz Torabi and Garrett Warnell and Peter Stone},
	year = {2018},
	booktitle = {International Joint Conference on Artificial Intelligence},
	url = {https://arxiv.org/abs/1805.01954}
}

@inproceedings{torabi2019recent,
	title = {Recent Advances in Imitation Learning from Observation},
	author = {Torabi, Faraz and Warnell, Garrett and Stone, Peter},
	year = {2019},
	booktitle = {International Joint Conference on Artificial Intelligence},
	url = {https://arxiv.org/abs/1905.13566}
}

@article{towers2024gymnasium,
	title = {Gymnasium: A Standard Interface for {RL} Environments},
	author = {Mark Towers and Ariel Kwiatkowski and Jordan Terry and John U. Balis and Gianluca De Cola and Tristan Deleu and Manuel Goulão and Andreas Kallinteris and Markus Krimmel and Arjun KG and Rodrigo Perez-Vicente and others},
	year = {2025},
	journal = {arXiv preprint arXiv:2407.17032},
	url = {https://arxiv.org/abs/2407.17032},
	eprint = {2407.17032},
	archiveprefix = {arXiv},
	primaryclass = {cs.LG}
}

@misc{uesato2022solving,
	title = {Solving Math Word Problems with Process-based and Outcome-based Feedback},
	author = {Jonathan Uesato and Nate Kushman and Ramana Kumar and H. Francis Song and Noah Yamamoto Siegel and Lisa Wang and Antonia Creswell and Geoffrey Irving and Irina Higgins},
	year = {2023},
	url = {https://openreview.net/forum?id=MND1kmmNy0O}
}

@book{vershynin2018high,
	title = {High-dimensional probability: An introduction with applications in data science},
	author = {Vershynin, Roman},
	year = {2018},
	publisher = {Cambridge University Press}
}

@inproceedings{viano2022proximal,
	title = {Proximal point imitation learning},
	author = {Viano, Luca and Kamoutsi, Angeliki and Neu, Gergely and Krawczuk, Igor and Cevher, Volkan},
	year = {2022},
	booktitle = {Advances in Neural Information Processing Systems},
	url = {https://proceedings.neurips.cc/paper_files/paper/2022/hash/9988f2c8e07c1f98af7ba9ca31ccae0b-Abstract-Conference.html}
}

@inproceedings{viano2024imitation,
	title = {Imitation Learning in Discounted Linear {MDP}s without exploration assumptions},
	author = {Luca Viano and Stratis Skoulakis and Volkan Cevher},
	year = {2024},
	booktitle = {International Conference on Machine Learning},
	url = {https://openreview.net/forum?id=DChQpB4AJy}
}

@inproceedings{viel2025soar,
	title = {{IL}-{SOAR}: Imitation Learning with Soft Optimistic  Actor cRitic},
	author = {Stefano Viel and Luca Viano and Volkan Cevher},
	year = {2025},
	booktitle = {International Conference on Machine Learning},
	url = {https://openreview.net/forum?id=NNr8DHb0L7}
}

@inproceedings{watson2023coherent,
	title = {Coherent Soft Imitation Learning},
	author = {Watson, Joe and Huang, Sandy and Heess, Nicolas},
	year = {2023},
	booktitle = {Advances in Neural Information Processing Systems},
	url = {https://proceedings.neurips.cc/paper_files/paper/2023/file/2f0435cffef91068ced08d7c7d8e643e-Paper-Conference.pdf}
}

@inproceedings{weisz2021query,
	title = {On Query-efficient Planning in {MDPs} under Linear Realizability of the Optimal State-value Function},
	author = {Weisz, Gellert and Amortila, Philip and Janzer, Barnab\'as and Abbasi-Yadkori, Yasin and Jiang, Nan and Szepesvari, Csaba},
	year = {2021},
	booktitle = {Conference on Learning Theory},
	url = {https://proceedings.mlr.press/v134/weisz21a.html}
}

@inproceedings{wu2024diffusing,
title={Diffusing States and Matching Scores: A New Framework for Imitation Learning},
author={Runzhe Wu and Yiding Chen and Gokul Swamy and Kiant{\'e} Brantley and Wen Sun},
booktitle={International Conference on Learning Representations},
year={2025},
url={https://openreview.net/forum?id=kWRKNDU6uN}
}

@inproceedings{wulfmeier2024imitating,
	title = {Imitating Language via Scalable Inverse Reinforcement Learning},
	author = {Markus Wulfmeier and Michael Bloesch and Nino Vieillard and Arun Ahuja and Jorg Bornschein and Sandy Huang and Artem Sokolov and Matt Barnes and Guillaume Desjardins and Alex Bewley and others},
	year = {2024},
	booktitle = {Advances in Neural Information Processing Systems},
	url = {https://openreview.net/forum?id=5d2eScRiRC}
}

@inproceedings{xu2023provably,
	title = {Provably Efficient Adversarial Imitation Learning with Unknown Transitions},
	author = {Xu, Tian and Li, Ziniu and Yu, Yang and Luo, Zhi-Quan},
	year = {2023},
	booktitle = {Conference on Uncertainty in Artificial Intelligence},
	url = {https://proceedings.mlr.press/v216/xu23c.html}
}

@inproceedings{xu2024provably,
	title = {Provably and Practically Efficient Adversarial Imitation Learning with General Function Approximation},
	author = {Tian Xu and Zhilong Zhang and Ruishuo Chen and Yihao Sun and Yang Yu},
	year = {2024},
	booktitle = {Advances in Neural Information Processing Systems},
	url = {https://openreview.net/forum?id=7YdafFbhxL}
}

@inproceedings{Ziebart:2008,
	title = {Maximum entropy inverse reinforcement learning},
	author = {Ziebart, B. D. and Maas, A. and Bagnell, J. A. and Dey, A. K.},
	year = {2008},
	booktitle = {AAAI Conference on Artificial Intelligence}
}

@inproceedings{li2024getting,
  title={Getting More Juice Out of the {SFT} Data: Reward Learning from Human Demonstration Improves {SFT} for {LLM} Alignment},
  author={Jiaxiang Li and Siliang Zeng and Hoi To Wai and Chenliang Li and Alfredo Garcia and Mingyi Hong},
  booktitle={Advances in Neural Information Processing Systems},
  year={2024},
  url={https://openreview.net/forum?id=orxQccN8Fm}
}

@inproceedings{jia2024adversarial,
  title={Adversarial Moment-Matching Distillation of Large Language Models},
  author={Chen Jia},
  booktitle={Advances in Neural Information Processing Systems},
  year={2024},
  url={https://openreview.net/forum?id=0VeSCjRDBy}
}

@article{sun2025inverse,
  title={Inverse reinforcement learning meets large language model post-training: Basics, advances, and opportunities},
  author={Sun, Hao and van der Schaar, Mihaela},
  journal={arXiv preprint arXiv:2507.13158},
  year={2025},
  url={https://arxiv.org/abs/2507.13158}
}

@inproceedings{cai2026escaping,
title={Escaping the Verifier: Learning to Reason via Demonstrations},
author={Locke Cai and Max Ryabinin and Ivan Provilkov},
booktitle={International Conference on Machine Learning},
year={2026},
url={https://openreview.net/forum?id=pS1khvoxHT}
}

@inproceedings{sikchi2023dual,
title={Dual {RL}: Unification and New Methods for Reinforcement and Imitation Learning},
author={Harshit Sikchi and Qinqing Zheng and Amy Zhang and Scott Niekum},
booktitle={International Conference on Learning Representations},
year={2024},
url={https://openreview.net/forum?id=xt9Bu66rqv}
}

@inproceedings{cundy2023sequencematch,
title={{SequenceMatch}: Imitation Learning for Autoregressive Sequence Modelling with Backtracking},
author={Chris Cundy and Stefano Ermon},
booktitle={International Conference on Learning Representations},
year={2024},
url={https://openreview.net/forum?id=FJWT0692hw}
}

@article{li2026revisiting,
  title={Revisiting {DAgger} in the Era of {LLM}-Agents},
  author={Li, Changhao and Qiang, Rushi and Huang, Jiawei and Gao, Chenxiao and Zhang, Chao and He, Niao and Dai, Bo},
  journal={arXiv preprint arXiv:2605.12913},
  year={2026},
  url={https://arxiv.org/abs/2605.12913}
}

@article{sriraman2026behavior,
      title={Behavior Cloning is Not All You Need: The Optimality of On-Policy Distillation for Noisy Expert Feedback}, 
      author={Ved Sriraman and Peihan Liu and Daniel Hsu and Adam Block},
      year={2026},
      journal = {arXiv preprint arXiv:2606.30923},
      eprint={2606.30923},
      archivePrefix={arXiv},
      primaryClass={cs.LG},
      url={https://arxiv.org/abs/2606.30923}, 
}

@article{zhang2026does,
  title={When Does Online Imitation Learning Help in {LLM} Post-Training? {The} Role of (Non-) Realizability Beyond Horizon},
  author={Zhang, Huaqing and Gai, Jingchu and Kim, Juno and Liu, Bingbin and Risteski, Andrej},
  journal={arXiv preprint arXiv:2606.30445},
  year={2026},
  url={https://arxiv.org/abs/2606.30445}
}

@inproceedings{cheng2022adversarially,
  title={Adversarially trained actor critic for offline reinforcement learning},
  author={Cheng, Ching-An and Xie, Tengyang and Jiang, Nan and Agarwal, Alekh},
  booktitle={International Conference on Machine Learning},
  year={2022},
  url={https://proceedings.mlr.press/v162/cheng22b.html}
}

@article{balcan2026online,
  title={Verify to Amplify: Improving Reasoning via Learned Chain-of-Thought Verification},
  author={Balcan, Maria-Florina and Blum, Avrim and Fragkia, Kiriaki and Li, Zhiyuan and Sharma, Dravyansh},
  journal={arXiv preprint arXiv:2603.03538},
  year={2026},
  url={https://arxiv.org/abs/2603.03538}
}
\bibliographystyle{plainnat}

\newpage
\appendix
\hypertarget{appendix-toc}{}
\renewcommand{\contentsname}{Contents of Appendix}
\addtocontents{toc}{\protect\setcounter{tocdepth}{2}}
{
  \setlength{\cftbeforesecskip}{.7em}
  \setlength{\cftbeforesubsecskip}{.4em}
  \tableofcontents
}

\clearpage
\arxiv{\enableappendixtocfooter}
\section{Notation Reference}
\label{sec:notation_reference}
The table summarizes notation. The last column points to definitions or key uses.

\begingroup
\small
\Crefname{section}{Sec.}{Secs.}
\Crefname{subsection}{Sec.}{Secs.}
\Crefname{appendixsection}{App.}{Apps.}
\Crefname{appendixsubsection}{App.}{Apps.}
\Crefname{assumption}{Asp.}{Asps.}
\makeatletter
\def\Cref@assumption@format#1#2#3{#2Asp.~#1#3}
\def\cref@assumption@format#1#2#3{#2Asp.~#1#3}
\makeatother
\Crefname{definition}{Def.}{Defs.}
\Crefname{theorem}{Thm.}{Thms.}
\Crefname{algorithm}{Alg.}{Algs.}
\Crefname{equation}{Eq.}{Eqs.}
\Crefname{Lem}{Lem.}{Lems.}
\newcommand{\compactrefconjunctions}{%
  \renewcommand{\crefpairconjunction}{ \&\nobreakspace}%
  \renewcommand{\creflastconjunction}{ \&\nobreakspace}%
  \renewcommand{\crefpairgroupconjunction}{ \&\nobreakspace}%
  \renewcommand{\creflastgroupconjunction}{, \&\nobreakspace}%
}
\setlength{\tabcolsep}{4pt}
\setlength{\LTpre}{0.5em}
\setlength{\LTpost}{0.5em}
\renewcommand{\arraystretch}{1.16}
\begin{longtable}{@{}L{0.225\linewidth}L{\dimexpr0.625\linewidth-4\tabcolsep\relax}>{\compactrefconjunctions}L{0.15\linewidth}@{}}
\caption{Notation reference.}
\label{tab:notation_reference}
\\
\toprule
Notation & Meaning & Reference \\
\midrule
\endfirsthead
\caption[]{Notation reference (continued).}
\\
\toprule
Notation & Meaning & Reference \\
\midrule
\endhead
\midrule
\multicolumn{3}{r}{Continued on next page}
\\
\endfoot
\bottomrule
\endlastfoot
\multicolumn{3}{@{}l}{\textbf{MDP and imitation learning setting}} \\
\addlinespace[0.25em]
$\cM{=}\spr{\cX, \cA, H, P, r, \initial}$ & State/action spaces $\cX,\cA$ ($A=\abs{\cA}$), horizon $H$, transitions $P$, rewards $r$, initial law $\initial$. & \Cref{sec:preliminaries} \\
$\pi = \spr{\pi_h}_{h=1}^H$ & Nonstationary Markov policy. $\pi_h\colon \cX \to \simplex\spr{\cA}$ is its decision rule at stage $h$. & \Cref{sec:preliminaries} \\
$\bbP^\pi,\ \bbE^\pi,\ J^\pi$ & Trajectory law, expectation, and expected return under $\pi$. & \Cref{sec:preliminaries} \\
$V_h^\pi,\ Q_h^\pi$ & State and state-action value functions at stage $h$. & \Cref{sec:preliminaries} \\
$d_h^\pi \spr{x},\ d_h^\pi \spr{x,a}$ & State and state-action occupancy measures at stage $h$. & \Cref{sec:preliminaries} \\
$\expert,\ \experth$, $\piout,\ \piouth$ & Expert and output policies, with their decision rules at stage $h$. & \Cref{sec:preliminaries} \\
$\varepsilon,\ \delta$ & Target suboptimality and failure probability. & \Cref{eq:suboptimality} \\
$\cDE,\ \XEih$, $\Xih,\ \AEih$ & Offline expert dataset/states, query states and expert actions. & \Cref{sec:il-background,sec:algo} \\
$\tauE$ & Offline trajectories, queried samples per stage, episodes in \Cref{def:interactive-il-algo}. & \Cref{sec:il-background,sec:algo} \\
$f\spr{x,\pi}$ & Action average $\sum_a\pi\spr{a\given x}f\spr{x,a}$. & \Cref{sec:preliminaries} \\
\addlinespace[0.5em]
\multicolumn{3}{@{}l}{\textbf{Value classes and realizability}} \\
\addlinespace[0.25em]
$\cQ,\ \cQ_h$ & Value-function class and its projection at stage $h$. & \Cref{sec:vb-il,sec:algo} \\
$\QMAX$ & Upper bound on total trajectory reward and candidate values. & \Cref{sec:preliminaries} \\
$\Pi,\ \cR$ & Policy/reward classes. Realizability means $\expert\in\Pi$, $r\in\cR$. & \Cref{ass:policy-realizability,ass:reward-realizability} \\
$\tQ^\expert$ & Witness for relaxed $\qexpert$-realizability. & \Cref{ass:relaxed} \\
$Q^{\Pi_{\cQ}}$-realizability & $Q^\pi\in\cQ$ for every policy in the specified class $\Pi_{\cQ}$. & \Cref{ass:q-picq-realizability} \\
\addlinespace[0.5em]
\multicolumn{3}{@{}l}{\textbf{Objectives, empirical estimates, and complexity}} \\
\addlinespace[0.25em]
$\cL_h^d \spr{p, Q}$ & Expert-minus-learner value under stagewise state law $d_h$. & \Cref{sec:algo} \\
$\cL_h^\pi$ & Shorthand for $\cL_h^{d^\pi}$. & \Cref{sec:proofs_ub} \\
$\hcL_h\spr{p,Q},\ \Delta\spr{\pi}$ & Empirical version of $\cL_h^{\piout}$ and its uniform estimation error at $\pi$. & \Cref{lem:pifirst} \\
$\regret^\pi_h \spr{x}$ & Cumulative policy regret against the expert at stage $h$, state $x$. & \Cref{lem:pifirst} \\
$\cN_\varepsilon \spr{S, \mathrm{d}},\ \cC_\varepsilon \spr{S, \mathrm{d}}$ & Covering number and an $\varepsilon$-cover. Omitted metric means $\norm{\cdot}_\infty$. & \Cref{def:covering} \\
$\maxcovering{\varepsilon}$ & Largest stagewise covering number: $\max_h\cN_\varepsilon\spr{\cQ_h,\norm{\cdot}_\infty}$. & \Cref{thm:ovi-main-covering} \\
$\maxcovering[\cQ\times\Pi_{\cQ},\rho]{r}$ & Product-class covering number $\max_h\cN_r\spr{\cQ_h\times\Pi_{\cQ,h},\rho}$. & \Cref{thm:adaptive-hybrid} \\
$\norm{\cdot}_{\infty,1},\ \rho$ & Supremum over states, $\ell_1$ over actions, and scaled value/policy product metric. & \Cref{eq:l-inf-one-norm,eq:moulin-stage-product-metric} \\
$\Pi_{\cQ},\ \Pi_{\cQ, h}$ & Softmax policies and stagewise projection, $K$/rate-dependent, with stagewise hybrids in \Cref{sec:adaptive}. & \Cref{lem:moulin-nonconvex-product-cover,ass:q-picq-realizability} \\
$r_\varepsilon,\ r_\varepsilon^{(1)},\ r_\varepsilon^{(\infty)}$ & Result-dependent covering radii in sample-complexity bounds. & \Cref{sec:proofs_ub,sec:ispil_plus,sec:break_lb_with_coverage} \\
\addlinespace[0.5em]
\multicolumn{3}{@{}l}{\textbf{Algorithmic and appendix-specific notation}} \\
\addlinespace[0.25em]
$K,\ \eta,\ \eta_k,\ \pi_h^k,\ Q_h^k$ & Stage iteration count, learning rates, and policy/value iterates. & \Cref{alg:interactive_finite-H,alg:arbitrary_nu} \\
$ Q^{1:K},\ \mathrm{LC},\ f$ & Value functions, linear combination, and link of a VI policy. & \Cref{def:valueIL} \\
$\Pi_{\mathrm{Alg}},\ C_\infty,\ C_1$ & Output-policy class, $L_\infty$ and $L_1$ expert-coverage bounds. & \Cref{asp:linf-coverage,asp:l1-coverage} \\
$d=\spr{d_h}_{h=1}^H$, $\hcL_h^d,\ \Delta_d\spr{\pi}$ & Stagewise sampling laws, possibly history-dependent. Empirical objective and estimation error. & \Cref{thm:arbitrary_nu}, proof \\
$\tauL,\ \Xjkh$ & \ISPILPLUS learner batch size per stage/iteration and fresh states. & \Cref{alg:interactive_finite-H-expert-buffer} \\
$\hcL_{h,k}\spr{p,Q_h},\ \Delta_{h,k}\spr{p}$ & \ISPILPLUS empirical objective, error uniform in $Q_h$ at fixed $h,k,p$. & \Cref{thm:ovi-plus-covering}, proof \\
$\alpha$ & Expert-state weight in the learner/expert occupancy mixture. & \Cref{alg:adaptive} \\
$\wh{\ell}_h^{\piout, k}\spr{Q}$ & Sampled loss in the $Q$-\ISPIL exponential-weights update. & \Cref{sec:q_first_ispil} \\
$w_h^k,\ \regret_h^w,\ \hQ_h^\expert$ & Cover weights, regret, and a closest cover element to $Q_h^\expert$. & \Cref{lem:Qfirst} \\
\end{longtable}
\endgroup

\clearpage
\section{Additional Related Work}
\label{sec:additional_related_work}
We provide an overview of closely related work. In \Cref{sec:theory_related}, we discuss theoretical results available for value-based IL methods, and in \Cref{sec:applied_related} we cover value-based IL algorithms that focus on empirical performance. We use ``value-based IL'' broadly to include reward-learning, inverse reinforcement learning, apprenticeship learning, moment matching, and Q-function-based methods, since all use learned reward/value information to derive a policy rather than directly fitting the expert action distribution.\loose

\subsection{Theoretical Guarantees for Imitation Learning}
\label{sec:theory_related}

As mentioned in the main text, our work aims to establish expert-sample complexity bounds when expert policy realizability need not hold. While we are the first to study this problem under $Q^\expert$-realizability alone, prior work has examined several alternative assumptions that we discuss below.

After behavior cloning was proposed in the seminal work of \citet{pomerleau1988alvinn,Pomerleau:1991}, a parallel line of work formalized \emph{inverse reinforcement learning} (IRL; \citealp{Russell:1998, Ng:2000}), whose goal is to infer a reward function that explains the expert's behavior, typically among many rewards for which the expert is optimal. Closely related work on \emph{apprenticeship learning} \citep{Abbeel:2004,Abbeel:2008,Syed:2007,Syed:2008,Ziebart:2008} uses this reward-based perspective to learn a policy that matches the expert's performance.

\citet{ross2014reinforcement} introduced \AGGREVATE, an interactive no-regret reduction that augments expert action queries with value information. That is, when the learner visits a state-action pair $\spr{x, a}$, $\qexpert \spr{x, a}$ is observed. This oracle is referred to as \emph{value feedback} and is stronger than our setting, where the learner can query expert actions but does not observe rewards or expert values. Their framework also yields practical algorithms compatible with neural-network-based function approximation; see \citet{sun2017deeply,cheng2018fast}. Value-based feedback, together with access to an interactive expert, is known to allow for better horizon dependence in error propagation \citep{ross2014reinforcement,sun2017deeply}, which can lead to representational benefits in policy-based IL \citep{foster2024behavior}.

The closest conceptual predecessor to \ISPIL is the moment-matching framework of \citet{swamy2021moments}, especially their ``On-Q'' template and the associated \DAEQUIL algorithm. This framework also includes reward and ``Off-Q'' moment-matching templates, and analyzes imitation gaps through \emph{moment recoverability} $\mu$: matching action-value moments to precision $\veps$ yields a suboptimality bound of order $\mu H \veps$, improving on the classical $H^2$ error-propagation factor in \BC-style analyses. Empirically, \citet{swamy2021moments} show that \DAEQUIL can outperform \BC and \DAGGER in a forest-navigation task, which they attribute to the mode-seeking behavior induced by action-value moment matching. However, despite the better error-propagation properties of interactive algorithms, \citet{foster2024behavior} showed that no interactive algorithm can uniformly improve upon an offline one across all possible policy classes. Therefore, the error-propagation-style analysis does not seem sufficient to capture the benefits of interaction. Our contribution is complementary: we instantiate the ``On-Q'' perspective with a computationally oracle-efficient algorithm, provide a sample-complexity analysis under $\qexpert$-realizability, and show that this representational assumption alone is insufficient in the offline setting.\loose

Next, we review value-based IL further, distinguishing (i) offline algorithms that operate on a fixed dataset of expert demonstrations from (ii) variants that can additionally access the environment and collect trajectories without observing rewards. Then, we discuss the main policy-based IL methods, followed by influential empirical work on value-based IL.\loose

\para{Value-based imitation learning with online MDP access or known dynamics}
It is useful to distinguish three closely related access models. First, classical apprenticeship-learning and IRL methods \citep{Abbeel:2004,Syed:2007} assumed knowledge of a linear reward class containing the true unknown reward, as well as perfect knowledge of the MDP dynamics. In tabular problems with unknown dynamics, \citet{Syed:2007} instead estimated the transition kernel only on state-action pairs frequently visited by the expert and sent all other transitions to a pessimistic dead state. They showed that this restricted model suffices to retain a guarantee, at the cost of a slightly worse sample-complexity bound. Also in the tabular setting, \citet{rajaraman2020toward} characterize worst-case offline imitation rates with quadratic horizon dependence $H^2$ and show that, for deterministic experts, access to the transition model can improve the horizon dependence by at least a $\sqrt{H}$ factor. \citet{rajaraman2021provablybreaking} later proved a matching lower bound in that setting.

A second line of work assumes reward-free online access to the MDP: the learner may roll out policies in the environment, but it does not observe the reward and typically obtains expert information through precollected demonstrations, expert occupancies, or expert states or features rather than through value feedback. \citet{xu2023provably} removed the known-transition assumption discussed above in tabular adversarial IL and proposed a method achieving the minimax optimal expert-sample complexity by combining expert demonstrations with reward-free environment interaction. \citet{Shani:2021} presented an approach based on alternating updates between sequences of rewards and policies. \citet{viano2022proximal} develop minimax/proximal-point methods for infinite-horizon IL in MDPs with linear, unknown rewards and dynamics under a strong exploratory assumption on the learner's policies, which was later removed by \citet{viano2024imitation}. \citet{moulin2025optimistically} further develop optimistic exploration tools for infinite-horizon linear MDPs and apply them to IL. In state- or feature-matching variants, such methods can sometimes use expert states or features rather than expert actions, connecting them to imitation from observation alone \citep{sun2019provably,torabi2018behavioral,torabi2019recent,kidambi2021mobile,viel2025soar}. Beyond linear function approximation, \citet{xu2024provably} extended the guarantees to the setting in which general function approximation is required to approximate the reward and the dynamics of the environment. Finally, recent work also studies instance-dependent and second-order guarantees (\ie, bounds that scale with the variance of the expert and learner value functions): \citet{wu2024diffusing} give first- and second-order bounds in a score-matching framework, while \citet{li2025near} prove near-optimal second-order guarantees for model-based adversarial IL.

Finally, another related but distinct line gives the learner online reward observations in addition to expert information, making the problem closer to RL with expert advice than to reward-free IL. For example, \citet{tiapkin2023regularized} studied the setting in which an offline record of state-action pairs collected from an optimal expert can be used to initialize the learner policy in RL. They show that in tabular and linear MDPs, this can achieve better sample complexity than RL without expert data. On the representation side, \citet{amortila2022few} study RL with ``one-step'' resets, access to an interactive optimal expert, and a class that realizes the optimal state-value function. They provide a computationally efficient algorithm to approximate the optimal policy, showing that expert advice can make otherwise difficult value-realizable settings tractable. This contrasts with planning or RL under value-function realizability alone, where query or computational barriers are known in related settings \citep{weisz2021query,liu2023exponential}.

These works are complementary to our setting: they exploit reward observations, known dynamics, reward-free environment interaction without learner-state expert action queries, or expert advice for RL, whereas \ISPIL uses only reward-free MDP rollouts together with expert action queries on learner-visited states.

\para{Value-based imitation learning with logged expert data}
The closest offline counterpart to our setting is \citet{moulin2025inverse}. Their algorithm, \SPOIL, shows that expert-policy realizability can be avoided in offline IL by imposing a closure-type value realizability assumption: the class $\cQ$ must realize $Q^\pi$ for every policy $\pi$ in the implicit policy class $\Pi_{\cQ}$ that the algorithm may optimize over or output. This assumption is qualitatively different from, and substantially stronger than, requiring only $\qexpert \in \cQ$, since it asks the value class to be closed under many learner-generated policies rather than only to contain the expert's value function. This distinction is important for interpreting our offline lower bound: the lower bound shows that $\qexpert$-realizability alone is insufficient for offline value-based IL in general, while the offline guarantees of \SPOIL rely on additional closure structure that is absent from our main setting.
Moreover, we point out that Adversarially Trained Actor Critic (\ATAC; \citet{cheng2022adversarially}), originally developed for offline RL, recovers \SPOIL when some hyperparameter ($\beta$ in their paper) is set to zero, and mirror descent is used as their generic online learning algorithm. \citet{cheng2022adversarially} made explicit that IL algorithms can be recovered from their formulation, but their analysis still relies on expert policy realizability.

A different logged-data setting was studied by \citet{joshi2025learning}, who model learning from correct demonstrations as offline IL in a contextual bandit, equivalently an MDP with horizon $H = 1$. They assume that the correctness, or reward, model belongs to a finite class, rather than assuming expert-policy realizability, and they design statistically efficient algorithms from offline demonstrations alone. Moreover, they improve the algorithm of \citet{Syed:2007} when instantiated with $H = 1$, thereby obtaining faster rates when the expert is optimal. Because $H = 1$, there is no sequential distribution shift and no nontrivial transition dynamics to learn or exploit.

A recent value-based IL method also investigates how value structure can mitigate compounding errors. \citet{xu2026non} show that \IQLEARN-style objectives can reduce to behavior cloning and still suffer compounding errors, and they introduce Bellman constraints as a mechanism for propagating value information from demonstrated states to states not covered by the demonstrations. In our setting, a variant of this idea could be used to filter a given realizable class $\cQ$ by removing functions with ``temporal'' variations greater than one; this is guaranteed to reduce the complexity of $\cQ$ while still providing realizability of $\qexpert$. In practice, this could be implemented by adding a regularization term to the \ISPIL objective that penalizes large temporal variations.

\para{Policy-based imitation learning}
Historically, the most common paradigm for tackling IL is to maximize the likelihood of the observed expert actions over a given policy class. In its offline form, \BC fits the expert action distribution by supervised learning on expert trajectories, with early \BC-style systems including \texttt{ALVINN} \citep{pomerleau1988alvinn,Pomerleau:1991}. Interactive variants such as \DAGGER---very popular in practice---and related no-regret reductions instead collect expert labels on states visited by the learner, thereby addressing the covariate shift that appears when an offline cloned policy is deployed on its own trajectory distribution \citep{Ross:2010,Ross:2011}. Theoretical guarantees are typically expressed in terms of the complexity of a policy class $\Pi$ that contains the expert policy $\expert$ (see, for example, \citealp{rajaraman2021value,swamy2022minimax}). \citet{li2022efficient} study this classification-based online IL viewpoint more generally, including nonrealizable settings. More recently, \citet{foster2024behavior} refined the classical picture by showing that, under suitable policy realizability and log-loss learnability assumptions, offline \BC can have much better horizon dependence than suggested by the standard compounding-error analysis.

Follow-up work investigates what happens when the expert policy is not exactly realizable by the learner's policy class. \citet{rohatgi2025computational} analyze autoregressive learning and IL under policy misspecification, showing that error amplification reappears in misspecified settings and that next-token or \BC-style objectives face computational-statistical tradeoffs. They also propose robustified \BC variants whose guarantees quantify the dependence on horizon and misspecification. \citet{espinosa2025efficient} introduce \texttt{GUITAR}, an efficient interactive IL algorithm under misspecification that uses a structural completeness condition and environment access initialized from states sampled from the expert occupancy measure, with guarantees controlled by an advantage-class notion of misspecification rather than by statistical divergences.\loose

Most recently, \citet{sriraman2026behavior} studied a setting in which expert samples may be corrupted: with some probability, the action shown to the learner is drawn from a different policy rather than the expert's. They show that any offline policy-based IL algorithm suffers exponential sample complexity, whereas their interactive algorithm, \texttt{NAIL}, avoids this exponential blow-up by additionally querying the expert.

In a spirit similar to ours, another recent work \citep{zhang2026does} studies the advantage of on-policy interaction, albeit from a policy-based viewpoint. They show that when the policy class does not realize the expert policy, a coverage coefficient akin to ours is unavoidable for offline imitation. In contrast, online distillation can avoid such dependence if the expert logits are \emph{aligned} with the reward direction. We consider their approach orthogonal to ours: they show under which assumptions existing algorithms are effective, whereas we introduce a setup ($\qexpert$-realizable MDPs) in which new algorithmic ideas are required to attain the best possible theoretical guarantees. Moreover, \citet{zhang2026does} show that on-policy distillation algorithms can produce students competitive with the expert in terms of accuracy while maintaining very different next-token distributions.

Finally, \citet{balcan2026online} show that a separation between the offline and the interactive setting also occurs when considering the problem of learning a process reward model from a limited number of offline or on-policy calls from an oracle, which can be cast as policy-based imitation learning. In particular, \citet{balcan2026online} show that a process reward model with optimal soundness/completeness tradeoff is possible only given on-policy access to the oracle.

To conclude our review of the theoretical literature in IL, we report a comparison of our results with the bounds obtained and the representational conditions required by other algorithms in \Cref{tab:related}.
\renewcommand{\arraystretch}{1.8}
\begin{table}[t]
    \centering
    \vspace{.4em}
    \caption{\label{tab:related} Comparison with related algorithms, assumptions, and expert-interaction requirements. To state the guarantees of \BC, we define a policy class $\PiE$ such that $\expert \in \PiE$. \citet{joshi2025learning} study finite reward classes and one-step decision problems. $\maxcovering{r_\varepsilon}$ denotes the largest covering number of the stagewise classes $\cQ_h$ in $\norm{\cdot}_{\infty}$ with spacing $r_\varepsilon = \cO \spr{\varepsilon^2}$. We define the class $\Pi_{\cQ} \ldef \scbr{\pi : \exists m \leq K, \forall h, \exists Q_h^1, \ldots, Q_h^m \in \cQ_h, \pi_h \spr{\cdot \given x} = \softmax \spr{\eta \sum_{k = 1}^m Q_h^k \spr{x, \cdot}}}$ for the algorithm's iteration budget $K$ and learning rate $\eta$. $Q^{\Pi_{\cQ}}$-realizability is the assumption that $Q^\pi \in \cQ$ for any $\pi \in \Pi_{\cQ}$.}
    \vspace{.4em}
    \resizebox{\textwidth}{!}{%
    \begin{tabular}{ccccc}
        \toprule
        \textbf{Algorithm} & \textbf{Representational assumptions} & \textbf{Comp. efficient} & \textbf{Interactive expert} & \textbf{Expert traj./queries} \\ \midrule
        \texttt{LogLossBC} & \multirow{2}{*}[\verticaloffset]{$\expert$-realizability} & \multirow{2}{*}[\verticaloffset]{\goodcell} & \multirow{2}{*}[\verticaloffset]{\badcell} & \multirow{2}{*}[\verticaloffset]{\large $\cO \spr[\Big]{\frac{H^2 \log \abs{\PiE}}{\varepsilon^2}}$} \\[\spacingfirstcol]
        \citep{foster2024behavior} & & & & \\ \midrule
        \texttt{MWAL} & Reward linear in $d$ features & \multirow{2}{*}[\verticaloffset]{\goodcell} & \multirow{2}{*}[\verticaloffset]{\badcell} & \multirow{2}{*}[\verticaloffset]{\large $\tcO \spr[\Big]{\frac{H^2 \log \spr{d}}{ \varepsilon^2}}$} \\[\spacingfirstcol]
        \citep{Syed:2007} & Known transitions & & & \\ \midrule
        \texttt{MUWU} & Reward realizability with class $\cR$ & \multirow{2}{*}[\verticaloffset]{\badcell} & \multirow{2}{*}[\verticaloffset]{\badcell} & \multirow{2}{*}[\verticaloffset]{\large $\tcO \spr[\Big]{\frac{\log \spr{\abs{\cR}}}{ \varepsilon^2}}$} \\[\spacingfirstcol]
        \citep{joshi2025learning} & Contextual bandits ($H=1$) & & & \\ \midrule
        \SPOIL & \multirow{2}{*}[\verticaloffset]{$Q^{\Pi_{\cQ}}$-realizability} & \multirow{2}{*}[\verticaloffset]{\goodcell} & \multirow{2}{*}[\verticaloffset]{\badcell} & \multirow{2}{*}[\verticaloffset]{\large $\tcO \spr[\Big]{\frac{H^5 \QMAX^4 \log \spr{A \maxcovering{r_\varepsilon}}}{ \varepsilon^4}}$} \\[\spacingfirstcol]
        \citep{moulin2025inverse} & & & & \\ \midrule
        $Q$-\ISPIL & \multirow{2}{*}[\verticaloffset]{$Q^\expert$-realizability} & \multirow{2}{*}[\verticaloffset]{\badcell} & \multirow{2}{*}[\verticaloffset]{\goodcell} & \multirow{2}{*}[\verticaloffset]{\large $\tcO \spr[\Big]{\frac{H^3 \QMAX^2 \log \maxcovering{\nicefrac{\varepsilon}{4 H}}}{\varepsilon^2}}$} \\[\spacingfirstcol]
        (\textbf{Ours}, \Cref{thm:Q_first_main}) & & & & \\ \midrule
        \ISPIL & \multirow{2}{*}[\verticaloffset]{$Q^\expert$-realizability} & \multirow{2}{*}[\verticaloffset]{\goodcell} & \multirow{2}{*}[\verticaloffset]{\goodcell} & \multirow{2}{*}[\verticaloffset]{\large $\tcO \spr[\Big]{\frac{H^5 \QMAX^4 \log \spr{A \maxcovering{r_\varepsilon}}}{\varepsilon^4}}^{\circ}$} \\[\spacingfirstcol]
        (\textbf{Ours}, \Cref{thm:ovi-main-covering}) & & & & \\ 
        \midrule
        \ISPILPLUS & \multirow{2}{*}[\verticaloffset]{$Q^\expert$-realizability} & \multirow{2}{*}[\verticaloffset]{\goodcell} & \multirow{2}{*}[\verticaloffset]{\goodcell} & \multirow{2}{*}[\verticaloffset]{\large $\tcO \spr[\Big]{\frac{ H^3 \QMAX^2 \log \spr{\maxcovering{\nicefrac{\varepsilon}{24 H}}}}{\varepsilon^2}}^{\dagger}$} \\[\spacingfirstcol]
        (\textbf{Ours}, \Cref{thm:ovi-plus-covering}) & & & & \\
        \midrule
        \SPOIL & $Q^\expert$-realizability & \multirow{2}{*}[\verticaloffset]{\goodcell} & \multirow{2}{*}[\verticaloffset]{\badcell} & \multirow{2}{*}[\verticaloffset]{\large $\tcO \spr[\Big]{\frac{C^4_{\infty} H^5 \QMAX^4 \log \spr{A \maxcovering{r_\varepsilon}}}{\varepsilon^4}}^{\diamond}$} \\[\spacingfirstcol]
        (\textbf{Ours}, \Cref{thm:SPOIL}) & Bounded $C_{\infty}$, optimal $\expert$ & & & \\ \midrule
        \texttt{Lower Bound} & \multirow{2}{*}[\verticaloffset]{Linear $Q^\expert$-realizability} & \multirow{2}{*}[\verticaloffset]{$-$} & \multirow{2}{*}[\verticaloffset]{\badcell} & \multirow{2}{*}[\verticaloffset]{\large $ \Omega \spr[\Big]{\frac{\abs{\cX}}{\varepsilon}}^{\smalltriangle}$} \\[\spacingfirstcol]
        (\textbf{Ours}, \Cref{thm:lower_any_algo}) & & & &\\ \bottomrule
    \end{tabular}
    }
    \vspace{.3em}
    \begin{minipage}{\textwidth}
        \footnotesize
        $^\circ$ If $\cQ$ is convex, the bound improves to $\tcO \spr{H^3 \QMAX^2 \log \spr{\maxcovering{r_\varepsilon}} \varepsilon^{-2}}$.\\
        $^\dagger$ The number of learner rollouts is of order $\tcO \spr{\varepsilon^{-4}}$, but expert queries are performed only along $\tcO \spr{\varepsilon^{-2}}$ of them.\\
        $^\diamond$ If $\cQ$ is convex, the bound improves to $\tcO \spr{C^2_\infty H^3 \QMAX^2 \log \spr{\maxcovering{r_\varepsilon}} \varepsilon^{-2}}$.\\
        $^{\smalltriangle}$ The lower bound holds even if the expert is deterministic and the learner is allowed to replay the expert actions at the states logged in the dataset.
    \end{minipage}
\end{table}
\renewcommand{\arraystretch}{1.0}

\subsection{Empirical Value-Based Imitation Learning}
\label{sec:applied_related}

Although early empirical IL successes include behavior-cloning systems such as \texttt{ALVINN} \citep{pomerleau1988alvinn,Pomerleau:1991}, empirical value-based IL largely grew out of IRL and apprenticeship learning. Early reward-learning methods include maximum-margin planning and maximum-entropy or relative-entropy IRL \citep{Ratliff:2006,Ziebart:2008,boularias2011relative}. These methods learn a reward or cost function that explains the expert's behavior and then derive a policy by planning or policy optimization under the learned objective. Deep variants, such as guided cost learning and maximum-entropy deep IRL, replace linear rewards with neural cost or reward functions \citep{finn2016guided,wulfmeier2015maximum}. However, these methods typically inherit the ``RL-in-the-loop'' structure of \citet{Ziebart:2008}, requiring repeated planning or policy-optimization steps under a changing learned reward, which can be expensive in high-dimensional environments.

Other methods replace explicit planning with adversarial or occupancy-matching objectives. \citet{Ho:2016} developed model-free apprenticeship-learning methods using policy-gradient optimization, while \citet{Ho:2016b} introduced generative adversarial IL (\texttt{GAIL}), which alternates between training a discriminator/cost signal and updating the policy, using Trust Region Policy Optimization (\texttt{TRPO}; \citealp{Schulman:2015}) in the original implementation. \texttt{GAIL} and related adversarial IL methods inspired follow-up work on learned rewards, discriminator-actor-critic methods, off-policy distribution matching, and stationary-distribution correction \citep{Fu:2018,Kostrikov:2019,Kostrikov:2020,Nachum:2019a}. These methods can be effective, but adversarial IL is often sample-hungry and can be difficult to stabilize because the policy optimizer must continually adapt to a nonstationary learned reward or discriminator.

More recent non-adversarial value- and Q-based methods aim to avoid some of this instability. For example, \texttt{SQIL} converts demonstrations into a sparse-reward RL problem \citep{Reddy:2020}, while the \IQLEARN method of \citet{Garg:2021} learns a single soft $Q$-function that implicitly represents both the reward and the policy. Such methods are closely aligned with the broad class of value-based IL algorithms considered in our lower bound: they extract a policy by applying a softmax, greedy, or related activation to a learned value function. \IQLEARN has been evaluated in simulated robotics tasks \citep{Garg:2021}; related Q-learning, IRL, and moment-matching ideas have also been adapted to autoregressive sequence modeling and LM distillation, including \texttt{SequenceMatch} \citep{cundy2023sequencematch}, scalable inverse soft-Q learning for language imitation \citep{wulfmeier2024imitating}, and adversarial moment-matching distillation \citep{jia2024adversarial}. Demonstration-based reward-learning and IRL-style approaches have also been applied to LM supervised fine-tuning (SFT) and reasoning without preference labels or verifiers \citep{li2024getting,cai2026escaping}. We refer to \citet{sun2025inverse} for a broader overview of IRL and IL in LM post-training, and to \citet{sikchi2023dual} for a unifying perspective on practical dual RL and IL methods.

\para{Policy-based imitation learning for language model fine-tuning}
Policy-based, \DAGGER-like on-policy distillation methods have also been applied to LM fine-tuning. \citet{agarwal2024policy} train student models on their own generated outputs using teacher feedback, and \citet{lu2025onpolicydistillation} explicitly frames on-policy distillation as querying or evaluating a teacher on states visited by the student. These methods provide a policy-based analogue of the interactive protocols studied in this paper. Their strong empirical performance motivates the broader use of interaction in distillation, but our results suggest a complementary hypothesis: when the student is much less expressive than the teacher, interactive value-based distillation may be especially useful because it can target value-equivalent behavior rather than the teacher's full action distribution. The Gym experiments in \Cref{fig:size_scaling} provide preliminary evidence for this. Testing the same mechanism in LM distillation remains an important empirical direction.


\clearpage
\section{Value-Based Imitation Learning versus Policy-Based: Motivating Examples}
\label{sec:qexpert_is_weaker}
In this section, we argue that $\qexpert$-realizability can be weaker than $\expert$-realizability in two important cases.
\begin{enumerate}
    \item In \Cref{sec:Qisweaker1}, we show that when the expert is optimal, any class $\PiE$ that realizes the expert policy $\expert$ can be preprocessed into a class $\cQ$ with $\abs{\cQ} \leq \abs{\PiE}$ that realizes $\qexpert$ in a relaxed sense.
    \item In \Cref{app:hard-decoder-exbmdp}, we give an example where the environment makes the state-action value functions of the constructed expert policies simple, regardless of how complex the policy class is. In particular, $\cQ$ can be easier to specify than $\PiE$ when the expert behavior may depend on uncontrollable exogenous noise, while $\qexpert$ does not: the reward and transition dynamics are unaffected by this noise, whereas the expert's control policy is.
\end{enumerate}

\subsection{\texorpdfstring{Constructing $\cQ$ from $\PiE$ Reduces Complexity}{Constructing Q from Pi Reduces Complexity}}
\label{sec:Qisweaker1}

We first show that, given a policy class $\PiE$ realizing an optimal expert policy $\expert$, one can construct a function class $\cQ$ that realizes $\qexpert$ in a relaxed sense. We then show that \ISPIL guarantees continue to hold under this relaxed realizability condition. These guarantees improve the representational-complexity term over those of \BC, \DAGGER, and other policy-based methods that aim to match the expert trajectory distribution in total variation: indeed, such methods incur a $\log \abs{\PiE}$ dependence, while \ISPIL pays a factor $\log \abs{\cQ}$, where $\cQ$ is constructed from $\PiE$ and has no larger cardinality, \ie, $\abs{\cQ} \leq \abs{\PiE}$.\loose

\para{Constructing $\cQ$ from a policy class that realizes the expert}
We preprocess $\PiE$ by keeping only the action supports induced by its policies. Specifically, for each $\pi \in \PiE$, let $\tQ^\pi = \spr{\tQ^\pi_h}_{h = 1}^H$ with $\tQ^\pi_h \colon \cX \times \cA \rightarrow \sbr{0, \QMAX}$ be defined by
\begin{equation*}
    \tQ^\pi_h \spr*{x, a}
    =
    \QMAX \mathds{1}_{\scbr{a \in \texttt{supp} \spr*{\pi_h \spr*{\cdot \given x}}}}\,,
\end{equation*}
for every $\spr{h, x, a} \in \sbr{H} \times \cX \times \cA$. We then set
\begin{equation*}
    \cQ \spr*{\PiE}
    =
    \scbr*{\tQ^\pi \colon \pi \in \PiE}.
\end{equation*}
Since infinitely many stochastic policies in $\PiE$ can share the same support, $\PiE$ may be infinite even when $\cQ \spr{\PiE}$ is finite under the construction above.

\para{Exact $\qexpert$-realizability may fail}
When the expert is (pointwise) optimal, $\expert$ is greedy with respect to $\qexpert$. The construction also includes the function corresponding to $\pi = \expert$,
\begin{equation*}
    \tQ^{\expert}_h \spr*{x, a}
    =
    \QMAX \mathds{1}_{\scbr{a \in \texttt{supp} \spr*{\experth \spr*{\cdot \given x}}}},
\end{equation*}
with respect to which $\expert$ is greedy. However, this does not imply that $\tQ^{\expert} = \qexpert$; in general, $\qexpert$ may not belong to $\cQ$. The key point is that \ISPIL only needs the weaker relationship between $\tQ^\expert$, $\qexpert$, and the expert's action support formalized as relaxed $\qexpert$-realizability in \Cref{ass:relaxed}. The support-indicator witness above satisfies this condition because it assigns $\QMAX$ to every expert-supported action and zero elsewhere, while $\qexpert$ takes values in $\sbr{0, \QMAX}$.

\subsubsection{\texorpdfstring{Guarantees for \ISPIL under relaxed $\qexpert$-realizability}{Guarantees for ISPIL under relaxed Q-expert-realizability}}
\label{sec:relaxedQrea}

In this section, we present guarantees for \ISPIL under a relaxed notion of $\qexpert$-realizability. At a high level, the condition asks $\cQ$ to contain a surrogate for $\qexpert$ that preserves the action comparisons used by \ISPIL. We show that this condition is sufficient when the expert is statewise optimal, and that expert optimality is necessary.
\begin{assumption}[Relaxed $\qexpert$-realizability] \label{ass:relaxed}
    We say that $\cQ$ satisfies relaxed $\qexpert$-realizability if there exists a sequence of functions $\tQ^\expert = \spr{\tQ^\expert_h}^H_{h = 1} \in \cQ$ such that, for every $\spr{h, x} \in \sbr{H} \times \cX$,
    \begin{equation*}
        \texttt{supp} \spr*{\experth \spr*{\cdot \given x}}
        \subseteq
        \argmax_{b \in \cA} \tQ^\expert_h \spr*{x, b},
    \end{equation*}
    and, for every $\spr{h, x, a} \in \sbr{H} \times \cX \times \cA$,
    \begin{align*}
        a \in \texttt{supp} \spr*{\experth \spr*{\cdot \given x}}
        &\Longrightarrow
        \tQ^{\expert}_h \spr*{x, a} \geq \qexperth \spr*{x, a}, \\
        a \notin \texttt{supp} \spr*{\experth \spr*{\cdot \given x}}
        &\Longrightarrow
        \tQ^{\expert}_h \spr*{x, a} \leq \qexperth \spr*{x, a}.
    \end{align*}
\end{assumption}

Under the relaxed $\qexpert$-realizability assumption above, we can prove the same decomposition as in \Cref{lem:pifirst} for the perfectly realizable case, \ie, $\qexpert \in \cQ$.

\begin{restatable}{Lem}{lempifirstrelaxed} \label{lem:pifirstrelaxed}
    Let \Cref{ass:relaxed} hold, and assume that the expert policy $\expert$ is statewise optimal. Let policies $\pi^1, \ldots, \pi^K$ and functions $Q^k_h \in \cQ_h$ be computed as in \cref{eq:updates-pi-first}. Define the $\pi$-regret at stage $h$ and state $x \in \cX$ as
    \begin{equation*}
        \regret^\pi_h \spr*{x}
        =
        \sumkK \inp*{Q^k_h \spr*{x, \cdot}, \experth \spr*{\cdot \given x} - \pi^k_h \spr*{\cdot \given x}},
    \end{equation*}
    define the output policy as $\piout = \frac1K \sumkK \pi^k$, and, for any policy $\pi$, define the estimation error as
    \begin{equation*}
        \Delta \spr*{\pi}
        =
        \max_{h \in \sbr*{H}}\sup_{Q_h \in \cQ_h} \abs*{\hcL_h \spr*{\pi_h, Q_h} - \cL_h^{\piout} \spr*{\pi_h, Q_h}}.
    \end{equation*}
    Then, it holds that
    \begin{equation*}
        J^{\expert} - J^{\piout}
        \leq
        \frac1K \sumhH \sum_{x \in \cX} d^{\piout}_h \spr*{x} \regret^\pi_h \spr*{x} + \frac{2 H}{K} \sumkK \Delta \spr*{\pi^k}.
    \end{equation*}
\end{restatable}

At this point, one can establish the same complexity guarantee as \ISPIL in the exactly realizable setting by bounding $\frac{2 H}{K} \sumkK \Delta \spr{\pi^k}$ with high probability, as in the proof of \Cref{thm:ovi-main-covering}. To conclude the section, we prove \Cref{lem:pifirstrelaxed}.

\begin{proof}[\pfref{lem:pifirstrelaxed}]
    The proof follows the proof of \Cref{lem:pifirst}, except that we first compare $\qexperth$ with the relaxed witness $\tQ^\expert_h$. For any policy $\pi$, we use the shorthand $\cL_h^{\pi} := \cL_h^{d^\pi}$ and we claim that, for every $h \in \sbr{H}$,
    \begin{equation} \label{eq:relaxed-Q-upper-bound}
        \cL_h^{\piout} \spr*{\piouth, \qexperth}
        \leq
        \cL_h^{\piout} \spr*{\piouth, \tQ^\expert_h}.
    \end{equation}
    To prove this claim, fix $h \in \sbr{H}$ and $x \in \cX$, write $\expertsup \ldef \texttt{supp} \spr{\experth \spr{\cdot \given x}}$, and set $\delta_h^x \spr{a} \ldef \experth \spr{a \given x} - \piouth \spr{a \given x}$. By definition of $\cL_h^{\piout}$,
    \begin{equation*}
        \cL_h^{\piout} \spr*{\piouth, \qexperth}
        =
        \sum_{x \in \cX} d^{\piout}_h \spr*{x} \sum_{a \in \cA} \delta_h^x \spr*{a} \qexperth \spr*{x, a}.
    \end{equation*}
    Fix a state $x$. Since $\expert$ is statewise optimal, every action in $\expertsup$ is greedy with respect to $\qexperth \spr{x, \cdot}$, that is, $\expertsup \subseteq \argmax_{b \in \cA} \qexperth \spr{x, b}$. Thus, by definition of $\delta_h^x$,
    \begin{align*}
        \sum_{a \in \cA} \delta_h^x \spr*{a} \qexperth \spr*{x, a}
        &=
        \sum_{a \in \expertsup} \delta_h^x \spr*{a} \qexperth \spr*{x, a} + \sum_{a \notin \expertsup} \delta_h^x \spr*{a} \qexperth \spr*{x, a} \\
        &=
        \max_{b \in \cA} \qexperth \spr*{x, b} \Bigl( 1 - \sum_{a \in \expertsup} \piouth \spr*{a \given x} \Bigr) - \sum_{a \notin \expertsup} \piouth \spr*{a \given x} \qexperth \spr*{x, a},
    \end{align*}
    where we used that $\experth \spr{a \given x} = 0$ for $a \notin \expertsup$, that $\sum_{a \in \expertsup} \experth \spr{a \given x} = 1$, and that all actions in $\expertsup$ maximize $\qexperth \spr{x, \cdot}$. By the relaxed realizability condition (\Cref{ass:relaxed}), we can replace $\qexperth$ with $\tQ^\expert_h$ as follows
    \begin{align*}
        \sum_{a \in \cA} \delta_h^x \spr*{a} \qexperth \spr*{x, a}
        &\overset{\text{(a)}}{\leq}
        \max_{b \in \cA} \tQ^\expert_h \spr*{x, b} \Bigl( 1 - \sum_{a \in \expertsup} \piouth \spr*{a \given x} \Bigr) - \sum_{a \notin \expertsup} \piouth \spr*{a \given x} \tQ^\expert_h \spr*{x, a} \\
        &\overset{\text{(b)}}{=}
        \sum_{a \in \expertsup} \delta_h^x \spr*{a} \tQ^\expert_h \spr*{x, a} - \sum_{a \notin \expertsup} \piouth \spr*{a \given x} \tQ^\expert_h \spr*{x, a} \\
        &\overset{\text{(c)}}{=}
        \sum_{a \in \cA} \delta_h^x \spr*{a} \tQ^\expert_h \spr*{x, a}.
    \end{align*}
    Step (a) uses that on the support $\expertsup$, $\tQ^\expert_h$ is at least $\qexperth$, so $\max_{b \in \cA} \tQ^\expert_h \spr{x, b} \geq \max_{b \in \cA} \qexperth \spr{x, b}$, that $1 - \sum_{a \in \expertsup} \piouth \spr{a \given x} = \sum_{a \notin \expertsup} \piouth \spr{a \given x} \ge 0$, and that outside the support $\expertsup$, $\tQ^\expert_h \spr{x, a} \leq \qexperth \spr{x, a}$ holds, so the inequality is reversed when multiplied by $- \piouth \spr{a \given x}$. Step (b) uses $\expertsup \subseteq \argmax_{b \in \cA} \tQ^\expert_h \spr{x, b}$ and $\sum_{a \in \expertsup} \experth \spr{a \given x} = 1$. Finally, (c) uses the definition of $\delta_h^x$ and the fact that $\delta_h^x \spr{a} = - \piouth \spr{a \given x}$ for $a \notin \expertsup$. Averaging the last display over $x \sim d^{\piout}_h$ gives \cref{eq:relaxed-Q-upper-bound}.

    Since $\tQ^\expert \in \cQ$ and hence $\tQ^\expert_h \in \cQ_h$ for each $h \in \sbr{H}$, we can now upper bound the suboptimality as in \Cref{lem:pifirst}. By the performance difference lemma (\Cref{lem:performance-difference}), and the inequality we just proved (\cref{eq:relaxed-Q-upper-bound}), we have
    \begin{align*}
        K \spr*{J^{\expert} - J^{\piout}}
        &=
        K \sumhH \cL_h^{\piout} \spr*{\piouth, \qexperth} \\
        &\leq
        K \sumhH \cL_h^{\piout} \spr*{\piouth, \tQ^\expert_h}.
    \end{align*}
    Furthermore, by the definitions of the policy $\piout$ and the estimation error $\Delta$, we have
    \begin{align*}
        K \spr*{J^{\expert} - J^{\piout}}
        &\leq
        \sumhH \sumkK \cL_h^{\piout} \spr*{\pi^k_h, \tQ^\expert_h} \\
        &\leq
        \sumhH \sumkK \hcL_h \spr*{\pi^k_h, \tQ^\expert_h} + H \sumkK \Delta \spr*{\pi^k} \\
        &\leq
        \sumhH \sumkK \hcL_h \spr*{\pi^k_h, Q^k_h} + H \sumkK \Delta \spr*{\pi^k} \\
        &\leq
        \sumhH \sumkK \cL_h^{\piout} \spr*{\pi^k_h, Q^k_h} + 2 H \sumkK \Delta \spr*{\pi^k}.
    \end{align*}
    Here, the first and third inequalities use the definition of $\Delta \spr{\pi^k}$, respectively with $\tQ^\expert_h \in \cQ_h$ and $Q^k_h \in \cQ_h$. The second inequality uses $\tQ^\expert_h \in \cQ_h$ (\cref{ass:relaxed}) and the fact that $Q^k_h$ is a best response to $\hcL_h \spr{\pi^k_h, \cdot}$ over $\cQ_h$. Finally, for each $h \in \sbr{H}$, expanding $\cL_h^{\piout}$ gives
    \begin{equation*}
        \sumkK \cL_h^{\piout} \spr*{\pi^k_h, Q^k_h}
        =
        \sum_{x \in \cX} d^{\piout}_h \spr*{x}
        \sumkK \inp*{Q^k_h \spr{x, \cdot}, \experth \spr{\cdot \given x} - \pi^k_h \spr{\cdot \given x}}
        =
        \sum_{x \in \cX} d^{\piout}_h \spr*{x} \regret^\pi_h \spr*{x}.
    \end{equation*}
    Dividing by $K$ concludes the proof.
\end{proof}

\subsubsection{Expert optimality is necessary for learning under relaxed $\qexpert$-realizability}

We conclude by showing that expert optimality is necessary. The failure mode for algorithms that output a VI policy under relaxed $\qexpert$-realizability with a suboptimal expert is that the relaxed condition forces the algorithm to treat all actions in the expert support \emph{equally}. For stochastic suboptimal experts, this can increase the probability of actions rarely played by the expert.\loose
\begin{theorem} \label{thm:suboptimalrelaxed}
    For any algorithm $\mathrm{Alg}$ outputting a VI policy (see \Cref{def:valueIL}, even if allowed expert interaction, there exists a $2$-action MDP $\cM$ with a suboptimal expert policy and a class $\cQ$ satisfying \Cref{ass:relaxed} such that, even if $\mathrm{Alg}$ performs infinitely many expert queries, $\bbE_{\piout \sim \mathrm{Alg}} \sbr{J^{\expert}_{\cM} - J^{\piout}_{\cM}} = \Omega \spr{H}$.
\end{theorem}

\begin{proof}[\pfref{thm:suboptimalrelaxed}]
    Consider a class of environments $\scbr{\cM_1, \cM_2}$ with action space $\cA = \scbr{a_1, a_2}$ and a single absorbing state $x$, \ie, both actions from $x$ transition back to $x$ with probability $1$. For $i \in \sbr{2}$, let $r \spr{x, a_i} = 1$ in $\cM_i$ and $r \spr{x, a} = 0$ for the other action $a \neq a_i$. In $\cM_i$, let the expert policy $\experti$ satisfy $\experti \spr{a_i \given x} = 0.9$ and $\experti \spr{a \given x} = 0.1$ for $a \neq a_i$. \Cref{ass:relaxed} is satisfied by the singleton class $\cQ = \scbr{\tQ}$, where $\tQ_h \spr{x, a} = H$ for all $\spr{h, a} \in \sbr{H} \times \cA$. Indeed, for both $\expertone$ and $\experttwo$, the expert support is the full action space $\cA$. Under the additional value-derived output restriction that the learner can only output a policy obtained by applying a scalar link function to a linear combination of functions in $\cQ$, the output policy is uniform over the action space because $\tQ_h$ assigns the same value to all actions, regardless of which environment the learner faces. Therefore,
    \begin{align*}
        \max_{i \in \sbr*{2}} \bbE_{\piout \sim \mathrm{Alg}} \sbr*{J^{\experti}_{\cM_i} - J^{\piout}_{\cM_i}}
        &\geq
        \frac12 \sum^2_{i = 1} \bbE_{\piout \sim \mathrm{Alg}} \sbr*{J^{\experti}_{\cM_i} - J^{\piout}_{\cM_i}}
        =
        0.9 H - 0.5 H
        =
        0.4 H.
    \end{align*}
    This concludes the proof.
\end{proof}

\subsection{\texorpdfstring{$\qexpert$-Realizability Can Be Strictly Weaker Than $\expert$-Realizability under Exogenous Noise}{Q-expert-Realizability Can Be Strictly Weaker Than Expert-Realizability under Exogenous Noise}}
\label{app:hard-decoder-exbmdp}

In this section, we give an example in which expert-policy realizability is more demanding than $\qexpert$-realizability. The example is an exogenous block MDP (ExBMDP), a standard rich-observation model in which the endogenous control problem is low-dimensional, but the observation-based expert policy must decode a high-dimensional observation. A typical example is a robot with high-dimensional camera observations that contain the (small) endogenous state and a background that evolves independently of the robot's actions (\eg, a television). We formalize the setting below, and the results are stated and proved in subsequent subsections.

One limitation of this example is that a single policy that always chooses action $0$ is optimal in every MDP in the family, so the construction concerns exact representational complexity rather than control hardness. We view this as acceptable because the purpose of the construction is to show that exact expert-policy realizability can be much more demanding than $\qexpert$-realizability.

The motivation is directly tied to the complexity terms that appear in sample-complexity guarantees. Policy-based methods such as \BC and \DAGGER aim to imitate the expert's action distribution. Under an expert-policy realizability assumption, their guarantees typically depend on the complexity of a policy class $\Pi$ satisfying $\expert \in \Pi$, through a term such as $\log \abs{\Pi}$ (or a log-covering number). In contrast, our value-based guarantees depend on the complexity of a value-function class $\cQ$ satisfying $\qexpert \in \cQ$, through a term such as $\log \abs{\cQ}$.

The construction below shows that these two quantities can be exponentially different. We construct a family of ExBMDPs $\scbr{\cM_b : b \in \cB_n}$ and deterministic, suboptimal experts $\scbr{\pi^b : b \in \cB_n}$ indexed by balanced Boolean decoders $b\colon \scbr{0, 1}^n \rightarrow \scbr{0, 1}$, where $n$ is the number of bits in the observation space at stage 1. Any expert-agnostic policy class that contains all experts in the family must have exponentially large metric entropy: for every $\veps \in \spr{0, 1 / 2}$,
\begin{equation*}
    \inf_{\Pi : \forall b \in \cB_n, \pi^b \in \Pi}
    \log_2 \cN_\veps \spr*{\Pi, d_\Pi}
    \ge
    2^n - n - 1\,,
\end{equation*}
where $d_\Pi$ is the sup-total variation metric. By contrast, a two-dimensional class $\cQ$ realizes the expert action-value function $Q^{\pi^b}_{\cM_b}$ for every $b \in \cB_n$, and
\begin{equation*}
    \log \cN_\veps \spr*{\cQ, \norm*{\cdot}_\infty}
    \le
    2 \log \spr*{3 / \veps}\,.
\end{equation*}
Informally, the construction is as follows. We consider an instance where the endogenous state is a single bit $s \in \scbr{0, 1}$, and the expert's endogenous policy simply plays action $a = s$ among the actions $\cA = \scbr{0, 1, 2}$. However, the learner does not observe $s$ directly. Instead, the learner observes a high-dimensional binary string $w \in \scbr{0, 1}^n$, and recovering $s$ from $w$ requires computing a decoder $b \spr{w}$. Thus, an observation-based policy that exactly imitates the expert must compute $\pi_{\experttag, 1} \spr{w} = b \spr{w}$, and the complexity of a policy class realizing the expert must scale with the complexity of the decoder. On the other hand, the $Q$-function does not need this decoder. At the first layer, actions $0$ and $1$ are both safe and value-equivalent, while action $2$ is unsafe. Hence $\qexpert$ only needs to know the value-relevant distinction of $\scbr*{0, 1}$ \vs $\scbr*{2}$. The hard decoder determines only the expert's tie-breaking between the two safe actions, which is what creates this separation. Every expert obtains return $1 / 2$, while the policy that always chooses action $0$ obtains the optimal return $1$. Thus the separation holds even though every expert is strictly suboptimal.

\subsubsection{Preliminaries}
\label{sec:preliminaries-hard-decoder}

\para{Exogenous block MDPs}
A finite-horizon exogenous block MDP (ExBMDP) consists of a tuple
\begin{equation*}
    \spr*{
        H,
        \spr*{\cS_h}_{h = 1}^H,
        \spr*{\Xi_h}_{h = 1}^H,
        \spr*{\cX_h}_{h = 1}^H,
        \cA,
        \spr*{T^{\texttt{endo}}_h}_{h = 0}^{H - 1},
        \spr*{T^{\texttt{exo}}_h}_{h = 0}^{H - 1},
        \spr*{g^{\texttt{obs}}_h}_{h = 1}^H,
        \spr*{r_h}_{h = 1}^H
    },
\end{equation*}
where $H$ is the horizon, and at each layer $h$, $\cS_h$ is the endogenous state space, $\Xi_h$ is the exogenous state space, $\cX_h$ is the observation space, $\cA$ is the action space, $T^{\texttt{endo}}_0$ and $T^{\texttt{exo}}_0$ are initial distributions, for $h \in \sbr{H - 1}$ the map $T^{\texttt{endo}}_h\colon \cS_h \times \cA \to \Delta \spr{\cS_{h+1}}$ is the endogenous transition kernel and $T^{\texttt{exo}}_h\colon \Xi_h \to \Delta \spr{\Xi_{h+1}}$ is the exogenous transition kernel, $g^{\texttt{obs}}_h\colon \cS_h \times \Xi_h \to \cX_h$ is the observation map, and $r_h\colon \cX_h \times \cA \to \sbr{0, 1}$ is the reward function.

The initial latent state is generated by $\mb{s}_1 \sim T^{\texttt{endo}}_0 \spr{\cdot \given \emptyset}$, $\mb{\xi}_1 \sim T^{\texttt{exo}}_0 \spr{\cdot \given \emptyset}$ independently, with the convention that $\cS_0 = \Xi_0 = \emptyset$. For $h \in \sbr{H - 1}$, after taking action $\mb{a}_h \in \cA$, the latent state evolves according to $\mb{s}_{h + 1} \sim T^{\texttt{endo}}_h \spr{\cdot \given \mb{s}_h, \mb{a}_h}$, and $\mb{\xi}_{h + 1} \sim T^{\texttt{exo}}_h \spr{\cdot \given \mb{\xi}_h}$. Thus only the endogenous state is causally affected by the action; the exogenous state evolves independently of the action.

The observation at layer $h$ is generated by a deterministic emission map $\mb{x}_h = g^{\texttt{obs}}_h \spr{\mb{s}_h, \mb{\xi}_h} \in \cX_h$. The model is decodable if for any $h$, there exists a decoder $\phi^\star_h\colon \cX_h \to \cS_h$ such that for any $s \in \cS_h$ and $\xi \in \Xi_h$, $\phi^\star_h \spr{g^{\texttt{obs}}_h \spr{s, \xi}} = s$. Finally, the rewards are endogenous if there exist functions $\bar r_h\colon \cS_h \times \cA \to \sbr{0, 1}$ such that $r_h \spr{x, a} = \bar r_h \spr{\phi^\star_h \spr{x}, a}$ for all $h \in \sbr{H}$, $x \in \cX_h$, and $a \in \cA$.

Because we work in a finite-horizon nonstationary MDP, policies and value
functions are also indexed by the layer $h$. Thus a policy is a sequence $\pi = \spr{\pi_h\colon \cX_h \to \Delta \spr{\cA}}_{h = 1}^H$. In the construction below, we use separate observation spaces $\cX_1$ and $\cX_2$. Equivalently, one can combine them into a single disjoint union $\cX = \cX_1 \sqcup \cX_2$.

\para{Balanced decoders}
For the construction below, we will consider the family of balanced Boolean functions as decoders, defined for any $n \ge 1$ as
\begin{equation*}
    \cB_n = \scbr*{b\colon \scbr*{0, 1}^n \to \scbr*{0, 1} : \abs*{b^{-1} \spr*{0}} = \abs*{b^{-1} \spr*{1}} = 2^{n - 1}}.
\end{equation*}
The number of balanced Boolean functions is $\abs*{\cB_n} = \binom{2^n}{2^{n - 1}}$. By unimodality of the binomial coefficient, and since the largest binomial coefficient is at least the average binomial coefficient, we have
\begin{equation*}
    \binom{2^n}{2^{n - 1}} = \max_{k \in \sbr*{2^n}} \binom{2^n}{k} \ge \frac{1}{2^n + 1} \sum_{k=0}^{2^n} \binom{2^n}{k} = \frac{2^{2^n}}{2^n + 1}.
\end{equation*}
Therefore $\log_2 \abs*{\cB_n} \ge 2^n - \log_2 \spr{2^n + 1} \ge 2^n - n - 1$.

\subsubsection{The construction}
\label{sec:construction-hard-exbmdp}

Fix an integer $n \ge 1$ and a decoder $b \in \cB_n$.

\para{Exogenous block MDP}
We define an ExBMDP $\cM_b$ with horizon $H = 2$ and action space $\cA = \scbr*{0, 1, 2}$. Actions $0$ and $1$ are safe first-layer actions, while action $2$ is unsafe.

Let $\cS = \cS_1 \sqcup \cS_2$ be a layered endogenous state space with $7$ states, where the first layer is $\cS_1 = \scbr*{0, 1}$ and the second layer is $\cS_2 = \cG \cup \scbr*{B}$, with $\cG = \scbr*{G_{s, p}: s \in \scbr*{0, 1}, \ p \in \scbr*{0, 1}}$. The state $G_{s, p}$ is a good state reached from first-layer endogenous state $s$ after taking safe action $p$. The state $B$ is a bad state reached after taking action $2$.

Let the exogenous state space be time-homogeneous and given by $\Xi_1 = \Xi_2 = \scbr*{0, 1}^{n - 1}$. Since $b$ is balanced, we have $\abs*{b^{-1} \spr{0}} = \abs*{b^{-1} \spr{1}} = \abs*{\Xi_1} = 2^{n - 1}$. Therefore, for each $s \in \scbr{0, 1}$, we can fix an arbitrary bijection $\eta^b_s\colon \Xi_1 \to b^{-1} \spr{s}$ that maps an $\spr{n - 1}$-bit exogenous variable to an $n$-bit observation string.

The observation spaces are $\cX_1 = \scbr*{0, 1}^n$ and $\cX_2 = \cS_2 \times \Xi_2$.

The endogenous dynamics are defined as follows. The initial endogenous state is sampled uniformly at random, $\mb{s}_1 \sim \mathrm{Unif} \spr{\scbr*{0, 1}}$, and the transition from layer $1$ to layer $2$ is deterministic:
\begin{equation*}
    \mb{s}_2 =
    \begin{cases}
        G_{\mb{s}_1, \mb{a}_1}, & \mb{a}_1 \in \scbr*{0, 1},\\
        B, & \mb{a}_1 = 2.
    \end{cases}
\end{equation*}
Thus, the first action genuinely affects the endogenous state: actions $0$ and $1$ lead to different good branches, while action $2$ leads to the bad state. This specifies a valid deterministic transition kernel from $\cS_1 \times \cA$ to $\cS_2$.

For the exogenous dynamics, the initial exogenous state is sampled independently of $\mb{s}_1$, also uniformly at random, $\mb{\xi}_1 \sim \mathrm{Unif} \spr{\Xi_1}$, and the exogenous process is persistent, \ie, $\mb{\xi}_2 = \mb{\xi}_1$. Equivalently, $T^{\texttt{exo}}_1 \spr{\xi' \given \xi} = \II{\xi' = \xi}$. The exogenous transition is independent of the action.

We define the observation map at stage $1$ for any $\spr{s, \xi} \in \cS_1 \times \Xi_1$ as $g^{\texttt{obs}}_{1, b} \spr{s, \xi} = \eta^b_s \spr{\xi}$. Hence, conditional on $s = 0$, the observed string is uniformly distributed over $b^{-1} \spr{0}$, and conditional on $s = 1$, it is uniformly distributed over $b^{-1} \spr{1}$. Since $\mb{s}_1$ and $\mb{\xi}_1$ are both uniform, every observation $w \in \scbr{0, 1}^n$ is reachable and equally likely at layer $1$: if $b \spr{w} = s$, then $\bbP \sbr{\mb{x}_1 = w} = \bbP \sbr{\mb{s}_1 = s} \bbP \sbr{\eta^b_s \spr{\mb{\xi}_1} = w} = \frac12 \cdot 2^{-(n - 1)} = 2^{-n}$. We can think of the observation map as hiding the endogenous state within an $n$-bit observation string via the mapping $\eta^b_s$. The observation map at stage $2$ is transparent: for any $\spr{y, \xi} \in \cS_2 \times \Xi_2$, $g^{\texttt{obs}}_{2, b} \spr{y, \xi} = \spr{y, \xi}$.

Define the decoder $\phi^\star_b$ by $\phi^\star_{1, b} \spr{w} = b \spr{w}$ for any $w \in \scbr*{0, 1}^n$, and $\phi^\star_{2, b} \spr{y, \xi} = y$ for any $y \in \cS_2, \ \xi \in \Xi_2$. Then, for every latent state $\spr{s_h, \xi_h}$, $\phi^\star_{h, b} \spr{g^{\texttt{obs}}_{h, b} \spr{s_h, \xi_h}} = s_h$. Indeed, at layer $1$, we have
\begin{equation*}
    \phi^\star_{1, b} \spr*{g^{\texttt{obs}}_{1, b} \spr*{s, \xi}} = \phi^\star_{1, b} \spr*{\eta^b_s \spr*{\xi}} = b \spr*{\eta^b_s \spr*{\xi}} = s,
\end{equation*}
because $\eta^b_s \spr{\xi} \in b^{-1} \spr{s}$. At layer $2$, the claim follows from the transparent observation map.

Finally, we define the reward function. At stage 1, the reward is zero everywhere: for any observation $x \in \cX_1$ and action $a \in \cA$, $r_1 \spr{x, a} = 0$. At stage $2$, for any observation $\spr{y, \xi} \in \cX_2$ and $a \in \cA$, we define $r_2 \spr{\spr{y, \xi}, a} = \II{y \in \cG} \spr*{\II{a = 0} + \frac12 \II{a = 1}}$. Thus action $0$ gives reward $1$ and action $1$ gives reward $1 / 2$ in good second-layer states, while all actions give reward $0$ in the bad state $B$. The reward is endogenous: taking $\bar r_1 \equiv 0$ and $\bar r_2 \spr{y, a} = \II{y \in \cG} \spr*{\II{a = 0} + \frac12 \II{a = 1}}$ gives $r_h \spr{x, a} = \bar r_h \spr{\phi^\star_{h, b} \spr{x}, a}$ for both layers. Together with the decoder identity above and the action-independent exogenous transition, this defines a valid decodable ExBMDP with endogenous rewards.

\para{Expert policy}
Define the endogenous expert policy $\bar \pi^b\colon \cS \to \cA$ by
\begin{equation*}
    \bar\pi^b_1 \spr*{s} = s, \quad s \in \scbr*{0, 1}, \quad\text{and}\quad \bar\pi^b_2 \spr*{y} = 1, \quad y \in \cS_2.
\end{equation*}
The corresponding observation-based expert policy $\pi^b\colon \cX \to \cA$ is
\begin{equation*}
    \pi^b_1 \spr*{w} = b \spr*{w}, \quad w \in \scbr*{0, 1}^n, \quad\text{and}\quad \pi^b_2 \spr*{y, \xi} = 1, \quad y \in \cS_2, \ \xi \in \Xi_2.
\end{equation*}
Thus $\pi^b$ is endogenous, since for any layer $h$ and observation $x \in \cX_h$, $\pi^b_h \spr{x} = \bar\pi^b_h \spr{\phi^\star_{h, b} \spr{x}}$. In endogenous coordinates, the expert is simple. In raw observation coordinates, the first-layer expert action is exactly the decoder $b$.

The expert always reaches a good state and then plays action $1$, so $J_{\cM_b}^{\pi^b} = 1 / 2$. By contrast, the policy that always plays action $0$ reaches a good state and receives reward $1$, which is the maximum possible return. Hence every expert is strictly suboptimal with suboptimality gap $1 / 2$.

\subsubsection{A two-dimensional class realizes the expert $Q$-function}
\label{sec:q-realizability-hard-decoder}

Consider the two-parameter time-indexed function class $\cQ = \scbr{Q^{\alpha, \beta}: \alpha, \beta \in \sbr{0, 1}}$, where, for observations $x \in \cX_1$, $\spr{y, \xi} \in \cX_2$, and action $a \in \cA$, the functions are defined as follows.
\begin{equation*}
    Q^{\alpha, \beta}_1 \spr*{x, a} = \frac{\alpha}{2}\,\II{a \in \scbr*{0, 1}}, \qquad\text{and}\qquad Q^{\alpha, \beta}_2 \spr*{\spr*{y, \xi}, a} = \beta\,\II{y \in \cG} \spr*{\II{a = 0} + \frac12 \II{a = 1}}.
\end{equation*}
Every function in $\cQ$ is bounded by $\QMAX = 1$. Importantly, the class $\cQ$ does not depend on $b$. At layer $1$, it does not distinguish the strings $w \in \scbr{0, 1}^n$ at all; it only distinguishes the safe actions $0, 1$ from the unsafe action $2$. For any function $Q \in \bbR^{\cX \times \cA}$, we will use the norm $\norm{Q}_\infty = \max_{h \in \scbr{1, 2}} \max_{x \in \cX_h} \max_{a \in \cA} \abs{Q_h \spr{x, a}}$. We start by showing that every expert in the family has the same Q-function, which is realizable by $\cQ$.

\begin{lemma}[$Q_{\cM_b}^{\pi^b}$-realizability] \label{lem:q-realizability-hard-decoder}
    For every $b \in \cB_n$, $Q_{\cM_b}^{\pi^b} = Q^{1, 1} \in \cQ$.
\end{lemma}

\begin{proof}[\pfref{lem:q-realizability-hard-decoder}]
    At layer $2$, there is no future reward. Thus, for every
    $\spr{y, \xi} \in \cX_2$ and $a \in \cA$,
    \begin{equation*}
        Q_{\cM_b, 2}^{\pi^b} \spr*{\spr*{y, \xi}, a} = r_2 \spr*{\spr*{y, \xi}, a} = \II{y \in \cG} \spr*{\II{a = 0} + \frac12 \II{a = 1}} = Q^{1, 1}_2 \spr*{\spr*{y, \xi}, a}.
    \end{equation*}

    At layer $1$, the immediate reward is zero. Fix any observation
    $w \in \scbr{0, 1}^n$. Its decoded endogenous state is $s = \phi^\star_{1, b} \spr{w} = b \spr{w}$. If $a \in \scbr{0, 1}$, then the next endogenous state is $G_{s, a} \in \cG$. At layer $2$, the expert plays action $1$, receiving reward $1 / 2$. Thus, for $a \in \scbr{0, 1}$, $Q_{\cM_b, 1}^{\pi^b} \spr{w, a} = 1 / 2$. If $a = 2$, then the next endogenous state is $B$. From $B$, every action receives reward $0$, so $Q_{\cM_b, 1}^{\pi^b} \spr{w, 2} = 0$. Hence
    \begin{equation*}
        Q_{\cM_b, 1}^{\pi^b} \spr*{w, a} = \frac12 \II{a \in \scbr*{0, 1}} = Q^{1, 1}_1 \spr*{w, a}.
    \end{equation*}
    Combining the two layers gives $Q_{\cM_b}^{\pi^b} = Q^{1, 1} \in \cQ$.
\end{proof}

We can then bound the covering number of the class $\cQ$ with the following lemma.
\begin{lemma}[Covering number of the $Q$-class] \label{lem:q-cover-hard-decoder}
    For every $\veps \in (0, 1]$, $\cN_\veps \spr{\cQ, \norm*{\cdot}_\infty} \le \spr{\frac{3}{\veps}}^2$.
\end{lemma}

\begin{proof}[\pfref{lem:q-cover-hard-decoder}]
    For any $\alpha, \beta, \alpha', \beta' \in \sbr{0, 1}$,
    \begin{equation*}
        \norm*{Q^{\alpha, \beta} - Q^{\alpha', \beta'}}_\infty \le \max \scbr*{\abs*{\alpha - \alpha'}, \abs*{\beta - \beta'}}.
    \end{equation*}
    Thus any $\veps$-grid of $\sbr{0, 1}^2$ in the $\ell_\infty$ norm induces an $\veps$-cover of $\cQ$. By \Cref{lem:linfty-ball-cover}, such a grid has cardinality at most $\spr{\ceil*{\frac{1}{\veps}} + 1}^2$. For $\veps \in (0, 1]$, $\ceil*{\frac{1}{\veps}} + 1 \le \frac{3}{\veps}$. This proves the claim.
\end{proof}

\subsubsection{Policy realizability requires an exponentially large class}

We show that any policy class $\Pi$ that contains all experts in the family must have an exponentially large covering number under the sup-total-variation metric $d_\Pi$. In particular, if $\Pi$ is finite, then it must have exponentially large cardinality. This shows that a policy realizability assumption is much stronger than a value-function realizability assumption, since the same two-dimensional class $\cQ$ realizes the expert $Q$-function for every expert in the family, while any single policy class realizing all experts must have exponential complexity.

For any policies $\pi$, $\pi'$, define the sup-total-variation metric
\begin{equation*}
    d_\Pi \spr*{\pi, \pi'} = \frac12 \max_{h \in \scbr*{1, 2}} \norm*{\pi_h - \pi'_h}_{\infty, 1},
\end{equation*}
where $\norm{\cdot}_{\infty, 1}$ is defined in \cref{eq:l-inf-one-norm}. For deterministic policies, this distance equals $1$ whenever the two policies choose different actions at some layer and observation. Let
\begin{equation*}
    \Pi_n
    =
    \scbr*{\pi^b : b \in \cB_n}
\end{equation*}
be the natural expert-agnostic policy class containing all possible experts in the hard-decoder family. We have the following result.
\begin{lemma}[Family-level policy lower bound] \label{lem:family-policy-lower-bound-hard-decoder}
    Let $\Pi \subset \simplex \spr{\cA}^\cX$ be any policy class. If $\Pi_n \subseteq \Pi$, then, for every $\veps \in \spr{0, 1 / 2}$, $\cN_\veps \spr{\Pi, d_\Pi} \ge \abs{\cB_n}$. In particular,
    \begin{equation*}
        \log_2 \cN_\veps \spr*{\Pi, d_\Pi} \ge 2^n - n - 1.
    \end{equation*}
    If $\Pi$ is finite, then also $\abs{\Pi} \ge \abs{\cB_n}$.
\end{lemma}

\begin{proof}[\pfref{lem:family-policy-lower-bound-hard-decoder}]
    If $b \neq b'$ are two different decoders in $\cB_n$, then there exists an observation $w \in \scbr{0, 1}^n$ such that $b \spr{w} \neq b' \spr{w}$. As a result, $\pi^b_1 \spr{w} = b \spr{w} \neq b' \spr{w} = \pi^{b'}_1 \spr{w}$. Therefore, $\pi^b \neq \pi^{b'}$. Hence the map $b \mapsto \pi^b$ is injective. Any finite policy class containing all experts $\scbr{\pi^b : b \in \cB_n}$ must therefore have cardinality at least $\abs{\cB_n}$.

    Moreover, if $b \neq b'$, then the deterministic policies $\pi^b$ and $\pi^{b'}$ choose different actions at some observation $w$ in stage $1$. Hence $d_\Pi \spr{\pi^b, \pi^{b'}} = 1$. Thus $\scbr{\pi^b : b \in \cB_n}$ is a $1$-packing under $d_\Pi$. A ball of radius $\veps < 1 / 2$ cannot contain two distinct elements of this packing. Therefore every $\veps$-cover of $\Pi$ must contain at least $\abs{\cB_n}$ elements. Finally, the lower bound on the log-covering number follows from the lower bound on $\abs{\cB_n}$ shown in \Cref{sec:preliminaries-hard-decoder}.
\end{proof}

\subsubsection{Separation theorem}

We now summarize the separation. The result compares the complexity terms that enter policy-based and value-based guarantees. It shows that on the hard-decoder family, (i) any expert-agnostic class satisfying expert-policy realizability uniformly over all decoders $b$ has exponential metric entropy, while (ii) a value-based guarantee under $\qexpert$-realizability depends on the complexity of a two-dimensional class realizing the expert action-value function.
\begin{theorem}[Hard-decoder ExBMDP separation] \label{thm:combined-hard-decoder-exbmdp}
    The ExBMDP family $\scbr*{\cM_b: b \in \cB_n}$ defined in \Cref{sec:construction-hard-exbmdp} satisfies the following properties.
    \begin{enumerate}
        \item \textbf{Uniform $\qexpert$-realizability.} For every balanced decoder $b \in \cB_n$, the expert policy $\pi^b$ is deterministic and endogenous in $\cM_b$. It is strictly suboptimal, with $J_{\cM_b}^{\pi^b} = 1 / 2$ and $\sup_{\pi} J_{\cM_b}^{\pi} = 1$. Moreover, the two-dimensional class $\cQ$ defined in \cref{sec:q-realizability-hard-decoder} realizes every expert action-value function in the family: for every $b \in \cB_n$, $Q_{\cM_b}^{\pi^b} = Q^{1, 1} \in \cQ$. In addition, for every $\veps \in (0, 1]$,
        \begin{equation*}
            \log \cN_\veps \spr*{\cQ, \norm*{\cdot}_\infty} \le 2 \log \spr*{\frac{3}{\veps}}.
        \end{equation*}

        \item \textbf{Family-level policy lower bound.} Let $\Pi \subset \simplex \spr{\cA}^\cX$ be any expert-agnostic policy class. If $\scbr{\pi^b : b \in \cB_n} \subseteq \Pi$, then, for every $\veps \in \spr{0, 1 / 2}$,
        \begin{equation*}
            \log_2 \cN_\veps \spr*{\Pi, d_\Pi} \ge 2^n - n - 1.
        \end{equation*}
    \end{enumerate}
\end{theorem}

\begin{proof}[\pfref{thm:combined-hard-decoder-exbmdp}]
    The validity of the ExBMDP construction and the fact that each $\pi^b$ is deterministic and endogenous follow from \Cref{sec:construction-hard-exbmdp}. The return gap is also established in that section. The identity $Q_{\cM_b}^{\pi^b} = Q^{1, 1} \in \cQ$ for every $b \in \cB_n$ follows from \Cref{lem:q-realizability-hard-decoder}. The covering-number bound for $\cQ$ follows from \Cref{lem:q-cover-hard-decoder}. The family-level policy lower bound follows from \Cref{lem:family-policy-lower-bound-hard-decoder}.
\end{proof}

\para{Why the policy needs the decoder but the $Q$-function does not}
At layer $1$, for any observation $w$, the expert's action is $\pi^b_1 \spr{w} = b \spr{w}$. Thus exact behavioral cloning must recover the decoder $b$. However, the expert $Q$-function at the same observation is $Q^{\pi^b}_{\cM_b, 1} \spr{w, 0} = Q^{\pi^b}_{\cM_b, 1} \spr{w, 1} = 1 / 2$, and $Q^{\pi^b}_{\cM_b, 1} \spr{w, 2} = 0$, which is independent of $b \spr{w}$. The decoder determines which safe action the expert chooses, but both safe actions have the same value. This is why it is easier to represent the expert's $Q$-function than to represent the expert's policy.

Here the expert is strictly suboptimal because it chooses action $1$ rather than action $0$ at the second layer. This does not affect the tie-breaking mechanism at the first layer: actions $0$ and $1$ have the same expert $Q$-value, and recovering the exact expert policy still requires the decoder-dependent rule $w \mapsto b \spr{w}$. Thus a simple $Q$-function can correspond to an enormous set of expert policies even when those experts are suboptimal. Making the expert's first-layer action uniquely optimal in this construction would force $Q^{\pi^b}$ to encode the decoder, eliminating the present decoder-independent witness class.

\para{Why the exogenous structure matters}
If the endogenous state $s$ were observed directly, the expert would be the constant-complexity policy $\bar\pi^b_1 \spr{s} = s$ and $\bar\pi^b_2 \spr{y} = 1$ on an endogenous state space of size $7$. The controlled dynamics and rewards are also simple in endogenous coordinates.

The complexity appears only because the observation map hides the endogenous bit inside exogenous variation, $w = \eta^b_s \spr{\xi} \in \scbr{0, 1}^n$. Recovering the endogenous bit from the raw observation requires computing $s = b \spr{w}$. Thus the observation-based expert policy factors as $\pi^b_1 \spr{w} = \bar\pi^b_1 \spr{\phi^\star_{1, b} \spr{w}} = b \spr{w}$. Expert-policy realizability therefore requires representing the hard decoder. By contrast, $Q^{\pi^b}$ is identical across decoders. This is why the same two-dimensional $Q$-class works uniformly for all decoders $b$, whereas any policy class that realizes all corresponding experts must be exponentially large.


\clearpage
\section{Experiment Details}
\label{sec:experiments_details}
This section includes additional experimental results and details on the experiments presented in the main text. To start with, we include some experiments in the instance we use for our lower bound to showcase the failure of offline value-based IL under $\qexpert$-realizability and no additional assumptions. As a representative offline algorithm that outputs a VI policy, we use \SPOIL \citep{moulin2025inverse}. These results are described in \Cref{subsec:lower_bounds_exp}. Next, in \Cref{subsec:gym} we provide additional details and comments about the Gymnasium experiments.

\subsection{Experiments on the Lower Bound Instance}
\label{subsec:lower_bounds_exp}

\begin{wrapfigure}[16]{r}{0.5\textwidth}
\vspace{-15pt}
    \begin{tcolorbox}[
        colback=blue!6!white,
        colframe=blue!50!black,
        boxrule=0.6pt,
        arc=2pt,
        width=\linewidth,
        before skip=0pt,
        after skip=0pt,
        left=2pt, right=2pt,
        top=2pt, bottom=2pt
    ]
    \centering
    \includegraphics[width=.9\linewidth]{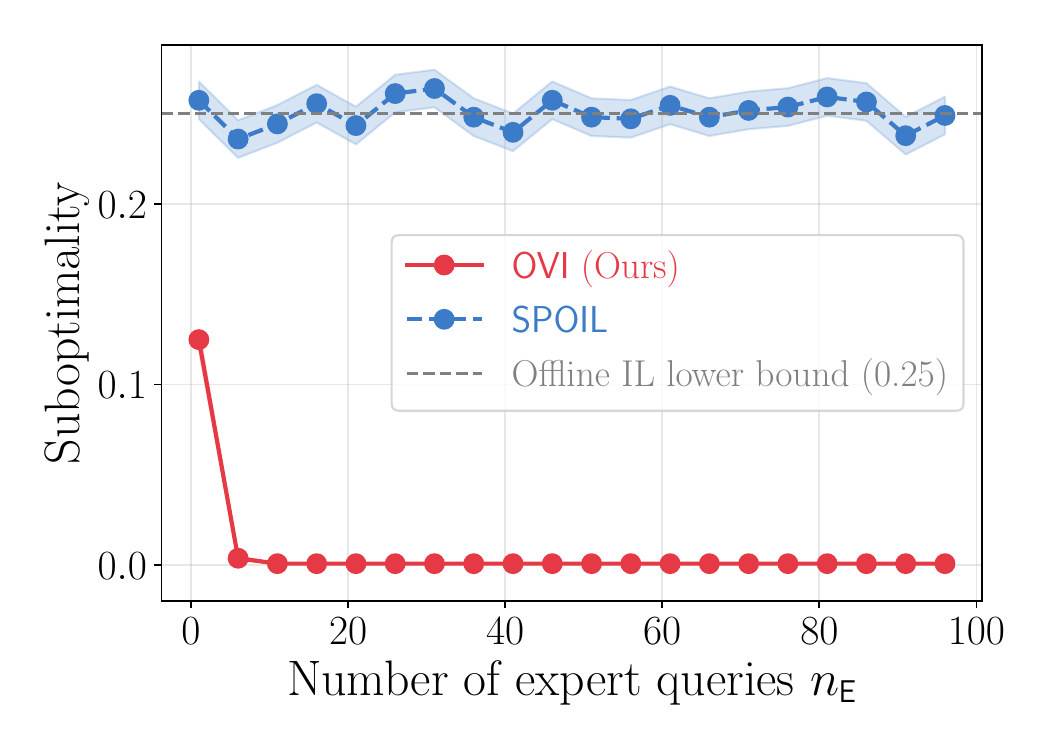}
    \caption{Experiment on the environment class used to prove \Cref{thm:lower}. \label{fig:exp1}}
    \end{tcolorbox}
\end{wrapfigure}

We start by investigating the performance of an offline VI-policies-based algorithm, such as \SPOIL. Our negative result in \Cref{thm:lower} predicts a constant suboptimality in at least one environment. To observe this prediction in practice, we repeat the experiment $500$ times, each time letting the adversary choose uniformly between the two environments $\cM_1$ and $\cM_2$ in the hard family used in the proof of \Cref{thm:lower}. We report their average suboptimality in \Cref{fig:exp1}, which is a lower bound on the worst-case suboptimality. Indeed, at least one of $\cM_1$ and $\cM_2$ has suboptimality no smaller than the average. As \Cref{fig:exp1} shows, \ISPIL reaches zero suboptimality, whereas \SPOIL remains at the constant suboptimality predicted by our construction.

\subsection{Implementation with Neural Networks}
\label{subsec:gym}

For the Gymnasium experiments, we consider a stationary approximation of \Cref{alg:interactive_finite-H}. At the beginning of each epoch,  we label one trajectory collected with the learner policy output by the previous epoch, $\pi^{\ell}$, rather than from $d^{\piout}_h$ as our theory for finite horizon would prescribe. We label such trajectories with expert actions. At this point, we compute the empirical approximation of the objective used by \ISPIL and compute an approximate saddle point by $K$ iterations of the online learning algorithms. The resulting pseudocode is given in \Cref{alg:approximateOVI}. For the experiments, we set $K=50$ while we vary the maximum number of queries along the learner trajectories as shown in \Cref{fig:gym}. The exponential weights update for the policy is approximated via \SOFTDQN \citep{haarnoja2017reinforcement} as in \citet{moulin2025inverse}. Moreover, we have observed that \ISPIL's performance improves when the state dataset is formed by aggregating both expert and learner trajectories. Such data mixing is common in on-policy distillation \citep{agarwal2024policy} and we prove that it extends the representational benefits of interaction; see \Cref{sec:adaptive}.\loose

In the easiest environments like \CARTPOLE and \ACROBOT, we find that 10 trajectories collected offline (for \SPOIL and \BC) or online and then labeled by the expert (for \ISPIL and \DAGGER) usually suffice to match the expert performance. In \Cref{fig:gym}, in more complicated environments such as \LUNARLANDER and \PENDULUM, we increased the number of expert samples, but we did not find it necessary to increase the value of $K$.

\begin{algorithm}[!ht]
    \caption{Stationary approximation of \ISPIL for neural networks. \label{alg:approximateOVI}}
    \begin{algorithmic}[1]
        \STATE Initialize policy weights $\psi^K_0$ and value network weights $\theta^K_0$.
        \FOR{$\ell= 1, \dots, \tauE$}
        \STATE Collect a state trajectory $(\mb{x}_{h,\ell})^H_{h=1}$ with the policy $\pi_{\psi^K_{\ell-1}}$.
        \STATE Label the states with expert actions $\mb{a}^{\experttag}_{h,\ell} \sim \experth(\cdot \given \mb{x}_{h, \ell})$ for each $h \in [H]$.
        \STATE \algcomment{Warm-start initialization of policy and Q network parameters.}
        \STATE $\theta^1_\ell = \theta^{K}_{\ell-1}$, $\psi^1_\ell = \psi^{K}_{\ell-1}$ 
        \FOR{$k=1, \dots, K-1$}
        \STATE \algcomment{Update the policy weights.} 
        \STATE Update $\psi^k_\ell$ via SoftDQN \citep{haarnoja2017reinforcement} using $Q_{\theta^{k}_\ell}$ as the $Q$-network.
        \STATE \algcomment{The argmax is approximated with several Adam iterations initialized at $\theta^k_{\ell}$.}
        \STATE $\theta^{k+1}_\ell \approxeq \argmax_{\theta \in \bbR^d} \sum^\ell_{s=1} \sumhH \br{ Q_\theta(\mb{x}_{h,s}, \mb{a}^{\experttag}_{h,s}) - Q_\theta(\mb{x}_{h,s}, \pi_{\psi^k_\ell})} $.
        \ENDFOR
        \ENDFOR
        \STATE Output the policy $\pi_{\psi^K_{\tauE}}$.
    \end{algorithmic}
\end{algorithm}

As baselines, we use \DAGGER \citep{Ross:2011}, an interactive method requiring policy realizability. We also use standard \BC \citep{Pomerleau:1991}, an offline method requiring expert-policy realizability, and \SPOIL \citep{moulin2025inverse}, an offline method that does not require interactive expert access but assumes the stronger $Q^{\Pi_{\cQ}}$-realizability condition, namely, $Q^\pi \in \cQ$ for every $\pi \in \Pi_{\cQ}$.

\begin{figure*}
    \centering
    \includegraphics[width=0.91\linewidth]{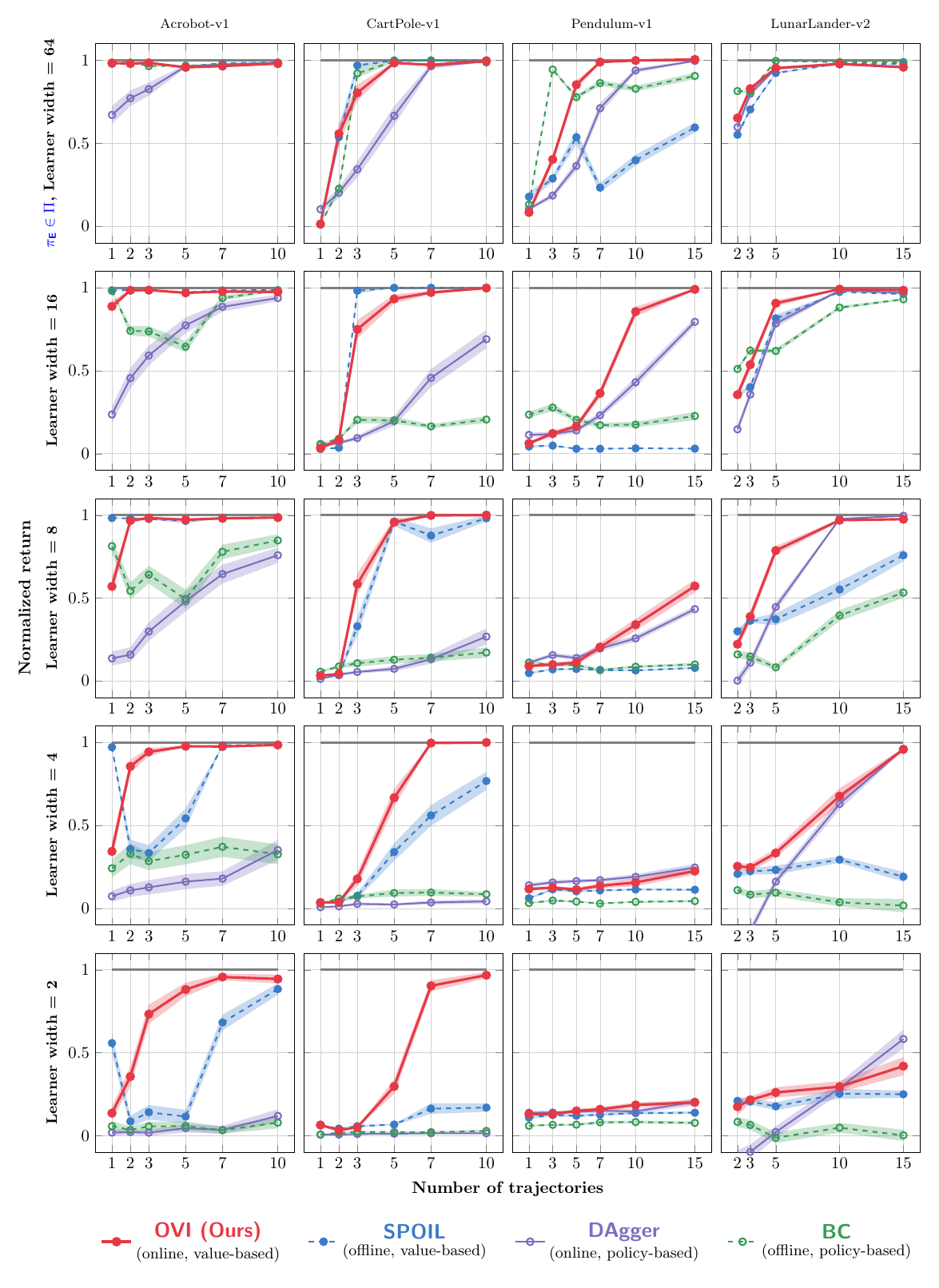}
    \caption{\label{fig:gym} Experiments in the \texttt{Gymnasium} library. The top row reports the setting in which the expert and learner have the same network architecture. In the second and subsequent rows, the learner's network width decreases, taking the values $\bc{16, 8, 4, 2}$ neurons. The $x$-axis reports the number of expert trajectories for offline methods (\BC and \SPOIL) and the number of learner trajectories labeled with expert actions for \DAGGER and \ISPIL.}
\end{figure*}

\subsection{Additional Details and Comments on the Gym Experiments}
\label{sec:more-gym}

We ran the above algorithms in some Gymnasium environments and report the results next. In particular, we first trained an expert network with 2 hidden layers with 64 neurons each via \DQN \citep{Mnih:2015} in \PENDULUM and \PPO \citep{Schulman:2017} in \CARTPOLE and \ACROBOT and then considered the following two settings. First, we considered a learner parameterized via exactly the same architecture (2 hidden layers with 64 neurons). In this case, expert realizability clearly holds. Therefore, \BC and \DAGGER are expected to work well. Second, we considered smaller learner networks (2 hidden layers with 16, 8, 4, or 2 neurons per layer). Expert realizability is a poorly motivated assumption in this second case, and we therefore expected methods like \SPOIL and \ISPIL to be more efficient as the expert--learner size gap increases.

The aforementioned hypotheses are indeed confirmed by our experiments in \CARTPOLE and \ACROBOT. When expert realizability holds (top row of \Cref{fig:gym}), we see that all four methods are essentially equivalent and all eventually manage to match the performance of the expert policy. By contrast, when the learner network becomes smaller, we see that the methods that avoid the expert-realizability assumption achieve better performance. Moreover, the benefits of interaction become evident for the largest size gap we tested (in the last row of \Cref{fig:gym}) in which \ISPIL outperforms \SPOIL. This suggests, therefore, that the network parametrization considered here is not powerful enough to represent the $Q$-functions of the policy sequence generated by the algorithm. An interesting additional observation is that in \CARTPOLE and \ACROBOT, value-based methods clearly outperform policy-based ones. In contrast, in \PENDULUM and \LUNARLANDER, the interactive methods, \ie, \ISPIL and \DAGGER, are the best ones. Therefore, for some environments it seems that interaction is key for performance, while in others value-based representation is more important. We think that such \emph{instance-dependent} phenomena raise interesting open questions.

The \PENDULUM environment naturally comes with continuous actions, but we adapted it to our setting via action discretization, a technique used in practice \citep{dadashi2021continuous} and recently analyzed by \citet{cao2026understanding}. Notably, in this case, we notice that \ISPIL outperforms the standard \SPOIL, potentially because \ISPIL imposes a weaker requirement on the representational power of the $\cQ$-class. The experts achieved the following average returns: $500$ in \CARTPOLE, $-77.1$ in \ACROBOT, $270$ in \LUNARLANDER and $-176.6$ in \PENDULUM. The zero point on the $y$-axis is set to the following returns achieved by very suboptimal policies: $-500$ in \ACROBOT, $8$ in \CARTPOLE, $-1650.0$ in \PENDULUM and $-600$ in \LUNARLANDER. For each environment, we ran each IL algorithm $50$ times with different random seeds from $0$ to $49$. As is common practice in Gymnasium environments \citep{Garg:2021}, we subsampled the expert trajectories for offline methods and retained expert-action labels along learner trajectories only at a fixed frequency for interactive methods. For sample accounting, each retained label counts as one expert sample for offline methods and one expert query for interactive methods. Thus, a trajectory of length $T$ contributes $\ceil{T/f}$ samples or queries at subsampling frequency $f$.

For all algorithms, we tuned their hyperparameters to optimize the performance with no expert--learner gap and then kept the same hyperparameters for the experiments with smaller learner network sizes. The detailed list of hyperparameters is in the \texttt{README} file accompanying our code. For \SPOIL and \ISPIL, we approximate the policy update via \SOFTDQN learning over the $Q$-values, not from rewards, but with the loss motivated by our theoretical analysis. The experiments can be run on CPUs. We ran them on a personal laptop. Note, however, that running all the experiments can take several hours.\loose


\clearpage
\section{Technical Tools}
\label{sec:technical_tools}
\subsection{Concentration}
\label{sec:technical-tools-concentration}

\begin{Lem}[Azuma--Hoeffding inequality; \citealp{Boucheron:2013}, Section~1.1] \label{lemma:azumahoeffding}
    Let $X_1, \ldots, X_n$ be a martingale difference sequence with respect to a filtration $\spr{\cF_i}^n_{i=1}$, \ie, $\bbE \sbr{X_i \given \cF_{i-1}} = 0$ and $\abs{X_i} \leq M$ almost surely for all $i \in \sbr{n}$.  Then, for any $t > 0$,
    \begin{equation*}
        \bbP \sbr*{\abs*{\frac1n \sumin {X_i}} > t}
        \leq
        2 e^{- \frac{n t^2}{2 M^2}}\,.
    \end{equation*}
    Equivalently, for any $\delta \in \spr{0, 1}$, with probability at least $1 - \delta$,
    \begin{equation*}
        \abs*{\frac1n \sumin {X_i}}
        \leq
        M \sqrt{\frac{2 \log \spr*{2 / \delta}}{n}}\,.
    \end{equation*}
\end{Lem}

We recap the definition of covering numbers.

\begin{definition}[Covering number] \label{def:covering}
    Let $\spr{M, \mathrm{d}}$ be a metric space, $S$ be a subset of $M$, and $\epsilon > 0$. A set $\cC_\epsilon \spr{S, \mathrm{d}}$ is an $\epsilon$-covering of $S$ if for any $x \in S$, there exists $y \in \cC_\epsilon \spr{S, \mathrm{d}}$ such that $\mathrm{d} \spr{x, y} \leq \epsilon$. The covering number of $S$, $\cN_\epsilon \spr{S, \mathrm{d}}$, is the minimum cardinality of any such covering of $S$. Moreover, we denote $\cC_\epsilon \spr{S} \ldef \cC_\epsilon \spr{S, \norm{\cdot}_\infty}$.
\end{definition}

We state the finite-horizon concentration lemma that will be used for the uniform covering argument. For decision rules $p, p'\colon \cX \to \simplex \spr{\cA}$, define the policy metric
\begin{equation} \label{eq:l-inf-one-norm}
    \norm*{p - p'}_{\infty,1}
    \ldef
    \max_{x \in \cX} \sum_{a \in \cA} \abs*{p \spr*{a \given x} - p' \spr*{a \given x}}.
\end{equation}
Given two functions $Q, Q' \in \scbr{f\colon \cX \times \cA \to \sbr{0, \QMAX}}$, and two decision rules $p, p' \in \simplex \spr{\cA}^\cX$, define the scaled product metric
\begin{equation} \label{eq:moulin-stage-product-metric}
    \rho \spr*{\spr*{Q, p}, \spr*{Q', p'}}
    \ldef
    \norm*{Q - Q'}_\infty + \QMAX \norm*{p - p'}_{\infty,1}.
\end{equation}

\begin{Lem}[Finite-horizon uniform concentration; analogue of \citealp{moulin2025inverse}, Lemma~5] \label{lem:moulin-fh-concentration}
    Fix nonempty deterministic classes $\Pi_h \subseteq \simplex \spr{\cA}^{\cX}$ and $\cQ_h \subseteq \scbr{Q_h\colon \cX \times \cA \to \sbr{0, \QMAX}}$ for every $h \in \sbr{H}$. For each $h$, let $\cF_{h - 1}$ denote the history available before the samples for stage $h$ are drawn. Let $d_h \in \simplex \spr{\cX}$ be $\cF_{h - 1}$-measurable, and suppose that, conditionally on $\cF_{h - 1}$,
    \begin{equation*}
        \Xih \sim d_h,
        \qquad
        \AEih \sim \experth \spr*{\cdot \given \Xih},
        \qquad
        i \in \sbr*{\tauE},
    \end{equation*}
    are independent. For any $p \in \Pi_h$ and $Q_h \in \cQ_h$, define
    \begin{align*}
        \cL_h^d \spr*{p, Q_h}
        &\ldef
        \sum_{x \in \cX} d_h \spr*{x} \inp*{Q_h \spr*{x, \cdot}, \experth \spr*{\cdot \given x} - p \spr*{\cdot \given x}},
        \\
        \hcL_h^d \spr*{p, Q_h}
        &\ldef
        \frac{1}{\tauE} \sum_{i = 1}^{\tauE} \spr*{Q_h \spr*{\Xih, \AEih} - Q_h \spr*{\Xih, p}}.
    \end{align*}
    Then, for any $r > 0$ and $\delta \in \spr{0, 1}$, with probability at least $1 - \delta$,
    \begin{equation} \label{eq:moulin-fh-concentration}
        \max_{h \in \sbr*{H}} \sup_{p \in \Pi_h} \sup_{Q_h \in \cQ_h} \abs*{\hcL_h^d \spr*{p, Q_h} - \cL_h^d \spr*{p, Q_h}}
        \leq
        4 r + \QMAX \sqrt{\frac{8 \log \spr*{2 \sumhH \cN_r \spr*{\cQ_h \times \Pi_h, \rho} / \delta}}{\tauE}}.
    \end{equation}
    If $\sumhH \cN_r \spr{\cQ_h \times \Pi_h, \rho} = \infty$, the bound is interpreted as vacuous.
\end{Lem}

\begin{proof}[\pfref{lem:moulin-fh-concentration}]
    The result is vacuous when the total covering number is infinite. Thus, assume that
    \begin{equation*}
        \sumhH \cN_r \spr*{\cQ_h \times \Pi_h, \rho} < \infty.
    \end{equation*}
    Fix $h$ and condition on $\cF_{h - 1}$. For a fixed pair $\spr{Q_h, p} \in \cQ_h \times \Pi_h$, define
    \begin{equation*}
        Z_i \spr*{Q_h, p}
        =
        Q_h \spr*{\Xih, \AEih} - Q_h \spr*{\Xih, p} - \cL_h^d \spr*{p, Q_h}.
    \end{equation*}
    Then $\spr{Z_i \spr{Q_h, p}}_{i = 1}^{\tauE}$ is a conditionally independent mean-zero sequence and $\abs{Z_i \spr{Q_h, p}} \leq 2 \QMAX$ almost surely. By the Azuma--Hoeffding inequality (\Cref{lemma:azumahoeffding}), for every $t > 0$,
    \begin{equation} \label{eq:moulin-fixed-pair}
        \bbP \sbr*{\abs*{\hcL_h^d \spr*{p, Q_h} - \cL_h^d \spr*{p, Q_h}} \geq t \given \cF_{h - 1}}
        \leq
        2 \exp \spr*{- \frac{\tauE t^2}{8 \QMAX^2}}.
    \end{equation}
    For each $h$, choose a finite set $\cC_{h,r} \subseteq \cQ_h \times \Pi_h$ with cardinality $\cN_r \spr{\cQ_h \times \Pi_h, \rho}$ such that every pair in $\cQ_h \times \Pi_h$ lies within distance $r$ of some element of $\cC_{h,r}$ under $\rho$. Applying \cref{eq:moulin-fixed-pair} to all elements of these finite sets, then taking expectations and a union bound over $h \in \sbr{H}$, gives an event of probability at least $1 - \delta$ on which, for all $h$ and all $\spr{\wt{Q}_h, \wt{p}} \in \cC_{h,r}$,
    \begin{equation} \label{eq:moulin-cover-event}
        \abs*{\hcL_h^d \spr*{\wt{p}, \wt{Q}_h} - \cL_h^d \spr*{\wt{p}, \wt{Q}_h}}
        \leq
        \QMAX \sqrt{\frac{8 \log \spr*{2 \sumhH \cN_r \spr*{\cQ_h \times \Pi_h, \rho} / \delta}}{\tauE}}.
    \end{equation}
    We work on this event and pass from the cover to the full class. Fix $h$, $Q_h \in \cQ_h$, and $p \in \Pi_h$, and choose $\spr{\wt{Q}_h, \wt{p}} \in \cC_{h,r}$ such that
    \begin{equation*}
        \norm*{Q_h - \wt{Q}_h}_\infty + \QMAX \norm*{p - \wt{p}}_{\infty,1}
        \leq
        r.
    \end{equation*}
    By the triangle inequality,
    \begin{equation*}
        \abs*{\hcL_h^d \spr*{p, Q_h} - \hcL_h^d \spr*{\wt{p}, \wt{Q}_h}}
        \leq
        \frac{1}{\tauE}\sum_{i = 1}^{\tauE} \abs*{Q_h \spr*{\Xih, \AEih} - \wt{Q}_h \spr*{\Xih, \AEih}} + \frac{1}{\tauE}\sum_{i = 1}^{\tauE} \abs*{Q_h \spr*{\Xih, p} - \wt{Q}_h \spr*{\Xih, \wt{p}}}.
    \end{equation*}
    For every state $x$,
    \begin{align*}
        \abs*{Q_h \spr*{x, p} - \wt{Q}_h \spr*{x, \wt{p}}}
        &\leq
        \abs*{Q_h \spr*{x, p} - \wt{Q}_h \spr*{x, p}} + \abs*{\wt{Q}_h \spr*{x, p} - \wt{Q}_h \spr*{x, \wt{p}}}
        \\
        &\leq
        \norm*{Q_h - \wt{Q}_h}_\infty + \QMAX \norm*{p - \wt{p}}_{\infty,1}.
    \end{align*}
    Therefore,
    \begin{equation} \label{eq:moulin-emp-lip}
        \abs*{\hcL_h^d \spr*{p, Q_h} - \hcL_h^d \spr*{\wt{p}, \wt{Q}_h}}
        \leq
        2 \norm*{Q_h - \wt{Q}_h}_\infty + \QMAX \norm*{p - \wt{p}}_{\infty,1}
        \leq
        2 r.
    \end{equation}
    The same deterministic calculation, with empirical averages replaced by expectation under $x \sim d_h$ and $a \sim \experth \spr{\cdot \given x}$, gives
    \begin{equation} \label{eq:moulin-true-lip}
        \abs*{\cL_h^d \spr*{p, Q_h} - \cL_h^d \spr*{\wt{p}, \wt{Q}_h}}
        \leq
        2 r.
    \end{equation}
    Combining \cref{eq:moulin-cover-event,eq:moulin-emp-lip,eq:moulin-true-lip} yields \cref{eq:moulin-fh-concentration}. Since $h$, $p$, and $Q_h$ were arbitrary, the proof is complete.
\end{proof}

For later use, we also record the covering bounds for the decision-rule classes induced by exponentiated-gradient updates. For function classes $\cQ_h \subseteq \scbr{Q_h\colon \cX \times \cA \to \sbr{0, \QMAX}}$, an integer $K \geq 0$, and $\eta > 0$, define
\begin{equation*}
    \Pi_\cQ
    \ldef
    \scbr*{\pi: \exists m \in \scbr*{0, \ldots, K}, \forall h, \exists Q_h^1, \ldots, Q_h^m \in \cQ_h,~~~\pi_h \spr*{\cdot \given x} = \softmax \spr[\Big]{\eta\ {\textstyle\sum_j} Q_h^j \spr*{x, \cdot}}}.
\end{equation*}
For each $h$, let $\Pi_{\cQ,h} \ldef \scbr{\pi_h : \pi \in \Pi_\cQ}$. When $m = 0$, the empty sum is the zero vector, so the corresponding decision rule is uniform.

\begin{Lem}[Nonconvex product covering; analogue of \citealp{moulin2025inverse}, Lemma~7] \label{lem:moulin-nonconvex-product-cover}
    For every $h \in \sbr{H}$ and every $\veps > 0$,
    \begin{equation} \label{eq:moulin-nonconvex-policy-cover}
        \cN_\veps \spr*{\Pi_{\cQ,h}, \norm*{\cdot}_{\infty,1}}
        \leq
        1 + \sum_{m = 1}^K \cN_{\veps / \spr*{\eta m}} \spr*{\cQ_h, \norm*{\cdot}_\infty}^m.
    \end{equation}
    Consequently, for the scaled product metric $\rho$ in \cref{eq:moulin-stage-product-metric}, for every $r > 0$,
    \begin{equation} \label{eq:moulin-nonconvex-product-cover}
        \cN_r \spr*{\cQ_h \times \Pi_{\cQ,h}, \rho}
        \leq
        \inf_{\alpha \in \spr*{0, r}} \cN_\alpha \spr*{\cQ_h, \norm*{\cdot}_\infty} \spr*{1 + \sum_{m = 1}^K \cN_{\spr*{r - \alpha} / \spr*{\QMAX \eta m}} \spr*{\cQ_h, \norm*{\cdot}_\infty}^m}.
    \end{equation}
    In particular, the simpler bound
    \begin{equation} \label{eq:moulin-nonconvex-product-cover-simple}
        \cN_r \spr*{\cQ_h \times \Pi_{\cQ,h}, \rho}
        \leq
        \spr*{K + 1} \cN_{\frac{r}{2 \max \scbr*{1, \QMAX \eta K}}} \spr*{\cQ_h, \norm*{\cdot}_\infty}^{K + 1}
    \end{equation}
    also holds.
\end{Lem}

\begin{proof}[\pfref{lem:moulin-nonconvex-product-cover}]
    Fix $h$ and $m \in \scbr{0, \ldots, K}$. If $m = 0$, the only decision rule represented with $m$ functions is the uniform decision rule. Suppose now that $m \geq 1$, and let $\cC_s$ be an $s$-covering of $\cQ_h$ under $\norm{\cdot}_\infty$. Consider any $p \in \Pi_{\cQ,h}$ represented with this value of $m$ by $Q_h^1, \ldots, Q_h^m \in \cQ_h$, and choose $\wt{Q}_h^j \in \cC_s$ such that $\norm{Q_h^j - \wt{Q}_h^j}_\infty \leq s$ for all $j \in \sbr{m}$. Define
    \begin{equation*}
        \wt{p} \spr*{\cdot \given x}
        =
        \softmax \spr*{\eta \sum_{j = 1}^m \wt{Q}_h^j \spr*{x, \cdot}}.
    \end{equation*}
    By 1-Lipschitz continuity of the softmax (\cref{lem:softmax-lipschitz}) and the triangle inequality, for every $x \in \cX$,
    \begin{align*}
        \norm*{p \spr*{\cdot \given x} - \wt{p} \spr*{\cdot \given x}}_1
        &\leq
        \eta \sum_{j = 1}^m \norm*{Q_h^j - \wt{Q}_h^j}_\infty
        \leq
        \eta m s,
    \end{align*}
    where we used the definitions of $\wt{Q}_h^1, \ldots, \wt{Q}_h^m$ in the last inequality. Taking $s = \veps / \spr{\eta m}$ gives a $\veps$-covering for the decision rules represented with this value of $m$, of cardinality at most $\cN_{\veps / \spr{\eta m}} \spr{\cQ_h, \norm{\cdot}_\infty}^m$. Taking the union over $m = 0, \ldots, K$ proves \cref{eq:moulin-nonconvex-policy-cover}.

    For the product bound, fix $\alpha \in \spr{0, r}$. Take an $\alpha$-covering of $\cQ_h$ under $\norm{\cdot}_\infty$ and an $\spr{r - \alpha} / \QMAX$-covering of $\Pi_{\cQ,h}$ under $\norm{\cdot}_{\infty,1}$. By definition of $\rho$, their Cartesian product is an $r$-covering of $\cQ_h \times \Pi_{\cQ,h}$ under $\rho$. Applying \cref{eq:moulin-nonconvex-policy-cover} with $\veps = \spr{r - \alpha} / \QMAX$, and then taking the infimum over $\alpha$, gives \cref{eq:moulin-nonconvex-product-cover}.

    To obtain \cref{eq:moulin-nonconvex-product-cover-simple}, take $\alpha = r / 2$ in \cref{eq:moulin-nonconvex-product-cover}. If $K = 0$, the claim follows immediately. Suppose then that $K \geq 1$. The radius $r / \spr{2 \max \scbr{1, \QMAX \eta K}}$ is at most $r / 2$ and at most $r / \spr{2 \QMAX \eta m}$ for every $1 \leq m \leq K$. Thus, by monotonicity of covering numbers,
    \begin{equation*}
        \cN_{r / 2} \spr*{\cQ_h, \norm*{\cdot}_\infty}
        \leq
        \cN_{\frac{r}{2 \max \scbr*{1, \QMAX \eta K}}} \spr*{\cQ_h, \norm*{\cdot}_\infty},
        \qquad
        \cN_{\frac{r}{2 \QMAX \eta m}} \spr*{\cQ_h, \norm*{\cdot}_\infty}^m
        \leq
        \cN_{\frac{r}{2 \max \scbr*{1, \QMAX \eta K}}} \spr*{\cQ_h, \norm*{\cdot}_\infty}^K.
    \end{equation*}
    Since covering numbers of nonempty classes are at least one, \cref{eq:moulin-nonconvex-product-cover-simple} follows.
\end{proof}

\begin{Lem}[Convex product covering; analogue of \citealp{moulin2025inverse}, Lemma~8] \label{lem:moulin-convex-product-cover}
    Assume that each $\cQ_h$ is convex. Then, for every $h \in \sbr{H}$ and every $\veps > 0$,
    \begin{equation} \label{eq:moulin-convex-policy-cover}
        \cN_\veps \spr*{\Pi_{\cQ,h}, \norm*{\cdot}_{\infty,1}}
        \leq
        1 + \sum_{m = 1}^K \cN_{\veps / \spr*{\eta m}} \spr*{\cQ_h, \norm*{\cdot}_\infty}.
    \end{equation}
    Consequently, for the scaled product metric $\rho$ in \cref{eq:moulin-stage-product-metric}, for every $r > 0$,
    \begin{equation} \label{eq:moulin-convex-product-cover}
        \cN_r \spr*{\cQ_h \times \Pi_{\cQ,h}, \rho}
        \leq
        \inf_{\alpha \in \spr*{0, r}} \cN_\alpha \spr*{\cQ_h, \norm*{\cdot}_\infty} \spr*{1 + \sum_{m = 1}^K \cN_{\spr*{r - \alpha} / \spr*{\QMAX \eta m}} \spr*{\cQ_h, \norm*{\cdot}_\infty}}.
    \end{equation}
    In particular, the simpler bound
    \begin{equation} \label{eq:moulin-convex-product-cover-simple}
        \cN_r \spr*{\cQ_h \times \Pi_{\cQ,h}, \rho}
        \leq
        \spr*{K + 1} \cN_{\frac{r}{2 \max \scbr*{1, \QMAX \eta K}}} \spr*{\cQ_h, \norm*{\cdot}_\infty}^{2}
    \end{equation}
    also holds.
\end{Lem}

\begin{proof}[\pfref{lem:moulin-convex-product-cover}]
    Fix $h$ and $m \in \scbr{0, \ldots, K}$. The case $m = 0$ again gives the uniform decision rule. Suppose that $m \geq 1$, and let $p \in \Pi_{\cQ,h}$ be represented with this value of $m$ by $Q_h^1, \ldots, Q_h^m \in \cQ_h$. Define
    \begin{equation*}
        \wb{Q}_h
        =
        \frac{1}{m} \sum_{j = 1}^m Q_h^j.
    \end{equation*}
    By convexity, $\wb{Q}_h \in \cQ_h$, and $p \spr{\cdot \given x} = \softmax \spr{\eta m \wb{Q}_h \spr{x, \cdot}}$ for every $x \in \cX$. Let $\cC_s$ be an $s$-covering of $\cQ_h$ under $\norm{\cdot}_\infty$, choose $\wt{Q}_h \in \cC_s$ such that $\norm{\wb{Q}_h - \wt{Q}_h}_\infty \leq s$, and define
    \begin{equation*}
        \wt{p} \spr*{\cdot \given x}
        =
        \softmax \spr*{\eta m \wt{Q}_h \spr*{x, \cdot}}.
    \end{equation*}
    By 1-Lipschitz continuity of the softmax (\cref{lem:softmax-lipschitz}) and the definition of $\wt{Q}_h$, for every $x \in \cX$,
    \begin{equation*}
        \norm*{p \spr*{\cdot \given x} - \wt{p} \spr*{\cdot \given x}}_1
        \leq
        \eta m \norm*{\wb{Q}_h - \wt{Q}_h}_\infty
        \leq
        \eta m s.
    \end{equation*}
    Taking $s = \veps / \spr{\eta m}$ gives a $\veps$-covering for the decision rules represented with this value of $m$, of cardinality at most $\cN_{\veps / \spr{\eta m}} \spr{\cQ_h, \norm{\cdot}_\infty}$. Taking the union over $m = 0, \ldots, K$ proves \cref{eq:moulin-convex-policy-cover}. The product bound follows from the same Cartesian-product argument used in the proof of \Cref{lem:moulin-nonconvex-product-cover}.

    To obtain \cref{eq:moulin-convex-product-cover-simple}, take $\alpha = r / 2$ in \cref{eq:moulin-convex-product-cover}. If $K = 0$, the claim follows immediately. Suppose then that $K \geq 1$. The radius $r / \spr{2 \max \scbr{1, \QMAX \eta K}}$ is at most $r / 2$ and at most $r / \spr{2 \QMAX \eta m}$ for every $1 \leq m \leq K$. Thus, by monotonicity of covering numbers,
    \begin{equation*}
        \cN_{r / 2} \spr*{\cQ_h, \norm*{\cdot}_\infty}
        \leq
        \cN_{\frac{r}{2 \max \scbr*{1, \QMAX \eta K}}} \spr*{\cQ_h, \norm*{\cdot}_\infty},
        \qquad
        \cN_{\frac{r}{2 \QMAX \eta m}} \spr*{\cQ_h, \norm*{\cdot}_\infty}
        \leq
        \cN_{\frac{r}{2 \max \scbr*{1, \QMAX \eta K}}} \spr*{\cQ_h, \norm*{\cdot}_\infty}.
    \end{equation*}
    Since covering numbers of nonempty classes are at least one, \cref{eq:moulin-convex-product-cover-simple} follows.
\end{proof}

We also recall the following standard bound.

\begin{lemma}[Covering number of an $\ell_\infty$ ball; \citealp{vershynin2018high}, Section~4.2] \label{lem:linfty-ball-cover}
    Let $d \in \bbN$, $R > 0$, and let $\fkB_\infty^d \spr{R}$ denote the $\ell_\infty$ ball of radius $R$ in $\bbR^d$. Then, for every $\veps > 0$,
    \begin{equation*}
        \cN_\veps \spr*{\fkB_\infty^d \spr*{R}, \norm*{\cdot}_\infty}
        \le
        \spr*{\ceil*{\frac{2R}{\veps}} + 1}^d.
    \end{equation*}
\end{lemma}

\subsection{Optimization}

We will need the following regret bounds for mirror descent and follow-the-regularized-leader (FTRL), which are standard in the literature. We report them here for completeness.

\begin{Lem}[Simplified version of \citealp{orabona2023modern}, Theorem~6.11] \label{lemma:mirror_orabona}
    Let us consider a non-empty closed convex set $V$, an arbitrary sequence of adaptively chosen loss vectors $\spr{\ell_k}^K_{k=1}$ such that $\norm{\ell_k}_{\infty} \leq \ell_{\max}$, and let $D\colon V \times \mathrm{int} \spr{V} \rightarrow \bbR$ be a Bregman divergence induced by a $\lambda$-strongly convex function in the $\ell_1$ norm. Then, for all $u \in V$, the sequence $\spr{x_k}^K_{k=1}$ generated for each $k$ by
    \begin{equation*}
        x_{k + 1}
        =
        \argmin_{v \in V} \scbr*{\inp{\ell_k, v} + \frac1\eta D \spr*{v, x_k}}
    \end{equation*}
    for an initial $x_1 \in \mathrm{int} \spr{V}$ and under the assumption that $x_k \in \mathrm{int} \spr{V}$ for every $k \in \sbr{K + 1}$ satisfies
    \begin{equation*}
        \sumkK \inp*{\ell_k, x_k - u}
        \leq
        \frac{D \spr*{u, x_1}}{\eta} + \frac{\eta K \ell^2_{\max}}{2 \lambda}\,.
    \end{equation*}
\end{Lem}

\begin{Lem}[FTRL over the simplex] \label{lemma:FTRL}
    Let $\simplex \spr{\cA}$ be the simplex over a discrete action space $\cA$, and for any $x \in \simplex \spr{\cA}$, define the Shannon entropy as $H \spr{x} = - \sum_{a \in \cA} x \spr{a} \log x \spr{a}$. Consider the sequence $\spr{x_k}^K_{k = 1}$ generated via FTRL using the potential
    \begin{equation*}
        \psi_k \spr*{x}
        =
        \frac{- H \spr*{x} - \min_{x' \in \simplex \spr*{\cA}} \spr*{- H \spr*{x'}}}{\eta_k}.
    \end{equation*}
    That is, for a certain sequence of cost vectors $\ell_1, \ldots, \ell_K$ we have
    \begin{equation*}
        x_k
        =
        \argmin_{x \in \simplex \spr*{\cA}} \inp*{x, \sum^{k - 1}_{k' = 1} \ell_{k'}} + \psi_k \spr*{x}
        =
        \frac{e^{- \eta_k \sum^{k - 1}_{k' = 1} \ell_{k'}}}{\sum_{a \in \cA} e^{- \eta_k \sum^{k - 1}_{k' = 1} \ell_{k'} \spr*{a}}}\,.
    \end{equation*}
    Then, if $\norm{\ell_k}_\infty \leq \ell_{\max}$ and $\eta_k = \sqrt{\frac{\log \spr{A}}{\ell_{\max}^2 k}}$, it holds that for any $K \in \bbN$ and any $x^\star \in \simplex \spr{\cA}$,
    \begin{equation*}
        \sumkK \inp*{\ell_k, x_k - x^\star}
        \leq
        3 \sqrt{\ell_{\max}^2 \log \abs*{\cA} K}.
    \end{equation*}
\end{Lem}
\begin{proof}[\pfref{lemma:FTRL}]
    Let us define the function $F_k \spr{x} = \sum^{k - 1}_{k' = 1} \inp{\ell_{k'}, x} + \psi_k \spr{x}$. Then, by definition, we have
    \begin{equation*}
        - \sumkK \inp*{\ell_k, x^\star}
        =
        \psi_{K + 1} \spr*{x^\star} - F_{K + 1} \spr*{x^\star}\,.
    \end{equation*}
    Adding and subtracting $F_{K + 1} \spr{x_{K + 1}}$ and $F_1 \spr{x_1}$ and noticing that $x_1$ minimizes $F_1 = \psi_1$ over the simplex, we obtain
    \begin{equation*}
        - \sumkK \inp*{\ell_k, x^\star}
        =
        \psi_{K + 1} \spr*{x^\star} - F_{K + 1} \spr*{x^\star} - F_{K + 1} \spr*{x_{K + 1}} + F_{K + 1} \spr*{x_{K + 1}} - \min_{x \in \simplex \spr*{\cA}} \psi_1 \spr*{x} + F_1 \spr*{x_1}.
    \end{equation*}
    Now, writing the difference $F_1 \spr{x_1} - F_{K + 1} \spr{x_{K + 1}}$ as a telescoping sum, we have
    \begin{equation*}
        - \sumkK \inp*{\ell_k, x^\star}
        =
        \psi_{K + 1} \spr*{x^\star} - F_{K + 1} \spr*{x^\star} + \sumkK \spr*{F_k \spr*{x_k} - F_{k + 1} \spr*{x_{k + 1}}} + F_{K + 1} \spr*{x_{K + 1}} - \min_{x \in \simplex \spr*{\cA}} \psi_1 \spr*{x}\,,
    \end{equation*}
    Finally, adding $\sumkK \inp{\ell_k, x_k}$ to both sides, we obtain
    \begin{align*}
        \sumkK \inp*{\ell_k, x_k - x^\star}
        &=
        \psi_{K + 1} \spr*{x^\star} - F_{K + 1} \spr*{x^\star} + \sumkK \spr*{F_k \spr*{x_k} - F_{k + 1} \spr*{x_{k + 1}} + \inp*{\ell_k, x_k}} \\
        &\quad+ F_{K + 1} \spr*{x_{K + 1}}
        - \min_{x \in \simplex \spr*{\cA}} \psi_1 \spr*{x}.
    \end{align*}
    Using that $\psi_k \spr{x} \geq 0$ for all $k$ and $x \in \simplex \spr{\cA}$, and $F_{K + 1} \spr{x^\star} \geq F_{K + 1} \spr{x_{K + 1}}$ (by definition of $x_{K + 1}$), we obtain
    \begin{align*}
        \sumkK \inp*{\ell_k, x_k - x^\star}
        &\leq
        \psi_{K + 1} \spr*{x^\star} + \sumkK \spr*{F_k \spr*{x_k} - F_{k + 1} \spr*{x_{k + 1}} + \inp*{\ell_k, x_k}}.
    \end{align*}
    At this point, note
    \begin{align*}
        F_k &\spr*{x_k} - F_{k + 1} \spr*{x_{k + 1}} + \inp*{\ell_k, x_k} \\
        &\overset{\scriptscriptstyle\mathrm{(a)}}{=}
        F_k \spr*{x_k} + \inp*{\ell_k, x_k} - \spr*{F_{k + 1} \spr*{x_{k + 1}} + \inp*{\ell_k, x_{k + 1}}} + \inp*{\ell_k, x_{k + 1}} \\
        &\overset{\scriptscriptstyle\mathrm{(b)}}{=}
        F_k \spr*{x_k} + \inp*{\ell_k, x_k} - \spr*{F_k \spr*{x_{k + 1}} + \inp*{\ell_k, x_{k + 1}}} + \inp*{\ell_k, x_{k + 1}} + F_k \spr*{x_{k + 1}} - F_{k + 1} \spr*{x_{k + 1}} \\
        &\overset{\scriptscriptstyle\mathrm{(c)}}{=}
        F_k \spr*{x_k} + \inp*{\ell_k, x_k} - \spr*{F_k \spr*{x_{k + 1}} + \inp*{\ell_k, x_{k + 1}}} \\
        &\quad+
        F_{k + 1} \spr*{x_{k + 1}} - \psi_{k + 1} \spr*{x_{k + 1}} + \psi_k \spr*{x_{k + 1}} - F_{k + 1} \spr*{x_{k + 1}} \\
        &\overset{\scriptscriptstyle\mathrm{(d)}}{=}
        F_k \spr*{x_k} + \inp*{\ell_k, x_k} - \spr*{F_k \spr*{x_{k + 1}} + \inp*{\ell_k, x_{k + 1}}} - \psi_{k + 1} \spr*{x_{k + 1}} + \psi_k \spr*{x_{k + 1}}\,,
    \end{align*}
    where we (a) added and removed a term $\inp{\ell_k, x_{k + 1}}$, (b) added and removed a term $F_k \spr{x_{k + 1}}$, (c) used the fact that $\inp{\ell_k, x_{k + 1}} + F_k \spr{x_{k + 1}} = F_{k + 1} \spr{x_{k + 1}} - \psi_{k + 1} \spr{x_{k + 1}} + \psi_k \spr{x_{k + 1}}$ by definition of $F_k$, and (d) rearranged terms.
    
    For any $x \in \bbR^A$, we define $\iota_{\simplex \spr{\cA}} \spr{x} = 0$ if $x \in \simplex \spr{\cA}$ and $\iota_{\simplex \spr{\cA}} \spr{x} = \infty$ otherwise. Then, since both $x_{k + 1}$ and $x_k$ are in $\simplex \spr{\cA}$, it holds that
    \begin{align*}
        F_k \spr*{x_k}
        &- F_{k + 1} \spr*{x_{k + 1}} + \inp*{\ell_k, x_k} \\
        &=
        F_k \spr*{x_k} + \inp*{\ell_k, x_k} + \iota_{\simplex \spr*{\cA}} \spr*{x_k} - \spr*{F_k \spr*{x_{k + 1}} + \inp*{\ell_k, x_{k + 1}} + \iota_{\simplex \spr*{\cA}} \spr*{x_{k + 1}}} \\
        &\quad- \psi_{k + 1} \spr*{x_{k + 1}} + \psi_k \spr*{x_{k + 1}}\,.
    \end{align*}
    Since $\psi_k$ is $\eta_k^{-1}$-strongly convex, the function $F_k + \inp{\ell_k, \cdot} + \iota_{\simplex \spr{\cA}}$ is also $\eta_k^{-1}$-strongly convex. Hence, writing $\partial$ for the subdifferential, for any subgradient $g_k \in \scbr{\ell_k} + \partial \spr{F_k + \iota_{\simplex \spr{\cA}}} \spr{x_k}$, we have
    \begin{align*}
        F_k \spr*{x_k} &+ \inp*{\ell_k, x_k} + \iota_{\simplex \spr*{\cA}} \spr*{x_k} - \spr*{F_k \spr*{x_{k + 1}} + \inp*{\ell_k, x_{k + 1}} + \iota_{\simplex \spr*{\cA}} \spr*{x_{k + 1}}} \\
        &\leq
        - \inp*{g_k, x_{k + 1} - x_k} - \frac{1}{2 \eta_k} \norm*{x_k - x_{k + 1}}_1^2\,.
    \end{align*}
    Since $x_k = \argmin_{x \in \simplex \spr{\cA}} \spr{F_k + \iota_{\simplex \spr{\cA}}} \spr{x}$, we obtain that $0 \in \partial \spr{F_k + \iota_{\simplex \spr{\cA}}} \spr{x_k}$, which implies that we can choose $g_k = \ell_k$.
    Plugging this back into the previous display, we obtain that
    \begin{align*}
        F_k \spr*{x_k}
        &- F_{k + 1} \spr*{x_{k + 1}} + \inp*{\ell_k, x_k} \\
        &\leq
        \inp*{\ell_k, x_k - x_{k + 1}} - \frac{1}{2 \eta_k} \norm*{x_k - x_{k + 1}}_1^2 - \psi_{k + 1} \spr*{x_{k + 1}} + \psi_k \spr*{x_{k + 1}} \\
        &\overset{\scriptscriptstyle\mathrm{(a)}}{\leq}
        \norm*{\ell_k}_\infty \norm*{x_k - x_{k + 1}}_1 - \frac{1}{2 \eta_k} \norm*{x_k - x_{k + 1}}_1^2 - \psi_{k + 1} \spr*{x_{k + 1}} + \psi_k \spr*{x_{k + 1}} \\
        &\overset{\scriptscriptstyle\mathrm{(b)}}{\leq}
        \frac{\eta_k}{2} \norm*{\ell_k}^2_{\infty} - \psi_{k + 1} \spr*{x_{k + 1}} + \psi_k \spr*{x_{k + 1}} \\
        &\overset{\scriptscriptstyle\mathrm{(c)}}{\leq}
        \eta_k \norm*{\ell_k}^2_{\infty}\,,
    \end{align*}
    where we also used (a) Hölder's inequality, (b) Young's inequality, and (c) the fact that the potential is nondecreasing. Therefore, the regret bound becomes
    \begin{align*}
        \sumkK \inp*{\ell_k, x_k - x^\star}
        &\leq
        \psi_{K + 1} \spr*{x^\star} + \sumkK \eta_k \norm*{\ell_k}^2_{\infty}.
    \end{align*}
    By definition of the potential $\psi_{K + 1}$ and the facts that $- H \spr{x} \leq 0$ and $H \spr{x} \leq \log \spr{A}$ for all $x \in \simplex \spr{\cA}$, we obtain
    \begin{equation*}
        \psi_{K + 1} \spr*{x}
        =
        \frac{- H \spr*{x} - \min_{x \in \simplex \spr*{\cA}} \spr*{- H \spr*{x}}}{\eta_{K + 1}}
        \leq
        \frac{\max_{x \in \simplex \spr*{\cA}} H \spr*{x}}{\eta_{K + 1}}
        \leq
        \frac{\log \spr*{A}}{\eta_{K + 1}}\,.
    \end{equation*}
    Therefore, we obtain
    \begin{align*}
        \sumkK \inp*{\ell_k, x_k - x^\star}
        &\leq
        \frac{\log \spr*{A}}{\eta_{K + 1}} + \sumkK \eta_k \norm*{\ell_k}^2_{\infty}.
    \end{align*}
    Then, choosing $\eta_k = \sqrt{\frac{\log \spr{A}}{\ell_{\max}^2 k}}$ and using the fact that $\norm{\ell_k}^2_{\infty} \leq \ell_{\max}^2$, we obtain that
    \begin{equation*}
        \sumkK \inp*{\ell_k, x_k - x^\star}
        \leq
        \sqrt{\ell_{\max}^2 \log \spr*{A} \spr*{K + 1}} + \sumkK \sqrt{\frac{\ell_{\max}^2 \log \spr*{A}}{k}}
        \leq
        3 \sqrt{\ell_{\max}^2 \log \spr*{A} K}.
    \end{equation*}
    Here, the last inequality follows by factoring out $\sqrt{\ell_{\max}^2 \log \spr{A}}$ and using $\sumkK \frac{1}{\sqrt{k}} \leq 1 + \int_1^K \frac{1}{\sqrt{x}} \diff x = 2 \sqrt{K} - 1$, and $\sqrt{K + 1} \leq \sqrt{K} + 1$. This concludes the proof.
\end{proof}

\begin{Lem}[Softmax Lipschitzness] \label{lem:softmax-lipschitz}
    For any finite action set $\cA$ and any $u, v \in \bbR^{\cA}$,
    \begin{equation*}
        \norm*{\softmax \spr*{u} - \softmax \spr*{v}}_1
        \leq
        \norm*{u - v}_\infty.
    \end{equation*}
\end{Lem}

\begin{proof}[\pfref{lem:softmax-lipschitz}]
    Since the gradient of the log-sum-exp function is the softmax function, the desired inequality is exactly the $1$-smoothness of log-sum-exp with respect to the $\ell_\infty$ norm, whose dual norm is $\ell_1$. By \citet[Theorem 6.26]{orabona2023modern}, this smoothness follows from the $1$-strong convexity of the convex conjugate of log-sum-exp with respect to the dual $\ell_1$ norm. This conjugate is the negative entropy on the simplex, and its $1$-strong convexity in the $\ell_1$ norm follows from \citet[Lemma 6.33]{orabona2023modern}.
\end{proof}

\subsection{Performance Difference Lemma}

We use the following version of the performance difference lemma~\citep{howard1960dynamic,kakade2002approximately}: it expresses the performance gap between two policies as a sum of advantages of one policy under the occupancy of the other.\loose

\begin{restatable}[Performance difference lemma]{Lem}{pdl} \label{lem:performance-difference}
    Let $\pi$ and $\pi'$ be arbitrary policies. Then,
    \begin{equation} \label{eq:online-pdl}
        J^{\pi'} - J^\pi
        =
        \sumhH \sum_{x \in \cX} d_h^\pi \spr*{x} \sum_{a \in \cA} \spr*{\pi_h' \spr*{a \given x} - \pi_h \spr*{a \given x}} Q_h^{\pi'} \spr*{x, a}.
    \end{equation}
\end{restatable}

\begin{proof}[\pfref{lem:performance-difference}]
    By considering the Bellman equations for the policy $\pi'$, for any state-action pair $\spr{x, a}$ and any $h \in \sbr{H}$ we have that
    \begin{equation*}
        Q^{\pi'}_h \spr*{x, a}
        =
        r_h \spr*{x, a} + \sum_{x' \in \cX} P_h \spr*{x' \given x, a} V^{\pi'}_{h + 1} \spr*{x'}.
    \end{equation*}
    Then, taking the expectation with respect to the state-action occupancy measure of the policy $\pi$, we have
    \begin{align*}
        \sum_{x, a} d^\pi_h \spr*{x, a} Q^{\pi'}_h \spr*{x, a}
        &=
        \sum_{x, a} d^\pi_h \spr*{x, a} r_h \spr*{x, a} + \sum_{x, a} d^\pi_h \spr*{x, a} \sum_{x'} P_h \spr*{x' \given x, a} V^{\pi'}_{h + 1} \spr*{x'}.
    \end{align*}
    Then, by swapping the sums and using the flow conditions satisfied by the occupancy measure of $\pi$, \ie, $d^\pi_{h + 1} \spr{x} = \sum_{x', a'} P_h \spr{x \given x', a'} d^\pi_h \spr{x', a'}$, observe that
    \begin{align*}
        \sum_{x, a} d^\pi_h \spr*{x, a} \sum_{x'} P_h \spr*{x' \given x, a} V^{\pi'}_{h + 1} \spr*{x'}
        &=
        \sum_{x'} \spr*{\sum_{x, a} P_h \spr*{x' \given x, a} d^\pi_h \spr*{x, a}} V^{\pi'}_{h + 1} \spr*{x'} \\
        &=
        \sum_{x'} d^\pi_{h + 1} \spr*{x'} V^{\pi'}_{h + 1} \spr*{x'}.
    \end{align*}
    This implies that
    \begin{align*}
        \inp*{d^\pi_h, Q^{\pi'}_h} - \inp*{d^\pi_h, V^{\pi'}_h}
        &=
        \inp*{d^\pi_h, r_h} + \inp*{d^\pi_{h + 1}, V^{\pi'}_{h + 1}} - \inp*{d^\pi_h, V^{\pi'}_h}.
    \end{align*}
    Then, summing over $h \in \sbr{H}$ and recalling that $V^{\pi'}_{H + 1} = 0$ and that $\initial = d^\pi_1$, we obtain that
    \begin{align*}
        \sumhH \spr*{\inp*{d^\pi_h, Q^{\pi'}_h} - \inp*{d^\pi_h, V^{\pi'}_h}}
        =
        \sumhH \inp*{d^\pi_h, r_h} - \inp*{\initial, V^{\pi'}_1}
        =
        J^\pi - J^{\pi'}.
    \end{align*}
    Finally, rearranging, we have
    \begin{align*}
        J^\pi - J^{\pi'}
        &=
        \sumhH \spr*{\inp*{d^\pi_h, Q^{\pi'}_h} - \inp*{d^\pi_h, V^{\pi'}_h}} \\
        &=
        \sumhH \sum_x d^\pi_h \spr*{x} \spr*{\sum_a \pi_h \spr*{a \given x} Q^{\pi'}_h \spr*{x, a} - V^{\pi'}_h \spr*{x}} \\
        &=
        \sumhH \sum_x d^\pi_h \spr*{x} \sum_a \spr*{\pi_h \spr*{a \given x} - \pi'_h \spr*{a \given x}} Q^{\pi'}_h \spr*{x, a},
    \end{align*}
    where we used $d_h^\pi \spr{x, a} = d^\pi_h \spr{x} \pi_h \spr{a \given x}$ and the definition of $V_h^{\pi'}$. Multiplying by $-1$ concludes the proof.\loose
\end{proof}


\newpage
\part{Proofs of Main Results}

In this section, we include the proofs omitted from the main text. \Cref{sec:technical_tools} provides technical results in concentration and optimization that will be used to prove the sample complexity guarantees of \Cref{alg:interactive_finite-H} in \Cref{sec:proofs_ub}, and to prove the lower bound in \Cref{sec:proofs_lb}.

\section{Proofs from \texorpdfstring{\cref{sec:ub}}{Section}}
\label{sec:proofs_ub}
As discussed in \Cref{sec:algo}, \ISPIL{} reduces the suboptimality gap to a sequence of saddle-point problems involving the objective $\cL_h^d$ and learns the output policy in a layer-wise manner. We recall here only the additional notation needed for the proofs. For any (partial) policy $\pi = \spr{\pi_1, \ldots, \pi_{h-1}}$, we denote $\cL_h^\pi = \cL_h^{d^\pi}$ for convenience. When learning layer $h$, the previously computed decision rules $\piout_1, \ldots, \piout_{h - 1}$ allow us to sample from $d_h^{\piout}$. Thus, for each $i \in \sbr{\tauE}$, \ISPIL{} independently samples
\begin{equation*}
    \Xih \sim d^{\piout}_h
    \quad\text{and}\quad
    \AEih \sim \experth \spr*{\cdot \given \Xih}.
\end{equation*}
For $h = 1$, we can sample from $d_1^{\piout} = \initial$. For subsequent stages $h$, we can sample from $d_h^{\piout}$ by rolling in with the previously computed policies $\piout_1, \ldots, \piout_{h - 1}$. Given these samples, for any decision rule $p$ and function $f \in \bbR^{\cX \times \cA}$, we define the empirical estimate of $\cL_h^{\piout}$ as
\begin{equation*}
    \hcL_h \spr*{p, f}
    =
    \frac{1}{\tauE} \sum^{\tauE}_{i = 1} \spr*{f \spr*{\Xih, \AEih} - f \spr*{\Xih, p}}.
\end{equation*}
For the version of \ISPIL analyzed in \Cref{thm:ovi-main-covering}, the empirical game is solved by having the $Q$-player best respond to $\hcL_h \spr{\pi_h^k, \cdot}$, while the $\pi$-player uses online mirror ascent \citep{Beck:2003} with negative entropy as a regularizer. Fixing a stage $h$, for each $k$,
\begin{equation}\label{eq:updates-pi-first}
    Q^k_h \in \argmax_{Q_h \in \cQ_h} \hcL_h \spr*{\pi^k_h, Q_h},
    \qquad\text{and}\qquad
    \pi^{k + 1}_h \propto \pi_h^k \odot \exp \spr*{\eta Q_h^k},
\end{equation}
where $\eta > 0$ is the learning-rate parameter and $\odot$ denotes the element-wise product.

Having recalled the algorithm structure, we can now move towards proving our main result: interaction enables computationally and statistically efficient value-based imitation learning algorithms such as our \ISPIL.

Before proving \Cref{thm:ovi-main-covering}, we derive \Cref{lem:pifirst}, which motivates the $\pi$-player update that will be used in the proof of \Cref{thm:ovi-main-covering}. It shows that if the $Q$-player best responds to the $\pi$-player, then the suboptimality gap can be bounded by the regret of the $\pi$-player plus an estimation error term that can be controlled with a large enough number of expert queries $\tauE$.
\begin{restatable}{Lem}{lempifirst} \label{lem:pifirst}
    Let \cref{asp:q-expert-realizability} hold. Let policies $\pi_h^1, \ldots, \pi_h^K$ and functions $Q_h^1, \ldots, Q_h^K \in \cQ$ be computed as in \cref{eq:updates-pi-first}. We define the $\pi$-regret at stage $h$ and state $x$ as $\regret^\pi_h \spr{x} = \sumkK \inp{Q^k_h \spr{x, \cdot}, \experth \spr{\cdot \given x} - \pi^k_h \spr{\cdot \given x}}$. Furthermore, for any policy $\pi$, we define the estimation error as $\Delta \spr{\pi} = \max_{h \in \sbr{H}} \sup_{Q_h \in \cQ_h} \abs{\hcL_h \spr{\pi_h, Q_h} - \cL_h^{\piout} \spr{\pi_h, Q_h}}$, where $\piouth = \frac1K \sumkK \pi^k_h$. Then,\loose
    \begin{equation*}
        J^{\expert} - J^{\piout}
        \leq
        \frac1K \sumhH \sum_{x \in \cX} d^{\piout}_h \spr*{x} \regret^\pi_h \spr*{x} + \frac{2 H}{K} \sumkK \Delta \spr*{\pi^k}.
    \end{equation*}
\end{restatable}

\begin{proof}[\pfref{lem:pifirst}]
    By the performance difference lemma (\cref{lem:performance-difference}) with $\pi = \piout$ and $\pi' = \expert$, and by the definition of $\cL_h^{\piout}$, we have $J^{\expert} - J^{\piout} = \sumhH \cL_h^{\piout} \spr{\piouth, \qexperth}$. Since $\cL_h^{\piout} \spr{p, Q}$ is affine in the decision rule $p$ and $\piouth = \frac1K \sumkK \pi^k_h$, this gives
    \begin{equation*}
        J^{\expert} - J^{\piout}
        =
        \sumhH \cL_h^{\piout} \spr*{\piouth, \qexperth}
        =
        \frac1K \sumhH \sumkK \cL_h^{\piout} \spr*{\pi^k_h, \qexperth}.
    \end{equation*}
    Then, we can upper bound the suboptimality as follows
    \begin{align*}
        K \spr*{J^{\expert} - J^{\piout}}
        &=
        \sum_{h, k} \cL_h^{\piout} \spr*{\pi^k_h, \qexperth} \overset{\text{(a)}}{\leq} \sum_{h, k} \hcL_h \spr*{\pi^k_h, \qexperth} + H \sum_k \Delta \spr*{\pi^k} \\
        &\overset{\text{(b)}}{\leq}
        \sum_{h, k} \hcL_h \spr*{\pi^k_h, Q^k_h} + H \sum_k \Delta \spr*{\pi^k}
        \overset{\text{(a)}}{\leq}
        \sum_{h, k} \cL_h^{\piout} \spr*{\pi^k_h, Q^k_h} + 2 H \sum_k \Delta \spr*{\pi^k},
    \end{align*}
    where step (a) uses the definition of $\Delta \spr{\pi^k}$ together with $\qexperth \in \cQ_h$ in the first application and $Q^k_h \in \cQ_h$ in the second, and step (b) uses that $Q_h^k$ is a best response to $\hcL_h \spr{\pi_h^k, \cdot}$ over $\cQ_h$. Finally, by expanding the definition of $\cL_h^{\piout}$ and exchanging the finite sums over $k$, $x$, and $a$,
    \begin{equation*}
        \sumkK \cL_h^{\piout} \spr{\pi^k_h, Q^k_h}
        =
        \sum_{x \in \cX} d^{\piout}_h \spr{x} \sumkK \inp{Q^k_h \spr{x, \cdot}, \experth \spr{\cdot \given x} - \pi^k_h \spr{\cdot \given x}}
        =
        \sum_{x \in \cX} d^{\piout}_h \spr{x} \regret^\pi_h \spr{x}.
    \end{equation*}
    This concludes the proof.
\end{proof}

Next, we use \Cref{lem:pifirst} to derive a sample complexity guarantee for \ISPIL. At a high level, the proof intuition is to notice that \Cref{lem:pifirst} reduces the control of the suboptimality gap to two terms: the average regret of the policy player and the average estimation error of the empirical saddle-point objective. Thus, it suffices to ensure that $\regret^\pi_h \spr{x}$ grows sublinearly in $K$ for each stage $h$ and state $x \in \cX$, which is achieved by the no-regret update in \Cref{alg:interactive_finite-H}. The estimation error term is controlled by a uniform concentration argument over the $Q$-class and over the policies that can be played by our algorithm. In particular, each iterate belongs to the policy class $\Pi_{\cQ} \ldef \scbr{\pi : \exists m \leq K, \forall h, \exists Q^1_h, \ldots, Q^m_h \in \cQ_h, ~~~\pi_h \spr{a \given x} = \softmax \spr{\sum_k Q^k_h \spr{x, \cdot}}_a}$ for some $K = \mathrm{poly} \spr{H, \log A, \QMAX, \varepsilon^{-1}}$, and concentration uniformly over $\Pi_{\cQ} \times \cQ$ controls $\frac1K \sumkK \Delta \spr{\pi^k}$ with a sufficiently large number of expert queries $\tauE$. We make the proof formal in the next subsection.\loose

\subsection{Proof of \texorpdfstring{\Cref{thm:pi_first_main}}{Theorem}
(Sample Complexity Guarantee for \ISPIL)}
\label{sec:proofs-ub-main}

In this section, we present the proof of our main upper bound.

\pifirstmain*

\begin{proof}[\pfref{thm:pi_first_main}]
    The proof is a direct consequence of \Cref{thm:ovi-main-covering} below, where the covering number argument simplifies to a union bound over the function class $\cQ$.
\end{proof}

Below, we present our general result that extends the sample complexity analysis to more general function classes. In particular, \Cref{thm:ovi-main-covering} shows that the sample complexity of \ISPIL is controlled by the largest $\ell_\infty$-covering number of $\cQ_h$ over $h \in \sbr{H}$, and that when $\cQ$ is convex, the sample complexity improves to $\tcO \spr{H^3 \QMAX^2 \log \spr{\maxcovering{r_\varepsilon} / \delta} \varepsilon^{-2}}$ for some $r_\varepsilon = \cO \spr{\varepsilon^2}$.
When the classes are infinite, the best responses in the $Q$-player update need not exist automatically. Throughout this result, we consider classes for which the maximizers in \Cref{alg:interactive_finite-H} are well defined. This holds, for instance, when each $\cQ_h$ is compact and the empirical objective optimized by the $Q$-player is continuous on $\cQ_h$. 
\begin{theorem}[Sample complexity of \ISPIL] \label{thm:ovi-main-covering}
    Let \Cref{asp:q-expert-realizability} hold. For every $r > 0$, let $\maxcovering{r}$ be the largest $\ell_\infty$-covering number (see \Cref{def:covering}) of $\cQ_h$ over $h \in \sbr{H}$, and choose $r_{\varepsilon} = \cO \spr{\varepsilon^2 / \spr{H^2 \QMAX \log A}}$. Then, for any $\varepsilon, \delta \in (0, 1)$, \ISPIL (\Cref{alg:interactive_finite-H}) with parameters
    \begin{equation*}
        \eta = \sqrt{\frac{\log \spr*{A}}{K \QMAX^2}}, K = \cO \spr*{\frac{H^2 \QMAX^2 \log \spr*{A}}{\varepsilon^2}}, \quad\text{and}\quad \tauE = \tcO \spr*{\frac{H^4 \QMAX^4 \log \spr*{A} \log \spr*{\maxcovering{r_\varepsilon} / \delta}}{\varepsilon^4}},
    \end{equation*}
    outputs a policy $\piout$ such that, with probability at least $1 - \delta$, $J^\expert - J^{\piout} \leq \varepsilon$ after
    \begin{equation*}
        \tcO \spr*{ \frac{H^5 \QMAX^4 \log \spr*{A} \log \spr*{\maxcovering{r_\varepsilon} / \delta}}{\varepsilon^4}} \qquad \text{expert queries}.
    \end{equation*}
    Moreover, when $\cQ$ is convex, the same guarantee holds with $\tauE = \tcO \spr{H^2 \QMAX^2 \log \spr{\maxcovering{r_\varepsilon} / \delta} \varepsilon^{-2}}$, and the sample complexity improves to $\tcO \spr{H^3 \QMAX^2 \log \spr{\maxcovering{r_\varepsilon} / \delta} \varepsilon^{-2}}$.
\end{theorem}

\begin{proof}[\pfref{thm:ovi-main-covering}]
    To prove the sample complexity guarantees of \Cref{alg:interactive_finite-H}, we start from the decomposition given in \Cref{lem:pifirst}.
    \begin{equation*}
        J^{\expert} - J^{\piout}
        \leq
        \frac1K \sumhH \sum_{x \in \cX} d^{\piout}_h \spr*{x} \regret^\pi_h \spr*{x} + \frac{2 H}{K} \sumkK \Delta \spr*{\pi^k}.
    \end{equation*}
    We need to control the regret of the policy player $\regret^\pi_h$, defined for any state $x$ as
    \begin{equation*}
        \regret^\pi_h \spr*{x}
        =
        \sumkK \inp*{Q^k_h \spr*{x, \cdot}, \experth \spr*{\cdot \given x} - \pi^k_h \spr*{\cdot \given x}},
    \end{equation*}
    and we need to control the average estimation error $\frac1K \sumkK \Delta \spr{\pi^k}$, where, for any policy $\pi$, we define
    \begin{equation*}
        \Delta \spr*{\pi}
        =
        \max_{h \in \sbr*{H}}\sup_{Q_h \in \cQ_h} \abs*{\hcL_h \spr*{\pi_h, Q_h} - \cL_h^{\piout} \spr*{\pi_h, Q_h}}.
    \end{equation*}

    \para{Controlling the regret $\regret^\pi_h$}
    In order to control $\regret^\pi_h$, we use online mirror ascent with negative-entropy regularizer, which corresponds to the following update rule for each stage $h \in \sbr{H}$, iteration $k \in \sbr{K}$, state $x \in \cX$ and action $a \in \cA$,
    \begin{equation*}
        \pi^{k + 1}_h \spr*{a \given x}
        \propto
        \pi^k_h \spr*{a \given x} e^{\eta Q^k_h \spr*{x, a}}.
    \end{equation*}
    Fix $h$ and $x$. Applying \Cref{lemma:mirror_orabona} to the simplex with losses $\ell_k = - Q^k_h \spr{x, \cdot}$ and comparator $u = \experth \spr{\cdot \given x}$ gives
    \begin{equation*}
        \regret^\pi_h \spr*{x}
        \leq
        \frac{\log \abs*{\cA}}{\eta} + \eta K \QMAX^2\,,
    \end{equation*}
    where we used that the initial decision rule is uniform, so the negative-entropy Bregman divergence to any comparator is at most $\log \abs{\cA}$, and that each $Q_h \in \cQ_h$ satisfies $\norm{Q_h}_\infty \leq \QMAX$, as assumed in \Cref{asp:q-expert-realizability}. Plugging the value of $\eta$ from the statement of \Cref{thm:ovi-main-covering}, $\eta = \spr{\log \spr{A} / K}^{1/2} \QMAX^{-1}$, we get that $\regret^\pi_h \spr{x} \leq 2 \sqrt{K \QMAX^2 \log \abs{\cA}}$. Since $d_h^{\piout}$ is a probability distribution for every $h$, setting $K = 36 H^2 \QMAX^2 \log \spr{A} \varepsilon^{-2}$ implies
    \begin{equation*}
        \frac1K \sumhH \sum_{x \in \cX} d^{\piout}_h \spr*{x} \regret^\pi_h \spr*{x}
        \leq
        2 \sqrt{\frac{H^2 \QMAX^2 \log \abs*{\cA}}{K}}
        =
        \frac{\varepsilon}{3}.
    \end{equation*}

    \para{Controlling the estimation error}
    It remains to control the average estimation error $\frac1K \sumkK \Delta \spr{\pi^k}$. To this end, recall that, in \Cref{alg:interactive_finite-H}, we build the empirical estimator $\hcL_h$ by sampling $\tauE$ states from $d^{\piout}_h$. This sampling is possible because $d^{\piout}_h$ depends only on $\piout_1$, $\ldots$, $\piout_{h - 1}$, which have already been computed. We then query the expert for an action at each sampled state. For each $i \in \sbr{\tauE}$, we sample
    \begin{equation*}
        \Xih \sim d^{\piout}_h
        \quad\text{and}\quad
        \AEih \sim \experth \spr*{\cdot \given \Xih}.
    \end{equation*}
    For every $k \in \sbr{K}$, the policy $\pi^k$ belongs to the class $\Pi_{\cQ} \subset \Delta \spr{\cA}^\cX$ defined below.
    \begin{equation*}
        \Pi_\cQ
        \ldef
        \scbr*{\pi: \exists m \in \scbr*{0, \ldots, K}, \forall h, \exists Q_h^1, \ldots, Q_h^m \in \cQ_h,~~~\pi_h \spr*{\cdot \given x} = \softmax \spr[\Big]{\eta\ {\textstyle\sum_j} Q_h^j \spr*{x, \cdot}}}.
    \end{equation*}
    For each $h$, write $\Pi_{\cQ,h} \ldef \scbr{\pi_h : \pi \in \Pi_{\cQ}}$ for the corresponding class of decision rules. We upper bound the average estimation error by
    \begin{equation*}
        \frac1K \sumkK \Delta \spr*{\pi^k}
        \leq
        \sup_{\pi \in \Pi_{\cQ}} \Delta \spr*{\pi}
        \leq
        \max_{h \in \sbr*{H}} \sup_{p \in \Pi_{\cQ, h}} \sup_{Q_h \in \cQ_h} \abs*{\hcL_h \spr*{p, Q_h} - \cL_h^{\piout} \spr*{p, Q_h}}.
    \end{equation*}
    To control the upper bound, we verify that we can apply the concentration argument from \Cref{lem:moulin-fh-concentration}. Fix $h$, $p \in \Pi_{\cQ,h}$, and $Q_h \in \cQ_h$, and define
    \begin{equation*}
        Y_i \spr*{Q_h, p}
        \ldef
        Q_h \spr*{\Xih, \AEih} - Q_h \spr*{\Xih, p} - \cL_h^{\piout} \spr*{p, Q_h}.
    \end{equation*}
    Conditionally on the history before the samples for stage $h$ are drawn, $d_h^{\piout}$ is fixed and $\spr{Y_i \spr{Q_h, p}}_{i = 1}^{\tauE}$ is a martingale difference sequence. Moreover, since $Q_h \in \sbr{0, \QMAX}^{\cX \times \cA}$, both $Q_h \spr{\Xih, \AEih}$ and $Q_h \spr{\Xih, p}$ belong to $\sbr{0, \QMAX}$, while $\cL_h^{\piout} \spr{p, Q_h}$ belongs to $\sbr{-\QMAX, \QMAX}$. Hence
    \begin{equation*}
        \abs*{Y_i \spr*{Q_h, p}}
        \leq
        2 \QMAX \qquad\text{almost surely.}
    \end{equation*}
    Applying \Cref{lem:moulin-fh-concentration} with $d_h = d_h^{\piout}$ and $\Pi_h = \Pi_{\cQ,h}$ therefore gives, for any radius $r > 0$, with probability at least $1 - \delta$,
    \begin{equation} \label{eq:estimation-error-bound-proof-main}
        \sup_{\pi \in \Pi_{\cQ}} \Delta \spr*{\pi}
        \leq
        4 r + \QMAX \sqrt{\frac{8 \log \spr*{\frac{2}{\delta} \sum_{h = 1}^H \cN_r \spr*{\cQ_h \times \Pi_{\cQ,h}, \rho}}}{\tauE}}.
    \end{equation}
    Taking $r = \varepsilon / \spr{24 H}$ and
    \begin{equation*}
        \tauE
        \geq
        \frac{288 H^2 \QMAX^2}{\varepsilon^2} \log \spr*{\frac{2}{\delta} \sum_{h = 1}^H \cN_{\varepsilon / \spr*{24 H}} \spr*{\cQ_h \times \Pi_{\cQ,h}, \rho}}
    \end{equation*}
    yields $H \sup_{\pi \in \Pi_{\cQ}} \Delta \spr{\pi} \leq \frac{\varepsilon}{3}$ with probability at least $1 - \delta$.

    \para{Putting everything together}
    Plugging the two previous bounds in the decomposition shown in \Cref{lem:pifirst}, and using $r = \varepsilon / \spr{24 H}$ in \cref{eq:estimation-error-bound-proof-main}, with probability at least $1 - \delta$, we have
    \begin{align*}
        J^{\expert} - J^{\piout}
        &\leq
        \frac1K \sumhH \sum_{x \in \cX} d^{\piout}_h \spr*{x} \regret^\pi_h \spr*{x} + \frac{2 H}{K} \sumkK \Delta \spr*{\pi^k} \\
        &\leq
        2 \sqrt{\frac{H^2 \QMAX^2 \log \abs*{\cA}}{K}}
        +
        \frac{\varepsilon}{3}
        +
        2 H \QMAX
        \sqrt{
            \frac{
                8 \log \spr*{
                    \frac{2}{\delta} \sum_{h = 1}^H
                    \cN_{\varepsilon / \spr*{24 H}}
                    \spr*{\cQ_h \times \Pi_{\cQ,h}, \rho}
                }
            }{\tauE}
        }
        \\
        &\leq
        \frac{\varepsilon}{3} + \frac{\varepsilon}{3} + \frac{\varepsilon}{3},
    \end{align*}
    where the last line uses $K = 36 H^2 \QMAX^2 \log \spr{A} \varepsilon^{-2}$ and
    \begin{equation*}
        \tauE
        \geq
        \frac{288 H^2 \QMAX^2}{\varepsilon^2}
        \log \spr*{
            \frac{
                2 \sum_{h = 1}^H
                \cN_{\varepsilon / \spr*{24 H}}
                \spr*{\cQ_h \times \Pi_{\cQ,h}, \rho}
            }{\delta}
        }.
    \end{equation*}
    Next, we consider separately the convex and nonconvex cases for the classes $\cQ_h$, and show that the sample complexity improves in the convex case.

    \para{Convex case}
    Assume first that each $\cQ_h$ is convex. Fix $h \in \sbr{H}$. By \Cref{lem:moulin-convex-product-cover}, more precisely by \cref{eq:moulin-convex-product-cover-simple}, for every radius $r > 0$,
    \begin{equation*}
        \cN_r \spr*{\cQ_h \times \Pi_{\cQ,h}, \rho}
        \leq
        \spr*{K + 1} \cN_{\frac{r}{2 \max \scbr*{1, \QMAX \eta K}}} \spr*{\cQ_h, \norm*{\cdot}_\infty}^{2}.
    \end{equation*}
    If $\varepsilon \geq \QMAX$, reward normalization gives $J^\expert - J^{\piout} \leq \QMAX \leq \varepsilon$, so the performance claim is immediate. We may therefore assume that $\varepsilon < \QMAX$. For the parameter choice in the theorem,
    \begin{equation*}
        \QMAX \eta K
        =
        \sqrt{K \log \abs*{\cA}}
        =
        \frac{6 H \QMAX \log \abs*{\cA}}{\varepsilon}.
    \end{equation*}
    Since $H \geq 1$ and $A \geq 2$, we then have $\QMAX \eta K > 6 H \log 2 > 1$. Hence $\max \scbr{1, \QMAX \eta K} = \QMAX \eta K$ in the remaining covering estimates. We now plug this stagewise bound into the estimation guarantee. With the radius $r = \varepsilon / \spr{24 H}$ used in \cref{eq:estimation-error-bound-proof-main}, the product bound above implies
    \begin{align*}
        \sum_{h = 1}^H
        \cN_{\varepsilon / \spr*{24 H}}
        \spr*{\cQ_h \times \Pi_{\cQ,h}, \rho}
        &\leq
        H \spr*{K + 1}
        \maxcovering{\frac{\varepsilon}{48 H \QMAX \eta K}}^{2}.
    \end{align*}
    Combining the bound above with the sufficient condition on $\tauE$ in the preceding paragraph, it is enough to take
    \begin{equation*}
        \tauE
        \geq
        \frac{288 H^2 \QMAX^2}{\varepsilon^2}
        \log \spr*{
            \frac{2 H \spr*{K + 1}}{\delta}
            \maxcovering{\frac{\varepsilon}{48 H \QMAX \eta K}}^{2}
        }.
    \end{equation*}
    Substituting the values of $K$ and $\QMAX \eta K$, we define
    \begin{equation*}
        r_\varepsilon
        \ldef
        \frac{\varepsilon}{48 H \QMAX \eta K}
        =
        \frac{\varepsilon^2}{288 H^2 \QMAX \log \abs*{\cA}}.
    \end{equation*}
    The following order of expert queries per stage then suffices.
    \begin{align*}
        \tauE
        &=
        \cO \spr*{
            \frac{H^2 \QMAX^2}{\varepsilon^2}
            \log \spr*{
                \frac{H \spr*{1 + 36 H^2 \QMAX^2 \log \abs*{\cA} / \varepsilon^2}}{\delta}
                \maxcovering{r_\varepsilon}^{2}
            }
        }
        \\
        &=
        \tcO \spr*{
            \frac{H^2 \QMAX^2}{\varepsilon^2}
            \log \spr*{
                \frac{1}{\delta}
                \maxcovering{r_\varepsilon}
            }
        }.
    \end{align*}
    Since the algorithm queries the expert at $H$ stages, the total number of expert queries is
    \begin{equation*}
        H \tauE
        =
        \tcO \spr*{
            \frac{H^3 \QMAX^2}{\varepsilon^2}
            \log \spr*{
                \frac{1}{\delta}
                \maxcovering{r_\varepsilon}
            }
        }.
    \end{equation*}

    \para{Nonconvex classes}
    We now turn to the general case. Fix $h \in \sbr{H}$. By \Cref{lem:moulin-nonconvex-product-cover}, more precisely by \cref{eq:moulin-nonconvex-product-cover-simple}, and since the case split above gives $\QMAX \eta K > 1$, for every radius $r > 0$,
    \begin{equation*}
        \cN_r \spr*{\cQ_h \times \Pi_{\cQ,h}, \rho}
        \leq
        \spr*{K + 1}
        \cN_{\frac{r}{2 \QMAX \eta K}}
        \spr*{\cQ_h, \norm*{\cdot}_\infty}^{K + 1}.
    \end{equation*}
    With the radius $r = \varepsilon / \spr{24 H}$ used in \cref{eq:estimation-error-bound-proof-main}, the product bound above implies
    \begin{align*}
        \sum_{h = 1}^H
        \cN_{\varepsilon / \spr*{24 H}}
        \spr*{\cQ_h \times \Pi_{\cQ,h}, \rho}
        &\leq
        H \spr*{K + 1}
        \maxcovering{r_\varepsilon}^{K + 1}.
    \end{align*}
    Combining the bound above with the sufficient condition on $\tauE$ in the preceding paragraph, it is enough to take
    \begin{equation*}
        \tauE
        \geq
        \frac{288 H^2 \QMAX^2}{\varepsilon^2}
        \log \spr*{
            \frac{2 H \spr*{K + 1}}{\delta}
            \maxcovering{r_\varepsilon}^{K + 1}
        }.
    \end{equation*}
    Substituting the values of $K$ and $\QMAX \eta K$, the following order of expert queries per stage suffices:
    \begin{align*}
        \tauE
        &=
        \cO \spr*{
            \frac{H^2 \QMAX^2}{\varepsilon^2}
            \log \spr*{
                \frac{H \spr*{1 + 36 H^2 \QMAX^2 \log \abs*{\cA} / \varepsilon^2}}{\delta}
                \maxcovering{r_\varepsilon}^{1 + 36 H^2 \QMAX^2 \log \abs*{\cA} / \varepsilon^2}
            }
        }
        \\
        &=
        \tcO \spr*{
            \frac{H^4 \QMAX^4 \log \abs*{\cA}}{\varepsilon^4}
            \log \spr*{
                \frac{1}{\delta}
                \maxcovering{r_\varepsilon}
            }
        }.
    \end{align*}
    Since the algorithm queries the expert at $H$ stages, the total number of expert queries is
    \begin{equation*}
        H \tauE
        =
        \tcO \spr*{
            \frac{H^5 \QMAX^4 \log \abs*{\cA}}{\varepsilon^4}
            \log \spr*{
                \frac{1}{\delta}
                \maxcovering{r_\varepsilon}
            }
        }.
    \end{equation*}
    This concludes the proof.
\end{proof}


\subsection{Proof of \texorpdfstring{\cref{thm:reward-realizability-lb}}{Theorem} (Reward Realizability Is Insufficient)}
\label{sec:reward_realizability_lb}
So far, we showed that, under $\qexpert$-realizability, interactive IL is possible but offline IL is not without additional coverage. It is natural to ask whether one can weaken the representation requirement even further. A tempting candidate is reward realizability: when $H = 1$, $\qexpert$-realizability reduces to reward realizability, and \citet{joshi2025learning} show that reward realizability is enough for statistical learnability in contextual bandits, although their algorithm is computationally inefficient.\loose

The theorem below shows that, in general MDPs, reward realizability alone cannot yield a representation-dependent guarantee even with interaction: the reward class is a singleton, but the learner still needs coverage of the rare decision states. This motivates the stronger assumption used in the main text, $\qexpert$-realizability, which stores additional information about the transitions that a reward function alone does not provide.

\rewardrealizabilitylb*

\para{Comparison with $\qexpert$-realizability}
This lower bound is specific to reward realizability and should not be read as a contradiction to the guarantees under \Cref{asp:q-expert-realizability}. In the construction below, where each MDP is characterized by a bit vector $b \in \scbr{0, 1}^X$, the common reward only identifies $x^i_+$ as rewarding and $x^i_-$ as non-rewarding but it does not reveal which action from $x^i_{\start}$ reaches $x^i_+$. By contrast, for each instance $\cM_b$, the expert value function satisfies $Q_{\cM_b, 1}^{\pi_\experttag^b} \spr{x^i_{\start}, a} = \mathbbm{1}\scbr{a = b_i}$. Thus $\qexpert$ stores the information about which action to take at the starting state itself. Equivalently, a value class that $\qexpert$-realizes all MDPs in this family must contain one distinct function for each $b \in \scbr{0, 1}^X$, so its complexity already scales with $X$. Reward realizability avoids this complexity entirely by using a singleton class, but then the learner receives no representation of the control information needed at unobserved states.

\begin{proof}[\pfref{thm:reward-realizability-lb}]
    Set $q = 4 \varepsilon / X$, which satisfies $q \leq 1 / X$ because $\varepsilon \leq 1 / 4$. We first define the hard family $\cH = \cH_{X, q}$. Let the action space be $\cA = \scbr{0, 1}$. For each $i \in \scbr{0, \ldots, X}$, the state space contains a start state $x^i_{\start}$ and two terminal states $x^i_+$ and $x^i_-$. The initial distribution is
    \begin{equation*}
        \initial \spr*{x^0_{\start}} = 1 - X q,
        \qquad\text{and}\qquad
        \initial \spr*{x^i_{\start}} = q \quad \text{for } i \in \sbr*{X}.
    \end{equation*}
    For $b = \spr{b_1, \ldots, b_X} \in \scbr{0, 1}^X$, set $b_0 = 0$. In the MDP $\cM_b$, the first transition is deterministic and given by
    \begin{equation*}
        P^b_1 \spr*{x^i_+ \given x^i_{\start}, a} = \mathbbm{1} \scbr*{a = b_i},
        \qquad\text{and}\qquad
        P^b_1 \spr*{x^i_- \given x^i_{\start}, a} = \mathbbm{1} \scbr*{a \neq b_i}.
    \end{equation*}
    All stage-two behavior is fixed independently of $b$. For concreteness, let the expert's terminal action be $0$, and take any fixed terminal transitions. The reward is common to all instances. For any index $i$ and action $a$,
    \begin{equation*}
        r^\star_1 \spr{x^i_{\start}, a} = 0,
        \qquad
        r^\star_2 \spr{x^i_+, a} = 1,
        \qquad
        r^\star_2 \spr{x^i_-, a} = 0.
    \end{equation*}
    Thus the singleton class $\cR = \scbr{r^\star}$ realizes every instance in the family. Let the expert in $\cM_b$ be the deterministic policy $\pi_{\experttag, 1}^b \spr{x^i_{\start}} = b_i$, with the fixed terminal behavior described above. The expert reaches $x^i_+$ from every $x^i_{\start}$, so it is optimal and $J_{\cM_b}^{\pi_\experttag^b} = 1$. The family is illustrated in \Cref{fig:reward-realizability-hard-family}.

    \begin{figure}[t]
        \centering
        \resizebox{\textwidth}{!}{%
        \begin{tikzpicture}[
            x=0.95cm, y=0.95cm,
            state/.style={circle, draw=black!80, thick, minimum size=1.0cm, inner sep=1pt, align=center, font=\footnotesize, fill=gray!5},
            plus/.style={state, draw=green!50!black, fill=green!8},
            minus/.style={state, draw=red!70!black, fill=red!7},
            root/.style={state, draw=black!80, fill=blue!4},
            arr/.style={->, >={Stealth[round]}, line width=0.95pt, black!80, shorten >=1pt, shorten <=1pt},
            lbl/.style={font=\scriptsize, align=center, fill=white, inner sep=1.5pt},
            prob/.style={font=\footnotesize, align=center}
        ]
            \begin{scope}
                \node[root] (s0) at (0,0) {$x^0_{\start}$};
                \node[plus] (g0) at (2.65,0.9) {$x^0_+$};
                \node[minus] (z0) at (2.65,-0.9) {$x^0_-$};
                \node[prob] at (1.33,1.75) {$\initial \spr*{x^0_{\start}} = 1 - Xq$};
                \draw[arr] (s0) -- node[lbl, above, pos=0.32, yshift=8pt] {$a = b_0 = 0$} (g0);
                \draw[arr] (s0) -- node[lbl, below, pos=0.32, yshift=-5pt] {$a \neq b_0$} (z0);
            \end{scope}
            \begin{scope}[xshift=5.35cm]
                \node[root] (s1) at (0,0) {$x^1_{\start}$};
                \node[plus] (g1) at (2.65,0.9) {$x^1_+$};
                \node[minus] (z1) at (2.65,-0.9) {$x^1_-$};
                \node[prob] at (1.33,1.75) {$\initial \spr*{x^1_{\start}} = q$};
                \draw[arr] (s1) -- node[lbl, above, pos=0.32, yshift=5pt] {$a = b_1$} (g1);
                \draw[arr] (s1) -- node[lbl, below, pos=0.32, yshift=-5pt] {$a \neq b_1$} (z1);
            \end{scope}
            \begin{scope}[xshift=9.85cm]
                \foreach \dx in {-0.18,0,0.18} {\fill[black!60] (\dx,0) circle[radius=1.2pt];}
            \end{scope}
            \begin{scope}[xshift=12.25cm]
                \node[root] (si) at (0,0) {$x^i_{\start}$};
                \node[plus] (gi) at (2.65,0.9) {$x^i_+$};
                \node[minus] (zi) at (2.65,-0.9) {$x^i_-$};
                \node[prob] at (1.33,1.75) {$\initial \spr*{x^i_{\start}} = q$};
                \draw[arr] (si) -- node[lbl, above, pos=0.32, yshift=5pt] {$a = b_i$} (gi);
                \draw[arr] (si) -- node[lbl, below, pos=0.32, yshift=-5pt] {$a \neq b_i$} (zi);
            \end{scope}
        \end{tikzpicture}}
        \caption{Instance from the family used in \Cref{thm:reward-realizability-lb}. Every MDP in the family uses the same reward, and, for $i \in \sbr{X}$, the hidden bit $b_i \in \scbr{0, 1}$ selects which action reaches $x^i_+$ from $x^i_{\start}$. It remains hidden unless $x^i_{\start}$ is observed during training.}
        \label{fig:reward-realizability-hard-family}
    \end{figure}
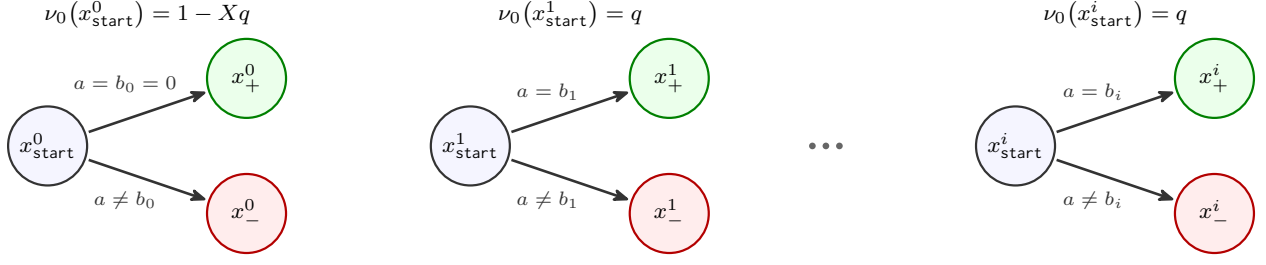

    For any policy $\pi$, starting from $x^i_{\start}$ yields terminal reward one if and only if the first action equals $b_i$. Hence
    \begin{equation} \label{eq:reward-realizability-gap}
        J_{\cM_b}^{\pi_\experttag^b} - J_{\cM_b}^{\pi}
        =
        1 - \sum_{i = 0}^{X} \initial \spr*{x^i_{\start}} \pi_1 \spr*{b_i \given x^i_{\start}}
        =
        \sum_{i = 0}^{X} \initial \spr*{x^i_{\start}} \pi_1 \spr*{\cA \setminus \scbr*{b_i} \given x^i_{\start}}.
    \end{equation}
    We lower bound the average suboptimality gap over a random instance. Let $B = \spr{B_1, \ldots, B_X} \sim \mathrm{Unif} \spr{\scbr{0, 1}^X}$, and run $\texttt{Alg}$ on $\cM_B$. Let $\xi$ be an independent random seed encoding all of the algorithm's internal randomness, and let $\cD$ denote the complete interaction transcript available to the learner after the fixed $\tauE$ episodes. For each $i \in \sbr{X}$, let $V_i$ be the event that $x^i_{\start}$ appears as the initial state in one of the fixed $\tauE$ interaction rounds.

    Although $B_i$ is sampled uniformly initially, we later condition on the dataset used to choose $\piout$. This conditioning can matter because $\cD$ is generated in the MDP whose expert policy and transitions depend on $B_i$. If $V_i$ occurs, then $\cD$ contains a transition out of $x^i_{\start}$. No matter whether an expert action is observed or not, the transition out of $x^i_{\start}$ reveals which action reaches $x^i_+$. On $V_i^c$, however, $\cD$ contains no observation from the $i$-th component of the construction: the state $x^i_{\start}$ can only appear as an initial state, and no transition from any other component reaches it. Thus, conditional on $V_i^c$, $\cD$ may depend on the coordinates $B_j$ for $j \neq i$, but the conditional law of $(\cD,\xi)$ does not depend on $B_i$. Since $B_i$ is uniform and independent of $B_j$ for $j \neq i$, even after seeing the dataset and knowing $V_i^c$, we still have no information about $B_i$. Whenever $\bbP \sbr{V_i^c} > 0$, this gives
    \begin{equation*}
        \bbP \sbr*{B_i = 0 \given \cD, V_i^c, \xi}
        =
        \bbP \sbr*{B_i = 1 \given \cD, V_i^c, \xi}
        =
        \frac12.
    \end{equation*}
    Conditional on $\cD$, $V_i^c$, and $\xi$, the output policy is fixed. We define $p_i \ldef \piout_1 \spr{1 \given x^i_{\start}}$. The quantity that appears in the suboptimality gap for the $i$-th component is $\piout_1 \spr{\cA \setminus \scbr{B_i} \given x^i_{\start}}$, which is the probability that the output policy chooses the wrong first action at $x^i_{\start}$. If $B_i = 0$, action $1$ is wrong, and this quantity equals $p_i$. If $B_i = 1$, action $0$ is wrong, and this quantity equals $1 - p_i$.

    When $\bbP \sbr{V_i^c} > 0$, averaging over $\cD$ and $\xi$ and multiplying by $\bbP \sbr{V_i^c}$ gives the identity below; when $\bbP \sbr{V_i^c} = 0$, the identity is trivial. Hence
    \begin{equation} \label{eq:reward-realizability-half-error}
        \bbE \sbr*{
            \piout_1 \spr{\cA \setminus \scbr{B_i} \given x^i_{\start}}
            \mathbbm{1}\scbr{V_i^c}
        }
        =
        \frac12 \bbP \sbr*{V_i^c}.
    \end{equation}
    Because $\tauE$ is a fixed deterministic budget and the initial state in each interaction round is sampled independently from $\initial$, the event $V_i^c$ has probability
    \begin{equation*}
        \bbP \sbr*{V_i^c}
        =
        \spr{1 - q}^{\tauE}.
    \end{equation*}
    Combining \cref{eq:reward-realizability-gap,eq:reward-realizability-half-error} with this identity and discarding the nonnegative contribution of $x^0_{\start}$ gives
    \begin{align}
        \bbE \sbr*{J_{\cM_B}^{\pi_\experttag^B} - J_{\cM_B}^{\piout}}
        &\geq
        \sum_{i = 1}^{X}
        q \bbE \sbr*{
            \piout_1 \spr{\cA \setminus \scbr{B_i} \given x^i_{\start}}
            \mathbbm{1}\scbr{V_i^c}
        }
        =
        \frac{X q}{2} \spr{1 - q}^{\tauE}.
        \label{eq:reward-realizability-average-risk}
    \end{align}
    Since the left-hand side averages uniformly over $B$, there exists a fixed $b \in \scbr{0, 1}^X$ such that
    \begin{equation*}
        \bbE \sbr*{J_{\cM_b}^{\pi_\experttag^b} - J_{\cM_b}^{\piout}}
        \geq
        \frac{X q}{2} \spr{1 - q}^{\tauE}.
    \end{equation*}
    If $\tauE < X / \spr{8 \varepsilon}$, Bernoulli's inequality gives
    \begin{equation*}
        \spr{1 - q}^{\tauE}
        =
        \spr{1 - \frac{4 \varepsilon}{X}}^{\tauE}
        \geq
        1 - \frac{4 \varepsilon \tauE}{X}
        >
        \frac12.
    \end{equation*}
    Therefore, $\cM_b$ satisfies
    \begin{equation*}
        \bbE \sbr*{J_{\cM_b}^{\pi_\experttag^b} - J_{\cM_b}^{\piout}}
        >
        \frac{X}{2} \cdot \frac{4 \varepsilon}{X} \cdot \frac12
        =
        \varepsilon.
    \end{equation*}
    Thus, achieving expected suboptimality at most $\varepsilon$ on every instance requires $\tauE \geq X / \spr{8 \varepsilon}$. Since $\abs{\cX} = 3 X + 3$, the sample-complexity statement follows.

\end{proof}


\clearpage
\section{Proofs from \texorpdfstring{\cref{sec:lb}}{Section} (\texorpdfstring{$\qexpert$}{Q-expert}-Realizability Is Insufficient to Learn Offline)}
\label{sec:proofs_lb}
Here, we show that an offline precollected expert dataset is insufficient for efficient IL (without paying for the expert policy hypothesis class) in MDPs where only $\qexpert$ is realizable. Our lower bound holds against all IL algorithms that output a value-induced policy, as defined in \Cref{def:valueIL}. We recall the definition of such algorithms next.

\para{Offline Value-Based Imitation Learning Algorithms} An algorithm $\mathrm{Alg}$ is said to be a value-based IL algorithm if the following conditions hold.
\begin{enumerate}
    \item $\mathrm{Alg}$ receives an expert dataset $\cDE$ and a state-action value function class $\cQ$.
    
    \item For some $K \in \bbN \cup \scbr{\infty}$, $\mathrm{Alg}$ chooses a sequence $Q_1, Q_2, \ldots, Q_K \in \cQ$ and coefficients $\spr{w_k}_{k = 1}^K$ defining the linear combination $\mathrm{LC} \spr{\spr{Q_k}_{k = 1}^K} = \sum_{k = 1}^K w_k Q_k$. When $K = \infty$, this series must converge absolutely at every stage-state-action tuple.
    
    \item $\mathrm{Alg}$ outputs the policy
    \begin{equation*}
        \piout \spr*{a \given x}
        \propto
        f \spr*{\mathrm{LC} \spr*{\spr*{Q_k}^K_{k = 1}} \spr*{x, a}}\,,
    \end{equation*}
    for any state-action pair $\spr{x, a}$, with the requirement that the normalizing denominator is strictly positive and finite at every stage-state pair.
\end{enumerate}
The pointwise-limit convention in \Cref{def:valueIL} also includes valid pointwise limits of policies of this form.

Next, we show that the class of algorithms that output a VI policy is general. It includes, for instance, \SPOIL \citep{moulin2025inverse}, while the pointwise-limit convention also represents greedy policies under suitable conditions.

\para{\SPOIL uses VI policies}
Choosing $f \spr{\cdot} = \exp \spr{\cdot}$ and taking $\mathrm{LC}$ to be the sum multiplied by a scalar $\eta \geq 0$ gives
\begin{equation*}
    \piout \spr*{a \given x}
    =
    \softmax \spr*{\eta {\textstyle\sumkK} Q_k \spr*{x, \cdot}}_a,
\end{equation*}
which captures the output strategy of \SPOIL and \Cref{alg:interactive_finite-H}.

\para{Greedy policies as limits of VI policies}
In the contextual-bandit setting of \citet{joshi2025learning}, we have $H = 1$, so we can omit the subscript $h$ and consider a reward class $\cR$ rather than a value function class $\cQ$. Their learner forms a reward $r \spr{x, a} = \sum_{\tilde{r} \in \cR} w \spr{\tilde{r}} \tilde{r} \spr{x, a}$ using learned weights $w \in \Delta \spr{\cR}$ and uses a predictor that is greedy with respect to $r$. Such a predictor is not a finite-temperature VI policy: for any finite $\eta$, choosing $f_\eta \spr{z} = \exp \spr{\eta z}$ gives the softmax VI policy $\pi_\eta \spr{\cdot \given x} = \softmax \spr{\eta r \spr{x, \cdot}}$, which generally assigns positive mass to nongreedy actions. Under our standing assumptions that $\cA$ is finite and $r$ is finite-valued, however, the softmax normalizer is strictly positive and finite, and these policies converge pointwise as $\eta \to \infty$.
\begin{equation*}
    \lim_{\eta \to \infty}
    \pi_\eta \spr*{a \given x}
    =
    \frac{
        \mathds{1}_{\scbr*{a \in \argmax_{b \in \cA} r \spr*{x, b}}}
    }{
        \abs*{\argmax_{b \in \cA} r \spr*{x, b}}
    }.
\end{equation*}
Consequently, a uniformly tie-broken greedy policy belongs to the VI class through its pointwise-limit convention. A deterministic greedy policy does as well when the maximizing action is unique. Other tie-breaking rules are covered only if they can be realized by an admissible sequence of VI scores. This gives a limiting connection to the roundwise greedy extraction step of \citet{joshi2025learning}; it does not assert that their final mixed output policy is itself a VI policy.

\para{Algorithms used in practice output VI policies}
Beyond prior work in IL theory \citep{moulin2025optimistically}, algorithms used in practice in IL, such as \texttt{IQ-LEARN} \citep{Garg:2021} and \texttt{CSIL} \citep{watson2023coherent}, are also value-based imitation algorithms. Indeed, these algorithms choose a single $Q$-function from the class $\cQ$ and output $\piout = \softmax \spr{\alpha Q \spr{x, a}}$. This is recovered by the definition of algorithms outputting VI policies, choosing $\alpha$ to be the weight of the linear combination and taking $f \spr{\cdot} = \exp \spr{\cdot}$.

\subsection{Proof of \texorpdfstring{\Cref{thm:lower}}{Theorem} (Lower Bound Against Offline Algorithms that Output VI Policies)}
\label{sec:proof_lower_bound}

We prove \Cref{thm:lower}, which shows nonidentifiability within the  class of algorithms outputting VI policies: even with infinite offline data, no such algorithm can guarantee error below $1/4$.

\Lower*

\begin{proof}[\pfref{thm:lower}]
    For clarity, we carry out the proof using centered rewards $r_h \in \sbr{-1, 0}$. Adding $1$ to every reward yields an equivalent MDP with rewards in $\sbr{0, 1}$. This shift preserves optimal policies, suboptimality gaps, and all action comparisons. Shifting the value class accordingly yields a class of the same cardinality that lies in $\sbr{0, 2}$ and satisfies $\qexpert$-realizability with $\QMAX = 2$. Although an arbitrary link function need not be translation invariant, every function in the shifted class still assigns equal values to the two actions at $x_{\start}$, and the proof treats the action distribution at $x_-$ as arbitrary. Hence the argument below applies unchanged. Formally, $a_1 \equiv +$ and $a_2 \equiv -$ are the two global actions; for every unspecified state-action-stage tuple, we assign zero centered reward and a transition to the absorbing state $x_{\texttt{end}}$, whose features are zero.

    \begin{figure}[htbp]
        \centering
        \begin{subfigure}[b]{0.45\textwidth}
            \centering
            \begin{tikzpicture}[
                x=0.90cm, y=0.90cm,
                state_node/.style={circle, draw=black!80, thick, minimum size=1.0cm, inner sep=1pt, align=center, font=\footnotesize, fill=gray!5},
                expert_arrow/.style={->, >={Stealth[round, bend]}, line width=0.95pt, blue!80!black, shorten >=1pt, shorten <=1pt},
                std_arrow/.style={->, >={Stealth[round, bend]}, thick, black!80, shorten >=1pt, shorten <=1pt},
                edge_label/.style={font=\scriptsize, align=center, fill=white, inner xsep=1.5pt, inner ysep=0.7pt},
                plus_label/.style={edge_label, pos=0.5, yshift=-5.2pt, fill opacity=0, text opacity=1, inner sep=0pt},
                minus_label/.style={edge_label, pos=0.53, xshift=3.2pt, yshift=3.2pt, fill opacity=0, text opacity=1, inner sep=0pt},
                upper_action_label/.style={edge_label, pos=0.48, above=2.3pt, xshift=-2.2pt},
                lower_action_label/.style={edge_label, pos=0.72, below=12pt, xshift=4.0pt, inner xsep=2.8pt, inner ysep=1.1pt}
            ]
                \node[state_node] (start) at (0,0) {$x_{\start}$};
                \node[state_node] (xminus) at (1.50,-1.35) {$x_-$};
                \node[state_node] (xend) at (3.55,0) {$x_{\texttt{end}}$};

                \draw[expert_arrow] (start) to[bend left=32] node[plus_label, text=blue!80!black] {$+$} (xend);
                \draw[std_arrow] (start) -- node[minus_label] {$-$} (xminus);
                \draw[std_arrow] (xminus) to[out=36, in=-145] node[upper_action_label] {$a_1$} (xend);
                \draw[std_arrow] (xminus) to[out=-4, in=-115] node[lower_action_label] {$a_2,\ r=-1$} (xend);
            \end{tikzpicture}
            \caption{MDP $\cM_1$: Action $a_1$ is optimal at $x_-$.}
            \label{fig:scenario1}
        \end{subfigure}
        \hfill 
        \begin{subfigure}[b]{0.45\textwidth}
            \centering
            \begin{tikzpicture}[
                x=0.90cm, y=0.90cm,
                state_node/.style={circle, draw=black!80, thick, minimum size=1.0cm, inner sep=1pt, align=center, font=\footnotesize, fill=gray!5},
                expert_arrow/.style={->, >={Stealth[round, bend]}, line width=0.95pt, blue!80!black, shorten >=1pt, shorten <=1pt},
                std_arrow/.style={->, >={Stealth[round, bend]}, thick, black!80, shorten >=1pt, shorten <=1pt},
                edge_label/.style={font=\scriptsize, align=center, fill=white, inner xsep=1.5pt, inner ysep=0.7pt},
                plus_label/.style={edge_label, pos=0.5, yshift=-5.2pt, fill opacity=0, text opacity=1, inner sep=0pt},
                minus_label/.style={edge_label, pos=0.53, xshift=3.2pt, yshift=3.2pt, fill opacity=0, text opacity=1, inner sep=0pt},
                upper_action_label/.style={edge_label, pos=0.48, above=2.3pt, xshift=-2.2pt},
                lower_action_label/.style={edge_label, pos=0.72, below=12pt, xshift=4.0pt, inner xsep=2.8pt, inner ysep=1.1pt}
            ]
                \node[state_node] (start) at (0,0) {$x_{\start}$};
                \node[state_node] (xminus) at (1.50,-1.35) {$x_-$};
                \node[state_node] (xend) at (3.55,0) {$x_{\texttt{end}}$};

                \draw[expert_arrow] (start) to[bend left=32] node[plus_label, text=blue!80!black] {$+$} (xend);
                \draw[std_arrow] (start) -- node[minus_label] {$-$} (xminus);
                \draw[std_arrow] (xminus) to[out=36, in=-145] node[upper_action_label] {$a_2$} (xend);
                \draw[std_arrow] (xminus) to[out=-4, in=-115] node[lower_action_label] {$a_1,\ r=-1$} (xend);
            \end{tikzpicture}
            \caption{MDP $\cM_2$: Action $a_2$ is optimal at $x_-$.}
            \label{fig:scenario2}
        \end{subfigure}
        \caption{\label{fig:hardMDPenv} Family of hard tasks. Unless otherwise specified, the reward is $0$. The blue path is the one taken by the deterministic expert. $x_{\texttt{end}}$ is an absorbing state.}
    \end{figure}
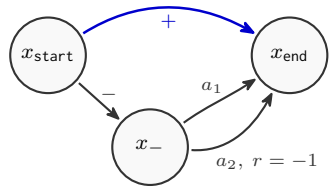
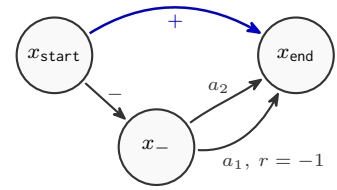

    We consider the family of hard MDPs $\cF = \scbr{\cM_1, \cM_2}$ depicted in \Cref{fig:hardMDPenv}. We consider the following $2$-dimensional features that satisfy the $Q^\expert$-realizability assumption.
    \begin{itemize}
        \item $\phi \spr{x_{\start}, +} = \phi \spr{x_{\start}, -} = \sbr{0, 0}\transpose$,
        \item $\phi \spr{x_-, a_1} = \sbr{0, 1}\transpose$,
        \item $\phi \spr{x_-, a_2} = \sbr{1, 0}\transpose$.
    \end{itemize}
    If the expert policy satisfies $\expertone \spr{+ \given x_{\start}} = 1$ and $\expertone \spr{a_1 \given x_-} = 1$, then we have
    \begin{itemize}
        \item $Q^\expertone_{\cM_1} \spr{x_{\start}, +} = Q^\expertone_{\cM_1} \spr{x_{\start}, -} = 0$,
        \item $Q^\expertone_{\cM_1} \spr{x_-, a_1} = 0$,
        \item $Q^\expertone_{\cM_1} \spr{x_-, a_2} = -1$.
    \end{itemize}
    Thus, by choosing $\theta_1 = \sbr{-1, 0}$, we can realize the expert action-value function as $Q^\expertone_{\cM_1} \spr{x, a} = \phi \spr{x, a}\transpose \theta_1$ for all $\spr{x, a} \in \cX \times \cA$. For the second MDP in $\cF$, $\cM_2$, the optimal expert policy satisfies $\experttwo \spr{+ \given x_{\start}} = 1$ and $\experttwo \spr{a_2 \given x_-} = 1$. Taking $\theta_2 = \sbr{0, -1}$, we can similarly write $Q^\experttwo_{\cM_2} \spr{x, a} = \phi \spr{x, a}\transpose \theta_2$, so $\cM_2$ is also $\qexpert$-realizable with the same features. Therefore, the two MDPs differ only in which action is good at $x_-$: $a_1$ in $\cM_1$ and $a_2$ in $\cM_2$.

    Fix any dataset size $\tauE \in \bbN^\star$. We now start proving our result. For the moment, we consider the case of deterministic value-based offline IL algorithms that choose a sequence of state-action value functions without randomization. Moreover, we consider the class
    \begin{equation*}
        \cQ
        =
        \scbr*{Q^\expertone_{\cM_1}, Q^\experttwo_{\cM_2}},
    \end{equation*}
    where $Q^\expertone_{\cM_1} \spr{x, a} = \phi \spr{x, a}\transpose \theta_1$ and $Q^\experttwo_{\cM_2} \spr{x, a} = \phi \spr{x, a}\transpose \theta_2$ for all $x, a$. Notice that in both $\cM_1$ and $\cM_2$, we have that  $Q^\expertone_{\cM_1} \spr{x_{\start}, a} = Q^\experttwo_{\cM_2} \spr{x_{\start}, a} = 0$ for all $a \in \scbr{+, -}$.

    Under this setting, no matter how $\mathrm{Alg}$ selects the sequence $Q_1, \dots, Q_K$, the output policy at the starting state is uniform: $\piout \spr{\cdot \given x_{\start}} = \mathrm{Unif} \spr{\scbr{+, -}}$. Indeed, every $Q \in \cQ$ assigns the same value to the two actions $+$ and $-$ at $x_{\start}$. Hence any well-defined linear combination $\mathrm{LC} \spr{Q^{1:K}}$ also assigns the same scalar to the two start actions, including when $K = \infty$. Since $f$ is a function and the normalizing denominator is positive by definition of a valid output policy, the two numerator terms are equal and nonzero. Therefore, the policy $\piout$ induced by any linear combination $\mathrm{LC}$ and function $f$ in \Cref{def:valueIL} guarantees $\piout \spr{+ \given x_{\start}} = \piout \spr{- \given x_{\start}} = \frac12$. Moreover, taking pointwise limits or averaging policies obtained from different choices of sequences in $\cQ$ still gives a uniform policy.
    
    It follows that, with probability $1/2$, the next state is $x_-$. Suppose that the sequence $Q_1, \dots, Q_K$ chosen by $\mathrm{Alg}$ induces probabilities $\piout \spr{a_1 \given x_-} = p$ and $\piout \spr{a_2 \given x_-} = 1 - p$. The choice of $Q_1, \dots, Q_K$, and hence of $p$, cannot depend on whether the underlying MDP is $\cM_1$ or $\cM_2$, because the expert dataset is identical in the two cases, even with infinite data. Therefore, the suboptimality at $x_-$, averaged over the MDP class $\cF$, equals $1/2$ as shown next.
    \begin{align*}
        \frac12 \sum^2_{i = 1} \spr*{Q^\experti_{\cM_i} \spr*{x_-, a_i} - Q^\experti_{\cM_i} \spr*{x_-, \piout}}
        &=
        \frac12 \sum^2_{i = 1} \spr*{- Q^\experti_{\cM_i} \spr*{x_-, \piout}} \\
        &=
        \frac{p}{2} \sum^2_{i = 1} \spr*{- r_{\cM_i} \spr*{x_-, a_1}} + \frac{1 - p}{2} \sum^2_{i = 1} \spr*{- r_{\cM_i} \spr*{x_-, a_2}} \\
        &=
        \frac{p}{2} + \frac{1 - p}{2} \\
        &=
        \frac12.
    \end{align*}
    Fix $i \in \scbr*{1, 2}$. Applying the performance difference lemma (\cref{lem:performance-difference}) with $\pi = \piout$ and $\pi' = \experti$ to the MDP whose initial distribution is a point mass at $x_{\start}$ gives the following identity. Suppressing the stage index as above, the only state-occupancy terms that can contribute are $d^{\piout}_1 \spr{x_{\start}} = 1$ and $d^{\piout}_2 \spr{x_-} = \piout \spr{- \given x_{\start}} = 1 / 2$. All terms at $x_{\texttt{end}}$ vanish because the expert $Q$-values there are zero. Therefore,
    \begin{align*}
        V^{\experti}_{\cM_i} \spr*{x_{\start}} - V^{\piout}_{\cM_i} \spr*{x_{\start}}
        &=
        \underbrace{1}_{d^{\piout}_1 \spr{x_{\start}}} \sum_{a \in \scbr*{+, -}} \spr*{\experti \spr*{a \given x_{\start}} - \piout \spr*{a \given x_{\start}}} Q^\experti_{\cM_i} \spr*{x_{\start}, a} \\
        &\quad
        + \underbrace{\frac12}_{d^{\piout}_2 \spr{x_-}} \sum_{a \in \scbr*{a_1, a_2}} \spr*{\experti \spr*{a \given x_-} - \piout \spr*{a \given x_-}} Q^\experti_{\cM_i} \spr*{x_-, a} \\
        &=
        0 + \frac12 \spr*{Q^\experti_{\cM_i} \spr*{x_-, a_i} - \sum_{a \in \scbr*{a_1, a_2}} \piout \spr*{a \given x_-} Q^\experti_{\cM_i} \spr*{x_-, a}} \\
        &=
        \frac12 \spr*{Q^\experti_{\cM_i} \spr*{x_-, a_i} - Q^\experti_{\cM_i} \spr*{x_-, \piout}}.
    \end{align*}
    Separately, the worst-case suboptimality gap is at least the average of the two suboptimality gaps over the two instances, so combining the two derivations above gives
    \begin{align*}
        \max_{\cM \in \scbr*{\cM_1, \cM_2}} V^\expert_\cM \spr*{x_{\start}} - V^{\piout}_\cM \spr*{x_{\start}}
        &\geq
        \frac12 \sum^2_{i = 1} \spr*{V^{\experti}_{\cM_i} \spr*{x_{\start}} - V^{\piout}_{\cM_i} \spr*{x_{\start}}} \\
        &=
        \frac12 \sum^2_{i = 1} \frac12 \spr*{Q^\experti_{\cM_i} \spr*{x_-, a_i} - Q^\experti_{\cM_i} \spr*{x_-, \piout}} \\
        &=
        \frac12 \cdot \underbrace{\frac12 \sum^2_{i = 1} \spr*{Q^\experti_{\cM_i} \spr*{x_-, a_i} - Q^\experti_{\cM_i} \spr*{x_-, \piout}}}_{= 1 / 2 \text{ by the preceding display}}
        =
        \frac{1}{4}.
    \end{align*}
    The preceding display holds for every deterministic offline algorithm that outputs a VI policy under the uniform distribution over the two instances in $\cF$. Therefore, by Yao's minimax principle (see, \eg, \citealp{bubeck2012regret}), the same lower bound holds against possibly randomized algorithms:
    \begin{equation*}
        \max_{\cM \in \scbr*{\cM_1, \cM_2}} \bbE_{\piout \sim \mathrm{Alg}} \sbr*{V^\star_\cM \spr*{x_{\start}} - V^{\piout}_\cM \spr*{x_{\start}}}
        \geq
        \frac{1}{4},
    \end{equation*}
    where we used the equalities $V^\expertone_{\cM_1} = V^\star_{\cM_1}$ and $V^\experttwo_{\cM_2} = V^\star_{\cM_2}$, which follow by optimality of the experts.
\end{proof}

\subsection{Proof of \texorpdfstring{\Cref{thm:lower_any_algo}}{Theorem} (Lower Bound Against Any Offline IL Algorithm)}

In the above construction, we could learn by memorizing one expert action in $x_{\start}$ and playing it at deployment time. However, we can build a harder extended MDP by repeating the hard instance from \Cref{fig:hardMDPenv} several times, as shown in \Cref{fig:extended_hard}. This lets us derive an information-theoretic lower bound against any offline IL algorithm. For algorithms outside the class of algorithms that output a VI policy, the problem is no longer nonidentifiable. However, the required dataset size scales linearly with the number of states, while the class $\cQ$ realizing $\qexpert$ has cardinality independent of $\abs{\cX}$.\loose

\LowerAny*

\begin{proof}[\pfref{thm:lower_any_algo}]
    As in the preceding proof of \Cref{thm:lower}, we use centered rewards $r_h \in \sbr{-1, 0}$ for convenience. Applying the same shifts and conventions yields the required formal construction without affecting the argument.\loose

    We consider a family $\cF = \bc{\cM^{\mathrm{ext}}_j}^{2^{X+1}}_{j=1}$ in which each of the $2^{X+1}$ extended MDPs has $3 X + 3$ states and is obtained by considering $X + 1$ copies of the structure in \Cref{fig:hardMDPenv}, with starting states labeled as $x^i_{\start}$ for $i \in \scbr{0, \dots, X}$. For any $\varepsilon \leq \frac18$, denote $q = \frac{8 \varepsilon}{X}$ and consider the following initial distribution
    \begin{equation*}
        \initial \spr*{x^0_{\start}} = 1 - X q,
        \qquad \text{and} \qquad
        \initial \spr*{x^i_{\start}} = q
        \quad \text{for all } i \neq 0.
    \end{equation*}
    To construct the MDP family, we consider that each of the $X$ copies indexed by $i \neq 0$ can be in two modes:
    \begin{itemize}
        \item In mode $+$, we consider that the action $+$ goes straight from $x^i_{\start}$ to $x^i_{\mathsf{end}}$, without passing through $x^i_{-}$.
        \item In mode $-$, instead the action $-$ goes straight from $x^i_{\start}$ to $x^i_{\mathsf{end}}$, without passing through $x^i_{-}$.
    \end{itemize}
    Considering all possible combinations of modes for each copy except the first, \ie, for $i=1, \dots, X$, we obtain $2^X$ MDPs. Moreover, for each of these $2^X$ transition configurations, we can choose, as in the proof of \Cref{thm:lower}, the reward parameters to be $\theta_1 = [-1, 0]$ or $\theta_2 = [0, -1]$. Therefore, we obtain a family of $2^{X+1}$ MDPs. A visual representation of an extended MDP with parameters $\theta_1$ is given in \Cref{fig:extended_hard}.
    
    The family of $2^{X+1}$ optimal expert policies is defined so that the expert policy in the copy with root state $x^i_{\start}$ in the MDP $\cM^{\mathrm{ext}}_j$ takes the optimal action that does not pass through $x^i_{-}$. That is, the expert chooses action $+$ if the $i^{\mathrm{th}}$ copy is in mode $+$ in the MDP $\cM^{\mathrm{ext}}_j$ and action $-$ if the copy is in mode $-$. At each $x^i_-$, the expert takes an action that is optimal for the chosen reward parameter, while at every absorbing state all expert policies take the same fixed action.
    
    Since all the experts are optimal, we denote the expert action-value function in the MDP $\cM^{\mathrm{ext}}_j$ by $Q^\star_{\cM^{\mathrm{ext}}_j}$ and the state value function by $V^\star_{\cM^{\mathrm{ext}}_j}$. At this point, note that for each $j \in [2^{X+1}]$, we have that $Q^\star_{\cM^{\mathrm{ext}}_j}$ is affine in the features
    \begin{equation*}
        \phi \spr{x^i_{\start}, +} = \phi \spr{x^i_{\start}, -} = \sbr{0, 0}\transpose,
        \quad
        \phi \spr{x^i_-, a_1} = \sbr{0, 1}\transpose,
        \quad\text{and}\quad
        \phi \spr{x^i_-, a_2} = \sbr{1, 0}\transpose,
    \end{equation*}
    with parameters either $\theta_1$ or $\theta_2$. Denoting $Q_\theta \spr{x, a} = \phi \spr{x, a}\transpose \theta$, the class $\cQ = \bc{Q_{\theta_1}, Q_{\theta_2}}$ satisfies $\qexpert$-realizability in $\cM^{\mathrm{ext}}_j$ for every $j \in [2^{X+1}]$ and has cardinality independent of $X$. Formally, we have $Q^\star_{\cM^{\mathrm{ext}}_j} \in \cQ$ for any $j\in[2^{X+1}]$.
    
    In each MDP, the expert from the root states $x^0_{\start}, \dots, x^X_{\start}$ only takes the action that avoids the states $x^0_{-}, \dots, x^X_{-}$, so it does not reveal which action (between $a_1$ and $a_2$) is optimal in those states and the learner cannot determine whether the reward parameter vector is $\theta_1$ or $\theta_2$.
    \begin{figure}[!t]
        \centering
        \begin{tikzpicture}[
            x=0.90cm, y=0.90cm,
            state/.style={circle, draw=black!80, thick, minimum size=1.0cm, inner sep=1pt, align=center, font=\footnotesize, fill=gray!5},
            seen_state/.style={state, draw=green!50!black, fill=green!8},
            unseen_state/.style={state, draw=red!70!black, fill=red!6},
            expert_arrow/.style={->, >={Stealth[round, bend]}, line width=0.95pt, green!50!black, shorten >=1pt, shorten <=1pt},
            learner_arrow/.style={->, >={Stealth[round, bend]}, line width=0.95pt, red!70!black, dashed, shorten >=1pt, shorten <=1pt},
            std_arrow/.style={->, >={Stealth[round, bend]}, thick, black!80, shorten >=1pt, shorten <=1pt},
            edge_label/.style={font=\scriptsize, align=center, fill=white, inner xsep=1.5pt, inner ysep=0.7pt},
            plus_label/.style={edge_label, pos=0.5, yshift=-5.2pt, fill opacity=0, text opacity=1, inner sep=0pt},
            minus_label/.style={edge_label, pos=0.53, xshift=3.2pt, yshift=3.2pt, fill opacity=0, text opacity=1, inner sep=0pt},
            upper_action_label/.style={edge_label, pos=0.48, above=2.3pt, xshift=-2.2pt},
            lower_action_label/.style={edge_label, pos=0.72, below=12pt, xshift=4.0pt, inner xsep=2.8pt, inner ysep=1.1pt},
            prob_label/.style={font=\footnotesize, align=center, text=black, inner sep=1pt}
        ]

            \begin{scope}
                \node[seen_state] (s0) at (0,0) {$x^0_{\start}$};
                \node[state] (xm0) at (1.50,-1.35) {$x^0_-$};
                \node[state] (xe0) at (3.55,0) {$x^0_{\texttt{end}}$};

                \node[prob_label] at (1.77,1.25) {$\initial(x^0_{\start}) = 1 - X q$};
                \draw[expert_arrow] (s0) to[bend left=32] node[plus_label, text=green!50!black] {$+$} (xe0);
                \draw[std_arrow] (s0) -- node[minus_label] {$-$} (xm0);
                \draw[std_arrow] (xm0) to[out=36, in=-145] node[upper_action_label] {$a_1$} (xe0);
                \draw[std_arrow] (xm0) to[out=-4, in=-115] node[lower_action_label] {$a_2,\ r=-1$} (xe0);
            \end{scope}

            \begin{scope}[xshift=5.00cm]
                \node[seen_state] (s1) at (0,0) {$x^1_{\start}$};
                \node[state] (xm1) at (1.50,-1.35) {$x^1_-$};
                \node[state] (xe1) at (3.55,0) {$x^1_{\texttt{end}}$};

                \node[prob_label] at (1.77,1.25) {$\initial(x^1_{\start}) = q$};
                \draw[expert_arrow] (s1) to[bend left=32] node[plus_label, text=green!50!black] {$+$} (xe1);
                \draw[std_arrow] (s1) -- node[minus_label] {$-$} (xm1);
                \draw[std_arrow] (xm1) to[out=36, in=-145] node[upper_action_label] {$a_1$} (xe1);
                \draw[std_arrow] (xm1) to[out=-4, in=-115] node[lower_action_label] {$a_2,\ r=-1$} (xe1);
            \end{scope}

            \begin{scope}[xshift=9.75cm]
                \foreach \dotx in {-0.16,0,0.16} {
                    \fill[black!55] (\dotx,0) circle[radius=1.1pt];
                }
            \end{scope}

            \begin{scope}[xshift=11.35cm]
                \node[unseen_state] (si) at (0,0) {$x^i_{\start}$};
                \node[state] (xmi) at (1.50,-1.35) {$x^i_-$};
                \node[state] (xei) at (3.55,0) {$x^i_{\texttt{end}}$};

                \node[prob_label] at (1.77,1.25) {$\initial(x^i_{\start}) = q$};
                \draw[learner_arrow] (si) to[bend left=32] node[plus_label, text=red!70!black] {$+$} (xei);
                \draw[learner_arrow] (si) -- node[minus_label, text=red!70!black] {$-$} (xmi);
                \draw[std_arrow] (xmi) to[out=36, in=-145] node[upper_action_label] {$a_1$} (xei);
                \draw[std_arrow] (xmi) to[out=-4, in=-115] node[lower_action_label] {$a_2,\ r=-1$} (xei);
            \end{scope}
        \end{tikzpicture}
        \caption{Representation of a MDP in the class $\cF$. The states highlighted in green appear in $\cDE$, so the learner can deterministically replay the expert action $+$. The red states are absent from $\cDE$. Thus, no expert guidance is available there, and, conditional on starting at $x^i_{\start}$, the learner reaches $x^i_{-}$ with probability at least $\nicefrac{1}{2}$ in some environment in $\cF$, where it plays an action that is suboptimal under one of the two reward parameter choices $\theta_1$ and $\theta_2$. Here, $q = 8 \varepsilon / X$, so $\initial$ assigns mass $1 - X q$ to $x^0_{\start}$ and mass $q$ to every other root state. This initial distribution follows the lower-bound construction for tabular offline IL from \citet{rajaraman2020toward} and ensures that guaranteeing expected suboptimality below $\varepsilon$ uniformly over $\cF$ requires $\tauE = \Omega \spr{\abs{\cX} \varepsilon^{-1}}$.\loose \label{fig:extended_hard}}
    \end{figure}
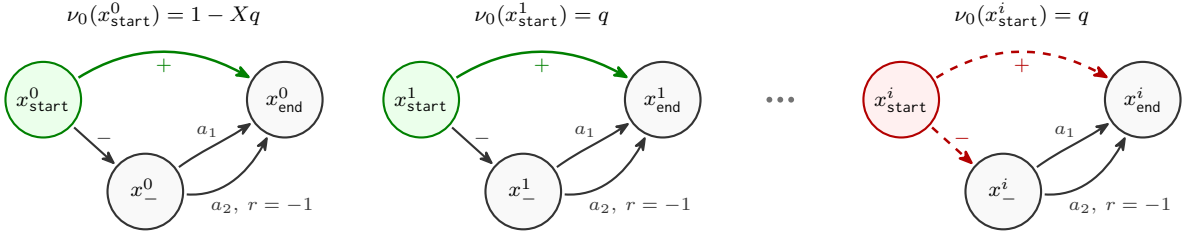

    For the averaging argument, draw an environment $\cM$ uniformly from $\cF$, equivalently by drawing its modes $\sigma_1, \dots, \sigma_X$ independently and uniformly from $\bc{+, -}$ and its reward parameter $\Theta$ independently and uniformly from $\bc{\theta_1, \theta_2}$. Here, $\sigma_i$ specifies the mode of copy $i$, so action $\sigma_i$ bypasses $x^i_-$, while $\Theta$ is the reward parameter shared by all copies and determines which action is optimal at each $x^i_-$. Let $\xi$ be a random seed, independent of the environment and expert data, that encodes all of the learner's internal randomness, and define $\mathsf{T} \ldef \spr{\cDE, \xi}$. Thus, $\mathsf{T}$ consists exactly of the expert dataset and the learner's random seed. Fix $i \in \sbr{X}$ and consider the event $\cE_i \ldef \scbr{x^i_{\start} \notin \cDE}$. On $\cE_i$, the dataset contains no observation from copy $i$, so it does not depend on $\sigma_i$. Moreover, it does not depend on $\Theta$: the expert never visits any of the states $x^0_-, \dots, x^X_-$, rewards are unobserved, and all expert policies take the same fixed action at the absorbing states. Since $\xi$ is independent of the environment and expert data, $\mathsf{T}$ is therefore independent of $\spr{\sigma_i, \Theta}$ conditional on $\cE_i$, even though it may reveal the modes of other sampled copies. Consequently, conditional on any realization $\mathsf{T} = t$ and $\cE_i$, the variables $\sigma_i$ and $\Theta$ remain independent and uniform, while the learner's output policy may depend arbitrarily on $t$.
    Fix such a transcript $t$, write $\piout_t$ for the resulting output policy, and let $p_t \ldef \piout_{t,1} \spr{+ \given x^i_{\start}}$ and $u_t \ldef \piout_{t,2} \spr{a_1 \given x^i_-}$. Averaging over $\sigma_i$, the probability of reaching $x^i_-$ is $\spr{1 - p_t} / 2 + p_t / 2 = 1 / 2$. Independently, averaging over $\Theta$, the suboptimality at $x^i_-$ is $\spr{1 - u_t} / 2 + u_t / 2 = 1 / 2$. Therefore, for every such $t$,
    \begin{equation*}
        \bbE
        \sbr*{
            V^\star_{\cM} \spr*{x^i_{\start}}
            -
            V^{\piout_t}_{\cM} \spr*{x^i_{\start}}
            \given
            \mathsf{T} = t, \cE_i
        }
        =
        \spr*{\frac{1 - p_t}{2} + \frac{p_t}{2}}
        \spr*{\frac{1 - u_t}{2} + \frac{u_t}{2}}
        =
        \frac14.
    \end{equation*}
    For each $j$, let $\bbE_j$ denote expectation under the transcript distribution in $\cM^{\mathrm{ext}}_j$ and the learner's internal randomness, and write $\Delta_{j,i} \ldef V^\star_{\cM^{\mathrm{ext}}_j} \spr{x^i_{\start}} - V^{\piout}_{\cM^{\mathrm{ext}}_j} \spr{x^i_{\start}}$. Since $\cE_i$ depends only on the initial-state draws, it has probability $\spr{1 - q}^{\tauE}$ in every environment. Averaging the preceding identity over $\mathsf{T}$ therefore yields the unconditional bound
    \begin{equation*}
        \frac{1}{2^{X+1}} \sum_{j = 1}^{2^{X+1}}
        \bbE_j
        \sbr*{
            \Delta_{j,i}
            \mathds{1}_{\cE_i}
        }
        \geq
        \frac14 \spr*{1 - q}^{\tauE}.
    \end{equation*}
    Optimality implies $\Delta_{j,i} \geq 0$ at every root state. Thus, dropping the $i = 0$ term and retaining only the terms on $\cE_i$, for any $j \in \sbr{2^{X+1}}$ we have
    \begin{align*}
        \inp*{\initial, V_{\cM^{\mathrm{ext}}_j}^{\star} - V_{\cM^{\mathrm{ext}}_j}^{\piout}}
        &=
        \sum^{X}_{i = 0} \initial \spr*{x^i_{\start}} \spr*{V_{\cM^{\mathrm{ext}}_j}^{\star} \spr*{x^i_{\start}} - V_{\cM^{\mathrm{ext}}_j}^{\piout} \spr*{x^i_{\start}}} \\
        &\geq
        \sum^{X}_{i = 1} \initial \spr*{x^i_{\start}} \spr*{V_{\cM^{\mathrm{ext}}_j}^{\star} \spr*{x^i_{\start}} - V_{\cM^{\mathrm{ext}}_j}^{\piout} \spr*{x^i_{\start}}} \\
        &=
        \sum^{X}_{i = 1} q \spr*{V_{\cM^{\mathrm{ext}}_j}^{\star} \spr*{x^i_{\start}} - V_{\cM^{\mathrm{ext}}_j}^{\piout} \spr*{x^i_{\start}}} \\
        &\geq
        \sum^{X}_{i = 1} q \spr*{V_{\cM^{\mathrm{ext}}_j}^{\star} \spr*{x^i_{\start}} - V_{\cM^{\mathrm{ext}}_j}^{\piout} \spr*{x^i_{\start}}} \mathds{1}_{\cE_i}.
    \end{align*}
    Taking $\bbE_j$ in the preceding inequality, averaging over the environments, and applying the unconditional bound above gives
    \begin{align*}
        \max_{j \in \sbr{2^{X+1}}}
        &\bbE_j
        \sbr*{\inp*{\initial, V_{\cM^{\mathrm{ext}}_j}^{\star} - V_{\cM^{\mathrm{ext}}_j}^{\piout}}} \\
        &\geq
        \frac{1}{2^{X+1}} \sum_{j = 1}^{2^{X+1}}
        \bbE_j
        \sbr*{\inp*{\initial, V_{\cM^{\mathrm{ext}}_j}^{\star} - V_{\cM^{\mathrm{ext}}_j}^{\piout}}} \\
        &\geq
        \sum^{X}_{i = 1} q
        \frac{1}{2^{X+1}} \sum_{j = 1}^{2^{X+1}}
        \bbE_j \sbr*{\Delta_{j,i} \mathds{1}_{\cE_i}} \\
        &\geq
        \frac{q}{4}\sum^{X}_{i = 1} \spr*{1 - q}^{\tauE}
        \\
        &=
        \frac{X}{4} q \spr*{1 - q}^{\tauE}.
    \end{align*}
    Then, if $\tauE \leq X / \spr{32 \varepsilon}$, Bernoulli's inequality gives $\spr{1 - q}^{\tauE} = \spr{1 - \frac{8 \varepsilon}{X}}^{\tauE} \geq 1 - \frac{8 \varepsilon \tauE}{X} \geq \frac34$. Therefore,
    \begin{equation*}
        \max_{\cM \in \cF} \bbE \sbr*{\inp*{\initial, V_{\cM}^{\star} - V_{\cM}^{\piout}}} \geq \frac{X}{4} \cdot \frac{8 \varepsilon}{X} \cdot \frac34
        = \frac{3 \varepsilon}{2}
        > \varepsilon,
    \end{equation*}
    Thus, recalling $3 \spr{X + 1} = \abs{\cX}$, achieving $\varepsilon$-suboptimality requires $\tauE = \Omega \spr{\nicefrac{\abs{\cX}}{\varepsilon}}$ for any offline IL algorithm.\loose
\end{proof}


\newpage
\part{Additional Results}

We collect here additional results omitted from the main text. We begin by completing \cref{sec:ub} with a variant of \ISPIL. In \Cref{sec:q_first_ispil}, we show that reversing the order in which the players update in \ISPIL yields an algorithm that is less computationally efficient, but provides statistical benefits for nonconvex $\cQ$ classes and allows us to learn a stationary solution in discounted MDPs.\loose

In addition to the lower bound from \cref{sec:lb}, we show in \cref{sec:break_lb_with_coverage} that it is possible to learn offline if one assumes the (optimal) expert covers the learner's distribution.

Moreover, in \Cref{sec:adaptive_algo}, we show that interaction enables efficient learning even when the class $\cQ$ realizes either $\qexpert$ or the $Q$-value of any policy that the algorithm might produce, without the learner knowing which case holds. The key algorithmic idea is to sample states from a mixture of the expert and learner occupancy measures, a technique proven successful in the context of language models \citep{agarwal2024policy,li2026revisiting}. This adaptive variant of \ISPIL therefore applies to a strictly larger set of classes $\cQ$ than that covered by the offline algorithm of \citet{moulin2025inverse}, revealing a previously unknown benefit of interaction.

\section{\ISPILPLUS: Improved Bound on the Number of Expert Queries}
\label{sec:ispil_plus}
Here, we present \ISPILPLUS, which reduces the number of expert queries by reusing expert-labeled data while drawing fresh unlabeled states to estimate the learner term. Suppressing logarithmic and value-class complexity factors, the expert-query bound for general value-function classes improves from order $\varepsilon^{-4}$ to order $\varepsilon^{-2}$, while the learner-rollout bound retains an $\varepsilon^{-4}$ dependence. We first present the pseudocode for \ISPILPLUS in \Cref{alg:interactive_finite-H-expert-buffer} and explain its design, then establish its guarantees.

\subsection{Design of \ISPILPLUS}

All samples within a batch are drawn independently conditional on the history before that batch. As for \ISPIL, we assume that the $Q$-player's best responses exist.

\begin{algorithm}[!h]
    \caption{\textbf{\ISPILPLUS}: \textbf{O}n-Policy \textbf{V}alue-Based \textbf{I}mitation Learning with Expert Sample Buffer \label{alg:interactive_finite-H-expert-buffer}}
    \begin{algorithmic}[1]
        \STATE \textbf{input:} Learning rate $\eta$, iterations $K$, expert queries per stage $\tauE$, learner batch size per iteration $\tauL$.
        \FOR{$h = 1, \ldots, H$}
        \STATE Create a dataset at stage $h$: for $i \in \sbr{\tauE}$, sample $\Xih \sim d^{\piout}_h$, and query $\AEih \sim \experth \spr{\cdot \given \Xih}$.
        \STATE Initialize $\pi^1_h = \mathrm{Unif} \spr{\cA}$.
        \FOR{$k = 1, \ldots, K$}
        \STATE \algcomment{Sample a fresh batch for estimating the learner term}
        \STATE For $j \in \sbr{\tauL}$, sample $\Xjkh \sim d^{\piout}_h$.
        \STATE Set $Q^{k}_h \in \argmax_{Q_h \in \cQ_h} \underbrace{\frac{1}{\tauE}\sum^{\tauE}_{i=1} Q_h \spr{\Xih, \AEih}}_{\text{Expert term}} - \underbrace{\frac{1}{\tauL}\sum^{\tauL}_{j=1} Q_h \spr{\Xjkh, \pi^k_h}}_{\text{Learner term}}$.
        \STATE Set $\pi^{k+1}_h \spr{a \given x} \propto \pi^k_h \spr{a \given x} e^{\eta Q^k_h \spr{x, a}}$, for any state-action pair $\spr{x, a}$.
        \ENDFOR
        \STATE Create the output policy at layer $h$: $\piouth = \frac1K \sumkK \pi^k_h$.
        \ENDFOR
    \end{algorithmic}
\end{algorithm}

The main idea behind \ISPILPLUS is to estimate the expert and learner terms in the $Q$-player's objective using separate samples. At each layer $h$, the previously learned decision rules fix $d^{\piout}_h$. This allows us to collect a batch of expert-labeled data $\scbr{\spr{\Xih, \AEih}}_{i = 1}^{\tauE}$ before starting the inner loop over $k \in \sbr{K}$. We reuse this batch to estimate the expert term in the objective of the $Q$-player, \ie, $\sum_{x \in \cX} d^{\piout}_h \spr{x} Q_h \spr{x, \experth}$. For each fixed $Q_h$, the centered expert-sample contributions form a martingale difference sequence when the batch is revealed sequentially, conditional on the history before that batch. Since this term does not involve the current decision rule $\pi^k_h$, uniform concentration over $\cQ_h$ suffices to control its estimation error throughout the inner loop. By contrast, the learner term depends on $\pi^k_h$, which is updated within the inner loop. When the same states are reused to estimate this term, as in \Cref{alg:interactive_finite-H}, the current decision rule depends on those states. Even for a fixed $Q_h$, the centered sample contributions need not form a martingale difference sequence when evaluated at this data-dependent decision rule. A concentration bound for a fixed decision rule therefore does not directly apply. The analysis of \ISPIL handles this dependence with an estimation bound for $\sum_{x \in \cX} d^{\piout}_h \spr{x} Q_h \spr{x, \pi_h}$ that holds uniformly, \ie, for all $Q_h \in \cQ_h$ and $\pi_h \in \Pi_{\cQ,h}$. In \Cref{alg:interactive_finite-H-expert-buffer}, we instead collect a fresh batch of $\tauL$ states $\scbr{\Xjkh}_{j = 1}^{\tauL}$ from $d^{\piout}_h$ at each iteration $k \in \sbr{K}$ to estimate $\sum_{x \in \cX} d^{\piout}_h \spr{x} Q_h \spr{x, \pi^k_h}$. Since the batch is sampled after $\pi^k_h$ is computed, we can condition on the preceding history and treat $\pi^k_h$ as fixed. For each fixed $Q_h$, the centered contributions from this fresh batch then form a martingale difference sequence. We therefore only need uniform concentration over $\cQ_h$, without uniformity over the learner policy class. Collecting these states only requires rolling in with the previously learned decision rules $\piout_1, \ldots, \piout_{h - 1}$ and requires no expert queries. These additional samples therefore do not increase the expert-query complexity.

\subsection{Guarantees for \ISPILPLUS}

We now state a guarantee that improves upon the expert-query bound for \ISPIL.
\begin{theorem}[Sample complexity of \ISPILPLUS] \label{thm:ovi-plus-covering}
    Let \Cref{asp:q-expert-realizability} hold. For every $r > 0$, let $\maxcovering{r}$ be the largest $\ell_\infty$-covering number (see \Cref{def:covering}) of $\cQ_h$ over $h \in \sbr{H}$. Then, for any $\varepsilon \in \spr{0, \min \scbr{1, \QMAX}}$ and $\delta \in \spr{0, 1}$, \ISPILPLUS (\Cref{alg:interactive_finite-H-expert-buffer}) with parameters
    \begin{align*}
        \eta
        &=
        \sqrt{\frac{\log \spr*{A}}{K \QMAX^2}},
        \qquad
        K
        =
        \ceil*{\frac{36 H^2 \QMAX^2 \log \spr*{A}}{\varepsilon^2}},
        \\
        \tauL = \tauE
        &=
        \ceil*{\frac{288 H^2 \QMAX^2 \log \spr*{2 H \spr*{K + 1} \maxcovering{\nicefrac{\varepsilon}{24 H}} / \delta}}{\varepsilon^2}},
    \end{align*}
    outputs a policy $\piout$ such that, with probability at least $1 - \delta$, $J^\expert - J^{\piout} \leq \varepsilon$ using
    \begin{equation*}
        \tcO \spr*{\frac{H^3 \QMAX^2 \log \spr*{2 \maxcovering{\nicefrac{\varepsilon}{24 H}} / \delta}}{\varepsilon^2}}
        \qquad
        \text{expert queries}.
    \end{equation*}
\end{theorem}

For $\varepsilon \geq \QMAX$, any policy already satisfies the performance guarantee, so no expert queries are needed.

Although the expert-query rate improves from $\tcO \spr{\varepsilon^{-4}}$ to $\tcO \spr{\varepsilon^{-2}}$ when comparing to \Cref{thm:ovi-main-covering}, the number of learner rollouts remains of order $\tcO \spr{\varepsilon^{-4}}$. Indeed, the total number of learner rollouts is
\begin{equation*}
    H \spr*{\tauE + K \tauL}.
\end{equation*}
This is of order
\begin{equation*}
    \tcO \spr*{\frac{H^5 \QMAX^4 \log \spr*{A} \log \spr*{2 \maxcovering{\nicefrac{\varepsilon}{24 H}} / \delta}}{\varepsilon^4}}.
\end{equation*}

An avenue for reducing this rollout cost is to recursively estimate changes in the learner term between consecutive policies on the same fresh states. The small changes induced by the softmax updates could allow smaller batches, with periodic full estimates or damping to control accumulated error. This could improve the explicit $\varepsilon^{-4}$ rollout dependence while preserving the $\varepsilon^{-2}$ expert-query dependence and the linear maximization oracle. We leave the development and analysis of this variance-reduced extension to future work.

\begin{proof}[\pfref{thm:ovi-plus-covering}]
    As in the analysis of \ISPIL, we first separate the performance gap into policy regret and estimation error, then control these two terms. For a decision rule $p\colon \cX \to \simplex \spr{\cA}$ and a function $Q_h \in \cQ_h$, define the iteration-dependent empirical objective and recall its population counterpart:
    \begin{align*}
        \hcL_{h, k} \spr*{p, Q_h}
        &\ldef
        \frac{1}{\tauE} \sum_{i = 1}^{\tauE} Q_h \spr*{\Xih, \AEih} - \frac{1}{\tauL} \sum_{j = 1}^{\tauL} Q_h \spr*{\Xjkh, p},
        \\
        \cL_h^{\piout} \spr*{p, Q_h}
        &\ldef
        \sum_{x \in \cX} d_h^{\piout} \spr*{x} \inp*{Q_h \spr*{x, \cdot}, \experth \spr*{\cdot \given x} - p \spr*{\cdot \given x}}.
    \end{align*}
    For each stage $h \in \sbr{H}$, iteration $k \in \sbr{K}$, and decision rule $p$, let
    \begin{equation*}
        \Delta_{h, k} \spr*{p}
        \ldef
        \sup_{Q_h \in \cQ_h} \abs*{\hcL_{h, k} \spr*{p, Q_h} - \cL_h^{\piout} \spr*{p, Q_h}}.
    \end{equation*}
    Unlike in the analysis of \ISPIL, the empirical objective $\hcL_{h, k}$ and the estimation error $\Delta_{h, k}$ depend explicitly on $k$ because a fresh learner batch is drawn at each iteration. Repeating the argument for the regret and estimation error decomposition (\Cref{lem:pifirst}), with $Q_h^k$ best responding to $\hcL_{h, k} \spr{\pi_h^k, \cdot}$ over $\cQ_h$ at each iteration, gives
    \begin{equation} \label{eq:ovi-plus-start}
        J^\expert - J^{\piout}
        \leq
        \frac{1}{K} \sumhH \sum_{x \in \cX} d^{\piout}_h \spr*{x} \regret^\pi_h \spr*{x}
        +
        \frac{2}{K} \sumhH \sumkK \Delta_{h, k} \spr*{\pi_h^k}.
    \end{equation}
    Here, for fixed $h$ and $x$, the policy player's regret is the sum below. The policy update is entropic mirror descent with loss vectors $\ell_k = -Q_h^k \spr{x, \cdot}$ and comparator $\experth \spr{\cdot \given x}$. The mirror descent regret bound (\Cref{lemma:mirror_orabona}) therefore gives
    \begin{equation*}
        \regret^\pi_h \spr*{x}
        =
        \sumkK \inp*{Q^k_h \spr*{x, \cdot}, \experth \spr*{\cdot \given x} - \pi^k_h \spr*{\cdot \given x}}
        \leq
        \frac{\log \abs*{\cA}}{\eta} + \eta K \QMAX^2.
    \end{equation*}
    Indeed, the uniform initial decision rule has negative-entropy Bregman divergence at most $\log \spr{A}$ to the comparator, all iterates have full support, and $\norm{Q_h^k}_\infty \leq \QMAX$. This regret bound holds for every realized sequence of critics, including adaptively chosen ones. Substituting $\eta$ gives $\regret^\pi_h \spr{x} \leq 2 \QMAX \sqrt{K \log \spr{A}}$. Since each $d_h^{\piout}$ is a probability distribution, the choice of $K$ implies
    \begin{equation} \label{eq:ovi-plus-regret-pi}
        \frac{1}{K} \sumhH \sum_{x \in \cX} d^{\piout}_h \spr*{x} \regret^\pi_h \spr*{x}
        \leq
        2 H \QMAX \sqrt{\frac{\log \spr*{A}}{K}}
        \leq
        \frac{\varepsilon}{3}.
    \end{equation}

    \para{Controlling the estimation error}
    It remains to control $\Delta_{h, k} \spr{\pi_h^k}$ simultaneously over $h \in \sbr{H}$ and $k \in \sbr{K}$. By the triangle inequality, we can separate the expert and learner terms:
    \begin{align*}
        \Delta_{h, k} \spr*{\pi_h^k}
        &\leq
        \sup_{Q_h \in \cQ_h}
        \abs*{\frac{1}{\tauE} \sum^{\tauE}_{i = 1} Q_h \spr*{\Xih, \AEih} - \sum_{x \in \cX} d^{\piout}_h \spr*{x} Q_h \spr*{x, \experth}}
        \\
        &\quad+
        \sup_{Q_h \in \cQ_h}
        \abs*{\frac{1}{\tauL} \sum^{\tauL}_{j = 1} Q_h \spr*{\Xjkh, \pi^k_h} - \sum_{x \in \cX} d^{\piout}_h \spr*{x} Q_h \spr*{x, \pi^k_h}}.
    \end{align*}
    Fix a stage $h$. Let $\cF_{h, 0}$ be the history before collecting the expert buffer at that stage. The distribution $d_h^{\piout}$ depends only on the previously learned decision rules $\piout_1, \ldots, \piout_{h - 1}$, so it is $\cF_{h, 0}$-measurable. For $i \in \sbr{\tauE}$, let $\cF_{h, i}$ be the sigma-algebra generated by $\cF_{h, 0}$ and the first $i$ expert-labeled pairs at stage $h$. For any fixed function $Q_h\colon \cX \times \cA \to \sbr{0, \QMAX}$, define
    \begin{equation*}
        Y_i \spr*{Q_h}
        \ldef
        Q_h \spr*{\Xih, \AEih} - \sum_{x \in \cX} d^{\piout}_h \spr*{x} Q_h \spr*{x, \experth}.
    \end{equation*}
    Conditional on the preceding history, the sampling rule for the expert buffer gives
    \begin{equation*}
        \bbE \sbr*{Y_i \spr*{Q_h} \given \cF_{h, i - 1}} = 0,
        \qquad
        \abs*{Y_i \spr*{Q_h}} \leq \QMAX
        \quad \text{almost surely}.
    \end{equation*}
    Thus, the expert increments form a martingale difference sequence with respect to $\spr{\cF_{h, i}}_{i = 0}^{\tauE}$.

    For the learner term, fix $k \in \sbr{K}$ and let $\cG_{h, k, 0}$ be the history immediately before drawing the fresh learner batch at iteration $k$. This history includes the expert buffer, the preceding learner batches, and the already computed decision rule $\pi_h^k$. Both $d_h^{\piout}$ and $\pi_h^k$ are $\cG_{h, k, 0}$-measurable. For $j \in \sbr{\tauL}$, let $\cG_{h, k, j}$ be the sigma-algebra generated by $\cG_{h, k, 0}$ and the first $j$ states of this fresh batch. Define
    \begin{equation*}
        \tilde{Y}_j^k \spr*{Q_h}
        \ldef
        Q_h \spr*{\Xjkh, \pi_h^k} - \sum_{x \in \cX} d^{\piout}_h \spr*{x} Q_h \spr*{x, \pi_h^k}.
    \end{equation*}
    Because the batch is drawn after $\pi_h^k$ is fixed, its conditional sampling rule gives
    \begin{equation*}
        \bbE \sbr*{\tilde{Y}_j^k \spr*{Q_h} \given \cG_{h, k, j - 1}} = 0,
        \qquad
        \abs*{\tilde{Y}_j^k \spr*{Q_h}} \leq \QMAX
        \quad \text{almost surely}.
    \end{equation*}
    The learner increments therefore form a martingale difference sequence with respect to $\spr{\cG_{h, k, j}}_{j = 0}^{\tauL}$. In both cases, the bound is $\QMAX$ because the function evaluation and its conditional mean lie in $\sbr{0, \QMAX}$. Applying the Azuma--Hoeffding inequality (\Cref{lemma:azumahoeffding}) conditional on the corresponding initial history gives, for every $t > 0$,
    \begin{align*}
        \bbP \sbr*{\abs*{\frac{1}{\tauL} \sum_{j = 1}^{\tauL} \tilde{Y}_j^k \spr*{Q_h}} > t \given \cG_{h, k, 0}}
        \leq
        2 e^{-\frac{\tauL t^2}{2 \QMAX^2}},
        \quad\text{and}\quad
        \bbP \sbr*{\abs*{\frac{1}{\tauE} \sum_{i = 1}^{\tauE} Y_i \spr*{Q_h}} > t \given \cF_{h, 0}}
        \leq
        2 e^{-\frac{\tauE t^2}{2 \QMAX^2}}.
    \end{align*}
    Taking expectations gives the same unconditional tail bounds. Set $t = \varepsilon / \spr{12 H}$ and $r_\varepsilon = \varepsilon / \spr{24 H}$. For each $h$, choose a deterministic $r_\varepsilon$-cover $\cC_{h, r_\varepsilon}$ of $\cQ_h$ in the $\ell_\infty$ norm, with $\abs{\cC_{h, r_\varepsilon}} \leq \maxcovering{r_\varepsilon}$ and all centers taking values in $\sbr{0, \QMAX}$. Such centers can be obtained by clipping any minimal cover pointwise to this interval, which preserves the covering radius and cannot increase its cardinality. The fixed-function bounds above therefore apply to every center.

    A union bound over the cover centers, the $H$ expert buffers, and the $H K$ learner batches gives a total failure probability at most
    \begin{equation*}
        2 H \spr*{K + 1} \maxcovering{r_\varepsilon} e^{-\frac{\tauL t^2}{2 \QMAX^2}}
        \leq
        \delta,
    \end{equation*}
    where we used $\tauL = \tauE$ and the sample size in the theorem statement. Hence, on a single event of probability at least $1 - \delta$, the expert and learner deviations are each at most $t$ on their respective covers, simultaneously for all $h$ and $k$. In particular, for every $h \in \sbr{H}$ and $k \in \sbr{K}$,
    \begin{align*}
        \sup_{Q_h \in \cC_{h, r_\varepsilon}}
        \abs*{\frac{1}{\tauE} \sum_{i = 1}^{\tauE} Y_i \spr*{Q_h}} + \sup_{Q_h \in \cC_{h, r_\varepsilon}}
        \abs*{\frac{1}{\tauL} \sum_{j = 1}^{\tauL} \tilde{Y}_j^k \spr*{Q_h}}
        \leq
        2 t = \frac{\varepsilon}{6 H}.
    \end{align*}

    To extend this event from the covers to the full classes, fix $Q_h \in \cQ_h$ and choose a cover representative $\tilde{Q}_h \in \cC_{h, r_\varepsilon}$ with $\norm{Q_h - \tilde{Q}_h}_\infty \leq r_\varepsilon$. Each empirical average and each population expectation changes by at most $r_\varepsilon$ when replacing $Q_h$ by $\tilde{Q}_h$. Consequently,
    \begin{equation*}
        \abs*{\frac{1}{\tauE} \sum_{i = 1}^{\tauE} {Y_i \spr*{Q_h}} - \frac{1}{\tauE} \sum_{i = 1}^{\tauE} {Y_i \spr*{\tilde{Q}_h}}}
        \leq
        2 r_\varepsilon,
    \end{equation*}
    and
    \begin{equation*}
        \abs*{\frac{1}{\tauL} \sum_{j = 1}^{\tauL} \tilde{Y}_j^k \spr*{Q_h} - \frac{1}{\tauL} \sum_{j = 1}^{\tauL} \tilde{Y}_j^k \spr*{\tilde{Q}_h}}
        \leq
        2 r_\varepsilon.
    \end{equation*}
    These approximation bounds hold for every decision rule, including the data-dependent $\pi_h^k$. Combining them with the simultaneous cover bounds and the triangle inequality yields
    \begin{equation} \label{eq:ovi_plus_est_bound}
        \max_{h \in \sbr*{H},\, k \in \sbr*{K}} \Delta_{h, k} \spr*{\pi_h^k}
        \leq
        2 t + 4 r_\varepsilon
        =
        \frac{\varepsilon}{3 H}.
    \end{equation}
    Combining the performance-gap decomposition \eqref{eq:ovi-plus-start}, the policy-regret bound \eqref{eq:ovi-plus-regret-pi}, and the uniform estimation bound \eqref{eq:ovi_plus_est_bound}, we obtain, on the same event of probability at least $1 - \delta$,
    \begin{equation*}
        J^\expert - J^{\piout}
        \leq
        \frac{\varepsilon}{3} + 2 H \max_{h \in \sbr*{H},\, k \in \sbr*{K}} \Delta_{h, k} \spr*{\pi_h^k}
        \leq
        \varepsilon.
    \end{equation*}
    The expert is queried only when collecting the buffers, so the total number of expert queries is $H \tauE$. Since $\varepsilon < \QMAX$, rounding the sample counts changes their order only by constant factors. Substituting the stated sample size gives the claimed expert-query bound, with the additional logarithmic dependence on $H$ and $K$ absorbed into $\tcO$.
\end{proof}

\clearpage
\section{\texorpdfstring{$Q$-\ISPIL: An Inefficient Algorithm with Improved Statistical Guarantees}{Q-ISPIL: An Inefficient Algorithm with Improved Statistical Guarantees}}
\label{sec:q_first_ispil}
In this section, we present an alternative algorithmic scheme that first updates the $Q$-variables via a no-regret algorithm. The resulting method, $Q$-\ISPIL, is computationally inefficient when $\cQ$ is continuous because it requires discretizing $\cQ$, but it enjoys better statistical rates when $\cQ$ is nonconvex. When the $Q$-player moves first, the policy player $\pi^k_h$ can be chosen greedily with respect to $Q^k_h$. For each $\spr{k, h} \in \sbr{K} \times \sbr{H}$, define
\begin{equation*}
    \Xkh
    \sim
    d^{\piout}_h,
    \qquad
    \AEkh
    \sim
    \experth \spr*{\cdot \given \Xkh},
    \qquad
    \wh{\ell}_h^{\piout, k} \spr*{Q}
    =
    Q \spr*{\Xkh, \pi^k_h} - Q \spr*{\Xkh, \AEkh}.
\end{equation*}
To complete the design of $Q$-\ISPIL, we use \Cref{lem:Qfirst} below and note that the sequence $\spr{Q^k_h}^K_{k = 1}$ can be chosen to minimize a regret with respect to losses $\wh{\ell}_h^{\piout, 1}, \ldots, \wh{\ell}_h^{\piout, K}$. In the following, for any $h \in \sbr{H}$, $\cC_{\epsilon} \spr{\cQ_h}$ denotes an $\epsilon$-cover of $\cQ_h$, and we abbreviate $\maxcovering{\epsilon} \ldef \max_{h \in \sbr{H}} \cN_{\epsilon} \spr{\cQ_h}$. Then, for a weight distribution $w^k_h \in \simplex \spr{\cC_{\varepsilon / \spr{4 H}} \spr{\cQ_h}}$ over a covering set $\cC_{\varepsilon / \spr{4 H}} \spr{\cQ_h}$ of $\cQ_h$, we define $Q^k_h$ as the corresponding weighted average. For any state-action pair $\spr{x, a} \in \cX \times \cA$,
\begin{equation*}
    Q^k_h \spr*{x, a}
    =
    \sum_{Q \in \cC_{\varepsilon / \spr*{4 H}} \spr*{\cQ_h}} w^k_h \spr*{Q} Q \spr*{x, a}.
\end{equation*}
The next lemma specifies how to choose the weights.

\begin{algorithm}[!ht]
    \caption{\texttt{Q-}\ISPIL: $Q$-First On-Policy Value-Based Imitation Learning \label{alg:exp_weights_finite_H}}
    \begin{algorithmic}[1]
        \STATE \textbf{input:} Learning rate $\eta$, iterations $K$, finite covering sets $\cC_{\varepsilon / \spr{4 H}} \spr{\cQ_h} \subseteq \cQ_h$ for $h \in \sbr{H}$.
        \STATE For any $h \in \sbr{H}$ and $Q \in \cC_{\varepsilon / \spr{4 H}} \spr{\cQ_h}$, set $w^1_h \spr{Q} = 1 / \cN_{\varepsilon / \spr{4 H}} \spr{\cQ_h}$.
        \STATE For any $h \in \sbr{H}$ and $\spr{x, a} \in \cX \times \cA$, set $Q^1_h \spr{x, a} = \sum_{Q \in \cC_{\varepsilon / \spr{4 H}} \spr{\cQ_h}} w^1_h \spr{Q} Q \spr{x, a}$.
        \STATE For any $h \in \sbr{H}$ and $x \in \cX$, set $\pi^1_h \spr{\cdot \given x} \in \argmax_{p \in \simplex \spr{\cA}} \inp*{p, Q^1_h \spr{x, \cdot}}$.
        \FOR{$h = 1, \dots, H$}
        \FOR{$k = 1, \dots, K$}
        \STATE Sample $\Xkh \sim d^{\piout}_h$, $\AEkh \sim \experth \spr{\cdot \given \Xkh}$.
        \STATE For any $Q \in \cC_{\varepsilon / \spr{4 H}} \spr{\cQ_h}$, set $w^{k+1}_h \spr{Q} \propto w^k_h \spr{Q} e^{\eta \spr{Q \spr{\Xkh, \AEkh} - Q \spr{\Xkh, \pi^k_h}}}$.
        \STATE For any $\spr{x, a} \in \cX \times \cA$, set $Q^{k+1}_h \spr{x, a} = \sum_{Q \in \cC_{\varepsilon / \spr{4 H}} \spr{\cQ_h}} w^{k+1}_h \spr{Q} Q \spr{x, a}$.
        \STATE For any $x \in \cX$, set $\pi^{k+1}_h \spr{\cdot \given x} \in \argmax_{p \in \simplex \spr{\cA}} \inp*{p, Q^{k+1}_h \spr{x, \cdot}} $.
        \ENDFOR
        \STATE Set $\piouth \spr{a \given x} = \frac1K \sumkK \pi^k_h \spr{a \given x}$ for any $\spr{x, a} \in \cX \times \cA$.
        \ENDFOR
    \end{algorithmic}
\end{algorithm}

\begin{restatable}{Lem}{Qfirst} \label{lem:Qfirst}
    Let \Cref{asp:q-expert-realizability} hold, and fix an accuracy parameter $\varepsilon > 0$. For each $\spr{k, h} \in \sbr{K} \times \sbr{H}$, suppose that $w^k_h \in \simplex \spr{\cC_{\varepsilon / \spr{4 H}} \spr{\cQ_h}}$ and that $Q^k_h$ is the corresponding weighted average. For each $h \in \sbr{H}$, define $\pi^k_h \spr{\cdot \given x} \in \argmax_{p \in \simplex \spr{\cA}} \inp{p, Q^k_h \spr{x, \cdot}}$, and set $\piouth \spr{a \given x} = \frac1K \sum^K_{k = 1} \pi^k_h \spr{a \given x}$ for any state-action pair $\spr{x, a} \in \cX \times \cA$. Let $\hQ^\expert_h \in \cC_{\varepsilon / \spr{4 H}} \spr{\cQ_h}$ be a closest element to $Q^\expert_h$ for each $h \in \sbr{H}$. Then, with probability at least $1 - \delta$, we have
    \begin{equation*}
        J^\expert - J^{\piout}
        \leq
        \frac{H}{K} \max_{h \in \sbr*{H}} \regret^w_h + 8 H \QMAX \sqrt{\frac{2 \log \spr*{2 H / \delta}}{K}} + \frac{\varepsilon}{2},
    \end{equation*}
    where the regret of the sequence $\spr{w^k_h}^K_{k = 1}$ is defined by
    \begin{equation*}
        \regret^w_h
        =
        \sumkK \spr*{\sum_{Q \in \cC_{\varepsilon / \spr*{4 H}} \spr*{\cQ_h}} w^k_h \spr*{Q} \wh{\ell}_h^{\piout, k} \spr*{Q} - \wh{\ell}_h^{\piout, k} \spr*{\hQ^\expert_h}}.
    \end{equation*}
\end{restatable}

\Cref{lem:Qfirst} gives the following bound on the performance gap with probability at least $1 - \delta$.
\begin{equation*}
    J^\expert - J^{\piout}
    \leq
    \frac{H}{K} \max_{h \in \sbr*{H}} \regret^w_h + 8 H \QMAX \sqrt{\frac{2 \log \spr*{2 H / \delta}}{K}} + \frac{\varepsilon}{2}.
\end{equation*}
$\regret^w_h$ can be easily controlled by updating $\spr{w^k_h}^K_{k = 1}$ via exponential weights. The resulting algorithm is in \Cref{alg:exp_weights_finite_H}. Therefore, using the decomposition in \Cref{lem:Qfirst} and showing that $\regret^w_h = \cO \spr{\sqrt{K}}$ via a standard online learning bound for the regret of the exponential weights algorithm (see, \eg, \citealp{Cesa-Bianchi:2006,orabona2023modern}), we obtain the following guarantees.

\begin{restatable}{theorem}{Qfirstmain}\label{thm:Q_first_main}
    Let \Cref{asp:q-expert-realizability} hold. For any $\varepsilon, \delta \in \spr{0, 1}$, recall that $\maxcovering{\varepsilon / \spr*{4 H}} \ldef \max_{h \in \sbr*{H}} \cN_{\varepsilon / \spr*{4 H}} \spr*{\cQ_h}$. Run $Q$-\ISPIL (\Cref{alg:exp_weights_finite_H}) with
    \begin{equation*}
        \eta
        =
        \sqrt{\frac{\log \maxcovering{\varepsilon / \spr*{4 H}}}{K \QMAX^2}},
        \qquad
        K
        =
        \cO \spr*{\frac{H^2 \QMAX^2 \log \spr*{\maxcovering{\varepsilon / \spr*{4 H}} H / \delta}}{\varepsilon^2}},
    \end{equation*}
    Then, the algorithm outputs a policy $\piout$ such that, with probability at least $1 - \delta$, we have $J^\expert - J^{\piout} \leq \varepsilon$ after
    \begin{equation*}
        \cO \spr*{\frac{H^3 \QMAX^2 \log \spr*{\maxcovering{\varepsilon / \spr*{4 H}} H / \delta}}{\varepsilon^{2}}} \qquad \text{expert queries.}
    \end{equation*}
\end{restatable}

Despite having strong statistical guarantees, \Cref{alg:exp_weights_finite_H} is not computationally attractive when the class is large or continuous because the construction of the covering set is cumbersome. We compare $Q$-\ISPIL and \ISPIL below.

\para{Comparison of the computational complexity}
As mentioned, \Cref{alg:interactive_finite-H} is better suited for a practical implementation because it does not require discretizing the class $\cQ$. While we think that the linear case could be made efficient, this approach is clearly not scalable when $\cQ$ corresponds to the class of weights for modern architectures. $Q$-\ISPIL would learn a distribution over a set of exponentially many weight configurations. In contrast, \ISPIL is compatible with the common practice of specifying a loss function and backpropagating through it.

\para{Comparison of the statistical complexity}
When the class $\cQ$ is convex, both algorithms achieve a rate of order $\cO \spr{\varepsilon^{-2}}$. A rate difference emerges in the nonconvex case in which $Q$-\ISPIL attains optimal rates in $\varepsilon$ for both expert queries and learner rollouts: \Cref{alg:exp_weights_finite_H} has an expert query complexity of order $\cO \spr{\varepsilon^{-2}}$, while \Cref{alg:interactive_finite-H} has a suboptimal $\cO \spr{\varepsilon^{-4}}$ rate. Moreover, the learner rollouts required by $Q$-\ISPIL is of order $\cO \spr{\varepsilon^{-2}}$ while \ISPIL and \ISPILPLUS require $\cO \spr{\varepsilon^{-4}}$. \loose

To summarize, $Q$-\ISPIL (\Cref{alg:exp_weights_finite_H}) has better statistical properties but is computationally inefficient compared to \ISPIL (\Cref{alg:interactive_finite-H}) because it requires discretizing the class $\cQ$.

\subsection{Proof of \texorpdfstring{\Cref{lem:Qfirst}}{Lemma} (Suboptimality Gap Decomposition for \texorpdfstring{$Q$-\ISPIL}{Q-ISPIL})}

\begin{proof}[\pfref{lem:Qfirst}]
    By the performance difference lemma (\Cref{lem:performance-difference}) and by the definition of $\piouth$, we have
    \begin{align*}
        J^\expert - J^{\piout}
        &=
        \sumhH \cL_h^{\piout} \spr*{\piouth, Q^\expert_h} \\
        &=
        \frac1K \sumhH \sumkK \cL_h^{\piout} \spr*{\pi^k_h, Q^\expert_h} \\
        &=
        \frac1K \sumhH \sumkK \cL_h^{\piout} \spr*{\pi^k_h, Q^k_h} + \frac1K \sumhH \sumkK \cL_h^{\piout} \spr*{\pi^k_h, Q^\expert_h - Q^k_h}.
    \end{align*}
    Since $\pi^k_h \spr{\cdot \given x} \in \argmax_{p \in \simplex \spr{\cA}} \inp{p, Q^k_h \spr{x, \cdot}}$, we have $\cL_h^{\piout} \spr{\pi^k_h, Q^k_h} \leq 0$. Hence,
    \begin{equation*}
        J^\expert - J^{\piout}
        \leq
        \frac1K \sumhH \sumkK \cL_h^{\piout} \spr*{\pi^k_h, Q^\expert_h - Q^k_h}.
    \end{equation*}
    Let $r > 0$ and let $\hQ^\expert_h \in \cC_r \spr{\cQ_h}$ be such that $\norm{\hQ^\expert_h - Q^\expert_h}_\infty \leq r$ for each $h \in \sbr{H}$. For any decision rule $p$ and functions $Q_h, Q'_h$,
    \begin{equation*}
        \abs*{\cL_h^{\piout} \spr*{p, Q_h - Q'_h}}
        \leq
        \norm*{Q_h - Q'_h}_\infty
        \sum_{x \in \cX} d_h^{\piout} \spr*{x} \norm*{\experth \spr*{\cdot \given x} - p \spr*{\cdot \given x}}_1
        \leq
        2 \norm*{Q_h - Q'_h}_\infty.
    \end{equation*}
    Thus, replacing $Q^\expert_h$ by $\hQ^\expert_h$ costs at most $2 r$ for each pair $\spr{k, h}$, hence at most $2 H r$ after averaging over $k$ and summing over $h$. Replacing $Q^\expert_h$ by $\hQ^\expert_h$ in the previous display gives
    \begin{equation} \label{eq:first}
        J^\expert - J^{\piout}
        \leq
        \frac1K \sumhH \sumkK \cL_h^{\piout} \spr*{\pi^k_h, \hQ^\expert_h - Q^k_h} + 2 H r.
    \end{equation}
    For each $\spr{k, h} \in \sbr{K} \times \sbr{H}$, define the sampled loss
    \begin{equation*}
        \wh{\ell}_h^{\piout, k} \spr*{Q}
        =
        Q \spr*{\Xkh, \pi^k_h} - Q \spr*{\Xkh, \AEkh},
        \qquad
        \Xkh \sim d^{\piout}_h,
        \qquad
        \AEkh \sim \experth \spr*{\cdot \given \Xkh}.
    \end{equation*}
    Then,
    \begin{equation*}
        \bbE \sbr*{\wh{\ell}_h^{\piout, k} \spr*{Q^k_h} - \wh{\ell}_h^{\piout, k} \spr*{\hQ^\expert_h} \given \cF_{k - 1, h}}
        =
        \cL_h^{\piout} \spr*{\pi^k_h, \hQ^\expert_h - Q^k_h},
    \end{equation*}
    where $\cF_{k - 1, h}$ contains all randomness from previous layers and from the first $k - 1$ samples at layer $h$. Conditionally on $\cF_{k - 1, h}$, the distribution $d_h^{\piout}$ is fixed because it only depends on the already computed policies $\piout_1, \ldots, \piout_{h - 1}$, and $Q^k_h$ and $\pi^k_h$ are also fixed because they are functions of the previous samples at layer $h$. Therefore, the increments
    \begin{equation*}
        M^k_h
        =
        \cL_h^{\piout} \spr*{\pi^k_h, \hQ^\expert_h - Q^k_h} - \spr*{\wh{\ell}_h^{\piout, k} \spr*{Q^k_h} - \wh{\ell}_h^{\piout, k} \spr*{\hQ^\expert_h}}
    \end{equation*}
    form a martingale difference sequence in $k$, for each fixed $h \in \sbr{H}$. Moreover, since $\hQ^\expert_h$ belongs to the cover and $Q^k_h$ is a convex combination of cover elements, both functions are bounded in sup-norm by $\QMAX$. Thus, $\norm{\hQ^\expert_h - Q^k_h}_\infty \leq 2 \QMAX$, and both $\abs{\cL_h^{\piout} \spr{\pi^k_h, \hQ^\expert_h - Q^k_h}}$ and $\abs{\wh{\ell}_h^{\piout, k} \spr{Q^k_h} - \wh{\ell}_h^{\piout, k} \spr{\hQ^\expert_h}}$ are at most $2 \norm{\hQ^\expert_h - Q^k_h}_\infty \leq 4 \QMAX$. Hence $\abs{M^k_h} \leq 8 \QMAX$ almost surely. By the Azuma-Hoeffding inequality and a union bound over $h \in \sbr{H}$, with probability at least $1 - \delta$, for all $h \in \sbr{H}$,
    \begin{equation*}
        \frac1K \sumkK M^k_h
        \leq
        8 \QMAX \sqrt{\frac{2 \log \spr*{2 H / \delta}}{K}}.
    \end{equation*}
    Summing over $h \in \sbr{H}$ incurs an extra $H$ factor in the upper bound. Finally, using $Q^k_h = \sum_{Q \in \cC_{r} \spr{\cQ_h}} w^k_h \spr{Q} Q$ and taking $r = \varepsilon / \spr{4 H}$, we obtain from \Cref{eq:first}
    \begin{align*}
        J^\expert - J^{\piout}
        &\leq
        \frac1K \sumhH \sumkK \spr*{\sum_{Q \in \cC_{r} \spr*{\cQ_h}} w^k_h \spr*{Q} \wh{\ell}_h^{\piout, k} \spr*{Q} - \wh{\ell}_h^{\piout, k} \spr*{\hQ^\expert_h}} + 8 H \QMAX \sqrt{\frac{2 \log \spr*{2 H / \delta}}{K}} + 2 H r \\
        &=
        \frac1K \sumhH \regret^w_h + 8 H \QMAX \sqrt{\frac{2 \log \spr*{2 H / \delta}}{K}} + \frac{\varepsilon}{2} \\
        &\leq
        \frac{H}{K} \max_{h \in \sbr*{H}} \regret^w_h + 8 H \QMAX \sqrt{\frac{2 \log \spr*{2 H / \delta}}{K}} + \frac{\varepsilon}{2},
    \end{align*}
    which gives the stated form.
\end{proof}

\subsection{Proof of \texorpdfstring{\Cref{thm:Q_first_main}}{Theorem} (Sample Complexity Guarantee for \texorpdfstring{$Q$-\ISPIL}{Q-ISPIL})}

\begin{proof}[\pfref{thm:Q_first_main}]
    From the decomposition proven in \Cref{lem:Qfirst}, with probability at least $1 - \delta$, we have
    \begin{equation*}
        J^\expert - J^{\piout}
        \leq
        \frac{H}{K} \max_{h \in \sbr*{H}} \regret^w_h + 8 H \QMAX \sqrt{\frac{2 \log \spr*{2 H / \delta}}{K}} + \frac{\varepsilon}{2}.
    \end{equation*}
    Thus, it remains to control the regret $\regret^w_h$ for each $h \in \sbr{H}$. The update in \Cref{alg:exp_weights_finite_H} can be written for any $Q \in \cC_{\varepsilon / \spr{4 H}}\spr{\cQ_h}$ as
    \begin{equation*}
        w^{k + 1}_h \spr*{Q}
        \propto
        w^k_h \spr*{Q} e^{\eta \spr*{Q \spr*{\Xkh, \AEkh} - Q \spr*{\Xkh, \pi^k_h}}}.
    \end{equation*}
    Equivalently, this is exponential weights with sampled losses $\wh{\ell}_h^{\piout, k}$. Moreover, by the definition of the function class, each $Q_h \in \cQ_h$ satisfies $\norm{Q_h}_\infty \leq \QMAX$, so $\abs{\wh{\ell}_h^{\piout, k} \spr{Q_h}} \leq 2 \QMAX$ for all $Q_h \in \cQ_h$. If $\maxcovering{\varepsilon / \spr{4 H}} = 1$, every stagewise cover is a singleton, so the prescribed learning rate is $\eta = 0$ but $w^k_h$ assigns unit mass to $\hQ^\expert_h$ for every $h$ and $k$. Hence, $\regret^w_h = 0$ for all $h$, and the regret bound below holds directly without invoking \Cref{lemma:mirror_orabona}. We may therefore suppose that $\maxcovering{\varepsilon / \spr{4 H}} \geq 2$, so $\eta > 0$. For any stage with a singleton cover, the same zero-regret conclusion holds. Fix a stage $h \in \sbr{H}$ with a non-singleton cover and apply \Cref{lemma:mirror_orabona} to the simplex $V = \simplex \spr{\cC_{\varepsilon / \spr{4 H}} \spr{\cQ_h}}$ with loss vectors $\ell_k \spr{Q} = \wh{\ell}_h^{\piout, k} \spr{Q}$. The update above is mirror descent with the negative entropy regularizer, whose Bregman divergence $D$ is the KL divergence. We use as comparator $\hQ^\expert_h \in \cC_{\varepsilon / \spr{4 H}} \spr{\cQ_h}$. Since $w^1_h$ is uniform over the cover, $D \spr{\mb{e}_{\hQ^\expert_h}, w^1_h} \leq \log \cN_{\varepsilon / \spr{4 H}} \spr{\cQ_h}$, where $\mb{e}_{\hQ^\expert_h}$ is the unit vector at $\hQ^\expert_h$. Thus, using the bounds $\norm{\ell_k}_\infty \leq 2 \QMAX$ and $\lambda \geq 1 / 2$, \Cref{lemma:mirror_orabona} gives
    \begin{equation*}
        \regret^w_h
        \leq
        \frac{\log \cN_{\varepsilon / \spr*{4 H}} \spr*{\cQ_h}}{\eta} + 4 \eta K \QMAX^2.
    \end{equation*}
    Using $\regret^w_h = 0$ for singleton stagewise covers and substituting $\eta = \sqrt{\log \maxcovering{\varepsilon / \spr{4 H}} / \spr{K \QMAX^2}}$ in the preceding bound for non-singleton covers, we obtain
    \begin{equation*}
        \max_{h \in \sbr*{H}} \regret^w_h
        \leq
        5 \sqrt{K \QMAX^2 \log \maxcovering{\varepsilon / \spr*{4 H}}}.
    \end{equation*}
    Combining the decomposition above with the regret bound gives
    \begin{align*}
        J^\expert - J^{\piout}
        &\leq
        5 H \sqrt{\frac{\QMAX^2 \log \maxcovering{\varepsilon / \spr*{4 H}}}{K}} + 8 H \QMAX \sqrt{\frac{2 \log \spr*{2 H / \delta}}{K}} + \frac{\varepsilon}{2} \\
        &\leq
        \spr*{5 + 8 \sqrt{2}} H \QMAX \sqrt{\frac{\log \spr*{\maxcovering{\varepsilon / \spr*{4 H}} 2 H / \delta}}{K}} + \frac{\varepsilon}{2}.
    \end{align*}
    Choosing $K$ so that
    \begin{equation*}
        K
        \geq
        4 \spr*{5 + 8 \sqrt{2}}^2
        \frac{H^2 \QMAX^2 \log \spr*{\maxcovering{\varepsilon / \spr*{4 H}} 2 H / \delta}}{\varepsilon^2}
    \end{equation*}
    makes the first term in the display above at most $\varepsilon / 2$, and therefore $J^\expert - J^{\piout} \leq \varepsilon$ with probability at least $1 - \delta$. Since \Cref{alg:exp_weights_finite_H} makes one expert query for each pair $\spr{k, h} \in \sbr{K} \times \sbr{H}$, the total number of expert queries is
    \begin{equation*}
        H K
        =
        \cO \spr*{H^3 \QMAX^2 \log \spr*{\maxcovering{\varepsilon / \spr*{4 H}} H / \delta} \varepsilon^{-2}}.
    \end{equation*}
\end{proof}

\subsection{Learning a Stationary Policy with \texorpdfstring{$Q$-\ISPIL}{Q-ISPIL}}
\label{sec:infinite_horizon}

In this section, we show that $Q$-\ISPIL can be used to learn a stationary policy that competes with the expert in the discounted infinite-horizon setting. We formalize the setting next.

\para{Infinite-horizon MDPs} We consider a discounted MDP $\cM = \spr{\cX, \cA, r, P, \gamma, \initial}$, where $\cX$ and $\cA$ are finite state and action spaces, $r\colon \cX \times \cA \to \sbr{0, 1}$ is the reward function, $P\colon \cX \times \cA \to \simplex \spr{\cX}$ is the transition kernel, $\gamma \in [0, 1)$ is the discount factor, and $\initial \in \simplex \spr{\cX}$ is the initial distribution. A policy is a mapping $\pi\colon \cX \to \simplex \spr{\cA}$. Rolling out $\pi$ generates a trajectory $\spr{\mb{x}_h, \mb{a}_h}^\infty_{h = 0}$ by drawing $\mb{x}_0 \sim \initial$, $\mb{a}_h \sim \pi \spr{\cdot \given \mb{x}_h}$, and $\mb{x}_{h + 1} \sim P \spr{\cdot \given \mb{x}_h, \mb{a}_h}$ for all $h \geq 0$. For any state-action pair $\spr{x, a} \in \cX \times \cA$, the discounted state-action value function is
\begin{equation*}
    Q^\pi_\gamma \spr*{x, a}
    =
    \bbE^\pi \sbr*{\sum^\infty_{h = 0} \gamma^h r \spr*{\mb{x}_h, \mb{a}_h} \given \mb{x}_0 = x, \mb{a}_0 = a}.
\end{equation*}
For any function $Q\colon \cX \times \cA \to \bbR$ and policy $\pi$, we use $Q \spr{x, \pi}$ to denote $\bbE_{\mb{a} \sim \pi \spr{\cdot \given x}} \sbr{Q \spr{x, \mb{a}}}$. The discounted state occupancy measure of $\pi$ is denoted by $d^\pi_\gamma \in \simplex \spr{\cX}$ and defined by
\begin{equation*}
    d^\pi_\gamma \spr*{x}
    =
    \spr*{1 - \gamma} \sum^\infty_{h = 0} \gamma^h \bbP^\pi \sbr*{\mb{x}_h = x}.
\end{equation*}
Finally, we denote the normalized expected return of $\pi$ by
\begin{equation*}
    J^\pi_\gamma
    =
    \spr*{1 - \gamma} \bbE^\pi \sbr*{\sum^\infty_{h = 0} \gamma^h r \spr*{\mb{x}_h, \mb{a}_h}}.
\end{equation*}

\para{Infinite-horizon $Q$-\ISPIL and guarantees}
For simplicity, we consider finite classes $\cQ$ in this section. The result can be extended to infinite classes via covering numbers, as in the finite-horizon case. We study the following infinite-horizon version of $Q$-\ISPIL.
\begin{algorithm}[!ht]
    \caption{Infinite-horizon $Q$-\ISPIL \label{alg:exp_weights}}
    \begin{algorithmic}[1]
        \STATE \textbf{input:} Learning rate $\eta$, iterations $K$, finite class $\cQ$.
        \STATE Set $w_1 \spr{Q} = 1 / \abs{\cQ}$ for all $Q \in \cQ$.
        \STATE For any $\spr{x, a} \in \cX \times \cA$, set $Q_1 \spr{x, a} = \sum_{Q \in \cQ} w_1 \spr{Q} Q \spr{x, a}$.
        \STATE For any $x \in \cX$, set $\pi_1 \spr{\cdot \given x} \in \argmax_{p \in \simplex \spr{\cA}} \inp*{p, Q_1 \spr{x, \cdot}}$.
        \FOR{$k = 1, \dots, K$}
        \STATE Sample $\mb{x}_k \sim d^{\pi_k}_\gamma$, $\mb{a}^\experttag_k \sim \expert \spr{\cdot \given \mb{x}_k}$.
        \STATE For any $Q \in \cQ$, set $w_{k + 1} \spr{Q} \propto w_k \spr{Q} e^{\eta \spr{Q \spr{\mb{x}_k, \mb{a}^\experttag_k} - Q \spr{\mb{x}_k, \pi_k}}}$.
        \STATE For any $\spr{x, a} \in \cX \times \cA$, set $Q_{k + 1} \spr{x, a} = \sum_{Q \in \cQ} w_{k + 1} \spr{Q} Q \spr{x, a}$.
        \STATE For any $x \in \cX$, set $\pi_{k + 1} \spr{\cdot \given x} \in \argmax_{p \in \simplex \spr{\cA}} \inp*{p, Q_{k + 1} \spr{x, \cdot}}$.
        \ENDFOR
        \STATE Draw $I \sim \mathrm{Unif} \spr{\sbr{K}}$ and output $\piout = \pi_I$.
    \end{algorithmic}
\end{algorithm}
For simplicity, we assume direct access to samples from $d^{\pi_k}_\gamma$. Under ordinary trajectory access, each sample can instead be generated by a geometric rollout, yielding expected environment-interaction complexity $\cO \spr*{K / \spr{1 - \gamma}}$, an additional factor of $1 / \spr{1 - \gamma}$. The algorithm still makes one expert query per round.

The resulting guarantee is as follows.
\begin{theorem}\label{thm:Qfirst_infinite}
    Assume $Q^\expert_\gamma \in \cQ$, $\abs{\cQ}\geq2$ and $\norm{Q}_\infty \leq \spr{1 - \gamma}^{-1}$ for all $Q \in \cQ$. For any $\varepsilon \in \spr{0, 1}$, run \Cref{alg:exp_weights} with
    \begin{equation*}
        \eta
        =
        \spr*{1 - \gamma} \sqrt{\frac{\log \abs{\cQ}}{K}},
        \qquad
        K
        =
        \cO \spr*{\frac{\log \abs{\cQ}}{\spr*{1 - \gamma}^2 \varepsilon^2}}.
    \end{equation*}
    Then, the algorithm outputs a policy $\piout$ such that
    \begin{equation*}
        \bbE \sbr*{J^\expert_\gamma - J^{\piout}_\gamma}
        \leq
        \varepsilon.
    \end{equation*}
\end{theorem}

\begin{proof}[\pfref{thm:Qfirst_infinite}]
    Since $\piout = \pi_I$ with $I \sim \mathrm{Unif} \spr{\sbr{K}}$, it is enough to control the average gap.\loose
    \begin{equation*}
        \bbE \sbr*{J^\expert_\gamma - J^{\piout}_\gamma}
        =
        \frac1K \sumkK \bbE \sbr*{J^\expert_\gamma - J^{\pi_k}_\gamma}.
    \end{equation*}
    For any fixed $k \in \sbr{K}$, the performance difference lemma in discounted MDPs \citep[Lemma~1]{moulin2025inverse} gives
    \begin{align*}
        J^\expert_\gamma - J^{\pi_k}_\gamma
        &=
        \sum_{x \in \cX} d^{\pi_k}_\gamma \spr*{x} \inp*{Q^\expert_\gamma \spr*{x, \cdot}, \expert \spr*{\cdot \given x} - \pi_k \spr*{\cdot \given x}} \\
        &=
        \sum_{x \in \cX} d^{\pi_k}_\gamma \spr*{x} \inp*{Q_k \spr*{x, \cdot}, \expert \spr*{\cdot \given x} - \pi_k \spr*{\cdot \given x}} \\
        &\quad+
        \sum_{x \in \cX} d^{\pi_k}_\gamma \spr*{x} \inp*{Q^\expert_\gamma \spr*{x, \cdot} - Q_k \spr*{x, \cdot}, \expert \spr*{\cdot \given x} - \pi_k \spr*{\cdot \given x}}.
    \end{align*}
    Because $\pi_k \spr{\cdot \given x}$ is greedy with respect to $Q_k \spr{x, \cdot}$, the first term is nonpositive. Therefore,
    \begin{equation*}
        J^\expert_\gamma - J^{\pi_k}_\gamma
        \leq
        \sum_{x \in \cX} d^{\pi_k}_\gamma \spr*{x} \inp*{Q^\expert_\gamma \spr*{x, \cdot} - Q_k \spr*{x, \cdot}, \expert \spr*{\cdot \given x} - \pi_k \spr*{\cdot \given x}}.
    \end{equation*}
    Define the sampled gains
    \begin{equation*}
        g_k \spr*{Q}
        =
        Q \spr*{\mb{x}_k, \mb{a}^\experttag_k} - Q \spr*{\mb{x}_k, \pi_k},
        \qquad
        \mb{x}_k \sim d^{\pi_k}_\gamma,
        \qquad
        \mb{a}^\experttag_k \sim \expert \spr*{\cdot \given \mb{x}_k}.
    \end{equation*}
    Let $\cF_{k - 1}$ be the history before drawing $\spr{\mb{x}_k, \mb{a}^\experttag_k}$, that is, the $\sigma$-field generated by the previous samples $\spr{\mb{x}_j, \mb{a}^\experttag_j}_{j < k}$ and any additional algorithmic randomness up to round $k - 1$. Although the distribution $d^{\pi_k}_\gamma$ changes with $k$, conditional on $\cF_{k - 1}$, the weight vector $w_k$, and hence $Q_k$, $\pi_k$, and $d^{\pi_k}_\gamma$, are fixed. The fresh sample is then drawn conditionally as $\mb{x}_k \sim d^{\pi_k}_\gamma$ and $\mb{a}^\experttag_k \sim \expert \spr{\cdot \given \mb{x}_k}$. Hence,
    \begin{equation*}
        \bbE \sbr*{g_k \spr*{Q^\expert_\gamma} - g_k \spr*{Q_k} \given \cF_{k - 1}}
        =
        \sum_{x \in \cX} d^{\pi_k}_\gamma \spr*{x} \inp*{Q^\expert_\gamma \spr*{x, \cdot} - Q_k \spr*{x, \cdot}, \expert \spr*{\cdot \given x} - \pi_k \spr*{\cdot \given x}}.
    \end{equation*}
    Taking expectations and summing over $k$ gives
    \begin{equation*}
        \sumkK \bbE \sbr*{J^\expert_\gamma - J^{\pi_k}_\gamma}
        \leq
        \bbE \sbr*{\sumkK \spr*{g_k \spr*{Q^\expert_\gamma} - g_k \spr*{Q_k}}}.
    \end{equation*}
    Let $w^\star = \mb{e}_{Q^\expert_\gamma}$ be the unit vector at $Q^\expert_\gamma$. Since $Q_k \spr{x, a} = \sum_{Q \in \cQ} w_k \spr{Q} Q \spr{x, a}$, the previous display becomes
    \begin{equation*}
        \sumkK \bbE \sbr*{J^\expert_\gamma - J^{\pi_k}_\gamma}
        \leq
        \bbE \sbr*{\sumkK \sum_{Q \in \cQ} \spr*{w^\star \spr*{Q} - w_k \spr*{Q}} g_k \spr*{Q}}.
    \end{equation*}
    The exponential weights update is mirror descent over the simplex $V = \simplex \spr{\cQ}$ with loss vectors $\ell_k \spr{Q} = - g_k \spr{Q}$. The regularizer is negative entropy, whose Bregman divergence $D$ is the KL divergence. We use the comparator $w^\star$, and since $w_1$ is uniform over $\cQ$, we have $D \spr{w^\star, w_1} \leq \log \abs{\cQ}$. Moreover, $\norm{\ell_k}_\infty \leq 2 \spr{1 - \gamma}^{-1}$, and the negative entropy is $1$-strongly convex in the $\ell_1$-norm. Thus, \Cref{lemma:mirror_orabona} yields, pathwise,
    \begin{equation*}
        \sumkK \sum_{Q \in \cQ} \spr*{w^\star \spr*{Q} - w_k \spr*{Q}} g_k \spr*{Q}
        \leq
        \frac{\log \abs{\cQ}}{\eta} + \frac{2 \eta K}{\spr*{1 - \gamma}^2}.
    \end{equation*}
    Assuming $\abs{\cQ} \ge 2$, we can set $\eta = \spr{1 - \gamma} \sqrt{\log \abs{\cQ} / K}$ to obtain
    \begin{equation*}
        \bbE \sbr*{J^\expert_\gamma - J^{\piout}_\gamma}
        \leq
        3 \sqrt{\frac{\log \abs*{\cQ}}{K \spr*{1 - \gamma}^2}}.
    \end{equation*}
    Choosing $K \geq \frac{9 \log \abs{\cQ}}{\spr{1 - \gamma}^2 \varepsilon^2}$ makes the right-hand side at most $\varepsilon$.
\end{proof}


\clearpage
\section{Breaking the Offline Lower Bound with Coverage}
\label{sec:break_lb_with_coverage}
This section shows that the offline lower bound does not apply when the expert is sufficiently exploratory, as measured by a coverage coefficient introduced in \Cref{asp:linf-coverage,asp:l1-coverage} below. We prove this through a generalized analysis of \ISPIL.\loose

\begin{assumption}[$L_\infty$-coverage] \label{asp:linf-coverage}
    Assume an algorithm $\mathrm{Alg}$ outputs a VI policy $\piout$ in a class $\Pi_{\mathrm{Alg}}$ such that, for some finite constant $C_\infty$, any stage $h$, and any state $x$, $\sup_{\pi \in \Pi_{\mathrm{Alg}}} d^\pi_h \spr{x} / d^\expert_h \spr{x} \leq C_\infty$.
\end{assumption}
\begin{assumption}[$L_1$-coverage] \label{asp:l1-coverage}
    Assume an algorithm $\mathrm{Alg}$ that outputs a VI policy $\piout$ in a class $\Pi_{\mathrm{Alg}}$ such that, for some finite constant $C_1$,
    \begin{equation*}
        \sup_{\pi \in \Pi_{\mathrm{Alg}}} \sumhH \sum_{x \in \cX} \frac{\spr*{d_h^\pi \spr*{x}}^2}{d_h^\expert \spr*{x}}
        \leq
        C_1.
    \end{equation*}
\end{assumption}
For both cases, we use the convention $0/0=0$.
We note that \Cref{asp:linf-coverage} implies \Cref{asp:l1-coverage} with $C_1 \leq H C_\infty$, but we keep the two assumptions separate because \Cref{thm:SPOIL_tighter} tracks both dependencies.
Under this setting our main theorem reads as follows.
\begin{theorem}\label{thm:SPOIL}
    Let $\varepsilon, \delta \in (0,1)$ and \Cref{asp:q-expert-realizability} hold. Assume the expert policy $\expert$ is optimal and that the algorithm \SPOIL \citep{moulin2025inverse} satisfies \Cref{asp:linf-coverage}. For every radius $r > 0$, write $\maxcovering{r} \ldef \max_{h \in \sbr{H}} \cN_r \spr{\cQ_h, \norm{\cdot}_\infty}$, and choose a radius $r_{\varepsilon} = \cO \spr{\varepsilon^2}$ small enough. Then, \SPOIL returns an $\varepsilon$-optimal policy with probability at least $1-\delta$ using $\tcO \spr{C^4_\infty H^5 \QMAX^4 \log \spr{A}\log\spr{ \maxcovering{r_{\varepsilon}} \delta^{-1}} \varepsilon^{-4}}$ precollected expert data. Moreover, if the class $\cQ$ is convex the bound improves to $\tcO \spr{C^2_\infty H^3 \QMAX^2 \log \spr{\maxcovering{r_{\varepsilon}} \delta^{-1}} \varepsilon^{-2}}$.
\end{theorem}

\Cref{thm:SPOIL} follows from the stronger \Cref{thm:SPOIL_tighter} proven below, which also gives a bound under \Cref{asp:l1-coverage}.

\subsection{\ISPIL with Arbitrary Sampling Distributions}

We first prove a generalization of \Cref{alg:interactive_finite-H} in which the state sampling distribution at stage $h$ may differ from $d^{\piout}_h$. This result will be used several times below.
\begin{algorithm}[!h]
    \caption{\ISPIL with arbitrary state sampling distribution $d$ \label{alg:arbitrary_nu}}
    \begin{algorithmic}[1]
        \STATE \textbf{input:} Sampling distributions $d = \scbr{d_h}^H_{h = 1}$, learning-rate schedule $\scbr{\eta_k}^{K + 1}_{k = 1}$, iterations $K$, number of expert queries per stage $\tauE$.
        \FOR{$h = 1, \dots, H$}
        \STATE Create the stage-$h$ dataset: for $i \in \sbr{\tauE}$, sample $\Xih \sim d_h$ and query $\AEih \sim \experth \spr{\cdot \given \Xih}$.
        \STATE Initialize $\pi^1_h = \mathrm{Unif} \spr{\cA}$.
        \FOR{$k = 1, \dots, K$}
        \STATE Set $Q^{k}_h \in \argmax_{Q_h \in \cQ_h} \sum^{\tauE}_{i = 1} \spr{Q_h \spr{\Xih, \AEih} - Q_h \spr{\Xih, \pi^{k}_h}}$.
        \STATE Set $\pi^{k + 1}_h \spr{a \given x} \propto \pi^1_h \spr{a \given x} \exp \spr{\eta_{k + 1} \sum^k_{k' = 1} Q^{k'}_h \spr{x, a}}$.
        \ENDFOR
        \STATE Set the output policy at stage $h$ to $\piouth \spr{a \given x} = \frac1K \sumkK \pi^k_h \spr{a \given x}$.
        \ENDFOR
    \end{algorithmic}
\end{algorithm}
\begin{theorem}[Sample complexity for \ISPIL with arbitrary sampling distributions] \label{thm:arbitrary_nu}
    For each $h \in \sbr{H}$, let $\cF_{h - 1}$ denote the history available before the stage-$h$ samples are drawn.
    Fix any sequence of state sampling distributions $d = \scbr{d_h}^H_{h = 1}$, where each $d_h$ may be random but is $\cF_{h - 1}$-measurable, and run \Cref{alg:arbitrary_nu} with $K, \tauE \in \bbN$ and the FTRL learning-rate schedule
    \begin{equation*}
        \eta_k = \sqrt{\frac{\log A}{k \QMAX^2}}
        \qquad
        \text{for every } k \geq 1.
    \end{equation*}
    For each $h \in \sbr{H}$, assume that, conditionally on $\cF_{h - 1}$, the pairs $\spr{\Xih, \AEih}_{i = 1}^{\tauE}$ are independent, with $\Xih \sim d_h$ and $\AEih \sim \experth \spr{\cdot \given \Xih}$.
    For every radius $r > 0$, write $\maxcovering{r} \ldef \max_{h \in \sbr{H}} \cN_r \spr{\cQ_h, \norm{\cdot}_\infty}$. Let $\varepsilon, \delta \in (0,1)$ and choose a radius $r_{\varepsilon} = \cO \spr{\varepsilon^2}$ small enough.
    Fix any comparator policy $\pi'$ such that $Q^{\pi'}_h \in \cQ_h$ for every $h \in \sbr{H}$. For the policy $\piout$ returned by \Cref{alg:arbitrary_nu}, we recall the definition
    \begin{equation*}
        \cL^{d}_h \spr*{\piouth, Q_h^{\pi'}}
        =
        \sum_{x \in \cX} d_h \spr*{x} \inp*{Q_h^{\pi'} \spr*{x, \cdot}, \experth \spr*{\cdot \given x} - \piouth \spr*{\cdot \given x}}.
    \end{equation*}
    If $\cQ$ is convex, then with probability at least $1 - \delta$,
    \begin{equation*}
        \sumhH \cL^d_h \spr*{\piouth, Q_h^{\pi'}}
        \leq
        \tcO \spr*{\sqrt{\frac{H^2 \QMAX^2 \log A}{K}} + \sqrt{\frac{H^2 \QMAX^2 \log \spr*{\maxcovering{r_{\varepsilon}} / \delta}}{\tauE}} + \varepsilon}.
    \end{equation*}
    If $\cQ$ is not assumed to be convex, then with probability at least $1 - \delta$,
    \begin{equation*}
        \sumhH \cL^{d}_h \spr*{\piouth, Q_h^{\pi'}}
        \leq
        \tcO \spr*{\sqrt{\frac{H^2 \QMAX^2 \log A}{K}} + \sqrt{\frac{K H^2 \QMAX^2 \log \spr*{\maxcovering{r_{\varepsilon}} / \delta}}{\tauE}} + \varepsilon}.
    \end{equation*}
    In particular, up to logarithmic factors, the convex bound is of order $\varepsilon$ by taking $K \gtrsim H^2 \QMAX^2 \log \spr{A} / \varepsilon^2$ and $\tauE \gtrsim H^2 \QMAX^2 \log \spr{\maxcovering{r_{\varepsilon}} / \delta} / \varepsilon^2$. In the nonconvex case, the same choice of $K$ and $\tauE \gtrsim K H^2 \QMAX^2 \log \spr{\maxcovering{r_{\varepsilon}} / \delta} / \varepsilon^2$ suffices.
\end{theorem}

\begin{proof}[\pfref{thm:arbitrary_nu}]
    Recall the definitions of the empirical objective and the uniform estimation error
    \begin{align*}
        \hcL^{d}_h \spr*{\pi_h, Q_h}
        &\ldef
        \tauE^{-1} \sum^{\tauE}_{i = 1} \spr*{Q_h \spr*{\Xih, \AEih} - Q_h \spr*{\Xih, \pi_h}},\\
        \Delta_{d} \spr*{\pi}
        &\ldef
        \max_{h \in \sbr*{H}}\sup_{Q_h \in \cQ_h} \abs*{\hcL^{d}_h \spr*{\pi_h, Q_h} - \cL^{d}_h \spr*{\pi_h, Q_h}}.
    \end{align*}
    Fix $h \in \sbr{H}$. Since $Q_h^{\pi'} \in \cQ_h$, the first inequality below follows from the definition of $\Delta_d \spr{\pi^k}$. The second inequality follows from the empirical maximization property defining $Q_h^k$, and the last inequality uses the definition of $\Delta_d \spr{\pi^k}$ again, now with $Q_h^k$:
    \begin{align*}
        \sumkK \cL^{d}_h \spr*{\pi^k_h, Q_h^{\pi'}}
        &\leq
        \sumkK \hcL^{d}_h \spr*{\pi^k_h, Q_h^{\pi'}} + \sumkK \Delta_{d} \spr*{\pi^k} \\
        &\leq
        \sumkK \hcL^{d}_h \spr*{\pi^k_h, Q^k_h} + \sumkK \Delta_{d} \spr*{\pi^k} \\
        &\leq
        \sumkK \cL^{d}_h \spr*{\pi^k_h, Q^k_h} + 2 \sumkK \Delta_{d} \spr*{\pi^k}.
    \end{align*}
    We now control the first term on the right-hand side using the FTRL regret bound. Fix $x \in \cX$ and define the loss vectors $\ell_{k, h, x} \spr{a} \ldef - Q_h^k \spr{x, a}$. Since $\pi_h^1$ is uniform, the update in \Cref{alg:arbitrary_nu} implies that, for every $k \geq 1$,
    \begin{equation*}
        \pi_h^k \spr{a \given x}
        =
        \frac{\exp \spr*{\eta_k \sum_{j < k} Q_h^j \spr{x, a}}}{\sum_{a' \in \cA} \exp \spr*{\eta_k \sum_{j < k} Q_h^j \spr{x, a'}}}
        =
        \frac{\exp \spr*{-\eta_k \sum_{j < k} \ell_{j, h, x} \spr{a}}}{\sum_{a' \in \cA} \exp \spr*{-\eta_k \sum_{j < k} \ell_{j, h, x} \spr{a'}}}.
    \end{equation*}
    Thus, for this fixed $h$ and $x$, the sequence $\spr{\pi_h^k \spr{\cdot \given x}}^K_{k = 1}$ is exactly the FTRL sequence of \Cref{lemma:FTRL} with comparator $\experth \spr{\cdot \given x}$. Moreover, $\norm{\ell_{k, h, x}}_\infty \leq \QMAX$ and $\eta_k = \sqrt{\log A / \spr{k \QMAX^2}}$. Applying \Cref{lemma:FTRL} pointwise in $x$ and then averaging over $d_h$, for any $K \in \bbN$, we obtain
    \begin{equation*}
        \sumkK \cL^{d}_h \spr*{\pi^k_h, Q^k_h}
        =
        \sum_{x \in \cX} d_h \spr*{x} \sumkK \inp*{Q^k_h \spr*{x, \cdot}, \experth \spr*{\cdot \given x} - \pi^k_h \spr*{\cdot \given x}}
        \leq
        3 \sqrt{\QMAX^2 K \log A},
    \end{equation*}
    where the last inequality uses that $d_h$ is a probability distribution. Compared with the proof of \Cref{thm:ovi-main-covering}, note that we use FTRL rather than online mirror descent.\footnote{This change gives anytime bounds that hold for every $K$, which is needed because we invoke \Cref{thm:arbitrary_nu} with different values of $K$.}

    It remains to control the estimation error. Since the learning rate varies with $k$, define the policy class
    \begin{align*}
        \Pi_\cQ
        \ldef
        \scbr*{\pi: \exists m \in \scbr*{0, \ldots, K}, \forall h, \exists Q_h^1, \ldots, Q_h^m \in \cQ_h, ~~~\pi_h \spr*{\cdot \given x} = \softmax \spr[\Big]{\eta_{m+1} {\textstyle\sum_{j = 1}^m} Q_h^j \spr*{x, \cdot}}}.
    \end{align*}
    For each $h$, let $\Pi_{\cQ,h} \ldef \scbr{\pi_h : \pi \in \Pi_\cQ}$. The iterates generated by \Cref{alg:arbitrary_nu} satisfy $\pi^k \in \Pi_\cQ$ for every $k \in \sbr{K}$, by taking $m = k - 1$.

    For every $h$, condition on $\cF_{h - 1}$. Then $d_h$ is fixed, and the pairs $\spr{\Xih, \AEih}_{i = 1}^{\tauE}$ are independent, with $\Xih \sim d_h$ and $\AEih \sim \experth \spr{\cdot \given \Xih}$. Applying the uniform concentration bound in \Cref{lem:moulin-fh-concentration} with $\Pi_h = \Pi_{\cQ,h}$ gives that, for any radius $r > 0$, with probability at least $1 - \delta$,
    \begin{equation*}
        \sumkK \Delta_{d} \spr*{\pi^k}
        \leq
        K \sup_{\pi \in \Pi_{\cQ}} \Delta_{d} \spr*{\pi}
        \leq
        K \spr*{
            4 r
            +
            \QMAX
            \sqrt{
                \frac{
                    8 \log \spr*{
                        2 \sumhH
                        \cN_r \spr*{\cQ_h \times \Pi_{\cQ,h}, \rho}
                        / \delta
                    }
                }{\tauE}
            }
        }.
    \end{equation*}
    Therefore, dividing by $K$ and using $\cL^{d}_h \spr{\piouth, Q^{\pi'}_h} = \frac1K \sumkK \cL^{d}_h \spr{\pi^k_h, Q^{\pi'}_h}$, we obtain
    \begin{equation*}
        \sumhH \cL^{d}_h \spr*{\piouth, Q^{\pi'}_h}
        \leq
        3 \sqrt{\frac{H^2 \QMAX^2 \log A}{K}}
        +
        8 H r
        +
        2 H \QMAX
        \sqrt{
            \frac{
                8 \log \spr*{
                    2 \sumhH
                    \cN_r \spr*{\cQ_h \times \Pi_{\cQ,h}, \rho}
                    / \delta
                }
            }{\tauE}
        }.
    \end{equation*}
    Next, we bound the covering number of the product class. Choose $r = \varepsilon / \spr{8 H}$. The product-cover lemmas in \Cref{lem:moulin-convex-product-cover} and \Cref{lem:moulin-nonconvex-product-cover} are stated for a fixed learning rate $\eta$, but their proofs apply almost verbatim to the present class. For policies represented by $m$ functions, the factor $\eta m$ is replaced by $\eta_{m + 1} m$. Thus, they naturally lead to a covering number with radius
    \begin{equation*}
        \frac{\varepsilon}{16 H \max \scbr*{1, \QMAX \max_{1 \leq m \leq K} m \eta_{m + 1}}}.
    \end{equation*}
    Increasing $K$ by an absolute constant if necessary, we may assume $K \log A \geq 1$. Then the learning-rate schedule gives
    \begin{equation*}
        \QMAX \max_{1 \leq m \leq K} m \eta_{m + 1}
        =
        K \sqrt{\frac{\log A}{K + 1}}
        \leq
        \sqrt{K \log A},
        \qquad
        \max \scbr*{1, \QMAX \max_{1 \leq m \leq K} m \eta_{m + 1}}
        \leq
        \sqrt{K \log A}.
    \end{equation*}
    Thus, for the parameter choices below, it is enough to use the smaller radius
    \begin{equation*}
        r_\varepsilon
        \ldef
        \frac{\varepsilon}{16 H \sqrt{K \log A}},
    \end{equation*}
    which only enlarges the covering number. After the choices of $K$ below, this radius satisfies $r_\varepsilon = \cO \spr{\varepsilon^2}$.

    \para{Convex case} If $\cQ$ is convex, then each projection $\cQ_h$ is convex. Hence, by \Cref{lem:moulin-convex-product-cover}, for every $h \in \sbr{H}$,
    \begin{equation*}
        \cN_r \spr*{\cQ_h \times \Pi_{\cQ,h}, \rho}
        \leq
        \spr*{K + 1}
        \cN_{r_\varepsilon} \spr*{\cQ_h, \norm*{\cdot}_\infty}^{2}.
    \end{equation*}
    Plugging this into the preceding estimation bound gives
    \begin{equation*}
        \sumhH \cL^{d}_h \spr*{\piouth, Q^{\pi'}_h}
        \leq
        3 \sqrt{\frac{H^2 \QMAX^2 \log A}{K}}
        +
        \varepsilon
        +
        2 H \QMAX
        \sqrt{
            \frac{
                8 \log \spr*{2 H \spr*{K + 1} \spr*{\maxcovering{r_\varepsilon}}^2 / \delta}
            }{\tauE}
        }.
    \end{equation*}
    Thus, the convex bound becomes
    \begin{equation*}
        \tcO \spr*{
            \sqrt{\frac{H^2 \QMAX^2 \log A}{K}}
            +
            \sqrt{\frac{H^2 \QMAX^2 \log \spr*{\maxcovering{r_\varepsilon} / \delta}}{\tauE}}
            +
            \varepsilon
        }.
    \end{equation*}
    In particular, choosing
    \begin{equation*}
        K
        \geq
        \frac{9 H^2 \QMAX^2 \log A}{\varepsilon^2},
        \qquad
        \tauE
        \geq
        \frac{32 H^2 \QMAX^2}{\varepsilon^2}
        \log \spr*{\frac{2 H \spr*{K + 1} \spr*{\maxcovering{r_\varepsilon}}^2}{\delta}}
    \end{equation*}
    makes the three terms in the previous display of order $\varepsilon$.

    \para{Nonconvex case}
    For arbitrary, not necessarily convex, classes $\cQ$, \cref{lem:moulin-nonconvex-product-cover} instead gives, for every $h \in \sbr{H}$,
    \begin{equation*}
        \cN_r \spr*{\cQ_h \times \Pi_{\cQ,h}, \rho}
        \leq
        \spr*{K + 1}
        \cN_{r_\varepsilon} \spr*{\cQ_h, \norm*{\cdot}_\infty}^{K + 1}.
    \end{equation*}
    Therefore,
    \begin{equation*}
        \sumhH \cL^{d}_h \spr*{\piouth, Q^{\pi'}_h}
        \leq
        3 \sqrt{\frac{H^2 \QMAX^2 \log A}{K}}
        +
        \varepsilon
        +
        2 H \QMAX
        \sqrt{
            \frac{
                8 \log \spr*{2 H \spr*{K + 1} \spr*{\maxcovering{r_\varepsilon}}^{K + 1} / \delta}
            }{\tauE}
        },
    \end{equation*}
    or, equivalently,
    \begin{equation*}
        \tcO \spr*{
            \sqrt{\frac{H^2 \QMAX^2 \log A}{K}}
            +
            \sqrt{\frac{K H^2 \QMAX^2 \log \spr*{\maxcovering{r_\varepsilon} / \delta}}{\tauE}}
            +
            \varepsilon
        }.
    \end{equation*}
    With the same choice of $K$ as above, taking
    \begin{equation*}
        \tauE
        \geq
        \frac{32 H^2 \QMAX^2}{\varepsilon^2}
        \log \spr*{\frac{2 H \spr*{K + 1} \spr*{\maxcovering{r_\varepsilon}}^{K + 1}}{\delta}}
        =
        \tcO \spr*{
            \frac{K H^2 \QMAX^2 \log \spr*{\maxcovering{r_\varepsilon} / \delta}}{\varepsilon^2}
        }
    \end{equation*}
    makes the three terms of the nonconvex bound of order $\varepsilon$. This concludes the proof.
\end{proof}

\subsection{Proof of \texorpdfstring{\Cref{thm:SPOIL}}{Theorem} (Offline IL with Coverage and Optimal Expert)}
\label{sec:offline_with_coverage}

We prove \Cref{thm:SPOIL} through the stronger \Cref{thm:SPOIL_tighter}. The sharper statement also gives a guarantee under \Cref{asp:l1-coverage} alone in terms of the $L_1$-coverage coefficient $C_1$, which can be much smaller than $H C_{\infty}$.
\begin{theorem}\label{thm:SPOIL_tighter}
    Let $\varepsilon, \delta \in (0,1)$. Let \Cref{asp:q-expert-realizability} hold. Assume the expert policy $\expert$ is optimal and that \SPOIL \citep{moulin2025inverse} satisfies at least one of \Cref{asp:l1-coverage,asp:linf-coverage}. Run \SPOIL (\ie, \Cref{alg:arbitrary_nu} with $d = d^\expert$) with the FTRL learning-rate schedule
    \begin{equation*}
        \eta_k
        =
        \sqrt{\frac{\log A}{k \QMAX^2}}
        \qquad
        \text{for every } k \geq 1.
    \end{equation*}
    For every radius $r > 0$, write $\maxcovering{r} \ldef \max_{h \in \sbr{H}} \cN_r \spr{\cQ_h, \norm{\cdot}_\infty}$. Choose radii $r_{\varepsilon}^{(1)} = \cO \spr{\varepsilon^4}$ and $r_{\varepsilon}^{(\infty)} = \cO \spr{\varepsilon^2}$ small enough.
    If $\cQ$ is arbitrary, the $C_1$-dependent choice below guarantees $J^\expert - J^{\piout} \leq \varepsilon$ with probability at least $1 - \delta$ under \Cref{asp:l1-coverage}, while the $C_\infty$-dependent choice guarantees the same under \Cref{asp:linf-coverage}:
    \begin{align*}
        K^{(1)}
        &=
        \tcO \spr*{\frac{C^2_1 H^2 \QMAX^4 \log A}{\varepsilon^4}},
        &
        \tauE^{(1)}
        &=
        \tcO \spr*{\frac{C^4_1 H^4 \QMAX^8 \log \spr*{A } \log\spr*{\maxcovering{r_{\varepsilon}^{(1)}} / \delta}}{\varepsilon^8}},
        \\
        K^{(\infty)}
        &=
        \tcO \spr*{\frac{C^2_\infty H^2 \QMAX^2 \log A}{\varepsilon^2}},
        &
        \tauE^{(\infty)}
        &=
        \tcO \spr*{\frac{C^4_\infty H^4 \QMAX^4 \log \spr*{A } \log\spr*{\maxcovering{r_{\varepsilon}^{(\infty)}} / \delta}}{\varepsilon^4}}.
    \end{align*}
    If both coverage assumptions hold, choosing the better of the two branches gives total expert sample complexity
    \begin{equation*}
        H \tauE
        =
        \tcO \spr*{
            \min \scbr*{
                \frac{C^4_1 H^5 \QMAX^8 \log \spr*{A } \log\spr*{\maxcovering{r_{\varepsilon}^{(1)}} / \delta}}{\varepsilon^8},
                \frac{C^4_\infty H^5 \QMAX^4 \log \spr*{A } \log\spr*{\maxcovering{r_{\varepsilon}^{(\infty)}} / \delta}}{\varepsilon^4}
            }
        }.
    \end{equation*}
    If $\cQ$ is convex, the same branch-wise guarantees hold under their respective coverage assumptions, with the same choices of $K^{(1)}$ and $K^{(\infty)}$, but with
    \begin{align*}
        \tauE^{(1)}
        &=
        \tcO \spr*{\frac{C^2_1 H^2 \QMAX^4 \log \spr*{\maxcovering{r_{\varepsilon}^{(1)}} / \delta}}{\varepsilon^4}},
        &
        \tauE^{(\infty)}
        &=
        \tcO \spr*{\frac{C^2_\infty H^2 \QMAX^2 \log \spr*{\maxcovering{r_{\varepsilon}^{(\infty)}} / \delta}}{\varepsilon^2}}.
    \end{align*}
    If both coverage assumptions hold, choosing the better of the two branches gives total expert sample complexity
    \begin{equation*}
        H \tauE
        =
        \tcO \spr*{
            \min \scbr*{
                \frac{C^2_1 H^3 \QMAX^4 \log \spr*{\maxcovering{r_{\varepsilon}^{(1)}} / \delta}}{\varepsilon^4},
                \frac{C^2_\infty H^3 \QMAX^2 \log \spr*{\maxcovering{r_{\varepsilon}^{(\infty)}} / \delta}}{\varepsilon^2}
            }
        }.
    \end{equation*}
    Here, the $\tcO$ notation in each branch hides logarithmic factors in $H$, $\QMAX$, and the corresponding coverage coefficient.
\end{theorem}

\begin{proof}[\pfref{thm:SPOIL_tighter}]
    Since the expert is optimal, $\expert$ is greedy with respect to $Q^\expert$ at every stage $h$ and state $x \in \mathrm{supp}(d^{\expert}_h)$. Hence, for any policy $\pi$, stage $h$, and state $x \in \mathrm{supp}(d^{\expert}_h)$,
    \begin{equation*}
        \inp*{Q_h^\expert \spr*{x, \cdot}, \experth \spr*{\cdot \given x} - \pi_h \spr*{\cdot \given x}}
        \geq
        0.
    \end{equation*}
    This nonnegativity is the only point at which expert optimality is used. It lets us remove absolute values in the change-of-measure arguments below. We prove two independent bounds: one under \Cref{asp:l1-coverage}, in terms of $C_1$, and one under \Cref{asp:linf-coverage}, in terms of $C_\infty$. When both assumptions hold, the guarantee takes the better of the two.

    \para{Bounds scaling with $C_1$}
    Let
    \begin{equation*}
        g_h \spr*{x}
        \ldef
        \inp*{Q_h^\expert \spr*{x, \cdot}, \experth \spr*{\cdot \given x} - \piouth \spr*{\cdot \given x}}.
    \end{equation*}
    Finiteness of $C_1$ implies that, for every $h$ and $x$, $d_h^{\piout} \spr{x} = 0$ whenever $d_h^\expert \spr{x} = 0$, so the following change of measure is well defined.
    The expert value function satisfies $Q_h^\expert \spr{x, a} \in \sbr{0, \QMAX}$ for every $h, x, a$. Together with the nonnegativity above, this gives $0 \leq g_h \spr{x} \leq \QMAX$, and hence $g_h \spr{x}^2 \leq \QMAX g_h \spr{x}$ for every $x \in \mathrm{supp}(d^{\expert}_h)$. By the performance difference lemma (\Cref{lem:performance-difference}), the Cauchy--Schwarz inequality, and \Cref{asp:l1-coverage}, we have
    \begin{align*}
        J^\expert - J^{\piout}
        &=
        \sumhH \sum_{x \in \cX} d_h^{\piout} \spr*{x} g_h \spr*{x} \\
        &\leq
        \sqrt{
            \sumhH \sum_{x \in \cX} \frac{\spr*{d_h^{\piout} \spr*{x}}^2}{d_h^\expert \spr*{x}}
        }
        \sqrt{
            \sumhH \sum_{x \in \cX} d_h^\expert \spr*{x} g_h \spr*{x}^2
        } \\
        &\leq
        \sqrt{
            C_1 \QMAX
            \sumhH \sum_{x \in \cX} d_h^\expert \spr*{x} g_h \spr*{x}
        }.
    \end{align*}
    The last summation is exactly the loss controlled by \Cref{thm:arbitrary_nu} with $d = d^\expert$ and $\pi' = \expert$:
    \begin{equation*}
        \sumhH \sum_{x \in \cX} d_h^\expert \spr*{x} g_h \spr*{x}
        =
        \sumhH \cL^\expert_h \spr*{\piouth, Q^\expert_h}.
    \end{equation*}
    For nonconvex $\cQ$, applying \Cref{thm:arbitrary_nu} with accuracy parameter $\varepsilon_0$ and associated covering radius $r_{\varepsilon_0} = \cO \spr{\varepsilon_0^2}$ gives, with probability at least $1 - \delta$,
    \begin{equation*}
        J^\expert - J^{\piout}
        \leq
        \tcO \spr*{
            \sqrt{
                C_1 \QMAX
                \spr*{
                    \sqrt{\frac{H^2 \QMAX^2 \log A}{K}}
                    +
                    \sqrt{\frac{K H^2 \QMAX^2}{\tauE} \log \spr*{\maxcovering{r_{\varepsilon_0}} / \delta}}
                    +
                    \varepsilon_0
                }
            }
        }.
    \end{equation*}
    Taking $\varepsilon_0 = \varepsilon^2 / \spr{C_1 \QMAX}$ and choosing $r_{\varepsilon}^{(1)} \leq r_{\varepsilon_0}$, the monotonicity of covering numbers and the absorption of logarithmic factors in $H$, $\QMAX$, and $C_1$ into $\tcO$ yield
    \begin{equation*}
        J^\expert - J^{\piout}
        \leq
        \tcO \spr*{
            \sqrt[4]{\frac{C^2_1 H^2 \QMAX^4 \log A}{K}}
            +
            \sqrt[4]{\frac{K C^2_1 H^2 \QMAX^4 \log \spr*{\maxcovering{r_{\varepsilon}^{(1)}} / \delta}}{\tauE}}
            +
            \varepsilon
        }.
    \end{equation*}
    Thus the $C_1$-tuned nonconvex choice in the theorem makes the right-hand side at most $\cO \spr{\varepsilon}$. If $\cQ$ is convex, the second term in \Cref{thm:arbitrary_nu} has no factor $K$. Repeating the same calculation gives
    \begin{equation*}
        J^\expert - J^{\piout}
        \leq
        \tcO \spr*{
            \sqrt[4]{\frac{C^2_1 H^2 \QMAX^4 \log A}{K}}
            +
            \sqrt[4]{\frac{C^2_1 H^2 \QMAX^4 \log \spr*{\maxcovering{r_{\varepsilon}^{(1)}} / \delta}}{\tauE}}
            +
            \varepsilon
        },
    \end{equation*}
    which gives the improved convex choice of $\tauE^{(1)}$.

    \para{Bounds scaling with $C_\infty$}
    The $C_\infty$ argument gives a better dependence on $1 / \varepsilon$ at the cost of replacing $C_1$ by the stronger coverage coefficient $C_\infty$. Since $g_h \spr{x} \geq 0$, \Cref{asp:linf-coverage} gives
    \begin{align*}
        J^\expert - J^{\piout}
        &=
        \sumhH \sum_{x \in \cX} d_h^{\piout} \spr*{x} g_h \spr*{x}
        \\
        &\leq
        \spr*{\max_{y \in \cX} \max_{h \in \sbr*{H}} \frac{d_h^{\piout} \spr*{y}}{d^\expert_h \spr*{y}}} \sumhH \sum_{x \in \cX} d^\expert_h \spr*{x}  \abs*{g_h \spr*{x}}
        \\
        &\leq
        C_\infty \sumhH \sum_{x \in \cX} d_h^\expert \spr*{x} g_h \spr*{x}
        =
        C_\infty \sumhH \cL^\expert_h \spr*{\piouth, Q^\expert_h}.
    \end{align*}
    Applying \Cref{thm:arbitrary_nu} with $d = d^\expert$, $\pi' = \expert$, and accuracy parameter $\varepsilon_0 = \varepsilon / C_\infty$, and choosing $r_{\varepsilon}^{(\infty)} \leq r_{\varepsilon_0}$ for the associated radius $r_{\varepsilon_0} = \cO \spr{\varepsilon_0^2}$, the monotonicity of covering numbers gives for nonconvex $\cQ$ that
    \begin{equation*}
        J^\expert - J^{\piout}
        \leq
        \tcO \spr*{
            \sqrt{\frac{C^2_\infty H^2 \QMAX^2 \log A}{K}}
            +
            \sqrt{\frac{K C^2_\infty H^2 \QMAX^2 \log \spr*{\maxcovering{r_{\varepsilon}^{(\infty)}} / \delta}}{\tauE}}
            +
            \varepsilon
        }.
    \end{equation*}
    This is at most $\cO \spr{\varepsilon}$ for the $C_\infty$-tuned nonconvex choice of $K^{(\infty)}$ and $\tauE^{(\infty)}$ in the theorem. If $\cQ$ is convex, the same calculation with the convex version of \Cref{thm:arbitrary_nu} gives
    \begin{equation*}
        J^\expert - J^{\piout}
        \leq
        \tcO \spr*{
            \sqrt{\frac{C^2_\infty H^2 \QMAX^2 \log A}{K}}
            +
            \sqrt{\frac{C^2_\infty H^2 \QMAX^2 \log \spr*{\maxcovering{r_{\varepsilon}^{(\infty)}} / \delta}}{\tauE}}
            +
            \varepsilon
        },
    \end{equation*}
    which yields the stated convex choice of $\tauE^{(\infty)}$. Finally, \SPOIL uses $\tauE$ expert samples at each of the $H$ stages, so the total number of precollected expert samples is $H \tauE$. Multiplying the branch-wise bounds on $\tauE$ by $H$ gives the two sample-complexity displays in the theorem.
\end{proof}


\clearpage
\section{On the Benefits of Mixing Expert and Learner Trajectories}
\label{sec:adaptive}
\label{sec:adaptive_algo}
In this section, we analyze the popular technique of mixing expert and learner data to create the empirical objective in LM distillation \citep{agarwal2024policy,li2026revisiting} and show that it leads to representational benefits. We develop an algorithm that is statistically efficient when the class $\cQ$ satisfies either $\qexpert$-realizability or the $Q^{\Pi_{\cQ}}$-realizability condition introduced by \citet{moulin2025inverse}, \emph{without knowing a priori which representational condition holds}.\loose

To obtain the anytime regret bounds needed in this section, we use a slightly different version of the $Q^{\Pi_{\cQ}}$-realizability condition. Compared with \citet{moulin2025inverse}, our version uses FTRL rather than OMD and allows the softmax weights to depend on a fixed learning-rate schedule $\eta_{1:K + 1}$.
\begin{assumption}[$Q^{\Pi_{\cQ}}$-realizability] \label{ass:q-picq-realizability}
    Fix $K = \mathrm{poly} \spr{H, \log A, \QMAX, \varepsilon^{-1}}$ and learning rates $\eta_{1:K + 1} \in \bbR^{K + 1}$, and define
    \begin{equation*}
        \Pi_{\cQ}
        \ldef
        \scbr*{\pi : \forall h \in \sbr*{H}, \exists m_h \leq K, \exists \spr*{Q^h_j}^{m_h}_{j = 1} \subset \cQ_h, ~~~~\pi_h \spr*{\cdot \given x} = \softmax \spr[\Big]{\eta_{m_h + 1} {\textstyle\sum_{j = 1}^{m_h}} Q^h_j \spr*{x, \cdot}}}.
    \end{equation*}
    The class $\cQ$ satisfies $Q^{\Pi_{\cQ}}$-realizability if $Q^\pi \in \cQ$ for every $\pi \in \Pi_{\cQ}$.
\end{assumption}
For each stage $h$, write $\Pi_{\cQ, h} \ldef \scbr*{\pi_h : \pi \in \Pi_{\cQ}}$. This notation is always understood with the same value of $K$ and the same learning-rate schedule $\eta_{1:K + 1}$ as the algorithm under consideration. This is a slightly broader policy class than the version with a single common value of $m$ shared by all stages. We use it because the hybrid policies in the proof may take different stages from different FTRL rounds, so the number of accumulated $Q$-functions can depend on $h$.

This assumption requires realizing the $Q$-functions of many policies, whereas $\qexpert$-realizability only requires realizing the expert's $Q$-function. Since the expert need not belong to $\Pi_{\cQ}$, $Q^{\Pi_{\cQ}}$-realizability does not imply $\qexpert$-realizability. Nevertheless, the upper bound of \citet{moulin2025inverse} implies that, under $Q^{\Pi_{\cQ}}$-realizability, there exists a policy in $\Pi_{\cQ}$ that is $\varepsilon$-optimal with respect to $\expert$.

As shown by \citet{moulin2025inverse}, offline IL is possible under $Q^{\Pi_{\cQ}}$-realizability, and therefore interactive IL is possible as well. However, \Cref{alg:exp_weights_finite_H,alg:interactive_finite-H} do not directly use this condition. We instead modify the state-sampling rule in \Cref{alg:interactive_finite-H}: for each $h \in \sbr{H}$, we sample states from a mixture of the expert occupancy $d^\expert_h$ and the learner occupancy $d^{\piout}_h$. Sampling from the former is possible in the interactive setting by rolling in with actions queried from the expert. Thus, \Cref{alg:adaptive} is an instance of \Cref{alg:arbitrary_nu} with sampling distribution
\begin{equation*}
    d_h
    =
    \spr*{1 - \alpha} d^{\piout}_h + \alpha d^\expert_h,
\end{equation*}
for an appropriate choice of $\alpha \in \sbr{0, 1}$. The mixture is the key adaptive device: the $d^{\piout}$ component gives coverage of the learner states needed under $\qexpert$-realizability, while the $d^\expert$ component keeps enough expert-state mass for the $Q^{\Pi_{\cQ}}$ argument.
\begin{algorithm}[!ht]
    \caption{\texttt{RAOVI:} Representation-Adaptive \ISPIL \label{alg:adaptive}}
    \begin{algorithmic}[1]
        \STATE \textbf{input:} Mixture parameter $\alpha \in \sbr{0, 1}$, learning rates $\spr{\eta_k}^{K + 1}_{k = 1}$, iterations $K$, dataset size per stage $\tauE$.
        \FOR{$h = 1, \dots, H$}
        \STATE Create the stage-$h$ dataset: for $i \in \sbr{\tauE}$, sample
        \begin{equation*}
            \Xih
            \sim
            \textcolor{black!10!orange}{\spr{1 - \alpha}d^{\piout}_h + \alpha d^\expert_h},
            \qquad
            \AEih
            \sim
            \experth \spr{\cdot \given \Xih}.
        \end{equation*}
        \STATE Initialize $\pi^1_h = \mathrm{Unif} \spr{\cA}$.
        \FOR{$k = 1, \dots, K$}
        \STATE Set $Q^{k}_h \in \argmax_{Q_h \in \cQ_h} \sum^{\tauE}_{i = 1} \spr{Q_h \spr{\Xih, \AEih} - Q_h \spr{\Xih, \pi^{k}_h}}$.
        \STATE Set $\pi^{k + 1}_h \spr{a \given x} \propto \pi^{1}_h \spr{a \given x} \exp \spr{\eta_{k + 1} \sum^k_{k' = 1} Q^{k'}_h \spr{x, a}}$.
        \ENDFOR
        \STATE Set $\piouth \spr{a \given x} = \frac1K \sumkK \pi^k_h \spr{a \given x}$.
        \ENDFOR
    \end{algorithmic}
\end{algorithm}
The change relative to \Cref{alg:arbitrary_nu} is highlighted in \textcolor{black!10!orange}{orange}. We have the following result.
\begin{theorem} \label{thm:adaptive-hybrid}
    Let the expert policy $\expert$ be optimal, $\varepsilon \in \spr*{0, \min \scbr*{1, 2 H \QMAX}}$, and $\delta \in \spr{0, 1}$. Run \Cref{alg:adaptive} with
    \begin{equation*}
        \alpha
        =
        1 - \frac{\varepsilon}{2 H \QMAX},
        \qquad
        \eta_k
        =
        \sqrt{\frac{\log A}{k \QMAX^2}}
        \qquad
        \text{for every } k \geq 1.
    \end{equation*}
    Then, with probability at least $1 - \delta$, each of the following guarantees holds under the corresponding condition.
    \begin{enumerate}
        \item Under $\qexpert$-realizability (\Cref{asp:q-expert-realizability}), $J^\expert - J^{\piout} \leq \varepsilon$ is guaranteed by the choices
        \begin{equation*}
            K
            =
            \tcO \spr*{\frac{H^4 \QMAX^4 \log A}{\varepsilon^4}},
            \qquad
            \tauE
            =
            \tcO \spr*{\frac{H^4 \QMAX^4 \log \spr*{2 H \maxcovering[\cQ \times \Pi_{\cQ}, \rho]{\varepsilon^2 / \spr*{8 H^2 \QMAX}} / \delta}}{\varepsilon^4}}.
        \end{equation*}

        \item Under $Q^{\Pi_\cQ}$-realizability (\Cref{ass:q-picq-realizability}), $J^\expert - J^{\piout} \leq \varepsilon$ is guaranteed by the choices
        \begin{equation*}
            K
            =
            \tcO \spr*{\frac{H^2 \QMAX^2 \log A}{\varepsilon^2}},
            \qquad
            \tauE
            =
            \tcO \spr*{\frac{H^2 \QMAX^2 \log \spr*{2 H \maxcovering[\cQ \times \Pi_{\cQ}, \rho]{\varepsilon / H} / \delta}}{\varepsilon^2}}.
        \end{equation*}
    \end{enumerate}
    Here, $\maxcovering[\cQ \times \Pi_{\cQ}, \rho]{r} \ldef \max_{h \in \sbr{H}} \cN_r \spr{\cQ_h \times \Pi_{\cQ, h}, \rho}$. The total number of expert queries is at most $H^2 \tauE$.
    Thus, without knowing which condition holds, one can choose $K$ and $\tauE$ to be the componentwise maximum of the two displayed requirements, and the corresponding guarantee applies under either condition.
\end{theorem}

We remark that the rates in the $\qexpert$-realizable case are not optimal. It is an interesting open question whether they can be improved while maintaining the faster rates in the $Q^{\Pi_{\cQ}}$-realizable case.

\begin{proof}[\pfref{thm:adaptive-hybrid}]
    We analyze each case separately. Throughout, let
    \begin{equation*}
        d_h
        =
        \alpha d^\expert_h + \spr*{1 - \alpha} d^{\piout}_h.
    \end{equation*}
    As in the previous cases, $d_h$ is measurable with respect to the history available before sampling at stage $h$, and conditionally on this history, the pairs $\spr{\Xih, \AEih}_{i = 1}^{\tauE}$ are independent with $\Xih \sim d_h$ and $\AEih \sim \experth \spr{\cdot \given \Xih}$, as required by \Cref{lem:moulin-fh-concentration}.
    Obtaining one labeled sample at stage $h$ requires at most $h$ expert queries, including the queries used to roll in under the expert. Summing over the $\tauE$ samples at each stage gives at most $\tauE \sum_{h = 1}^H h \leq H^2 \tauE$ expert queries.
    
    \para{Case 1: $\qexpert$-realizability}
    For every $h \in \sbr{H}$ and $x \in \cX$, define $g_h \spr{x} \ldef \inp{\qexpert_h \spr{x, \cdot}, \experth \spr{\cdot \given x} - \piouth \spr{\cdot \given x}}$. A one-step deviation argument using expert optimality gives $g_h \spr{x} \geq 0$ for $d_h^\expert$-almost every $x$. Therefore, the definition of $d_h$ and the performance difference lemma (\Cref{lem:performance-difference}) give
    \begin{align*}
        \sumhH \bbE_{x \sim d_h} \sbr*{g_h \spr*{x}}
        &=
        \spr*{1 - \alpha} \spr*{J^\expert - J^{\piout}}
        +
        \alpha \sumhH \bbE_{x \sim d_h^\expert} \sbr*{g_h \spr*{x}}
        \\
        &\geq
        \spr*{1 - \alpha} \spr*{J^\expert - J^{\piout}}.
    \end{align*}
    Consequently,
    \begin{align*}
        J^\expert - J^{\piout}
        &\leq
        \frac{1}{1 - \alpha} \sumhH \sum_{x \in \cX} d_h \spr*{x} \inp*{\qexpert_h \spr*{x, \cdot}, \experth \spr*{\cdot \given x} - \piouth \spr*{\cdot \given x}}
        \\
        &=
        \frac{1}{1 - \alpha} \sumhH \cL_h^d \spr*{\piouth, \qexpert_h}
        \\
        &\leq
        \frac{1}{1 - \alpha} \spr*{3 \sqrt{\frac{H^2 \QMAX^2 \log A}{K}} + 8 H r + 2 H \QMAX \sqrt{\frac{8 \log \spr*{2 \sumhH \cN_r \spr*{\cQ_h \times \Pi_{\cQ, h}, \rho} / \delta}}{\tauE}}},
    \end{align*}
    where the last inequality holds with probability at least $1 - \delta$ by rederiving the proof of \Cref{thm:arbitrary_nu} with sampling distribution $d_h = \alpha d^\expert_h + \spr{1 - \alpha} d^{\piout}_h$ and comparator $\pi' = \expert$, stopping at the intermediate product-cover bound, for any $r > 0$. Taking $r = \varepsilon^2 / \spr{8 H^2 \QMAX}$ gives $8 H r = \varepsilon^2 / \spr{H \QMAX}$ and covering radius $\varepsilon^2 / \spr{8 H^2 \QMAX}$ in the logarithm. With $\alpha = 1 - \varepsilon / \spr{2 H \QMAX}$, the choices of $K$ and $\tauE$ in the $\qexpert$-realizable branch make the right-hand side at most $\cO \spr{\varepsilon}$.\loose

    \para{Case 2: $Q^{\Pi_{\cQ}}$-realizability}
    The subtle point is that we should not apply the performance difference lemma directly with $Q^{\piout}$, because $\piout$ is a pointwise average of softmax policies and need not belong to $\Pi_{\cQ}$. Instead, we represent the trajectory distribution of $\piout$ as a mixture over stagewise hybrids of the iterates. Let $I = \spr{I_1, \dots, I_H}$ be sampled uniformly from $\sbr{K}^H$, and define the hybrid policy $\pi^I$ by
    \begin{equation*}
        \pi^I_h
        =
        \pi^{I_h}_h
        \qquad
        \text{for every } h \in \sbr*{H}.
    \end{equation*}
    We now justify carefully why these hybrids are related to the averaged policy. Let $\tau = \spr{\mb{x}_1, \mb{a}_1, \dots, \mb{x}_H, \mb{a}_H, \mb{x}_{H + 1}}$ be a trajectory under policy $\piout$. By the definition of $\piout$, we have
    \begin{align*}
        \bbP^{\piout} \spr*{\tau}
        &=
        \initial \spr*{\mb{x}_1} \prod^H_{h = 1} \sbr*{\piouth \spr*{\mb{a}_h \given \mb{x}_h} P_h \spr*{\mb{x}_{h + 1} \given \mb{x}_h, \mb{a}_h}}
        \\
        &=
        \initial \spr*{\mb{x}_1} \prod^H_{h = 1} \sbr*{\spr*{\frac1K \sum^K_{k_h = 1} \pi^{k_h}_h \spr*{\mb{a}_h \given \mb{x}_h}} P_h \spr*{\mb{x}_{h + 1} \given \mb{x}_h, \mb{a}_h}}
        \\
        &=
        \frac{1}{K^H} \sum_{i_1, \dots, i_H \in \sbr*{K}} \initial \spr*{\mb{x}_1} \prod^H_{h = 1} \sbr*{ \pi^{i_h}_h \spr*{\mb{a}_h \given \mb{x}_h} P_h \spr*{\mb{x}_{h + 1} \given \mb{x}_h, \mb{a}_h}}
        \\
        &=
        \bbE_I \sbr*{\bbP^{\pi^I} \spr*{\tau}}.
    \end{align*}
    The third line is where we use that the coordinates $I_1, \dots, I_H$ are sampled independently. If one sampled a single common index for all stages, the resulting average of products would not generally equal the product of averages defining $\piout$. Therefore, for any trajectory-level function $F$, and in particular for the cumulative reward $F \spr{\tau} = \sum^H_{h = 1} r_h \spr{\mb{x}_h, \mb{a}_h}$, we can average over the finite trajectory space to obtain
    \begin{equation*}
        \bbE^{\piout} \sbr*{F \spr*{\tau}}
        =
        \sum_{\tau} \bbP^{\piout} \spr*{\tau} F \spr*{\tau}
        =
        \bbE_I \sbr*{\sum_{\tau} \bbP^{\pi^I} \spr*{\tau} F \spr*{\tau}}
        =
        \bbE_I \sbr*{\bbE^{\pi^I} \sbr*{F \spr*{\tau}}}.
    \end{equation*}
    Taking $F$ to be the cumulative reward gives
    \begin{equation*}
        J^{\piout}
        =
        \bbE_I \sbr*{J^{\pi^I}}.
    \end{equation*}
    For every $I$, the hybrid policy $\pi^I$ belongs to $\Pi_{\cQ}$. Indeed, at each stage $h$, choose $m_h = I_h - 1$, the functions $Q^1_h, \ldots, Q^{I_h - 1}_h \in \cQ_h$ (empty if $I_h = 1$), and $\eta_{I_h} = \eta_{m_h + 1}$. By \Cref{ass:q-picq-realizability}, we have $Q^{\pi^I} \in \cQ$, and hence $Q^{\pi^I}_h \in \cQ_h$ for every $h$. Applying the performance difference lemma (\Cref{lem:performance-difference}) to each hybrid policy gives
    \begin{equation*}
        J^\expert - J^{\piout}
        =
        \bbE_I \sbr*{J^\expert - J^{\pi^I}}
        =
        \bbE_I \sbr*{\sumhH \sum_{x \in \cX} d^\expert_h \spr*{x} \inp*{Q^{\pi^I}_h \spr*{x, \cdot}, \experth \spr*{\cdot \given x} - \pi^{I_h}_h \spr*{\cdot \given x}}}.
    \end{equation*}
    Using $d^\expert_h = d_h + \spr{1 - \alpha} \spr{d^\expert_h - d^{\piout}_h}$, the right-hand side is equal to $T_1 + T_2$, where we define
    \begin{align*}
        T_1
        &\ldef
        \bbE_I \sbr*{\sumhH \sum_{x \in \cX} d_h \spr*{x} \inp*{Q^{\pi^I}_h \spr*{x, \cdot}, \experth \spr*{\cdot \given x} - \pi^{I_h}_h \spr*{\cdot \given x}}},
        \\
        T_2
        &\ldef
        \spr*{1 - \alpha} \bbE_I \sbr*{\sumhH \sum_{x \in \cX} \spr*{d^\expert_h \spr*{x} - d^{\piout}_h \spr*{x}} \inp*{Q^{\pi^I}_h \spr*{x, \cdot}, \experth \spr*{\cdot \given x} - \pi^{I_h}_h \spr*{\cdot \given x}}}.
    \end{align*}
    We first control $T_1$. For a fixed stage $h$, write $J \in \sbr{K}^{H - 1}$ for the collection of indices at all stages except $h$. For $k \in \sbr{K}$, let $I^{h, k, J}$ be the full index vector obtained by setting $I_h = k$ and $I_{-h} = J$, and define
    \begin{equation*}
        Q^{k, J}_h
        \ldef
        Q^{\pi^{I^{h, k, J}}}_h.
    \end{equation*}
    For every $h, k, J$, we have $Q^{k, J}_h \in \cQ_h$. Using the independence and uniformity of $I_h$, we can rewrite $T_1$ as
    \begin{equation*}
        T_1
        =
        \sumhH \bbE_{I_h, I_{-h}} \sbr*{ \cL^d_h \spr*{\pi^{I_h}_h, Q^{\pi^I}_h}}
        =
        \frac1K \sumhH \sumkK \bbE_J \sbr*{\cL^d_h \spr*{\pi^k_h, Q^{k, J}_h}}.
    \end{equation*}
    We now compare each realized comparator $Q^{k, J}_h$ to the empirical best response $Q^k_h$ before averaging over $J$. With a slight abuse of notation, recall that $\pi^k$ denotes the policy whose decision rule at stage $h$ is $\pi_h^k$. Define the same uniform estimation error used in \Cref{thm:arbitrary_nu} by
    \begin{equation*}
        \Delta_d \spr*{\pi}
        =
        \max_{h \in \sbr*{H}} \sup_{Q_h \in \cQ_h} \abs*{\hcL^d_h \spr*{\pi_h, Q_h} - \cL^d_h \spr*{\pi_h, Q_h}}.
    \end{equation*}
    For every fixed $h, k, J$, on this uniform-concentration event,
    \begin{align*}
        \cL^d_h \spr*{\pi^k_h, Q^{k, J}_h}
        \leq
        \hcL^d_h \spr*{\pi^k_h, Q^{k, J}_h} + \Delta_d \spr*{\pi^k}
        \leq
        \hcL^d_h \spr*{\pi^k_h, Q^k_h} + \Delta_d \spr*{\pi^k}
        \leq
        \cL^d_h \spr*{\pi^k_h, Q^k_h} + 2 \Delta_d \spr*{\pi^k}.
    \end{align*}
    The middle inequality is the empirical best-response property of $Q^k_h$, which we can use since $Q^{k, J}_h$ is one feasible element of $\cQ_h$ by \Cref{ass:q-picq-realizability}. Averaging the last display over $J$, summing over $h$ and $k$, and dividing by $K$, we obtain\loose
    \begin{align*}
        T_1
        &\leq
        \frac1K \sumhH \sumkK \cL^d_h \spr*{\pi^k_h, Q^k_h} + \frac{2H}{K} \sumkK \Delta_d \spr*{\pi^k}.
    \end{align*}
    The first term is controlled by the FTRL regret bound stage by stage (\Cref{lemma:FTRL}). Indeed, for each fixed $h$,
    \begin{equation*}
        \sumkK \cL^d_h \spr*{\pi^k_h, Q^k_h}
        =
        \sum_{x \in \cX} d_h \spr*{x} \sumkK \inp*{Q^k_h \spr*{x, \cdot}, \experth \spr*{\cdot \given x} - \pi^k_h \spr*{\cdot \given x}}
        \leq
        3 \sqrt{\QMAX^2 K \log A}.
    \end{equation*}
    Since $\pi^k$ is in $\Pi_\cQ$, applying \Cref{lem:moulin-fh-concentration} with $\Pi_h = \Pi_{\cQ, h}$ and radius $r > 0$ gives, on an event of probability at least $1 - \delta$,
    \begin{equation*}
        \sumkK \Delta_d \spr*{\pi^k}
        \leq
        4 K r + K \QMAX \sqrt{\frac{8 \log \spr*{2 \sumhH \cN_r \spr*{\cQ_h \times \Pi_{\cQ, h}, \rho} / \delta}}{\tauE}}.
    \end{equation*}
    Combining the previous three displays, taking $r = \varepsilon / H$, and bounding the sum by the maximum gives
    \begin{align*}
        T_1
        &\leq
        3 H \sqrt{\frac{\QMAX^2 \log A}{K}} + 8 \varepsilon + 2 H \QMAX \sqrt{\frac{8 \log \spr*{2 H \max_h \cN_{\varepsilon / H} \spr*{\cQ_h \times \Pi_{\cQ, h}, \rho} / \delta}}{\tauE}}.
    \end{align*}

    It remains to control $T_2$. Define, for each $h, x, I$,
    \begin{equation*}
        g_h^I \spr*{x}
        \ldef
        \inp*{Q^{\pi^I}_h \spr*{x, \cdot}, \experth \spr*{\cdot \given x} - \pi^{I_h}_h \spr*{\cdot \given x}}.
    \end{equation*}
    We have $Q^{\pi^I}_h \spr{x, a} \in \sbr{0, \QMAX}$ for every $x, a, h, I$. Therefore
    \begin{equation*}
        \abs*{g_h^I \spr*{x}}
        =
        \abs*{Q^{\pi^I}_h \spr*{x, \experth} - Q^{\pi^I}_h \spr*{x, \pi^{I_h}_h}}
        \leq
        \QMAX.
    \end{equation*}
    Using the triangle inequality and the previous inequality, we can bound the signed change-of-measure term explicitly:
    \begin{align*}
        \abs*{T_2}
        &\leq
        \spr*{1 - \alpha} \bbE_I \sbr*{\sumhH \abs*{\sum_{x \in \cX} \spr*{d^\expert_h \spr*{x} - d^{\piout}_h \spr*{x}} g_h^I \spr*{x}}}
        \\
        &\leq
        \spr*{1 - \alpha} \sumhH \bbE_I \sbr*{\sum_{x \in \cX} \abs*{d^\expert_h \spr*{x} - d^{\piout}_h \spr*{x}} \abs*{g_h^I \spr*{x}}}
        \\
        &\leq
        \spr*{1 - \alpha} \QMAX \sumhH \norm*{d^\expert_h - d^{\piout}_h}_1
        \\
        &\leq
        2 \spr*{1 - \alpha} H \QMAX.
    \end{align*}
    In the last line, we used that both $d^\expert_h$ and $d^{\piout}_h$ are probability distributions, so their $\ell_1$ distance is at most $2$. With $\alpha = 1 - \varepsilon / \spr{2 H \QMAX}$, the term $T_2$ is at most $\varepsilon$. The $Q^{\Pi_{\cQ}}$-realizable choices of $K$ and $\tauE$ make $T_1$ at most $\cO \spr{\varepsilon}$, proving the second branch for the same averaged output policy $\piout$.
\end{proof}


\end{document}
